\documentclass[11pt]{article}
\usepackage{preamble}
\usepackage[framemethod=TikZ]{mdframed}
\usepackage{multicol}
\newmdtheoremenv{program}{Program}
\usepackage{comment}

\title{Learning Polynomial Transformations}
\author{
    Sitan Chen\thanks{Email: \texttt{sitanc@berkeley.edu}} \\
    UC Berkeley
        \and 
    Jerry Li\thanks{Email: \texttt{jerrl@microsoft.com}} \\
    Microsoft Research
        \and
    Yuanzhi Li\thanks{Email: \texttt{yuanzhil@andrew.cmu.edu}} \\
    CMU
        \and
    Anru R. Zhang\thanks{Email: \texttt{anru.zhang@duke.edu}}\\
    Duke University
}

\newcommand{\calC}{\mathcal{C}}

\newcommand{\op}{\mathsf{op}}
\newcommand{\calG}{\mathcal{G}}

\newcommand{\sym}{\mathsf{sym}}

\newcommand{\mb}[1]{\mathbf{#1}}
\newcommand{\sort}[1]{\overline{\mb{#1}}}

\newcommand{\rchoose}{\binom{r+\omega-1}{\omega}}
\newcommand{\strings}{[r]^{\omega}}
\renewcommand{\epsilon}{\varepsilon}
\newcommand{\gaugedist}{d_{\mathsf{G}}}

\newcommand{\tmp}{\zeta}

\DeclareMathOperator{\sgn}{sgn}
\let\vec=\relax
\DeclareMathOperator{\vec}{vec}
\DeclareMathOperator{\mat}{mat}
\DeclareMathOperator{\ten}{tens}

\newcommand{\gap}{\upsilon}
\newcommand{\epsort}{\epsilon_{\mathsf{ort}}}
\newcommand{\epsmap}{\epsilon_{\mathsf{map}}}
\newcommand{\epsrel}{\epsilon_{\mathsf{id}}}
\newcommand{\epsnorm}{\epsilon_{\mathsf{norm}}}
\newcommand{\epsswap}{\epsilon_{\mathsf{swap}}}
\newcommand{\epsoffdiag}{\epsilon_{\mathsf{offdiag}}}
\newcommand{\epstrueort}{\epsilon^*_{\mathsf{ort}}}
\newcommand{\epsouter}{\epsilon_{\mathsf{out}}}
\newcommand{\epstrueouter}{\epsouter^*}
\newcommand{\epspair}{\epsilon_{\mathsf{pair}}}

\newcommand{\epsgram}{\epsilon_{\mathsf{gram}}}
\newcommand{\crude}{R'}
\newcommand{\symbound}{\Gamma}
\newcommand{\diag}{\mathrm{diag}}

\newcommand{\num}[1]{\#(\mb{#1})}

\makeatletter
\renewcommand{\paragraph}{%
  \@startsection{paragraph}{4}%
  {\z@}{1.75ex \@plus 1ex \@minus .2ex}{-1em}%
  {\normalfont\normalsize\bfseries}%
}
\makeatother

\DeclareSymbolFont{matha}{OML}{txmi}{m}{it}% txfonts
\DeclareMathSymbol{\varv}{\mathord}{matha}{118}

\begin{document}

\pagestyle{empty}
{
  \renewcommand{\thispagestyle}[1]{}
  \maketitle

    \begin{abstract}
        We consider the problem of learning high dimensional polynomial transformations of Gaussians. Given samples of the form $p(x)$, where $x\sim\calN(0,\Id_r)$ is hidden and $p: \R^r \to \R^d$ is a function where every output coordinate is a low-degree polynomial, the goal is to learn the distribution over $p(x)$.
        This problem is natural in its own right, but is also an important special case of learning deep generative models, namely pushforwards of Gaussians under two-layer neural networks with polynomial activations.
        Understanding the learnability of such generative models is crucial to understanding why they perform so well in practice.
        
        Our first main result is a polynomial-time algorithm for learning quadratic transformations of Gaussians in a smoothed setting.
        Our second main result is a polynomial-time algorithm for learning constant-degree polynomial transformations of Gaussian in a smoothed setting, when the rank of the associated tensors is small. In fact our results extend to any rotation-invariant input distribution, not just Gaussian. These are also the first end-to-end guarantees for learning a pushforward under a neural network with more than one layer.
        
        Along the way, we also give the first polynomial-time algorithms with provable guarantees for \emph{tensor ring decomposition}, a popular generalization of tensor decomposition that is used in practice to implicitly store large tensors \cite{zhao2016tensor}.
    \end{abstract}
    \newpage
  \tableofcontents
}

\clearpage
\pagestyle{plain}
% \clearpage
\pagenumbering{arabic}

%Main points:
%natural question, motivated by GANs
%distribution learning (different from GMMs/graphical models where the density function is %well-specified, captures low-dimensional manifold structure?)
%connection to tensor ring decomposition which had no provable algorithms
%    our algorithm recovers jennrich's in diagonal case

%ICA: earlier example of learning pushforwards, technically incomparable b/c they assume non-gaussian product to make the problem well-defined
%    - say we're considering a case of nonlinear ICA
%    - clarify why ICA on monomial basis doesn't solve our problem

%why smoothed instead of worst case:
%    - info theoretic shows parameter learning hard
%    mention hardness results for local PRGs (suggests worst-case functions, discrete distribution make the problem hard)
%    also that weird bivariate cubic thingy
    
%Techniques: using Sos to reason about unknown rotation

%Discussion about assumptions: "multi-view"

\section{Introduction}
In recent years, generative models such as variational auto-encoders  (VAEs)~\cite{kingma2013auto} and generative adversarial networks (GANs)~\cite{goodfellow2014generative} have exploded in popularity in practice as extraordinarily effective ways of modeling real-world data such as the distribution of natural images.
At their heart, such generative models attempt to learn a parametric transformation of a simple and relatively low dimensional distribution---typically chosen to be a standard normal Gaussian---into a complex, high-dimensional one.
The resulting distributions present a very rich family of distributions which dramatically differ in many ways from more classical generative models such as mixture models or graphical models.
However, despite their immense practical impact, very little is known about the learnability of such distributions from a theoretical perspective.

More formally, we consider the following problem.
We are given a class of functions $\mathcal{F}$ from $\R^r$ to $\R^d$, where $r \ll d$, and we are given samples of the form $f(x_1), \ldots, f(x_n)$, where $x_i \sim \calN(0, \Id_r)$, and $f$ is an unknown function in $\mathcal{F}$ (note we do not observe $x_1,\ldots,x_n$).
The goal is to output the description of some distribution over $\R^d$ which is close to the distribution of $f(x)$, for $x \sim \calN(0, \Id_r)$.

We will consider arguably the most basic class $\mathcal{F}$, namely functions $f: \R^r \to \R^d$ where each output coordinate is a (homogeneous) polynomial. That is, the main question we study is:
\vspace{-0.3em}
\begin{center}
    {\it When can we efficiently learn a high-dimensional polynomial transformation of a Gaussian?} 
\end{center}
\vspace{-0.3em}
In machine learning terminology, this problem can be stated as follows: when can we learn the pushforward of a one hidden layer neural network with polynomial activations? Note that while ReLU activations are more commonly used in GANs, it has been demonstrated that polynomial activations can also be used to generate images of nontrivial quality \cite[Figure 2]{li2020making}.

Despite the fundamental nature of this question, very little is understood about it.
The only provable results known for this problem~\cite{li2020making} only hold for extremely structured instances, and moreover, they require a conjectured structural result related to the identifiability of a certain tensor decomposition problem (see the discussion above Theorem 2 in~\cite{li2020making}).
To the best of our knowledge, to date, there are no algorithms with end-to-end provable guarantees for learning pushforwards of neural networks with more than one layer in any non-trivial setting. % in general settings. \sitan{unclear what "general settings" means here because we're assuming a specific activation. i think i prefer ``non-trivial,'' though it's somewhat aggressive}

\paragraph{Tensor ring decomposition.} We also consider a seemingly unrelated problem known as \emph{tensor ring decomposition} \cite{zhao2016tensor}.
Here, there are (unknown) matrices $Q^*_1, \ldots, Q^*_d \in \R^{r \times r}$, and the goal is to recover them up to trivial symmetries, given estimates for $\Tr (Q^*_a Q^*_b)$ and $\Tr (Q^*_a Q^*_b Q^*_c)$ for all $a, b, c$.
When $\brc{Q^*_a}$ are diagonal, this is equivalent to degree-3 tensor decomposition (Appendix~\ref{app:diagonal}). This problem is thus a natural ``non-commutative'' generalization of tensor decomposition.

Tensor ring decompositions, and related concepts such as hierarchical Tucker rank~\cite{ballani2013black,novikov2014putting} and tensor train decomposition~\cite{oseledets2010tt,oseledets2011tensor}, were first proposed in the condensed matter physics community~\cite{verstraete2004density}, and were later adopted in the neural network community as ways to concisely represent large tensors in a way which still allows for efficient linear algebraic computations~\cite{zhao2016tensor}.
Various heuristics have been proposed for this problem~\cite{zhao2016tensor,khoo2021efficient}, though to date, none of these come with provable guarantees for tensor ring decomposition in any nontrivial regime of parameters, and even in the noiseless setting.
This is in stark contrast to the state of affairs with traditional tensor decomposition, where for many settings, there are many polynomial time algorithms with provable guarantees, see e.g.~\cite{harshman1970foundations,leurgans1993decomposition,barak2015dictionary,ge2015decomposing,ma2016polynomial,hopkins2016fast,hopkins2019robust}.
This begs the natural question:
\vspace{-0.3em}
\begin{center}
    {\it When can we efficiently solve tensor ring decomposition?}
\end{center}
\vspace{-0.3em}

While this is of tremendous interest in its own right, our interest comes from the fact that this is fundamentally related to learning quadratic transformations of Gaussians.
Indeed, recovering the parameters of such a distribution from its moments of degree at most $3$ is exactly equivalent to solving noisy tensor ring decomposition (see Section~\ref{sec:connect})!
Understanding tensor ring decomposition thus seems like a necessary first step towards understanding our central learning problem.

\subsection{Our Contributions}

In this paper, we give the first efficient algorithms for learning high dimensional polynomial transformations of Gaussians, under mild non-degeneracy conditions that we demonstrate are satisfied with negligible failure probability in reasonable smoothed analysis settings.
Along the way, we also provide the first efficient algorithms for tensor ring decomposition under analogous conditions.

\paragraph{Efficient algorithms for quadratic transformations and tensor ring decomposition.}
Our first result is a polynomial time algorithm for learning smoothed (homogeneous) quadratic transformations of Gaussians, in sufficiently high dimensions:
\begin{theorem}[Informal, see Theorem~\ref{thm:main_quadratic}]
    \label{thm:main_quadratic_informal}
    For any $d\in\mathbb{N}$ sufficiently large and any $\epsilon > 0$, $1/\poly(d)$-smoothed quadratic transformations of Gaussian with input dimension $r = \widetilde{O}(\sqrt{d})$ are learnable (both in parameter distance and Wasserstein distance) to error $\epsilon$ in $\poly(r,1/\epsilon)\cdot d$ time and $\poly(r,1/\epsilon)$ samples with probability at least $1 - \exp(-\poly(r))$ over the smoothing.
\end{theorem}
\noindent To the best of our knowledge, this is the first end-to-end provable algorithmic result for learning pushforwards given by a neural networks with more than a single layer (see Section~\ref{sec:related} for further discussion). Note that the condition $r = \widetilde{O}(\sqrt{d})$ here means that the pushforward distribution is supported on a low-dimensional manifold, which is quite natural in practice~\cite{osher2017low}.

Our smoothed model is the standard one in which the instance is given by a small random perturbation of a worst-case instance (see Section~\ref{sec:definitions}).
As with many results in smoothed analysis, our results hold more generally under mild deterministic non-degeneracy conditions.

We complement this result with an information-theoretic lower bound (see Appendix~\ref{app:lbd}), which states that in the worst case, parameter learning for quadratic transformations requires exponentially many samples, even in one dimension.
Combined with cryptographic hardness results for density estimation of worst-case ReLU network transformations of Gaussians \cite{chen2022minimax}, this suggests that some beyond-worst-case assumptions are necessary to obtain efficient algorithms.
Intuitively, our non-degeneracy assumptions give us a ``blessing of dimensionality'' phenomenon which allows us to obtain multiple linearly independent ``views'' of the underlying transformation.

Theorem~\ref{thm:main_quadratic_informal} is based on the following new algorithm for tensor ring decomposition:
\begin{theorem}[Informal, see Theorem~\ref{thm:main_tensorring}]
    \label{thm:main_tensorring_informal}
    For any $d\in\mathbb{N}$ sufficiently large and any $\epsilon > 0$, given a $1 / \poly (d)$-smoothed instance of $\epsilon$-noisy tensor ring decomposition in dimension $r = \widetilde{O}(\sqrt{d})$, there is a polynomial time algorithm which recovers the unknown matrices to error $\poly(\epsilon, r)$ up to trivial symmetries in $\poly(r,1/\epsilon)\cdot d$ time with probability at least $1 - \exp(-\poly(r))$ over the smoothing.
\end{theorem}
\noindent Our algorithms for Theorems~\ref{thm:main_quadratic_informal} and Theorems~\ref{thm:main_tensorring_informal} are based on the Sum-of-Squares (SoS) ``proofs to algorithms'' framework, which in recent years has been applied to solve a number of high-dimensional statistical problems.
However, the design of our algorithm differs quite substantially from prior techniques used within this literature.
As we explain in Section~\ref{sec:overview}, the $\Tr(Q^*_aQ^*_b)$'s in tensor ring decomposition give us the unknown $r \times r$ matrices $Q^*_1, \ldots, Q^*_d$, up to a shared, unknown rotation, but as vectors in $r^2$ dimensions.
The heart of our algorithm is an SoS proof that the only such rotations which can additionally match the $\Tr(Q^*_aQ^*_bQ^*_c)$'s are in fact Kronecker powers of $r \times r$-dimensional rotation matrices.
In other words, up to gauge symmetry in the $r \times r$-dimensional space, the $r^2 \times r^2$-dimensional rotations which respect our constraints are unique, and moreover, SoS witnesses this fact.
Consequently, this implies that we can search for these rotations using an SoS program, and the result can be easily rounded to solve the overall problem.

\paragraph{Efficient algorithms for low-rank polynomial transformations.}
For our final result, we turn to polynomial transformations of higher degree. We show that (homogeneous) polynomial transformations of odd constant degree can be learned efficiently, as long as the transformation can be represented using low rank tensors.
Recall that any homogeneous degree $\omega$ polynomial $p: \R^r \to \R$ can be associated with a symmetric tensor $T: \R^r \to (\R^r)^{\otimes \omega}$, so that $p(x) = \iprod{T,x^{\otimes\omega}}$.
We say that a polynomial is rank $\ell$ if the associated tensor has symmetric rank $\ell$, and we say that a polynomial transformation $f: \R^r \to \R^d$ has rank $\ell$, if each output coordinate has rank $\ell$. From the perspective of neural networks, $\ell$ corresponds to the \emph{channels} of the hidden layer per neuron. Our main result here is:
\begin{theorem}[Informal, see Theorem~\ref{thm:main_lowrank}]\label{thm:informal_lowrank}
    There is an absolute constant $c > 0$ such that for any $d\in\mathbb{N}$ sufficiently large and any $\epsilon > 0$, $1/\poly(d)$-smoothed rank-$\ell = O(1)$ transformations of odd degree $\omega = O(1)$ with seed length $r = \widetilde{O}(d^{c/(\omega\ell)})$ are learnable (both in parameter distance and Wasserstein distance) to error $\epsilon$ in $\poly(r,1/\epsilon)\cdot d$ time and $\poly(r,1/\epsilon)$ samples with probability at least $1 - \exp(-\poly(r))$ over the smoothing.
\end{theorem}
At its heart, our algorithm follows the same rough structure as the one for the quadratic case, that is, we must show in SoS that the unknown rotation over $r^\omega$ dimensions which maps the ground truth to our estimates must arise as a Kronecker power of a rotation over $r$ dimensions.
However, the arguments here are much more subtle. For starters, for a high-degree polynomial transformation, even the low-order moments are unwieldy even to write down, let alone work with.

For this reason, unlike in the quadratic case, here we only work with second-order moments. In place of tensor ring decomposition, this leads to a new inverse problem that we call \emph{low-rank factorization}, which may be of independent interest: given unknown low-rank symmetric tensors $T^*_1,\ldots,T^*_d$, recover them from estimates of every $\iprod{T^*_a,T^*_b}$ up to the trivial $r\times r$ rotational symmetry (see Definition~\ref{def:lowrankfactorization}). A priori it is unclear why this should be possible, e.g. if $T^*_1,\ldots,T^*_d$ weren't constrained to be low-rank, then one could only hope to recover them up to a global $r^{\omega}\times r^{\omega}$ rotation. We show that surprisingly, the low-rank constraints force this rotation to be the Kronecker power of an $r\times r$ rotation (see Theorem~\ref{thm:main_push}). The proof of this is quite involved, in part because symmetric tensor rank, unlike matrix rank, is notoriously difficult to capture using simple polynomial constraints \cite{landsberg2013equations}. We refer the reader to Section~\ref{sec:overview} for more details.

Finally, we remark that all of our guarantees for learning transformations (Theorem~\ref{thm:main_quadratic_informal} and \ref{thm:informal_lowrank}) in fact hold for transformations of \emph{any} rotation-invariant input distribution with suitable moment bounds (see Sections~\ref{sec:quadratic_reduce}, \ref{sec:lowrank_reduce}, and \ref{sec:rotationinvariant}), not just of $\calN(0,\Id_r)$.

\section{Generative Model and Inverse Problems}
\label{sec:definitions}
In this section, we formally define the models we study throughout this paper.

\begin{definition}[Polynomial Transformations]
    For $\omega\ge 2$, a \emph{$d$-dimensional degree-$\omega$ transformation with seed length $r$} is a distribution $\calD$ over $\R^d$ specified by tensors $T^*_1,\ldots,T^*_d\in$ $(\R^r)^{\otimes d}$. To sample from $\calD$, one samples $x\sim \calN(0,\Id_r)$ and outputs $(\iprod{T^*_1,x^{\otimes \ell}},\ldots,\iprod{T^*_d,x^{\otimes\ell}})$. Equivalently, $\calD$ is the \emph{pushforward} of the standard Gaussian measure on $\R^r$ under the map $x\mapsto (\iprod{T^*_1,x^{\otimes \ell}},\ldots,\iprod{T^*_d,x^{\otimes\ell}})$.
    
    We will collectively refer to the tensors $T^*_1,\ldots,T^*_d$ as the \emph{polynomial network} specifying $\calD$. If $\omega = 2$, we will refer to $T^*_1,\ldots,T^*_d$ as $Q^*_1,\ldots,Q^*_d\in\R^{r\times r}$. If $T^*_1,\ldots,T^*_d$ are of rank $\ell$, then we will refer to $(T^*_1,\ldots,T^*_d)$ as a \emph{rank-$\ell$ polynomial network}.
\end{definition}

\noindent We will study the learnability of polynomial transformations in the following smoothed analysis settings. For quadratic transformations, we consider entrywise Gaussian perturbations.

\begin{definition}[Smoothed Quadratic Networks]\label{def:smoothed_1}
    Let $\rho > 0$. We say that a degree-2 polynomial network $Q^*_1,\ldots,Q^*_d\in\R^{r\times r}$ is \emph{$\rho$-fully-smoothed} if $Q^*_1,\ldots,Q^*_d$ were generated via the following experiment: for matrices $\overline{Q}_1,\ldots,\overline{Q}_d\in\R^{r\times r}$, each $Q^*_a$ is obtained by independently sampling a symmetric matrix $G_a$ whose diagonal and upper triangular entries are independent draws from $\calN(0,1)$
    % \footnote{That is, $G\in\R^{r\times r}$ is symmetric with upper diagonal entries sampled independently from $\calN(0,1/r)$ and diagonal entries sampled independently from $\calN(0,2/d)$.}
    and forming $Q^*_a \triangleq \overline{Q}_a + \frac{\rho}{r}\cdot G_a$.
    We refer to the matrices $\overline{Q}_1,\ldots,\overline{Q}_d$ as the \emph{base network}.
\end{definition}

\noindent For low-rank transformations, we consider perturbations of the rank-1 \emph{tensor components}.

\begin{definition}[Smoothed Low-Rank Networks]\label{def:smoothed_2}
    Let $\rho > 0$. We say that a rank-$\ell$ polynomial network $T^*_1,\ldots,T^*_d\in(\R^r)^{\otimes \omega}$ is \emph{$\rho$-componentwise-smoothed} if $T^*_1,\ldots,T^*_d$ were generated via the following experiment: for $r$-dimensional vectors $\brc{\overline{v}_{a,i}}_{a\in[d],i\in[\ell]}$, each $T^*_a$ is obtained by independently sampling $g_{a,1},\ldots,g_{a,\ell}\sim\calN(0,\Id_r)$ and forming $T^*_a \triangleq \sum^{\ell}_{i=1} (\overline{v}_{a,i} + \frac{\rho}{\sqrt{r}} \cdot g_{a,i})^{\otimes \omega}$. Similar to Definition~\ref{def:smoothed_1}, we refer to the tensors $\overline{T}_1,\ldots,\overline{T}_d$ given by $\overline{T}_a \triangleq \sum^{\ell}_{i=1} v^{\otimes\ell}_{a,i}$ as the \emph{base network}.
\end{definition}

\noindent In this paper we give guarantees for parameter learning polynomial transformations. First, there are some basic symmetries to be aware of. First, if $T$ and $T'$ differ by a skew-symmetric form, that is if $\sum_{\pi\in\calS_{\omega}} (T - T')_{i_{\pi(1)}\cdots i_{\pi(\omega)}} = 0$ for all $i_1,\ldots,i_\omega\in[r]$, then $\iprod{T,x^{\otimes\omega}}$ and $\iprod{T',x^{\otimes\omega}}$ are identical as polynomials in $x$. For this reason, we will henceforth assume without loss of generality that the network $T^*_1,\ldots,T^*_d$ consists of \emph{symmetric tensors}.

Additionally, because the input distribution $\calN(0,\Id_r)$ that is being pushed forward through the polynomial network is rotation-invariant, the network of a polynomial transformation is only identifiable up to a \emph{gauge symmetry}. Let $O(r)$ denote the group of orthogonal $r\times r$ matrices. Given a tensor $T\in(\R^r)^{\otimes\omega}$ and orthogonal matrix $U\in O(r)$, define the tensor (see Definition~\ref{def:transform})
\begin{equation}
    F_U(T) \in (\R^r)^{\otimes\omega}: \qquad F_U(T)_{i_1\cdots i_\omega} = \sum_{j_1,\ldots,j_\omega\in[r]} U_{i_1j_1}\cdots U_{i_\omega j_\omega} T_{j_1\cdots j_\omega} \ \ \forall i_1,\ldots,i_\omega\in[r]. \label{eq:FU_intro}
\end{equation}
Note that when $\omega = 2$ so that $T$ is an $r\times r$ matrix, then $F_U(T) = UTU^{\top}$ (see Example~\ref{example:rotation}). The following is immediate (see Appendix~\ref{app:defer_gauge}):

\begin{lemma}[Gauge symmetry]\label{lem:gauge}
    For any network $T^*_1,\ldots,T^*_d\in (\R^{\otimes r})^{\otimes\omega}$ and any orthogonal matrix $U\in O(r)$, the transformation specified by the polynomial network $T^{**}_1,\ldots,T^{**}_d$, where $T^{**}_a \triangleq F_{U^{\otimes \omega}}(T^*_a)$ is identical to the one specified by $T^*_1,\ldots,T^*_d$.
\end{lemma}

\noindent We thus formulate parameter learning as recovering the polynomial network modulo this freedom.

\begin{definition}[Parameter Distance]
    Given polynomial networks $T_1,\ldots,T_d$ and $T'_1,\ldots,T'_d$, define the \emph{parameter distance} $\gaugedist(\brc{T_a},\brc{T'_a})$ by $\gaugedist(\brc{T_a},\brc{T'_a}) \triangleq \min_{U\in O(r)} \max_{a\in[d]} \norm{F_{U^{\otimes\omega}}(T_a) - T'_a}_F$.
    % Note that $\gaugedist(\cdot,\cdot)$ is symmetric in its arguments.
\end{definition}

\noindent As we discuss in Section~\ref{sec:density}, parameter learning implies proper density estimation.

We note that in general it is not true that the parameters of a polynomial transformation must be identifiable up to gauge symmetry. For example, it was shown in \cite{grunbaum1975cubic} that there exist cubic polynomials $p, q: \R^2\to\R$ for which the corresponding pushforwards of $\calN(0,\Id_2)$ are identical as distributions, but for which $p$ and $q$ are \emph{not} equivalent up to gauge symmetry. Nevertheless, the fact that we are able to show identifiability up to gauge symmetry in smoothed settings suggests that such examples are quite pathological.

\subsection{Inverse Problems}
\label{sec:inverse_problems}

Our algorithms for parameter learning polynomial transformations are based on method of moments. In general, the intricate combinatorial structure of the higher-order moments of a polynomial transformation makes them quite difficult to work with, especially when the degree of the transformation itself is large. In this work however, we show that for smoothed networks, it suffices to work with moments up to degree at most three. That is, we show how to recover the parameters of a smoothed polynomial transformation $\calD$ using only estimates of the form $\E{z_a z_b z_c}, \E{z_a z_b}, \E{z_a}$ for $z\sim\calD$. As we show in Section~\ref{sec:connect}, these moments take a particular form so that the problem of reconstructing parameters from moments naturally gives rise to the following inverse problems.

\begin{definition}[Tensor Ring Decomposition]\label{def:tensorring}
    Let $\eta > 0$, and let $Q^*_1,\ldots,Q^*_d\in\R^{r\times r}$ be unknown symmetric matrices. Given as input a matrix $S\in\R^{d\times d}$ and a tensor $T\in\R^{d\times d\times d}$ satisfying
    \begin{equation}
        \abs*{\Tr(Q^*_a Q^*_b) - S_{a,b}} \le \eta \qquad  \text{and} \qquad
        \abs*{\Tr(Q^*_a Q^*_b Q^*_c) - T_{a,b,c}} \le \eta \ \ \forall \ a,b,c\in[d],
    \end{equation}
    the goal is to output $\wh{Q}_1,\ldots,\wh{Q}_d$ for which $\gaugedist(\brc{Q^*_a},\brc{\wh{Q}_a})$ is small.
\end{definition}

\begin{remark}
    This is slightly different from how tensor ring decomposition is traditionally posed \cite{zhao2016tensor} as usually one only assumes that $T$ is given. For learning polynomial transformations however, it is easy to get access to both $S$ and $T$, so we tailor our algorithms to Definition~\ref{def:tensorring}. %That said, in Section~\ref{sec:spectral}, we also give an algorithm for the more traditional setting of tensor ring decomposition where only $T$ is given.
\end{remark}

\noindent This specializes to the well-studied problem of symmetric tensor decomposition when $Q^*_a$ are diagonal: if $v_i\in\R^d$ denotes the vector with $a$-th entry $(Q^*_a)_{ii}$, then $T \approx \sum_i v_i^{\otimes 3}$ (see Appendix~\ref{app:diagonal}).

We also study the following (to our knowledge, new) variant of matrix factorization:

\begin{definition}[Low-Rank Factorization]\label{def:lowrankfactorization}
    Let $\eta > 0$, and let $T^*_1,\ldots,T^*_d\in(\R^r)^{\otimes\omega}$ be unknown symmetric tensors of rank $\ell$. Given a known positive definite matrix $\Sigma\in\R^{r^{\omega}\times r^{\omega}}$, let $\iprod{\cdot,\cdot}_{\Sigma}$ denote the associated inner product. Given as input a matrix $S\in\R^{d\times d}$ satisfying
    \begin{equation}
        \abs*{\iprod{\vec(T^*_a), \vec(T^*_b)}_{\Sigma} - S_{a,b}} \le \eta \ \ \forall \ a,b\in[d],
    \end{equation}
    the goal is to output $\wh{T}_1,\ldots,\wh{T}_d$ for which $\gaugedist(\brc{T^*_a},\brc{\wh{T}_a})$ is small.
\end{definition}

\noindent \emph{A priori}, it is not even clear that such a recovery guarantee is possible. Indeed, without the extra condition that $T^*_1,\ldots,T^*_d$ are low rank, the recovery goal in Definition~\ref{def:lowrankfactorization} is impossible, even for $\eta = 0$ and $\Sigma = \Id$. In that case, the constraints $\iprod{T^*_a,T^*_b} = S_{a,b}$ at best specify $\brc{T^*_a}$ up to an \emph{$r^{\omega}\times r^{\omega}$ rotation}, whereas in Definition~\ref{def:lowrankfactorization} we are interested in identifying up to an $r\times r$ rotation!
% specifies a \emph{symmetric matrix factorization} $S = NN^{\top}$, where the rows of $N\in\R^{d\times r^{\omega}}$ consist of the vectorizations of $T_1,\ldots,T_d$. But symmetric matrix factorization suffers from a very large symmetry: $N$ is only uniquely defined up to an \emph{$r^{\omega}\times r^{\omega}$ rotation}, whereas in Definition~\ref{def:lowrankfactorization} we are interested in identifying up to an $r\times r$ rotation!

In view of our application to polynomial transformations, we will be interested in $\Sigma$ given by $\Sigma = \E[x\sim D]{\vec(x)^{\otimes\omega} {\vec(x)^{\otimes\omega}}^{\top}}$ for rotation-invariant distributions $D$ over $\R^r$, e.g. $D = \calN(0,\Id)$.

\section{Technical Overview}
\label{sec:overview}

In this section we give a high-level overview of the key algorithmic ideas in this work. As our reduction from polynomial pushforwards to the inverse problems defined in Section~\ref{sec:inverse_problems} is straightforward (see Section~\ref{sec:connect}), here we focus on describing our algorithms for the inverse problems, namely tensor ring decomposition and low-rank factorization. For both of these, we will sketch how to prove that the underlying parameters ($\brc{Q^*_a}$ and $\brc{T^*_a}$ respectively) are information-theoretically \emph{identifiable} from the input, modulo gauge symmetry. As we show in Sections~\ref{sec:tensorring} and \ref{sec:push}, with significant care, these proofs of identifiability can be implemented in the SoS proof system and thus yield efficient algorithms; we discuss the main challenges for doing so at the end of this overview.

For simplicity, in this overview we focus on the noiseless setting, i.e. when $\eta = 0$ in Definitions~\ref{def:tensorring} and \ref{def:lowrankfactorization}, though in later sections we prove our guarantees for general $\eta$.

\paragraph{Overview notation.} Subscripts/superscripts denote row/column indices for matrices. Given $Q\in\R^{r\times r}$, $\vec(Q)\in\R^{r^2}$ denotes its flattening. $e_i\in\R^r$ denotes the $i$-th standard basis vector.

\subsection{Tensor Ring Decomposition}
\label{sec:tensorring_overview}

% Given a vector $v\in\R^{r^2}$, we use $\mat(v)$ to denote its reshaping into an $r\times r$ matrix.

\paragraph{Hidden $r^2\times r^2$ rotation.} Recall that in tensor ring decomposition, there are unknown symmetric matrices $Q^*_1,\ldots,Q^*_d$, and we want to recover them up to gauge symmetry given $\Tr(Q^*_aQ^*_b)$ and $\Tr(Q^*_aQ^*_bQ^*_c)$ for all $a,b,c$. 

First, as discussed above, the only information the $\Tr(Q^*_aQ^*_b)$'s provide is the angle between every pair of matrices \emph{regarded as an $r^2$-dimensional vector}. In particular, given only the $\Tr(Q^*_aQ^*_b)$'s, the best we can hope for is to estimate $\brc{Q^*_a}$ up to an \emph{$r^2\times r^2$ rotation}\footnote{Technically this is not quite true as $\brc{Q^*_a}$ do not span the space of all $r\times r$ matrices as they are symmetric. We defer the discussion of how we circumvent this issue to later in the overview.} (see Section~\ref{sec:hiddenrot_tensorring}). More formally, we can only hope to produce matrices $Q_1,\ldots,Q_d$ for which there exists some $r^2\times r^2$ orthogonal matrix $U$ satisfying $U\vec(Q^*_a) = \vec(Q_a)$ for $a = 1,\ldots,d$. Recalling \eqref{eq:FU_intro}, we denote this by
\begin{equation}
    F_U(Q^*_a) = Q_a \label{eq:FUsketch}
\end{equation}
(see Definition~\ref{def:transform}). An example of such a $U$ would be one corresponding to an \emph{$r\times r$ rotation}. Take any $r\times r$ orthogonal matrix $V$. We can check (see Example~\ref{example:rotation}) that the transformation sending any $Q$ to $VQV^{\top}$ can be expressed in terms of $F_U$ for $U$ given by the \emph{Kronecker square of $V$}. That is, if we index the rows and columns of $U$ by $[r]\times[r]$ and let the $(i,j)$-th column be given by the flattening of $V^i(V^j)^{\top}$, then $F_U(Q) = VQV^{\top}$. In this case, we say that $U$ \emph{arises from $V$}.

Note that $U$'s of this form comprise a vanishing fraction of all $r^2\times r^2$ orthogonal matrices. The bulk of our analysis is thus centered around proving that the remaining \emph{third-order} constraints in tensor ring decomposition, i.e. the $\Tr(Q^*_aQ^*_bQ^*_c)$'s, force $U$ to take this special form.

\paragraph{Using third-order constraints.} Note that we can interpret the $\Tr(Q^*_aQ^*_bQ^*_c)$'s as telling us the angle between $\vec(Q^*_a)$ and $\vec(Q^*_bQ^*_c)$ for any $a,b,c$. Using this, we can ensure that in addition to $U$ sending every $Q^*_a$ to $Q_a$, $U$ also sends every $Q^*_bQ^*_c$ to $Q_bQ_c$.

To unpack what additional information this implies about $U$, let us pretend for a moment that $Q^*_1,\ldots,Q^*_d$ consisted of the matrices $\brc{E_{ij}}$, where $E_{ij} = e_ie_j^{\top}$. For any $i,j$, we will refer to the corresponding $Q_a$ as $Q_{ij}$ so that $Q_{ij} = F_U(E_{ij})$. Then because $\vec(E_{ij})$ is simply the $(i,j)$-th standard basis vector in $\R^{r^2}$, we conclude that $Q_{ij}$ is the $(i,j)$-th column of $U$, reshaped into an $r\times r$ matrix. We will refer to this as $U^{ij}$.

Now what do the constraints $F_U(Q^*_aQ^*_b) = Q_aQ_b$ tell us? For any $i,j,j',k\in[r]$, note that $E_{ij}E_{j'k} = \bone{j = j'}\cdot E_{ik}$. So the fact that $Q_{ij}Q_{j'k} = F_U(E_{ij}E_{j'k})$ implies that
\begin{equation}
    U^{ij} U^{j'k} = \bone{j = j'} \cdot U^{ik}. \label{eq:Urelation_sketch}
\end{equation}

It turns out that even if $\brc{Q^*_a}$ are not given by $\brc{E_{ij}}$, under some mild non-degeneracy conditions on $\brc{Q^*_a}$ that are satisfied in the smoothed setting (see Part~\ref{assume:condnumber} of Assumption~\ref{assume:tensorring}) the hidden rotation $U$ will still satisfy \eqref{eq:Urelation_sketch}. The reason is as follows. First, we can write each $Q^*_a$ as a linear combination of $\brc{E_{ij}}$. Then for every triple $a,b,c\in[d]$, the constraints $F_U(Q^*_a) = Q_a$, $F_U(Q^*_b) = Q_c$, $F_U(Q^*_c) = Q_c$, and $F_U(Q^*_bQ^*_c) = Q_bQ_c$ altogether imply a quadratic relation on $U$ which is a \emph{linear combination of the relations~\eqref{eq:Urelation_sketch}}. Because the $Q^*_a$'s are sufficiently non-degenerate, then provided that $d$ is sufficiently large that the $Q^*_a$'s span the same space as the $E_{ij}$'s, these linear combinations of relations for different $a,b,c$ are sufficiently ``incoherent'' that they collectively imply the relations \eqref{eq:Urelation_sketch} (see Lemma~\ref{lem:main_identity} for a formal version of this argument).

\paragraph{Using the relations \eqref{eq:Urelation_sketch}.}

We now sketch how to argue, using the relations~\eqref{eq:Urelation_sketch}, that $U$ must arise from an $r\times r$ rotation. Recall this means we must argue that the matrices $U^{ij}$ are each given by the outer product of a pair of columns of some orthogonal matrix.

The main step is to argue that the matrices $U^{ij}$ are rank-1 matrices. From \eqref{eq:Urelation_sketch}, we have that $U^{ij} = U^{ii}U^{ij}$. Right-multiplying by ${U^{ij}}^{\top}$ on both sides and taking traces, we get
\begin{equation}
    \Tr(U^{ij}{U^{ij}}^{\top}) \le \Tr(U^{ii}U^{ij}{U^{ij}}^{\top}) \le \norm{U^{ii}}_F \norm{U^{ij}{U^{ij}}^{\top}}_F = \norm{U^{ij}{U^{ij}}^{\top}}_F, \label{eq:rank1sketch}
\end{equation}
where in the third step we used the fact that $U$ is orthogonal to conclude that $\norm{U^{ii}}_F = 1$. As $U^{ij}{U^{ij}}^{\top}$ is psd, the above inequality holds with equality, so $U^{ij}{U^{ij}}^{\top}$ is a rank-1 matrix, implying that $U^{ij}$ is as well (Lemma~\ref{lem:2by2} gives a formal version of this argument).

Having established that there exist unit vectors $\brc{v_{ij}, w_{ij}}$ for which $U^{ij} = v_{ij}w_{ij}^{\top}$, we can use \eqref{eq:Urelation_sketch} to narrow down what these vectors should be. For instance, \eqref{eq:Urelation_sketch} implies that $(U^{ii})^2 = U^{ii}$, so $v_{ii} = w_{ii}$. It also tells us that $U^{ii} U^{jj} = 0$ for $i\neq j$, so $\brc{v_{ii}}$ are orthonormal. Lastly, it tells us that $U^{ii}U^{ij} = U^{ij}$ and $U^{ij}U^{jj} = U^{ji}$, so $v_{ij} = v_{ii}$ and $w_{ij} = v_{jj}$. Put together, these imply that $U$ arises from an $r\times r$ rotation whose columns consist of $\brc{v_{ii}}$.

\paragraph{A catch: working with symmetric matrices.} Thus far, an important detail that we have swept under the rug is that because $Q^*_1,\ldots,Q^*_d$ are symmetric, there is actually some ambiguity in how to define the $r^2\times r^2$ matrix $U$ mapping every $Q^*_a$ to $Q_a$. For instance, given any such $U$, we could interchange the $(i,j)$-th and $(j,i)$-th columns (or more generally, replace them with arbitrary affine combinations of each other) and get a new matrix with the same property.

To resolve this ambiguity, we insist that the transformation $U$ satisfy $U^{ij} = U^{ji}$ for every $i = j$. Unfortunately, this comes at a cost: $U$ is no longer orthogonal. Additionally, the above argument for deducing the relations \eqref{eq:Urelation_sketch} no longer holds because the symmetric matrices $\brc{Q^*_a}$ do not span the same space as $\brc{E_{ij}}$, so we end up with a weaker family of relations (see \eqref{eq:matmul} in Lemma~\ref{lem:main_identity}).  

We nevertheless show how to use these weaker relations to bootstrap a new matrix out of $U$ and show that it satisfies all the desired properties from the discussion above, namely orthogonality, \eqref{eq:FUsketch}, and \eqref{eq:Urelation_sketch} (see Lemma~\ref{lem:simpler_identity}). We refer the readers to Section~\ref{sec:auxW} for the details.

% \paragraph{A spectral algorithm.} \TODO{}

\subsection{Low-Rank Factorization}
\label{sec:lowrank_overview}

Here we pursue the same strategy of showing the unknown (in this case $r^{\omega}\times r^{\omega}$) rotation mapping the ground truth to our estimates arises from an $r\times r$ rotation, with several essential differences.

\paragraph{Hidden rotation respecting $\Sigma$ norm.} Recall that in low-rank factorization, there are unknown symmetric tensors $T^*_1,\ldots,T^*_d$ of symmetric rank $\ell$, and we want to recover them up to gauge symmetry given $\iprod{T^*_a,T^*_b}_{\Sigma}$ where $\Sigma$ is some known psd matrix specifying the inner product. Our guarantee pertains to $\Sigma$ of the form $\Sigma = \E[g\sim D]{\vec(g)^{\otimes\omega}{\vec(g)^{\otimes\omega}}^{\top}}$ for any rotation-invariant distribution $D$ on $\R^r$. While the $\Sigma$-norm is no longer the Euclidean inner product, we can still hope to estimate $\brc{T^*_a}$ up to some $r^{\omega}\times r^{\omega}$ transformation $U$ that preserves the $\Sigma$-norm. As before, there is some ambiguity in defining $U$ because the $T^*_a$'s are symmetric, though our workaround for this (see Section~\ref{sec:hiddenrotation}) is similar in spirit to the one for tensor ring decomposition.

Ultimately our goal will still be to show that $U$ essentially arises from an $r\times r$ rotation. In the present setting, if we index the rows and columns of $U$ by $[r]^{\omega}$, this amounts to showing that there is some orthogonal $V\in\R^{r\times r}$ for which the $(i_1,\ldots,i_\omega)$-th column of $U$ is given by the flattening of the rank-1 tensor $V^{i_1}\otimes\cdots\otimes V^{i_\omega}$ for all $i_1,\ldots,i_\omega\in[r]$.

\paragraph{Rank-$\ell$-preserving transformations.} The key challenge that arises in low-rank factorization and not tensor ring decomposition is that we only have access to \emph{pairwise} information about $\brc{T^*_a}$. In the absence of third-order constraints that could allow us to prove an identity like \eqref{eq:Urelation_sketch}, we need to exploit the assumption that the unknown tensors $\brc{T^*_a}$ are low-rank.

In particular, the fact that $\brc{T^*_a}$ are low-rank and the fact that the estimates $T_a$ that we output should also be low-rank places nontrivial constraints on the form that $U$ can take. Intuitively, because $\brc{T^*_a}$ are ``random-looking'' under some mild non-degeneracy assumptions that are satisfied in the smoothed setting (see Assumption~\ref{assume:push}), if $d$ is sufficiently large then we expect that $U$ should send \emph{any} rank-$\ell$ tensor to a rank-$\ell$ tensor.

Reasoning about this in a way that is amenable to sum-of-squares is delicate, because tensor rank is notoriously worse-behaved than matrix rank. To get around this, we work with a relaxed notion of rank where we instead insist that any contraction of the tensor into an $r\times r$ matrix has rank $\ell$ (see Definition~\ref{def:Vrank}). For any fixed contraction, this amounts to a finite collection of polynomial identities corresponding to the vanishing of all $(\ell+1)\times(\ell+1)$ minors of the contraction. By polynomial anticoncentration, we show that because these identities are satisfied for $F_U(T^*_1),\ldots,F_U(T^*_d)$, they are also satisfied for $F_U(T)$ for any tensor $T$ of symmetric rank $\ell$ (see Lemma~\ref{lem:genericflatrank}, Lemma~\ref{lem:compsmooth}, and Lemma~\ref{lem:giant_mat}). 

In other words, $U$ is ``rank-$\ell$-preserving'' in the sense that it sends any tensor of symmetric rank $\ell$ to a tensor whose matrix contractions are of rank $\ell$.

\paragraph{Rank-$\ell$-preserving implies rank-1-preserving.} 

The reason it is useful for $U$ to be rank-$\ell$-preserving is that, as we show in Lemma~\ref{lem:rankrtorankr}, it additionally implies that $U$ is rank-$(\ell-1)$-preserving and thus, by induction, rank-1 preserving! Before we process the implications of the latter, we sketch the argument in Lemma~\ref{lem:rankrtorankr}. For simplicity, here we will consider the special case of $\omega = 2$, where our contraction-based notion of rank agrees with symmetric rank (i.e. matrix rank), and $r = \ell + 1$, though the argument also extends to any $\omega > 2$.

Starting with any rank-$(\ell-1)$ matrix $M\in\R^{(\ell+1)\times(\ell+1)}$, consider some rank-1 perturbation $c\cdot zz^{\top}$ that we will vary. By assumption, $F_U(M + zz^{\top}) = F_U(M) + F_U(c\cdot zz^{\top})$ has rank $\ell$, so $\det(F_U(M) + F_U(c\cdot zz^{\top})) = 0$. Formally differentiating this with respect to $c$ at $c = 0$ yields
\begin{equation}
\sum^{\ell+1}_{i=1}
\det\left(
    \begin{array}{c|c|c}
        % | & | & | \
        F_U(M)^{1:i-1} & F_U(zz^{\top})^{i}  & F_U(M)^{i+1:\ell+1}
        % | & | & |
    \end{array}
\right) = 0, \label{eq:sumdets}
\end{equation}
where $A^{i:j}$ denotes the matrix consisting of the $i$-th to $j$-th columns of $A$. In particular, we can take the Laplace expansion of the $i$-th determinant in \eqref{eq:sumdets} along the $i$-th column, and \eqref{eq:sumdets} then becomes a linear combination of all $\ell\times \ell$ minors of $F_U(M)$, where the coefficients of this linear combination are given by entries of $F_U(zz^{\top})$ (see Eq.~\eqref{eq:smallminors}). By taking many choices of $z$, we can ensure that sufficiently many different linear combinations of these minors vanish to imply that the minors themselves vanish. This shows that $F_U(M)$ is rank-$(\ell-1)$ as desired.

\paragraph{Using rank-1-preservation to conclude.} As one can show that our contraction-based notion of rank aligns with symmetric rank for rank-1 tensors (see Lemma~\ref{lem:Vrankimplies1rank} and Lemma~\ref{lem:perm}), we conclude that $U$ sends any symmetric rank-1 tensor to a symmetric rank-1 tensor. We now sketch how to use this, together with the fact that $U$ preserves the $\Sigma$-norm, to conclude that $U$ arises from an $r\times r$ rotation.
% Indeed, if $U$ mapped \emph{any rank-1 tensor} (not necessarily symmetric) to a rank-1 tensor, then for starters, because $U$ would map $e_{i_1}\otimes\cdots\otimes e_{i_\omega}$ (where $e_i\in\R^r$ is the $i$-th standard basis vector) to a rank-1 tensor for all $i_1,\ldots,i_\omega$, this would immediately imply that the columns of $U$ are flattenings of rank-1 tensors.

For starters, because $U$ sends the rank-1 tensor $e_i^{\otimes\omega}$ to a rank-1 tensor, this implies that there exist vectors $\brc{V^i}$ such that the reshaping of the $(i,\ldots,i)$-th column of $U$ is given by $U^{i\cdots i} = (V^i)^{\otimes\omega}$ for all $i$. Note that because $U$ preserves the inner product specified by $\Sigma$, we have $\E[g\sim D]{\iprod{V^i,g}^{\omega}\iprod{V^j,g}^{\omega}} = \E[g\sim D]{g_i^{\omega} g_j^{\omega}}$ for all $i,j$. Because $D$ is rotation-invariant, one can show that this implies the $v_i$'s must be orthonormal with respect to the Euclidean inner product (see Lemma~\ref{lem:diagonal_orth}).

It remains to argue that the other columns of $U$ are also rank-1 tensors whose factors are given by $\brc{V^i}$ (see Lemma~\ref{lem:tensorouter}). We accomplish this by using the fact that for various choices of $a_1,\ldots,a_r\in \R$, the image of the rank-1 tensor $(a_1,\ldots,a_r)^{\otimes\omega}$ under $U$, which is given by some linear combination of all $U^{i_1\cdots i_\omega}$'s, is rank-1. By varying $a_1,\ldots,a_r$, we can extract information about individual columns of $U$. 

To give a rough sense of how this would go, here we give a baby version of the argument. We show how to use rank-1-preservation to conclude, in the special case where $\omega = 2$ and $v_i = e_i$ for all $i$, that every $U^{ij} + U^{ji}$ is a multiple of $v_iv_j^{\top} + v_jv_i^{\top}$.  We know that the image of the rank-1 matrix $(e_i + a\cdot e_j)(e_i + a\cdot e_j)^{\top}$ under $U$, which is given by
\begin{equation}
    U^{ii} + a(U^{ij} + U^{ji}) + a^2 U^{jj} = (e_ie_i^{\top} + a^2e_je_j^{\top}) + a(U^{ij} + U^{ji}), \label{eq:image_sketch}
\end{equation}
is rank-1 for all $a\in\R$. Letting $M\triangleq U^{ij} + U^{ji}$, we find that the $2\times 2$ minor of \eqref{eq:image_sketch} given by rows/columns $i$ and $j$ is given by
\begin{equation}
    0 = (a M_{ii} + 1)(a M_{jj} + a^2) - a^2\cdot M_{ij}M_{ji} = aM_{jj} + a^2 (M_{ii}M_{jj} - M_{ij}M_{ji} + 1) + a^3 M_{ii}.
\end{equation}
As this holds for all $a\in\R$, this implies that every coefficient on the right-hand side vanishes, so $M_{ii} = M_{jj} = 0$ and $M_{ij}M_{ji} = 1$. In a similar fashion, by considering the other $2\times 2$ minors of \eqref{eq:image_sketch}, we can show that all other entries of $M$ are zero, and we conclude that $M = \pm(e_ie_j + e_je_i^{\top})$. The full argument for general $\omega$ is given in the proof of Lemma~\ref{lem:main_userank1_odd}.

% We could then use this, together with the fact that $U$ preserves the norm induced by $\Sigma$, to conclude that the columns of $U$ are orthonormal (see Corollary~\ref{cor:dd_orth_odd})

\subsection{Sum-of-Squares Algorithms}

\paragraph{Proofs to algorithms.} Our general approach for getting an algorithm out of all of this follows the usual SoS proofs-to-algorithms pipeline for statistical problems (see e.g. \cite{hopkins2018statistical}). We introduce SoS variables $\brc{Q_a}$ (resp. $\brc{T_a}$) corresponding to our estimates for the ground truth $\brc{Q^*_a}$ (resp. $\brc{T^*_a}$) and constrain them to possess the same properties as the ground truth. For instance, for low-rank factorization, we require that $\brc{T_a}$ are symmetric and satisfy $\iprod{T_a,T_b}_{\Sigma} = \iprod{T^*_a,T^*_b}_{\Sigma}$, and to constrain them to be low-rank, we also introduce SoS variables $\brc{v_{a,t}}_{a\in[d],t\in[\ell]}$ and insist that $T_a = \sum^{\ell}_{t=1} v_{a,t}^{\otimes\ell}$. The hope is to turn the arguments above into a low-degree SoS proof that $\brc{T_a}$ and $\brc{T^*_a}$ are equivalent up to gauge symmetry, and then to apply some simple rounding procedure to a pseudoexpectation satisfying the aforementioned constraints to extract estimates for $\brc{T^*_a}$.

This raises a number of challenges. How do we capture the $r^{\omega}\times r^{\omega}$ transformation $U$ from the preceding discussion in SoS? How do we encode the condition that $U$ has the structure of a Kronecker power of an $r\times r$ orthogonal matrix? And how do we actually round, given that everything is only specified up to gauge symmetry?

\paragraph{Implementing $U$ as an SoS variable.} For simplicity, we illustrate this in the setting of tensor ring decomposition. Having imposed the constraints $\Tr(Q_aQ_b) = \Tr(Q^*_aQ^*_b)$, we can rewrite these constraints as the matrix equality 
\begin{equation}
    NN^{\top} = N^*{N^*}^{\top},
\end{equation}
where $M, M^*\in\R^{d\times r^2}$ have $a$-th row given by $\vec(Q_a)$ and $\vec(Q^*_a)$ respectively. A linear transformation mapping every $Q^*_a$ into $Q_a$ can be thought of as a matrix $U$ for which $U{N^*}^{\top} = N^{\top}$. A natural way to construct such a matrix $U$ would be to define $U = N^{-1}{N^*}^{\top}$. Note that because $NN^{\top} = N^*{N^*}^{\top}$, it would follow that $U$ is an orthogonal matrix.

Of course this doesn't quite work as $N$ is an SoS variable and thus does not have a left-inverse, but this is easy to remedy by introducing an additional variable $L$ to the SoS program corresponding to this left-inverse and requiring that $LN = \Id$. We could then define $U$ to be $L{N^*}^{\top}$. We emphasize that $U$ should not be thought of as another variable in our SoS program; after all, $N^*$ is unknown to the algorithm designer, so the entries of $U$ are merely unknown linear forms in the SoS variable $L$. For this reason, $U$ is only referenced throughout the analysis of our SoS relaxation.

Finally, as discussed at the end of Section~\ref{sec:tensorring_overview}, there are some subtleties as $N$ has repeated columns because $\brc{Q_a}$ are symmetric, so strictly speaking it should not have a left-inverse. We discuss how to circumvent these issues in Section~\ref{sec:hiddenrot_tensorring} for tensor ring decomposition and Section~\ref{sec:hiddenrotation} for low-rank factorization.

\paragraph{Expressing Kronecker structure of $U$.} While most of the steps outlined in Sections~\ref{sec:tensorring_overview} and \ref{sec:lowrank_overview} proving various properties of $U$ are relatively straightforward to implement in SoS, e.g. \eqref{eq:Urelation_sketch} and \eqref{eq:rank1sketch} for tensor ring decomposition and rank preservation for low-rank factorization, it is less clear how to even express in SoS the main conclusion that we want to show about $U$, namely that it is the Kronecker power of an $r\times r$ orthogonal matrix $V$.

In particular, how do we express $V$? That is, how do we use the existing program variables to design a matrix $V$ for which we could hope to prove $U$ is its Kronecker power? For tensor ring decomposition, a natural candidate would be to take the $r\times r$ matrix $U^{ii}$ for every $i\in[r]$, pick one of its nonzero columns and normalize it to a unit vector $V^i$, and take $V$'s columns to consist of $V^i$'s. This does not quite work because the normalization step involves a rational function of the entries of the program variables. To fix this, we need to carry around these normalization factors when expressing the $U^{ij}$'s as outer products (e.g. Lemma~\ref{lem:outerproduct}).

While this turns out to be manageable for tensor ring decomposition, such an approach quickly becomes unwieldy for low-rank factorization where the degree $\omega$ can be arbitrary. Fortunately, for odd $\omega$, there is a simpler workaround. Heuristically, because we expect to have $U^{i\cdots i} = V^i\otimes\cdots\otimes V^i$ for $r\times r$ orthogonal matrix $V$, we also expect that $V^i$ is equal to the vector, call it $\wt{U}^i$, whose $j$-th entry is given by
\begin{equation}
    \wt{U}^i_j \triangleq \sum_{j_1,\ldots,j_{\floor{\omega/2}}\in[r]} U^{i\cdots i}_{j_1j_1\cdots j_{\floor{\omega/2}}j_{\floor{\omega/2}}j}
\end{equation}
(see \eqref{eq:wtU_def}). In particular, the entries of $\wt{U}$ are simply linear forms in those of $U$. For general odd $\omega$, our SoS proof that $U$ is a Kronecker power thus entails proving that $U$ is the Kronecker power of $\wt{U}$ (Lemma~\ref{lem:tensorouter}) and that $\wt{U}$ is orthogonal (Corollary~\ref{cor:dd_orth_odd}).

% easy for odd omega, annoying for even because of normalization

\paragraph{Rounding by breaking gauge symmetry.} Finally, we describe how to take a pseudodistribution satisfying the constraints of our SoS program and round to an integral solution. This is complicated by the fact that we can only hope to recover the ground truth up to gauge symmetry. We address this by breaking symmetry and imposing a small number of additional constraints to our SoS program. These constraints will ensure that the transformation $U$ is not just the Kronecker power of some $r\times r$ orthogonal matrix, but actually equal to the identity matrix (see Section~\ref{sec:diagonal}). This shows that $Q_a$ and $Q^*_a$ (or $T_a$ and $T^*_a$) are not only equivalent up to rotation, but \emph{equal}. At that point we can produce an integral solution simply by outputting the pseudoexpectations of $\brc{Q_a}$ or $\brc{T_a}$ (see e.g. Section~\ref{sec:puttogether}).

In tensor ring decomposition, a natural approach to ensure that $U$ is identity would be to further insist that one of the $Q_a$'s is diagonal with diagonal entries sorted in increasing order. The reason is that if the eigenvalues of $Q_a$ are all distinct (more precisely, well-separated to account for noise when $\eta > 0$), then the only way for $VQ_aV^{\top}$ to be equal to $Q_a$ for some $r\times r$ rotation $V$ would be for $V$ to be equal to $\pm \Id$ (and thus for $U$ to be $(\pm \Id)^{\otimes 2} = \Id$). This approach in fact already works in the smoothed setting.

To handle the slightly more general setting where $\brc{Q_a}$ are ``incoherent'' but have repeated eigenvalues, we slightly modify this by insisting that some suitable \emph{random linear combination} of the $Q_a$'s is diagonal with sorted diagonal entries. By carefully designing how this linear combination is sampled (see Appendix~\ref{app:forcegap}) we can ensure that it has sufficient eigengaps with high probability.

For low-rank factorization, we use a similar approach with various technical modifications to account for the fact that for $\omega > 2$, order-$\omega$ tensors do not have a suitable notion of eigengap. The details here are rather thorny and involve running \emph{two} SoS relaxations in succession. We defer an overview of these workarounds to Section~\ref{sec:breaksym} and \ref{sec:breaksos_push}.

% For low-rank factorization, a modification of this approach also works. Because we are now working with tensors, we no longer have a suitable notion of eigengap, so to break symmetry, we instead consider a certain matrices formed from polynomials in the entries of $\brc{T^*_a}$ (see \eqref{eq:Fdef}) and argue that some random linear combination of these matrices has eigengaps with high probability. This ensemble over linear combinations is more difficult to construct and requires first running a preliminary version of the SoS relaxation (Program~\ref{program:sos2}). After forming the linear combination, we can then run a second SoS relaxation incorporating the symmetry-breaking constraint that this linear combination of matrices is diagonal and sorted, at which point we can complete the argument (see Section~\ref{sec:breaksos_push} and the slightly more involved rounding procedure that is needed in Section~\ref{sec:push_put_together}).

\paragraph{Roadmap.} 
% In Section~\ref{sec:definitions} we formally define the probabilistic models we study, as well as the algorithmic questions of tensor ring decomposition and low-rank factorization. Then in Section~\ref{sec:overview} we give an overview of the techniques we develop to solve these problems. 
In Section~\ref{sec:related} we describe related work. After introducing notation and technical preliminaries in Section~\ref{sec:prelims} and establishing the reduction from learning polynomial transformations to tensor ring decomposition and low-rank factorization in Section~\ref{sec:connect}, we give our algorithms for the latter two problems in Sections~\ref{sec:tensorring} and \ref{sec:push}. Our guarantees depend on certain non-degeneracy assumptions about the input, and in Section~\ref{sec:conditions} we verify that these assumptions hold in the smoothed analysis setting and in Section~\ref{sec:final} we put everything together to prove our main results on learning smoothed polynomial transformations. In Appendix~\ref{app:diagonal} we establish an equivalence between a special case of learning quadratic transformations and tensor decomposition. In Appendix~\ref{app:lbd} we prove an exponential lower bound on the sample complexity of parameter learning one-dimensional quadratic transformations in the worst case. In Appendices~\ref{app:forcegap} to \ref{app:otherdefer} we provide proofs deferred from previous sections.

\begin{table}[h!]
\centering
\begin{tabular}{|c|c|c|c|}
\hline input dimension & output dimension & degree   & rank   \\ \hline
$r$             & $d$              & $\omega$ & $\ell$ \\\hline
\end{tabular}
\caption{Notation for the main parameters of a polynomial transformation}
\end{table}

\section{Related Work}
\label{sec:related}

There is a vast literature on density estimation of distributions, especially in high dimensions, to which we cannot do justice here. 
For conciseness we will only survey the most relevant work.

\paragraph{Learning latent variable models} Much of the recent algorithmic success in high dimensional distribution learning has been in developing efficient algorithms for a variety of latent variable models, such as mixture
models~\cite{dasgupta1999learning,achlioptas2005spectral,dasgupta2007probabilistic,arora2005learning,vempala2004spectral,feldman2008learning,kumar2010clustering,moitra2010settling,acharya2014near,bhaskara2014smoothed,anderson2014more,anandkumar2014tensor,ge2015learning,belkin2015polynomial,hardt2015tight,mixon2017clustering,diakonikolas2018list,hopkins2018mixture,regev2017learning,kothari2018robust,diakonikolas2020small} and graphical models~\cite{chow1968approximating,hoffgen1993learning,dasgupta1997sample,bresler2008reconstruction,wu2013learning,bresler2015efficiently,risteski2016calculate,arora2017provable,klivans2017learning,wu2019sparse,jain2019mean,goel2020learning,brustle2020multi,devroye2020minimax,bhattacharyya2020efficient,bresler2020learning,bhattacharyya2021near,diakonikolas2021outlier,boix2021chow,daskalakis2021sample}.
Of these works, we highlight the work on learning latent variable models in smoothed settings~\cite{hsu2013learning,anderson2014more,anandkumar2014tensor,bhaskara2014smoothed,ge2015learning,arora2017provable,bhaskara2019smoothed}, where a similar ``blessing of dimensionality'' phenomena to the one we observe can be seen.

However, there are important qualitative differences between these settings and the one we consider.
While our model can be viewed as a latent variable model, where the hidden variable is the unknown Gaussian, the main challenge of our work is to learn the transformation of the hidden variable, rather than the hidden variable itself.
This makes the problem take a qualitatively different form than much of the prior work.
From a technical perspective, another difference between our setting and much of the prior work on learning latent variable models is that the form of the pdf for our distributions is much more implicit; in particular, the relationship between moments of the distribution and the pdf is much less clean than (say) for Gaussian mixture models.

\paragraph{(Non-linear) independent component analysis} Independent component analysis as first proposed in~\cite{comon1994independent} is the question of learning an (unknown) linear transformation of a non-Gaussian, coordinate-wise independent random variable.
Here, the goal is to recover the underlying transformation as well as the original random variable (note that non-Gaussianity is necessary for this to be possible).
The literature on ICA is incredibly large, so we refer the reader to surveys of~\cite{hyvarinen2000independent,hyvarinen2002independent,comon2010handbook} and references within for a more detailed literature review.
We briefly note that to our knowledge, one cannot black-box apply a kernelized version of the algorithms for ICA such as~\cite{frieze1996learning,arora2012provable,arora2015simple,sun2016complete,li2017provable} to solve our problem, because in the polynomial kernel space, the resulting random variable does not satisfy coordinate-wise independence.

Of particular interest to us is the literature on non-linear ICA, which is very closely related to the learning problem we consider.
However, in non-linear ICA, the goal is not just to learn a distribution which is close to the ground truth, but in fact to recover the original (i.e. pre-transformation) latent variables.
Despite a substantial amount of interest in this model from the more applied side (see e.g.~\cite{hyvarinen2016unsupervised,hyvarinen2019nonlinear,khemakhem2020variational} and references therein), from a theoretical perspective, the problem remains relatively poorly understood without additional assumptions.
it is known that in the worst case, the latent variable is not identifiable~\cite{hyvarinen1999nonlinear}.
As another example of this phenomenon, note that the aforementioned counterexample of \cite{grunbaum1975cubic} from Section~\ref{sec:definitions} implies that for cubic transformations of Gaussians, the latent variable is not always identifiable.

Consequently, much of the literature has shifted to consider data with temporal structure, see e.g.~\cite{hyvarinen2016unsupervised,hyvarinen2019nonlinear}.
In contrast, we consider the standard i.i.d. model, but we make stronger parametric assumptions about the transformation, namely, that it is a low-degree polynomial.
In addition, we do not require that the latent variable be identifiable, as we only care about learning the underlying distribution, and not recovering the the latent variable.

\paragraph{Learning deep generative models} A full literature on the theory of learning deep generative models, and GANs in particular, is beyond the scope of this paper.
See e.g.~\cite{gui2021review} for a more in depth survey.
In terms of end-to-end learning guarantees with efficient algorithms, the literature is somewhat sparser.
To our knowledge, results are only known for relatively simple networks.
Much of the literature focuses on understanding when stochastic first order methods can learn the distribution on toy generative models~\cite{feizi2017understanding,daskalakis2017training,gidel2019negative,lei2020sgd,allen2021forward,jelassi2022adam}.
One line of work considers the problem of learning distributions generated by pushforwards of Gaussians one-layer neural networks with ReLU activations~\cite{wu2019learning,lei2020sgd}.
However, such distributions have a much simpler structure than the ones we consider in this paper, which correspond to two-layer neural networks (i.e. with one hidden layer).
Indeed, when the neural networks only have one layer, this means that the output of the distribution is very similar to a truncated Gaussian, and one can leverage techniques from the literature of learning from truncated samples~\cite{daskalakis2018efficient}.
However, such structure completely disappears with two layer neural networks. In that sense, our guarantee is the first end-to-end provable result for learning pushforwards under neural networks beyond a single layer.

Arguably the closest paper to ours is the recent work of~\cite{li2020making}.
This paper considers a very similar setting to ours, however, their result has a number of drawbacks compared to ours.
First, they assume that the hidden weight matrices are orthonormal; that is, the coordinates of their generative model are of the form $p(x) = \sum_{i = 1}^\ell a_i \iprod{u_i, x}^\omega$, where the $u_i$ are orthogonal unit vectors.
This is an incredibly brittle assumption, and their algorithm breaks even if the $u_i$ have inverse polynomially small correlations.
In particular, their assumption does not even hold in the smoothed setting we consider.
In contrast, we handle arbitrary low-rank tensors.
Second, their bounds scale exponentially with scale of $a_i$, whereas our bounds do not.
Finally, their provable guarantees are contingent on a conjectured identifiability assumption which they do not prove (see discussion above Theorem 2 in~\cite{li2020making}).
Therefore, they do not give end-to-end provable guarantees for their learning task.
In contrast, we give fully provable results for a significantly more general setting.
Indeed, much of the technical work in our paper comes down to giving a proof of identifiability for a more involved tensor decomposition-style problem.

The relative lack of algorithms for these learning tasks may be inherent, at least in some worst case sense.
Indeed, recent work of~\cite{chen2022minimax} demonstrates that learning the pushforwards of Gaussians under low-depth ReLU networks in Wasserstein distance is computationally intractable, under standard cryptographic assumptions. The starting point for their result is the observation that the assumption that ``local pseudorandom generators'' exist \cite{applebaum2006cryptography,applebaum2016cryptographic,goldreich2011candidate} implies that learning polynomial transformations of the uniform distribution over the hypercube is computationally intractable.
Alongside our information theoretic lower bound against parameter estimation for polynomial pushforwards (see Appendix~\ref{app:lbd}), this  gives evidence that some sort of smoothing assumptions are necessary to make the problem algorithmically tractable. 

On the flip side, there has been a lot of work on scrutinizing the ways in which the training dynamics for learning generative models in practice are aligned or misaligned with traditional statistical notions of distribution learning \cite{stanczuk2021wasserstein,fedus2017many,arora2018gans,arora2017generalization}, and relatedly, what it takes for minimax optimality (e.g. under the Wasserstein GAN objective) to actually ensure distribution learning \cite{bai2018approximability,liang2020well,singh2018nonparametric,uppal2019nonparametric,chen2020statistical,schreuder2021statistical,chen2022minimax}. While this suggests that a satisfactory theory for generative models may ultimately involve more than just distribution learning in the traditional sense, the basic algorithmic question considered in the present work, in addition to being natural in its own right, seems like a natural stepping stone towards such a theory.

\paragraph{Tensor ring decomposition} Tensor ring decomposition is an important instance of tensor network decomposition and arises as a prototypical model for periodic one-dimensional physical systems \cite{verstraete2004density}.
As alluded to previously, the tensor ring format, along with other dimension-reduced tensor representations such as the tensor train format~\cite{oseledets2010tt,oseledets2011tensor}, or those associated with Tucker rank or hierarchical Tucker rank~\cite{ballani2013black,novikov2014putting}, arose as ways of representing large tensors implicitly.
Unfortunately, unlike Tucker decomposition~ \cite{de2000best,zhang2018tensor}, hierarchical Tucker decomposition~\cite{grasedyck2010hierarchical}, or tensor-train decomposition~\cite{oseledets2011tensor,zhou2022optimal}, obtaining efficient algorithms with provable guarantees for tensor ring decomposition has proven quite challenging \cite{chen2020tensor}.
In part, this is because the notion of rank associated with tensor ring decomposition---in contrast to the other aforementioned representations---is unidentifiable in many scenarios~\cite{ye2018tensor}. 
While some heuristic algorithms for tensor ring decomposition have been proposed, such as those based on alternating least squares~\cite{zhao2016tensor,khoo2021efficient}, prior to our work, there were no known algorithms for the problem with end-to-end theoretical guarantees.

\paragraph{SoS for learning} From a technical point of view, our algorithms fit into the recent SoS ``proofs-to-algorithms'' paradigm for statistical inference problems (see e.g.~\cite{hopkins2018statistical} for a more thorough overview).
From a technical perpective, our problem is closest to the line of work using SoS and SoS-inspired algorithms to obtain efficient algorithms for a variety of tensor decomposition tasks~\cite{barak2015dictionary,ge2015decomposing,ma2016polynomial,hopkins2016fast,hopkins2019robust}.
However, our problem setting appears to be significantly more technically challenging, in large part because in addition to the usual permutational symmetry among components in tensor decomposition, there is an extra \emph{gauge symmetry} inherent to the problems we consider.
Even for tensor ring decomposition, which generalizes tensor decomposition, to our knowledge the techniques in these papers do not apply.

\section{Preliminaries}
\label{sec:prelims}

Given $n\in\mathbb{N}$, let $[n]$ denote the set $\brc{1,\ldots,n}$. Let $\calS_n$ denote the symmetric group on $n$ elements. Given $\pi\in\calS_n$, we let $\sgn(\pi)\in\brc{\pm 1}$ denote its parity.

\paragraph{Indexing notation.} We will use the following conventions extensively for indexing with and writing tuples. To index into a matrix $M\in\R^{r\times r}$, for any $i,j\in[r]$ we will refer to the entry in row $i$ and column $j$ interchangeably as $M^j_i$ or as $M_{ij}$.

Given a tuple $\mb{i} = (i_1,\ldots,i_\omega)$, let $\sort{i}$ denote the tuple given by sorting the entries of $\mb{i}$ in nondecreasing order, and let $\num{i}$ denote the number of tuples $\mb{j}$ for which $\sort{j} = \mb{i}$. We will refer to tuples of the form $\sort{i}$ as \emph{sorted tuples}. 

Given $s,t\in[\omega]$, we use $i_{s:t}$ to denote the substring $(i_s,i_{s+1},\ldots,i_t)$. If $\mb{i}$ is an element of $[\ell]^{\omega}$ consisting of $c_s$ copies of $s$ for every $s\in[\ell]$, we denote $\sort{i}$ by $1^{c_1}\cdots \ell^{c_\ell}$ (these two notations will only feature in parts of Section~\ref{sec:push}).

\subsection{Density Estimation}
\label{sec:density}

Here we note that our algorithms for parameter learning easily imply algorithms for proper density estimation. First, we formally define these two learning goals:

\begin{definition}[Parameter Learning]
    Given i.i.d. samples $z^{(1)},\ldots,z^{(n)}\in\R^d$ drawn from a transformation $\calD$ given by polynomial network $T^*_1,\ldots,T^*_d$, an algorithm is said to \emph{parameter learn to error $\epsilon$} if it outputs tensors $\wh{T}_1,\ldots,\wh{T}_d$ for which $\gaugedist(\brc{T^*_a},\brc{\wh{T}_a}) \le \epsilon$ with high probability.\footnote{Throughout, ``with high probability'' means ``with arbitrarily small constant failure probability,'' though our algorithms can be amplified to obtain any failure probability $\delta$ with standard clustering / hypothesis selection arguments.}
\end{definition}

\begin{definition}[Proper Density Estimation]
    Given i.i.d. samples $z^{(1)},\ldots,z^{(n)}\in\R^d$ drawn from some distribution $\calD$, an algorithm is said to solve \emph{proper density estimation to Wasserstein error $\epsilon$} if it outputs a description of a distribution $\wh{\calD}$ for which $W_1(\calD,\wh{\calD}) \le \epsilon$, where $W_1(\cdot,\cdot)$ denotes the Wasserstein-1 metric.
\end{definition}

\noindent The following lemma shows that an algorithm for parameter learning implies an algorithm for proper density estimation.

\begin{lemma}\label{lem:param_to_wasserstein}
    Let $T_1,\ldots,T_d\in (\R^r)^{\otimes\omega}$ and $T'_1,\ldots,T'_d\in(\R^r)^{\otimes\omega}$ be polynomial networks. If $\calD,\calD'$ are the transformations given by these two networks, then
    \begin{equation}
        W_1(\calD,\calD') \le \gaugedist(\brc{T_a},\brc{T'_a})\cdot \sqrt{d}\cdot O(\omega r)^{\omega/2}.
    \end{equation}
\end{lemma}

\noindent We defer the proof of this to Appendix~\ref{app:param_to_wasserstein}.

\subsection{Tensors and Linear Algebra Basics}

Given $i\in[r]$, let $e_i$ denote the $i$-th standard basis vector in $\R^r$ ($r$ will be clear from context). Let $O(r)$ denote the group of $r\times r$ orthogonal matrices.

\paragraph{Norms and eigenvalues.} Given vector $v$, let $\norm{v}_p$ denote its $\ell^p$ norm; when $p = 2$, we sometimes denote this by $\norm{v}$. Given square matrix $M$, let $\norm{M}$ or $\norm{M}_{\op}$ denote its operator norm, $\norm{M}_F$ its Frobenius norm, and $\norm{M}_{\max}$ the max-norm, that is, $\max_{i,j}|M_{ij}|$. We refer to the minimum distance between any two eigenvalues of $M$ as its \emph{minimum eigengap}. Given $M\in\R^{m\times n}$ for $m\ge n$, we refer to its $i$-th largest singular value as $\sigma_i(M)$; for $i = n$, we denote this by $\sigma_{\min}(M)$.

\paragraph{Tensor and Kronecker powers.} Given vector $v\in\R^r$ and $\omega\in\mathbb{N}$, let $v^{\otimes \omega}\in(\R^r)^{\otimes\omega}$ denote the $\omega$-th \emph{tensor power} of $v$, and let $(v^{\otimes\omega})_{\sym}$ denote its $\omega$-th symmetric power, that is, the $\rchoose$-dimensional vector whose $S$-th entry is $\prod_{s\in S} v_s$ for any multi-subset $S$ of $[r]$ of size $\omega$. Given a matrix $V\in\R^{r\times r}$, we use $V^{\otimes\omega}\in\R^{r^{\omega}\times r^{\omega}}$ to denote the $\omega$-th \emph{Kronecker power} of $V$, that is, the matrix whose $(\mb{i},\mb{j})$-th entry, for any $\mb{i},\mb{j}\in\strings$, is given by $\prod_{s\in[r]} V_{i_sj_s}$.

\paragraph{Reshapings.} Given vector $v\in\R^{r^2}$, let $\mat(v)\in\R^{r\times r}$ denote the \emph{matricization} of $v$, that is, the matrix whose $(i,j)$-th entry is given by the $((i-1)\cdot r + j)$-th entry of $v$ for all $i,j\in[r]$. Similarly, given vector $v\in\R^{r^{\ell}}$ for $\ell > 2$, let $\ten(v)\in (\R^r)^{\otimes\ell}$ denote the \emph{tensorization}, that is, the tensor whose $(i_1,\ldots,i_{\ell})$-th entry is given by the $(\sum^{\ell-1}_{j = 0} (i_j - 1)\cdot r^j + 1)$-th entry of $v$ for all $i,j\in[r]$. 

Given matrix $M\in\R^{r\times r}$, let $\vec(M)\in\R^{r^2}$ denote the \emph{vectorization} of $M$, that is, the vector whose entries consist of those of $M$ under the lexicographic ordering on $[r]^2$. Similarly, given $T\in(\R^r)^{\otimes\omega}$, let $\vec(T)\in\R^{r^{\omega}}$ denote the vectorization of $T$, that is, the vector whose entries consist of those of $T$ under the lexicographic ordering on $[r]^{\omega}$.

\begin{definition}[Slices and contractions]
    % Given a tensor $T\in(\R^d)^{\otimes \omega}$ and $i\in[\omega]$, let $\flat_i(T)\in\R^{d\times d^{\omega-1}}$ denote the \emph{$i$-th flattening} of $T$, that is, the matrix whose $(j_i,(j_1,\ldots,j_{i-1},j_{i+1},\ldots,j_{\omega}))$-th entry is given by the $(j_1,\ldots,j_{\omega})$-th entry of $T$. 
    Given $\mb{i}\in[r]^{\omega-j}$, define the \emph{slice} $T_{\mb{i}:\cdots:}\in\R^{r\times r}$ to be the order-$j$ tensor whose $(x_1,\ldots,x_j)$-th entry is $T_{\mb{i}x_1\cdots x_j}$. More generally, given $V\in(\R^r)^{\otimes\omega-j}$, define the \emph{contraction} $T(V,:,\ldots,:)$ to be the order-$j$ tensor whose $(x_1,\ldots,x_j)$-th entry is given by $\sum_{\mb{i}\in[r]^{\omega-j}} V_{\mb{i}} T_{\mb{i}x_1\cdots x_j}$. Note that $T_{\mb{i}:\cdots :} = T(\ten(e_{\mb{i}}),:,\ldots,:)$, where $\ten(e_{\mb{i}})\in(\R^r)^{\otimes\omega-j}$ is the tensorization of the $\mb{i}$-th standard basis vector.
\end{definition}

\noindent The following notion is crucial to our analysis:

\begin{definition}\label{def:transform}
    For any $\omega\in\mathbb{N}$, any $U\in\R^{r^{\omega}\times r^{\omega}}$ induces a linear map $F_U:(\R^{r})^{\otimes\omega}\to(\R^r)^{\otimes\omega}$ by sending $F_U(T) = \ten(U\vec(T))$. We sometimes say that $U$ \emph{maps} $T$ to $F_U(T)$ and refer to $U$ as an \emph{$r^\omega\times r^\omega$ transformation}.
\end{definition}

\begin{example}\label{example:rotation}
    An important instance of such a transformation is when $U$ is given by the Kronecker power of some $r\times r$ rotation, that is, when $U= V^{\otimes \omega}$ for some $V\in O(r)$. In this case, if we index the columns of $U$ by strings $\mb{i}\in\strings$, then by definition the $r\times\cdots\times r$ reshaping $U^{\mb{i}}$ of the $\mb{i}$-th column of $U$ is given by $U^{\mb{i}} = V^{i_1}\otimes\cdots\otimes V^{i_\omega}$.
    
    For instance, if $\omega = 2$ and $Q\in\R^{r\times r}$, then
    \begin{equation}
        F_U(Q) = \mat(U\vec(Q)) = \sum^d_{i,j = 1} \mat(U^{ij}) Q^j_i = \sum^d_{i,j=1} V^i Q^j_i (V^j)^{\top} = VQV^{\top}.
    \end{equation}
\end{example}

\paragraph{Symmetric tensors and symmetric rank.} We say that a tensor $T\in(\R^r)^{\otimes\omega}$ is symmetric if $T_{\mb{i}} = T_{\sort{i}}$ for all $\mb{i}\in\strings$. We also work with the following analogous notion for $r^{\omega}\times r^{\omega}$ matrices:

\begin{definition}\label{def:ultrasym}
    We say that a matrix $M\in\R^{r^{\omega}\times r^{\omega}}$ is \emph{ultra-symmetric} if for any permutations $\pi,\tau \in\calS_{\omega}$ and any $i_1,\ldots,i_{\omega},j_1,\ldots,j_{\omega}$, $M^{j_1\cdots j_{\omega}}_{i_1\cdots i_{\omega}} = M^{j_{\pi(1)}\cdots j_{\pi(\omega)}}_{i_{\tau(1)}\cdots i_{\tau(\omega)}}$.
    
    Given ultra-symmetric matrix $M$, define its \emph{symmetrization} to be the matrix $M_{\sym}\in\R^{n\times n}$ for $n = \rchoose$ with rows and columns indexed by tuples $(j_1,\ldots,j_{\omega})$ for $j_1\le \cdots \le j_{\omega}$ such that $(M_{\mathsf{sym}})^{j_1\cdots j_{\omega}}_{i_1\cdots i_{\omega}}  = \frac{1}{\omega!}\sum_{\pi\in\calS_{\omega}} M^{j_{\pi(1)}\cdots j_{\pi(\omega)}}_{i_1\cdots i_{\omega}}$.
\end{definition}

We say that a symmetric tensor $T\in(\R^r)^{\otimes\omega}$ has symmetric rank $\ell$ if it can be written as $T = \sum^\ell_{t=1} v_t^{\otimes\omega}$ for some $v_1,\ldots,v_\ell\in\R^r$. It is well-known (see e.g. Lemma 4.2 in \cite{comon2008symmetric}) that any symmetric tensor admits a decomposition of this form. In Appendix~\ref{app:defer_symmetric_inspan} we prove the following quantitative version of this fact:

\begin{lemma}[Decomposing symmetric tensors]\label{lem:symmetric_inspan}
    Let $\omega\in\mathbb{N}$ and define
    \begin{equation}
        \symbound \triangleq \exp (O(\omega^2 \log^2 \omega))  \label{eq:gamdef}\; .
    \end{equation} 
    For any tuple $j_1,\ldots,j_\omega\in[\omega]$, there exist $z_1,\ldots,z_s\in\S^{r-1}$ and $w\in\R^s$ for which
    \begin{equation}
        \sum^s_{i=1} w_i z^{\otimes \omega}_i = \frac{1}{\omega!} \sum_{\pi\in\calS_\omega} e_{j_{\pi(1)}}\otimes\cdots \otimes e_{j_{\pi(\omega)}}
    \end{equation}
    and such that $\norm{w}_1 \le \symbound$.
\end{lemma}

\paragraph{Norm bounds for interpolation.} We use the following bounds for expressing a vector as a linear combination of other vectors. We begin with the following standard fact about least-squares:

\begin{fact}[Minimum-norm solution]\label{fact:minnorm}
    Let $A\in\R^{m\times n}$ for $m > n$. If $\sigma_{\min}(M) \ge \kappa$ for $\kappa > 0$, then for any vector $v\in\R^n$, there is a $\lambda\in\R^m$ for which $\lambda^{\top}A = v$ and $\norm{\lambda}_2 \le \kappa^{-2}\norm{A}_{\op}\norm{v}$.
\end{fact}

\begin{proof}
    By assumption, $A^{\top}A$ is invertible, so define $\lambda \triangleq A(A^{\top}A)^{-1} v$. Then $\lambda^{\top}A = v$ by design. Furthermore, $\norm{\lambda}_2 \le \kappa^{-2}\norm{A}_{\op}\norm{v}$.
\end{proof}

% \begin{lemma}[See e.g. Lemma 4.2 in \cite{comon2008symmetric}]\label{lem:symmetric_inspan}
%     Given any symmetric tensor $T\in(\R^r)^{\otimes \omega}$, there exist vectors $v_1,\ldots,v_s\in\R^r$ for which $T = \sum^s_{i=0} v_i^{\otimes \omega}$.
% \end{lemma}

% \begin{lemma}\label{lem:vandermonde}
%     For any distinct $\alpha_1,\ldots,\alpha_m$, any $s,t\in[m]$, there exist coefficients $c_1,\ldots,c_m$ for which $\sum^{\ell+1}_{t=1} c_s \alpha_s^t = \bone{s = t}$.
% \end{lemma}

% \begin{proof}
%     This is an immediate consequence of the nonsingularity of square Vandermonde matrices with distinct columns.
% \end{proof}

\noindent Next we consider expressing a vector as a linear combination of rows of a Vandermonde matrix.

\begin{fact}\label{fact:vandermonde_interpolate}
    Let $V\in\R^{m\times m}$ be the Vandermonde matrix
    \begin{equation}
        \begin{pmatrix}
            1 & a_1 & \cdots & a_1^{m-1} \\
            1 & a_2 & \cdots & a_2^{m-1} \\
            \vdots & \vdots & \ddots & \vdots \\
            1 & a_m & \cdots & a_m^{m-1}
        \end{pmatrix}
    \end{equation}
    for $a_1,\ldots,a_m\in[0,1]$ satisfying $|a_i - a_j| > \zeta$ for all $i \neq j$. Then for any $w\in\R^m$, there is a $\lambda\in\R^m$ for which $\lambda^{\top} V = w$ and $\norm{\lambda} \le O(1/\zeta)^{2m-2}\cdot m\norm{w}$.
\end{fact}

\begin{proof}
    By Lemma 11 of \cite{gordon2020sparse}, $u^{\top} V^{\top}V u \ge \frac{1}{m}\cdot (\zeta/8)^{2m-2}\norm{u}^2 \ge (\zeta/16)^{2m-2}$ for all $u$, so $\sigma_{\min}(V)\ge (\zeta/16)^{m-1}$. On the other hand, $\norm{V}_{\op} \le \norm{V}_F \le m$. The lemma follows by Fact~\ref{fact:minnorm}.
\end{proof}

\noindent Lastly, we use Fact~\ref{fact:vandermonde_interpolate} to deduce the following. We defer its proof to Appendix~\ref{app:defer_general_vandermonde}.

\begin{corollary}\label{cor:general_vandermonde}
    For $D,e\in\mathbb{N}$, Let $\brc{c_{\alpha}}$ be coefficients, indexed by all $\alpha\in\brc{0,\ldots,e}^D$ for which $|\alpha| = e$. If for some $\nu > 0$ we have
    \begin{equation}
        -\nu \le \sum_{\alpha} c_{\alpha} \mb{z}_{\alpha} \le \nu \ \ \forall \ \mb{z}\in\bigl\{\frac{1}{e+1},\frac{2}{e+1},\ldots,1\bigr\}^D, \label{eq:allz}
    \end{equation}
    where $\mb{z}_{\alpha} \triangleq z^{\alpha_1}_1\cdots z^{\alpha_D}_D$, then for any $\alpha$ there is a linear combination of the constraints \eqref{eq:allz} for various choices of $\mb{z}$ which implies $|c_{\alpha}| \le O(e)^{\Theta(e D)}\cdot \nu$.
\end{corollary}

\subsection{High-Dimensional Probability}

We will use the following tail bounds and anticoncentration bounds:

\begin{fact}[Thin shell]\label{fact:shell}
    For $g\sim\calN(0,\Id_r)$, $\Pr*{\norm{g} \ge \sqrt{r} + \Omega(\sqrt{\log(1/\delta)})} \le \delta$.
\end{fact}

\begin{lemma}[Norm of Gaussian matrices, see e.g. Exercise 7.3.5 of \cite{vershynin2018high}]\label{lem:goe}
    There is an absolute constant $c > 0$ such that for symmetric matrix $G$ whose diagonal and upper triangular entries are independently sampled from $\calN(0,1)$, $\Pr{\norm{G}_{\op} \ge 2\sqrt{r} + t} \le 2\exp(-ct^2)$ for any $t > 0$.
\end{lemma}

\begin{fact}[Carbery-Wright]\label{fact:carberywright}
    There is an absolute constant $C > 0$ such that for any $\nu > 0$ and degree-$e$ polynomial $p:\R^d\to\R$, $\Pr[g\sim\calN(0,\Id)]{\abs{p(g)} \le \nu\cdot \Var{p(g)}^{1/2}} \le C\nu^{1/e}$.
\end{fact}

\begin{lemma}[Hypercontractivity]\label{lem:hypercontractivity}
    For any $e\in\mathbb{N}$, there is an absolute constant $c_e > 0$ such that if $p:\R^r\to\R$ is a polynomial of degree $e$, then $\Pr[g\sim\calN(0,\Id)]{|p(g) - \E{p}| \ge t\sqrt{\Var{p}}} \le \exp(-c_e t^{2/e})$.
\end{lemma}

% \begin{theorem}[Theorem 3.3 from \cite{sankar2006smoothed}]\label{thm:sst}
%     There is an absolute constant $c > 0$ such that for any matrix $M\in\R^{r\times r}$ and for $G\in\R^{r\times r}$ whose entries are independent draws from $\calN(0,1)$,
%     \begin{equation}
%         \Pr{\sigma_{\min}(M + \frac{\rho}{\sqrt{r}}\cdot G) \le c\cdot \rho\delta/r^2} \le \delta
%     \end{equation} for all $\delta > 0$.
% \end{theorem}

\begin{theorem}[Eigengaps of Gaussian matrices, special case of Theorem 2.6 from \cite{nguyen2017random}]\label{thm:gaps}
    Let $\gamma, c > 0$ be constants, and let $\alpha \triangleq 3(c+1)\max(1,2\gamma) + 5$, and let $G$ be a random symmetric matrix whose diagonal and upper-triangular entries are independent draws from $\calN(0,1)$.
    
    For any $M\in\R^{r\times r}$ satisfying $\norm{M}_{\mathsf{op}} \le r^{\gamma}$, the following holds with probability at least $1 - r^{1-c}$ over $G$: any two eigenvalues of $\overline{Q} + G$ differ by at least $r^{-\alpha}$.
\end{theorem}

\begin{theorem}[Smoothed analysis of tensor decomposition, special case of Theorem 2.1 from \cite{bhaskara2019smoothed}]\label{thm:decoupling}
    Let $\ell,r,\omega\in\mathbb{N}$ and $\rho > 0$. Given arbitrary $\overline{v}_1,\ldots,\overline{v}_{\ell}\in\R^r$, if $v_t\sim\calN(\overline{v}_t, \frac{\rho^2}{r}\Id)$ for every $t\in[\ell]$, then the matrix $M\in^{\rchoose\times k}$ whose $t$-th column is $\vec(v_t^{\otimes\omega})$ satisfies $\sigma_{\min}(M) \ge O\left(\frac{\rho}{r\sqrt{\ell}}\right)$ with probability at least $1 - \ell\exp(-\Omega(r^{0.9})))$ provided that $\ell \le r - r^{0.9}$.
\end{theorem}

\noindent We also need the following standard fact about the total variation distance between two Gaussians.

\begin{theorem}[TV between Gaussians]\label{thm:tvgaussian}
    Let $0 < \epsilon < 1/2$ and let $\Sigma\in\R^{n\times n}$ be a positive definite matrix. If $\norm{\Sigma - \Id}_F \le \epsilon$, then $d_{\mathrm{TV}}(\calN(0,\Sigma),\calN(0,\Id)) = \Theta(\epsilon)$.
\end{theorem}

\begin{proof}
    By Theorem 1.1 of \cite{devroye2018total}, if $\Sigma_1,\Sigma_2\in\R^{n\times n}$ are positive definite matrices such that $\Sigma^{-1}_1\Sigma_2 - \Id$ has eigenvalues $a_1,\ldots,a_n$, then $d_{\mathrm{TV}}(\calN(0,\Sigma_1),\calN(0,\Sigma_2)) = \Theta(\sqrt{\sum^n_{i=1} \lambda^2_i})$. Specializing this to $\Sigma_1 = \Sigma$ and $\Sigma_2 = \Id$, note that $\Sigma^{-1}_1 \Sigma_2 - \Id = \Sigma^{-1}_1 - \Id$. If $\epsilon_1,\ldots,\epsilon_n$ are the eigenvalues of $\Sigma - \Id$, then $(1 + \epsilon_1)^{-1} - 1,\ldots, (1+\epsilon_n)^{-1} - 1$ are the eigenvalues of $\Sigma^{-1} - \Id$. As $|\epsilon_i| \le \epsilon<1$, $|(1 + \epsilon_i)^{-1} - 1| \le 2|\epsilon_i|$, so $\norm{\Sigma^{-1}-\Id}^2_F \le 4\epsilon^2$.
\end{proof}

\noindent We need the following bound on the variance of a polynomial with input sampled from a spherical Gaussian with arbitrary mean. We defer the proof to Appendix~\ref{app:defer_variance_shift}.

\begin{lemma}[Lower bound on variance of Gaussian polynomial]\label{lem:variance_shift}
    For $a\in\R$, $b\in\R^r$, and any vector $p\in\S^{\rchoose-1}$ regarded as a degree-$\omega$, $r$-variate homogeneous polynomial, we have $\Var[g\sim\calN(0,\Id_r)]{p(ag+b)} \ge a^{2\omega}/\omega^{\omega/2}$.
\end{lemma}

Lastly, we will use the following polynomial identity involving Gaussian moments:

\begin{lemma}\label{lem:hermite}
    For any $r$-dimensional vectors $v,w$,
    \begin{equation}
        \E[g\sim\calN(0,\Id_r)]*{\iprod{v,g}^{\omega} \cdot \iprod{w,g}^{\omega}} = \omega!\sum^{\floor{\omega/2}}_{m=0} \binom{\omega}{m, m, \omega - 2m} \frac{1}{2^{2m}} \iprod{v,w}^{\omega - 2m} \norm{v}^{2m} \norm{w}^{2m}.
    \end{equation} This holds as a formal degree-$2\omega$ polynomial equality in the entries of $v,w$.
\end{lemma}

\begin{proof}
    Note that any perfect matching of $\omega$ copies of $v$ and $\omega$ copies of $w$ will pair up $\omega - 2m$ copies of $v$ with $\omega - 2m$ copies of $w$, $2m$ copies of $v$ with $2m$ copies of $v$, and $2m$ copies of $w$ with $2m$ copies of $w$, for some $0 \le m \le \floor{\omega/2}$. For each $m$, there are 
    \begin{equation}
        \left(\binom{\omega}{2m}\cdot (2m-1)!!\right)^2 \cdot (\omega-2m)! = \left(\frac{\omega!}{2^m\cdot m!\cdot (\omega-2m)!}\right)^2\cdot (\omega-2m)!=\omega!\binom{\omega}{m,m,\omega-2m}\frac{1}{2^{2m}}
    \end{equation}
    such perfect matchings. By Wick's theorem, the left-hand side is thus equal to the sum over perfect matchings of $2\omega$ elements of $\iprod{v,w}^m\iprod{v,v}^m\iprod{w,w}^m$, so the lemma follows.
    % For $\Sigma\triangleq \E{\vec(g^{\otimes\omega})\vec(g^{\otimes\omega})^{\top}}$. We can rewrite the left-hand side as $\vec(v^{\otimes\omega})^{\top} \Sigma\vec(w^{\otimes\omega})$. For any $\mb{i},\mb{j}\in\strings$, $\Sigma_{\mb{i},\mb{j}} = 
\end{proof}

\subsection{Sum-of-Squares}

\paragraph{SoS basics.} We begin with a brief overview of sum-of-squares (SoS). For a more detailed exposition of SoS, we refer the reader e.g. to \cite{barak2016proofs}.

\begin{definition}[Sum-of-squares proofs]\label{def:sosproof}
    Let $x_1,\ldots,x_n$ be \emph{indeterminates} (we also refer to these as \emph{variables}), and let \emph{program} $\calP$ be a collection of polynomial equations and inequalities $\brc{p_1(x)\ge 0,\ldots,p_m(x)\ge 0, q_1(x) = 0, \ldots, q_m(x) = 0}$ in these variables. Given a polynomial $p(x)$, we say that the inequality $p(x)\ge 0$ has a degree-$D$ SoS proof using $\calP$ if there exists a polynomial $q(x)$ in the ideal generated by $q_1(x),\ldots,q_m(x)$ at degree $D$, together with sum-of-squares polynomials $\brc{r_S(x)}_{S\subseteq[m]}$ (where the index $S$ ranges over multisets), such that
    \begin{equation}
        p(x) = q(x) + \sum_{S\subseteq[m]} r_S(x)\cdot \prod_{i\in S} p_i(x),
    \end{equation}
    and such that the degree of the polynomial $r_S(x)\cdot \prod_{i\in S}p_i(x)$ is at most $D$ for each multiset $S\subseteq[m]$. 
\end{definition}

\noindent A fact we will implicitly use throughout is that SoS proofs compose well:

\begin{fact}
    If there is a degree-$D$ SoS proof that $p(x) \ge 0$ using $\calP$, and a degree-$D'$ SoS proof using $\calP'$, then using the union of the constraints in $\calP$ and $\calP'$, there is a degree-$\max(D,D')$ SoS proof that $p(x)+q(x) \ge 0$ and a degree-$D'$ SoS proof that $p(x)q(x)\ge 0$.
\end{fact}

\noindent It is useful to work with the objects dual to SoS proofs, namely pseudodistributions. 

\begin{definition}[Pseudodistributions]
    A \emph{degree-$d$ pseudodistribution over variables $x_1,\ldots,x_n$} is a linear functional $\wt{\mathbb{E}}$ mapping degree-$D$ polynomials in $x_1,\ldots,x_n$ to reals which additionally satisfies the following properties:
    \begin{enumerate}
        \item Normalization: $\psE{1} = 1$.
        \item Positivity: $\psE{p(x)^2}$ for every degree polynomial $p$ of degree at most $D/2$.
    \end{enumerate}
    We will use the terms ``pseudistribution'' and ``pseudoexpectation'' interchangeably.
    
    We say that a degree-$D$ pseudodistribution $\wt{\mathbb{E}}$ \emph{satisfies} a program $\calP = \brc{p_1(x)\ge 0,\ldots,p_m(x)\ge 0, q_1(x) = 0,\ldots,q_m(x) = 0}$ if for every multiset $S\subseteq[m]$ and sum-of-squares polynomial $r(x)$ for which the degree of $r(x)\cdot \prod_{i\in S}p_i(x)$ is at most $D$, we have $\psE{r(x)\cdot \prod_{i\in S}p_i(x)} \ge 0$, and for every $q(x)$ in the ideal generated by $q_1,\ldots,q_m$ at degree $D$, we have $\psE{q(x)} = 0$.
\end{definition} 

\noindent The following important fact is an immediate consequence of duality for semidefinite programs:

\begin{fact}
    If there is a degree-$D$ SoS proof using the constraints of program $\calP$ that $p(x) \ge 0$, and $\wt{\mathbb{E}}$ is a degree-$D$ pseudoexpectation satisfying $\calP$, then $\wt{\mathbb{E}}$ satisfies $\calP\cup\brc{p(x)\ge 0}$. In particular, $\psE{p(x)} \ge 0$.
\end{fact}

\noindent For any $D\in\mathbb{N}$, given a program over variables $x_1,\ldots,x_n$, one can efficiently compute a degree-$D$ pseudoexpectation in time $n^{O(D)}$ \cite{Nesterov00,parrilo2000structured,Lasserre01,Shor87}.

% Finally, we have the following standard fact which follows by ``SoS Cauchy-Schwarz'' (see e.g. Lemma A.5 of \cite{barak2014rounding}) which will allow us to convert from pseudodistributions over solutions to a polynomial system into integral solutions:

% \begin{lemma}\label{lem:triv_round}
%     For any vector $w^*$ and degree-$D$ pseudoexpectation $\wt{\mathbb{E}}$ over vector-valued indeterminate $w$, we have that $\norm{\psE{w} - w^*}^2_2 \le\psE{\norm{w - w^*}^2_2}$.
% \end{lemma}

% \begin{proof}
%     By the dual definition of $L_2$ norm, the left-hand side can be expressed as $\sup_{v\in\mathbb{S}^{d-1}}\iprod{v,\psE{w} - w^*}^2$. For any $v\in\S^{d-1}$, $\iprod{v,\psE{w} - w^*}^2 = (\psE{\iprod{v,w - w^*}})^2 \le \psE{\iprod{v,w-w^*}^2} \le \psE{\norm{w - w^*}^2_2}$, where the first step is by linearity, the second step is by pseudoexpectation Cauchy-Schwarz (Lemma A.5 of \cite{barak2014rounding}), and the last inequality follows by the fact that Cauchy-Schwarz has a degree-2 SoS proof. Therefore, taking the maximum over all $v\in\S^{d-1}$ proves the inequality.
% \end{proof}

\paragraph{Elementary inequalities in SoS.} We now collect some useful basic inequalities provable in the SoS proof system. The following two facts respectively show that one can implicitly take $t$-th roots on both sides of an SoS inequality, see Appendices~\ref{app:defer_nonneg_power_of_two} and \ref{app:defer_root1} for proof.

\begin{fact}[Roots of zero]\label{fact:nonneg_power_of_two}
    Given $t\in\mathbb{N}$ and indeterminate $x$ satisfying constraint $-\epsilon^t \le x^t \le \epsilon^t$ for some scalar $\epsilon > 0$, there is a degree-$4t$ SoS proof that $x^2 \le \epsilon^2$.
\end{fact}

\begin{fact}[Roots of one]\label{fact:root1}
    Let $j\in\mathbb{N}$ and $0 \le \epsilon < 1$. For an indeterminate $x$ satisfying $-\epsilon \le x^j - 1 \le \epsilon$, there is a  degree-$4j$ SoS proof that $-3\epsilon \le x^2 - 1 \le 3\epsilon$. Furthermore, if $j$ is odd, then there is a degree-$4j$ SoS proof that $-2\epsilon \le x -1 \le 2\epsilon$.
\end{fact}

\noindent The next fact shows that one can also essentially divide by an SoS variable on both sides of an SoS inequality given the constraint that the variable is sufficiently positive.

\begin{fact}[Division on both sides]\label{fact:division}
    For indeterminates $x,y$ and scalar $\epsilon \ge 0$ sufficiently small:
    \begin{enumerate}
        \item If they satisfy the constraints $-\epsilon \le xy \le \epsilon$ and $x \ge \alpha$ for some scalar $\alpha > 0$, there is a degree-4 SoS proof that $-\epsilon/\alpha \le y \le \epsilon/\alpha$. \label{fact:divide_both_sides}
        \item If they satisfy the constraints $-\epsilon \le xy - 1 \le \epsilon$ and $-\delta \le x - 1 \le \delta$ for sufficiently small constant $\delta \ge 0$, then there is a degree-4 SoS proof that $-O(\epsilon+\delta) \le y - 1 \le O(\epsilon+\delta)$. \label{fact:divide_both_sides_big}
    \end{enumerate}
\end{fact}

\paragraph{Orthogonality, norms, and matrix rank.}

Next, we verify that SoS can reason about orthogonal matrices, sub-multiplicativity of matrix norms, and matrix rank. 

The following shows that in SoS that a matrix whose rows are approximately orthonormal must also have columns which are approximately orthonormal, see Appendix~\ref{app:defer_ortho_sos} for a formal proof:

\begin{lemma} \label{lem:ortho_sos}
    Let $0 \le \epsilon \le 1$. If $\brc{x_{ij}: 1 \le i,j \le d}$ is a collection of variables, then there is a degree-4 SoS proof using the constraints
    \begin{equation}
        -\epsilon \le \sum_j x_{ij}^2 - 1 \le \epsilon \ \forall \ i\in[d] \ \ \text{and} \ \ \epsilon \le \sum_j x_{ij} x_{i'j} \le \epsilon \ \forall \ i,i'\in[d] \label{eq:row_ortho}
    \end{equation}
    that
    \begin{equation}
        -4\sqrt{\epsilon d^3} \le \sum_i x_{ij}^2 - 1 \le 4\sqrt{\epsilon d^3} \ \forall \ j\in[d] \ \ \text{and} \ \ -2\sqrt{\epsilon} d \le \sum_i x_{ij} x_{ij'} \le 2\sqrt{\epsilon} d \ \forall \ j,j'\in[d]. \label{eq:column_ortho}
    \end{equation}
\end{lemma}

\noindent We defer a formal proof of this to Appendix~\ref{app:defer_ortho_sos}. Its proof uses the following standard fact that Cauchy-Schwarz has a degree-4 SoS proof:

\begin{fact}\label{fact:cauchy_schwarz}
    Given indeterminates $x_1,\ldots,x_d$ and $y_1,\ldots,y_d$, the following identity holds:
    \begin{equation}
        \sum_{i < j} (x_i y_j - x_j y_i)^2 = \biggl(\sum_i x_i^2\biggr)\biggl(\sum_i y_i^2\biggr) - \biggl(\sum_i x_i y_i\biggr)^2
    \end{equation}
\end{fact}

\noindent It is also easy to prove sub-multiplicativity of Frobenius norm in degree-4 SoS:

\begin{fact}\label{fact:sos_frobenius_mult}
    For any $d\times d$ matrices $M,N$ of indeterminates, there is a degree-4 SoS proof that $\norm{MN}^2_F \le \norm{M}^2_F \norm{N}^2_F$.
\end{fact}

\begin{proof}
    We have $\norm{MN}^2_F = \sum_{i,j} (\sum_k M_{ik} N_{kj})^2 \le \sum_{i,j} (\sum_k M^2_{ik})(\sum_k N^2_{kj}) = \norm{M}^2_F \norm{N}^2_F$.
\end{proof}

\noindent We can also use SoS to reason about low-rank matrices via vanishing of their minors:

\begin{fact}\label{fact:sos_rank}
    Let $\ell\in\mathbb{Z}$. For $r$-dimensional vector-valued indeterminates $v_1,\ldots,v_\ell$, define $M = \sum^{\ell+1}_{i=1} v_iv_i^{\top}$. For any $\brc{a_1,\ldots,a_{\ell+1}}, \brc{b_1,\ldots,b_{\ell+1}}\subset[r]$, there is a degree-$O(\ell)$ SoS proof that
    \begin{equation}
        \sum_{\pi\in\calS_{\ell+1}}\sgn(\pi)\prod^{\ell+1}_{s=1} M_{a_s,b_{\pi(s)}} = 0. \label{eq:determinant}
    \end{equation}
\end{fact}

\begin{proof}
    We can rewrite \eqref{eq:determinant} as
    \begin{equation}
        \sum_{\mb{i}\in[\ell]^{\ell+1}} \sum_{\pi\in\calS_{\ell+1}} \sgn(\pi) \prod^{\ell+1}_{s=1} (v_{i_s})_{a_s} (v_{i_s})_{b_{\pi(s)}} = \sum_{\mb{i}\in[\ell]^{\ell+1}} \prod^{\ell+1}_{s=1}(v_{i_s})_{a_s} \biggl( \sum_{\pi\in\calS_{\ell+1}} \sgn(\pi) \prod^{\ell+1}_{t=1} (v_{i_s})_{b_{\pi(t)}}\biggr).
    \end{equation}
    For any $\mb{i}\in[\ell]^{\ell+1}$, there is at least one element of $[\ell]$ that appears at least twice, so we conclude that $\sum_{\pi\in\calS_{\ell+1}} \sgn(\pi) \prod^{\ell+1}_{t=1} (v_{i_s})_{b_{\pi(t)}} = 0$ as desired.
\end{proof}

\paragraph{Approximation shorthand.}

Finally, we introduce some important shorthands. We use $a = b \pm c$ to denote the inequalities $-c \le a - b \le c$. Throughout this work, we will use the following basic fact under this notation:

\begin{fact}
    From the inequality $x^2 \le a^2$ for scalar $a > 0$, there is a degree-2 SoS proof that $x = \pm a$.
\end{fact}

\begin{proof}
    Observe that $2ax - a^2 \le x^2 - (x-a)^2 \le a^2$, so rearranging we conclude that $-a \le x \le a$, i.e. that $x = \pm a$.
\end{proof}

\noindent Given $r\times r$ matrices $A$ and $B$, we also use the shorthand $A \approx_{\epsilon^2} B$ to denote the polynomial inequality $\norm{A - B}^2_F \le \epsilon^2$. We collect some simple manipulations involving this shorthand:

\begin{fact}\label{fact:shorthand}
    For $i\in[m]$, let $A_i$ and $B_i$ be $r\times r$ matrices for which there is a degree-$d_i$ SoS proof that $A_i\approx_{\epsilon^2_i} B_i$ for some $\epsilon_i > 0$. Denote $A_1,B_1,\epsilon_1,d_1$ by $A,B,\epsilon,d$.
    \begin{enumerate}
        \item For any $r\times r$ matrix of indeterminates $C$, there is a degree-$(d+2)$ SoS proof in the entries of $A,B,C$ that $AC\approx_{\epsilon^2\norm{C}^2_F} BC$ and similarly that $CA\approx_{\epsilon^2\norm{C}^2_F} CB$. \label{shorthand:bothsides}
        \item If $B\approx_{\delta^2} C$, then there is a degree-$d$ Sos proof in the entries of $A,B,C$ that $A\approx_{2\epsilon^2 + 2\delta^2} C$. \label{shorthand:transitive}
        \item There is a degree-$(\max_i d_i + 2m - 2)$ SoS proof in the entries of $\brc{A_i,B_i}$ that $\prod_i A_i \approx_{\epsilon''^2} \prod_i B_i$ for $\epsilon''^2 \triangleq m\sum^m_{i=1} \epsilon^2_i \norm{A_1}^2_F\cdots\norm{A_{i-1}}^2_F \cdot \norm{B_{i+1}}^2_F \cdots \norm{B_m}^2_F$. \label{shorthand:multiply}
        \item For any $\lambda\in\R^m$, there is a degree-$(\max_i d_i)$ SoS proof in the entries of $\brc{A_i,B_i}$ that $\sum_i \lambda_i A_i$ $\approx_{\norm{\lambda}^2_2\cdot \sum_i \epsilon^2_i} \sum_i \lambda_i B_i$. \label{shorthand:lincombo}
    \end{enumerate}
\end{fact}

\begin{proof}
    The first part follows immediately from Fact~\ref{fact:sos_frobenius_mult}. The second part follows from Cauchy-Schwarz. For the third part, we have in degree-4 SoS that
    \begin{equation}
        \Big\|\prod_i A_i - \prod_i B_i\Big\|^2_F = \Big\|\sum^m_{i=1} A_1\cdots A_{i-1} (A_i-B_i) B_{i+1}\cdots B_m\Big\|^2_F \le \epsilon''^2
    \end{equation}
    where in the second step we used squared triangle inequality, and in the last step we used Fact~\ref{fact:sos_frobenius_mult} and Cauchy-Schwarz. The fourth part follows by Cauchy-Schwarz:
    \begin{align}
        \Big\|\sum_i \lambda_i (A_i - B_i)\Big\|^2 &\le \norm{\lambda}^2 \cdot \sum_i \norm{A_i - B_i}^2_F \le \norm{\lambda}^2 \cdot \sum_i \epsilon^2_i. \qedhere
    \end{align}
\end{proof}

\noindent The following is an immediate consequence of Cauchy-Schwarz:

\begin{fact}\label{fact:frob_to_trace}
    If $A\approx_{\epsilon^2} B$, then there is a degree-2 SoS proof in the entries of $A,B$ that $\Tr(A-B)^2 \le r\epsilon^2$ and in particular that $-\epsilon\sqrt{r} \le \Tr(A) - \Tr(B) \le \epsilon\sqrt{r}$.
\end{fact}

\section{Learning Polynomial Transformations}
\label{sec:connect}

In this section, we establish the connection between the inverse problems of Section~\ref{sec:inverse_problems}, tensor ring decomposition and low-rank factorization, to the problem of learning polynomial transformations:

\begin{theorem}\label{thm:quadratic_to_tensorring}
    Let $\epsilon > 0$. Suppose there is an algorithm for tensor ring decomposition (Definition~\ref{def:tensorring}) that, given as input $S,T$ satisfying     \begin{equation}
        \abs*{\Tr(Q^*_a Q^*_b) - S_{a,b}} \le \eta \ \ \forall \ a,b\in[d] \label{eq:estimate_second_moment}
    \end{equation}
    \begin{equation}
        \abs*{\Tr(Q^*_a Q^*_b Q^*_c) - T_{a,b,c}} \le \eta \ \ \forall \ a,b,c\in[d]. \label{eq:estimate_third_moment}
    \end{equation}
    for some $\eta = \eta(\epsilon)$, runs in time $T$ and with high probability outputs symmetric matrices $\wh{Q}_1,\ldots,\wh{Q}_d$ for which $\gaugedist(\brc{Q^*_a},\brc{\wh{Q}_a})\le \epsilon$.
    
    Then there is an algorithm for parameter learning the transformation $\calD$ given by quadratic network $Q^*_1,\ldots,Q^*_d$ to error $\epsilon$ with high probability that draws $O(r^3\radius^6\log^3(2d/\delta)/\eta(\epsilon)^2)$ samples and runs in time $T$. Furthermore, this algorithm also solves proper density estimation to Wasserstein error $O(\epsilon r\sqrt{d})$ with high probability.
\end{theorem}

\begin{theorem}\label{thm:lowrank_to_factorization}
    Let $\epsilon > 0$ and define
    \begin{equation}
        \Sigma\triangleq \E[g\sim\calN(0,\Id_r)]{g^{\otimes\omega}(g^{\otimes\omega})^{\top}}. \label{eq:Sigma_prelim_def}
    \end{equation}
    Suppose there is an algorithm for low-rank factorization (Definition~\ref{def:lowrankfactorization}) that, given as input $S$ satisfying
    \begin{equation}
        \abs*{\iprod{T^*_a,T^*_b}_{\Sigma} - S_{a,b}} \le \eta \ \ \forall \ a,b\in[d] \label{eq:estimate_second_moment_2}
    \end{equation}
    for some $\eta = \eta(\epsilon)$, runs in time $T$ and with high probability outputs $\wh{T}_1,\ldots,\wh{T}_d$ for which we have $\gaugedist(\brc{T^*_a},\brc{\wh{T}_a}) \le \epsilon$.
    
    Then there is an algorithm for parameter learning the transformation $\calD$ given by low-rank polynomial network $T^*_1,\ldots,T^*_d$ to error $\epsilon$ with high probability that draws $O(\omega r)^{2\omega} \radius^4 \log^{2\omega}(\delta/d)/\eta(\epsilon)^2$ samples and runs in time $T$. Furthermore, this algorithm also solves proper density estimation to Wasserstein error $O(\epsilon r\sqrt{d})$ with high probability.
\end{theorem}

\subsection{Quadratic Transformations}
\label{sec:quadratic_reduce}

Here we establish the connection between method of moments for learning quadratic transformations and tensor ring decomposition, and tensor ring decomposition and low-rank factorization. Throughout this section, let $\calD$ be a $d$-dimensional degree-2 transformation with seed length $r$ that is specified by the polynomial network $Q^*_1,\ldots,Q^*_d\in\R^{r\times r}$. 

\begin{lemma}\label{lem:moments}
    If $z$ is a sample from $\calD$, then for any $a,b,c\in[d]$,
    \begin{equation}
        2\Tr(Q^*_a Q^*_b) = \E{(z_a - \E{z_a})(z_b - \E{z_b})}
    \end{equation}
    \begin{equation}
        8\Tr(Q^*_a Q^*_b Q^*_c) = \E{(z_a-\E{z_a})(z_b-\E{z_b})(z_c-\E{z_c})}
    \end{equation}
\end{lemma}

\begin{proof}
    For any $a,b,c\in[d]$, we have by Isserlis' theorem that
    \begin{equation}
        \E{z_a} = \E[x\sim\calN(0,\Id)]{x^{\top} Q^*_a x} = \Tr(Q^*_a).
    \end{equation}
    \begin{equation}
        \E{z_a z_b} = \E[x\sim\calN(0,\Id)]{(x^{\top}Q^*_a x)\cdot (x^{\top}Q^*_b x)} = \Tr(Q^*_a)\Tr(Q^*_b) + 2\Tr(Q^*_a Q^*_b).
    \end{equation}
    \begin{align}
        \E{z_a z_b z_c} &= \E[x\sim\calN(0,\Id)]{(x^{\top}Q^*_a x)\cdot (x^{\top}Q^*_b x)\cdot (x^{\top}Q^*_c x)} \\
        &= \Tr(Q^*_a)\Tr(Q^*_b)\Tr(Q^*_c) + 2\Tr(Q^*_a) \Tr(Q^*_b Q^*_c) + 2\Tr(Q^*_b) \Tr(Q^*_a Q^*_c)\\
        & \qquad + 2\Tr(Q^*_c)\Tr(Q^*_a Q^*_b) + 8\Tr(Q^*_a Q^*_b Q^*_c),
    \end{align}
    where for the last identity, we used the fact that for any symmetric matrices $A_1,A_2,A_3$,
    \begin{equation}
        \Tr(A_1A_2A_3) = \Tr(A_{\pi(1)}A_{\pi(2)}A_{\pi(3)}) \ \ \forall \ \pi\in\calS_3.
    \end{equation}
    % As 
    % \begin{equation}
    %     \E{(z_a - \E{z_a})(z_b - \E{z_b})} = \E{z_a z_b} - \E{z_a}\E{z_b}
    % \end{equation}
    % \begin{equation}
    %     \E{(z_a-\E{z_a})(z_b-\E{z_b})(z_c-\E{z_c})} = \E{z_a z_b z_c} - \E{z_a}\E{z_b z_c} - \E{z_b}\E{z_a z_c} - \E{z_c}\E{z_a z_b} + 2\E{z_a}\E{z_b}\E{z_c},
    % \end{equation}
    The lemma follows immediately from the above moment calculations.
\end{proof}

\begin{lemma}[Empirical moment estimation]\label{lem:estimate_moments}
    For any $\eta,\delta > 0$, if $\norm{Q^*_a}^2_F \le \radius^2$ for all $a\in[d]$, there is an algorithm that takes $O(r^3\radius^6\log^3(2d/\delta)/\eta^2)$ samples from $\calD$ and with probability at least $1 - \delta$ outputs $S\in\R^{d\times d}$ and $T\in\R^{d\times d\times d}$ satisfying \eqref{eq:estimate_second_moment} and \eqref{eq:estimate_third_moment}.
\end{lemma}

We defer the proof of this to Appendix~\ref{app:defer_estimate_moments}. Theorem~\ref{thm:quadratic_to_tensorring} now immediately follows from Lemma~\ref{lem:estimate_moments}:

\begin{proof}[Proof of Theorem~\ref{thm:quadratic_to_tensorring}]
    The guarantee for parameter learning follows from Lemma~\ref{lem:estimate_moments}. The guarantee for proper density estimation follows from Lemma~\ref{lem:param_to_wasserstein}.
\end{proof}

\paragraph{Extending to general rotation-invariant seeds.} While it would appear that the reduction above makes use of the special structure of Gaussian moments, our approach easily extends to any rotation-invariant seed distribution $D$ which is reasonably concentrated so that the corresponding transformation moments can be estimated from samples as in Lemma~\ref{lem:estimate_moments}. The reason for this comes from the following elementary observation about moments of rotation-invariant distributions, whose proof we defer to Appendix~\ref{app:defer_rot_to_gaussian}.

\begin{lemma}\label{lem:rot_to_gaussian}
    For any rotation-invariant distribution $D$ over $\R^r$ and any degree-$e$ homogeneous polynomial $q:\R^r\to\R$, $\E[x\sim D]{q(x)} = C_{D,e}\cdot \E[g\sim\calN(0,\Id)]{q(g)}$ for $C_{D,e} \triangleq \frac{\Gamma(e/2)}{2^e\cdot \Gamma((r+e)/2)}\cdot \E[x\sim D]{\norm{x}^e}$.
\end{lemma}

\noindent So from the second-, fourth-, and sixth-order moments of any rotation-invariant $D$, we can extract the quantities $\Tr(Q^*_aQ^*_b)$ and $\Tr(Q^*_aQ^*_bQ^*_c)$ as in Lemma~\ref{lem:moments} even when $D$ is not $\calN(0,\Id)$, provided we know $\E[x\sim D]{\norm{x}^e}$ for $e = 2,4,6$. Regarding this last point, we note that it is entirely reasonable to assume that these quantities, in fact even a description of $D$ itself, is known to the algorithm designer: in the practice of generative models one has complete control over the seed distribution/prior that is used.

\subsection{Low-Rank Transformations}
\label{sec:lowrank_reduce}

Here we establish the connection between method of moments for learning low-rank transformations and low-rank factorization. Throughout this section, let $\calD$ be a $d$-dimensional degree-$\omega$ transformation with seed length $r$ that is specified by the low-rank polynomial network $T^*_1,\ldots,T^*_d\in(\R^r)^{\otimes\omega}$. 

\begin{lemma}\label{lem:moments2}
     If $z$ is a sample from $\calD$, then for any $a,b\in[d]$,
    \begin{equation}
        \iprod{T^*_a,T^*_b}_\Sigma = \E{z_a z_b},
    \end{equation}
    where $\Sigma$ is defined in \eqref{eq:Sigma_prelim_def}.
\end{lemma}

\begin{proof}
    This follows from
    \begin{align}
        \E{z_a z_b} &= \E[g\sim\calN(0,\Id_r)]{\iprod{T^*_a,g^{\otimes \omega}}\iprod{T^*_b,g^{\otimes\omega}}} = \vec(T^*_a)\E{g^{\otimes\omega}(g^{\otimes\omega})^{\top}}\vec(T^*_b). \qedhere
    \end{align}
\end{proof}

\begin{lemma}[Empirical moment estimation]\label{lem:estimate_moments_2}
    For any $\eta,\delta > 0$, if $\norm{T^*_a}^2_F \le \radius^2$ for all $a\in[d]$, there is an algorithm that takes $O(\omega r)^{2\omega} \radius^4 \log^{2\omega}(\delta/d)/\eta^2$ samples from $\calD$ and with probability at least $1 - \delta$ outputs $S\in\R^{d\times d}$ satisfying \eqref{eq:estimate_second_moment_2}.
\end{lemma}

We defer the proof of this to Appendix~\ref{app:defer_estimate_moments_2}. Theorem~\ref{thm:lowrank_to_factorization} now immediately follows from Lemma~\ref{lem:estimate_moments_2}:

\begin{proof}[Proof of Theorem~\ref{thm:lowrank_to_factorization}]
    The guarantee for parameter learning follows from Lemma~\ref{lem:estimate_moments_2}. The guarantee for proper density estimation follows from Lemma~\ref{lem:param_to_wasserstein}.
\end{proof}

\paragraph{Extending to general rotation-invariant seeds.} Note that Lemma~\ref{lem:moments2} makes no use of the fact that the transformation has seed distribution given by $\calN(0,\Id)$, so our reduction from learning low-rank transformations to low-rank factorization easily carries over to any known seed distribution $D$ which is sufficiently well-concentrated that the pairwise moments of $\calD$ can be estimated from samples as in Lemma~\ref{lem:estimate_moments_2} and for which the corresponding low-rank factorization problem with $\Sigma$ now given by $\E[x\sim D]{\vec(x^{\otimes\omega})\vec(x^{\otimes\omega})^{\top}}$ is tractable. As we show in Section~\ref{sec:rotationinvariant}, our algorithm for low-rank factorization applies to any $\Sigma$ of this form for which the seed distribution $D$ is \emph{rotation-invariant} and for which very mild condition number bounds hold. In Section~\ref{sec:rotationinvariant}, we also give an algorithm for low-rank factorization when $\Sigma = \Id$, which yields a learning algorithm for a certain family of \emph{inhomogeneous} polynomial transformations given by one hidden layer networks with Hermite polynomial activations (see Remark~\ref{remark:hermite}).

\section{Tensor Ring Decomposition}
\label{sec:tensorring}

Recall that in tensor ring decomposition (Definition~\ref{def:tensorring}), we are given $S\in\R^{d\times d}$ and $T\in\R^{d\times d\times d}$ such that there exist unknown symmetric matrices $Q^*_1,\ldots,Q^*_d\in\R^{r\times r}$ satisfying
\begin{equation}
    |S_{a,b} - \Tr(Q^*_a Q^*_b)| \le \eta \qquad \text{and} \qquad |T_{a,b,c} - \Tr(Q^*_a Q^*_b Q^*_c)| \le \eta  \ \ \forall \ a,b,c\in[d] \label{eq:assume_momentmatch}
\end{equation}
In this section we give a polynomial-time algorithm for recovering $Q^*_1,\ldots,Q^*_d$ from $S,T$ under the following assumptions:
\begin{assumption}\label{assume:tensorring}
    For parameters $\radius\ge 1$, $\kappa > 0$,
    \begin{enumerate}
        \item (Scaling) $\norm{Q^*_a}_F \le \radius$ for all $a\in[d]$. \label{assume:scale}
        % \item (Eigengap for $Q^*_1$) $Q^*_1$ has minimum eigengap at least $\gap$. \label{assume:breaksym}
        \item (Condition number bound) $\sigma_{\binom{r+1}{2}}(M^*) \ge \kappa$, where $M^*\in\R^{d\times \binom{r+1}{2}}$ is the matrix whose $(a,(i_1,i_2))$-th entry, for $a\in[d]$ and $1 \le i_1 \le i_2 \le r$, is given by $(Q^*_a)_{i_1 i_2}$.\label{assume:condnumber} % if $i_1 \neq i_2$ and $\frac{1}{\sqrt{2}}(Q^*_a)_{i_1 i_2}$ otherwise. 
    \end{enumerate}
\end{assumption}

\begin{remark}
    Readers familiar with the standard guarantees for Jennrich's algorithm will recognize that Part~\ref{assume:condnumber} of Assumption~\ref{assume:tensorring} is the tensor ring analogue of the condition number assumption in tensor decomposition. Namely, given an estimate of $\sum_i v_i^{\otimes 3}$, Jennrich's algorithm can recover $\brc{v_i}$ provided the matrix whose columns consist of $v_i$ is well-conditioned  (see e.g. \cite[Condition 2.2]{bhaskara2014smoothed}).\footnote{Technically if $\brc{Q^*_a}$ are all diagonal with $(Q^*_a)_{ii} = (v_i)_a$, Part~\ref{assume:condnumber} of Assumption~\ref{assume:tensorring} does not apply because $M^*$ will have many zero entries, but it is straightforward to modify our sum-of-squares algorithm to incorporate the assumption that $\brc{Q^*_a}$ are diagonal to recover the guarantees of Jennrich's algorithm.}
\end{remark}

\noindent One can readily check that Assumption~\ref{assume:tensorring} is gauge-invariant (see Appendix~\ref{app:tensorring_gauge} for the proof):

\begin{lemma}\label{lem:tensorring_gauge_invariant}
    If $\brc{Q^*_a}$ satisfy \eqref{eq:assume_momentmatch} and Assumption~\ref{assume:tensorring} with parameters $\radius,\kappa$, then $\brc{VQ^*_aV^{\top}}$ also satisfy \eqref{eq:assume_momentmatch} and Assumption~\ref{assume:tensorring} with the same parameters for any $V\in O(r)$.
\end{lemma}

\noindent Under Assumption~\ref{assume:tensorring}, we give an algorithm for tensor ring decomposition that runs in time polynomial in all parameters:

\begin{theorem}\label{thm:main_tensorring}
    For $d\ge\binom{r+1}{2}$, suppose $Q^*_1,\ldots,Q^*_d\in\R^{r\times r}$ satisfy Assumption~\ref{assume:tensorring} and $\eta \le O(\frac{\kappa^2}{rd^{3/2}})$, and we are given $S\in\R^{d\times d}$ and $T\in\R^{d\times d\times d}$ satisfying \eqref{eq:assume_momentmatch}. 
    
    Then there is an algorithm {\sc TensorRingDecompose}($S,T$) (see Algorithm~\ref{alg:tensorring}) which runs in time $\poly(d,r)$ and outputs $\wh{Q}_1,\ldots,\wh{Q}_d$ for which $\gaugedist(\brc{Q^*_a},\brc{\wh{Q}_a}) \le \poly(d,r,\radius,1/\kappa)\cdot \eta^c$ for some absolute constant $c > 0$, with high probability.
\end{theorem}

% The following lemma shows that we may assume without loss of generality that $Q^*_1$ is diagonal with diagonal entries sorted in nondecreasing order so that
% \begin{equation}
%     (Q^*_1)_{jj} \ge (Q^*_1)_{ii} + \gap \ \ \forall j > i. \label{eq:Qdiag_gap}
% \end{equation}

\paragraph{Section overview.} Our algorithm is based on rounding the solution to a suitable sum-of-squares relaxation. As such, our analysis is centered around exhibiting a low-degree sum-of-squares proof that the ground truth $\brc{Q^*_a}$ is identifiable from $S,T$. As discussed in Section~\ref{sec:overview}, the gauge symmetry inherent in tensor ring decomposition poses a major challenge for this, because $\brc{Q^*_a}$ is only identifiable up to a global rotation in $\R^r$. In Section~\ref{sec:breakground} we outline our strategy for ``breaking symmetry'' by imposing certain constraints on $\brc{Q^*_a}$ that are without loss of generality but which will \emph{uniquely} identify $\brc{Q^*_a}$. In Section~\ref{sec:sos_tensorring} we then formulate our sum-of-squares program which incorporates this symmetry-breaking strategy. 

The high-level strategy will be to introduce SoS variables $\brc{Q_a}$ that are constrained to have the same pairwise and three-wise moment bounds as in \eqref{eq:assume_momentmatch}, and we would like to prove the $\brc{Q_a}$ are close to $\brc{Q^*_a}$ in Frobenius norm. To show this, we would like to show that the $r^2\times r^2$ linear transformation that maps every $\vec(Q^*_a)$ to $\vec(Q_a)$ behaves like the Kronecker power $\Id^{\otimes 2}_r$. Because every $Q^*_a$ and $Q_a$ is symmetric, there is some ambiguity in formulating this transformation as an SoS variable (recall the discussion at the end of Section~\ref{sec:tensorring_overview} of the technical overview).

In Section~\ref{sec:hiddenrot_tensorring} we make a first attempt by constructing a certain auxiliary $r^2\times r^2$ matrix variable $U$ that, as we show in Section~\ref{sec:Uproperties}, behaves in some respects like this $r^2\times r^2$ transformation.
In Section~\ref{subsec:exploitthird}, we then use the third-order constraints in \eqref{eq:assume_momentmatch} to show that the entries of $U$ satisfy a certain collection of quadratic relations (Lemma~\ref{lem:main_identity}).

In Section~\ref{sec:auxW} we use these quadratic relations to refine $U$ to give another SoS auxiliary variable $W$ which better captures the $r^2\times r^2$ transformation and which also satisfies a similar collection of quadratic relations as $U$ (Lemma~\ref{lem:simpler_identity}). In Section~\ref{sec:diagonal}, we complete the analysis by implementing the aforementioned symmetry-breaking strategy in SoS to show that $W$ is approximately $\Id^{\otimes 2}_r$. In Section~\ref{sec:puttogether} we use this to give our main algorithm {\sc TensorRingDecompose} and prove Theorem~\ref{thm:main_tensorring}. Finally, in Section~\ref{sec:fpt}, we show how to improve the runtime of Theorem~\ref{thm:main_tensorring} to only depend \emph{linearly} on $d$.

\subsection{Breaking Gauge Symmetry for the Ground Truth}
\label{sec:breakground}

A natural approach for breaking symmetry would be to insist without loss of generality that, for instance, $Q^*_1$ is diagonal with sorted entries. If the eigenvalues of $Q^*_1$ are well-separated, then one can check that the only rotations $V\in O(r)$ for which $V^{\top}Q^*_a V = Q^*_a$ for all $a\in[d]$ are those for which $V$ is diagonal with diagonal entries in $\brc{\pm 1}$. If we could additionally insist that, say, the first row of $Q^*_2$ consisted entirely of \emph{strictly positive} entries, this would force $V$ to be the identity and completely break the gauge symmetry.

Of course, it could be that $Q^*_1$ and $Q^*_2$ don't meet the desired criteria for making such assumptions: $Q^*_1$ might have some repeated eigenvalues, or $Q^*_2$ might have a zero entry in its first row.\footnote{When $Q^*_1,Q^*_2$ are smoothed, this will not happen, but in this section we opt for an algorithm that can work under minimal non-degeneracy assumptions even when $\brc{Q^*_a}$ are not smoothed.} But the above strategy is certainly not specific to $Q^*_1$ or $Q^*_2$ or the choice of row in $Q^*_2$. Indeed, it would be enough for this to hold for some fixed linear combinations of $\brc{Q^*_a}$, instead of for $Q^*_1,$ and $Q^*_2$ respectively.

We show that under Assumption~\ref{assume:tensorring}, there is indeed a way to construct such linear combinations. In Appendix~\ref{app:forcegap}, we give an algorithm that takes in $S$ and outputs linear combinations of $\brc{Q^*_a}$ satisfying the desired properties, which we formalize in the definition below:

\begin{definition}\label{def:nondeg}
    We say that $\lambda,\mu\in\S^{d-1}$ are \emph{$\gap$-non-degenerate combinations of $Q^*_1,\ldots,Q^*_d$} if the following two properties hold for
    \begin{equation}
        Q^*_{\lambda} \triangleq \sum_{a\in[d]} \lambda_a Q^*_a \qquad \text{and} \qquad Q^*_{\mu} \triangleq \sum_{a\in[d]} \mu_a Q^*_a.\label{eq:Qstarlam}
    \end{equation}
    \begin{enumerate}
        \item $Q^*_{\lambda}$ has minimum eigengap at least $\gap$.
        \item Let $V^{\top}\Lambda V$ be the eigendecomposition of $Q^*_{\mu}$. Then every entry of $V Q^*_{\mu} V^{\top}$ has magnitude at least $\gap$.
    \end{enumerate}
\end{definition}

\noindent Because Assumption~\ref{assume:tensorring} is gauge-invariant by Lemma~\ref{lem:tensorring_gauge_invariant}, we can assume without loss of generality that $Q^*_{\lambda}$ defined in \eqref{eq:Qstarlam} is diagonal with entries sorted in nondecreasing order. As $Q^*_{\lambda}$ has minimum eigengap at least $\gap$,
\begin{equation}
    (Q^*_{\lambda})_{jj} \ge (Q^*_{\lambda})_{ii} + \gap \ \ \forall \ j > i. \label{eq:Qdiag_gap}
\end{equation}
After diagonalizing $Q^*_\lambda$, the second part of Definition~\ref{def:nondeg} implies that $|(Q^*_{\mu})_{ij}| \ge \gap$ for all $i,j\in[r]$. 

By applying one more joint rotation to $Q^*_1,\ldots,Q^*_d$ given by a diagonal matrix of $\pm 1$ entries, we can additionally assume that the first row of $Q^*_{\mu}$ consists of \emph{nonnegative} entries. That is,
\begin{equation}
    (Q^*_{\mu})_{1j} \ge \gap \ \ \forall \ j\in[r].\label{eq:Qstarfirstrow}
\end{equation}
In the sequel, we will show how to recover $Q^*_1,\ldots,Q^*_d$ in Frobenius norm (as opposed to just parameter distance) by insisting that our estimates also satisfy \eqref{eq:Qdiag_gap} and \eqref{eq:Qstarfirstrow}.

\subsection{A Sum-of-Squares Relaxation}
\label{sec:sos_tensorring}

To prove Theorem~\ref{thm:main_tensorring}, we will use the following sum-of-squares program:

\begin{program}\label{program:basic}
    \begin{center}
        \textsc{(Tensor Ring Decomposition)} \\
    \end{center}
    \noindent\textbf{Parameters:} $\lambda,\mu\in\S^{d-1}$, $S\in\R^{d\times d}$, $T\in\R^{d\times d\times d}$, $\radius\ge 1$, $\kappa, \gap > 0$.
    
    \noindent\textbf{Variables:} Let $Q_1,\ldots,Q_d$ be $r\times r$ matrix-valued variables, and let $L$ be an $\binom{r+1}{2}\times d$ matrix-valued variable. Let $M$ be the $d\times \binom{r+1}{2}$ matrix of indeterminates whose $(a,(i_1,i_2))$-th entry, for $a\in[d]$ and $1 \le i_1 \le i_2 \le r$, is given by $(Q_a)_{i_1 i_2}$. Also define $Q_{\lambda} \triangleq \sum^d_{a=1}\lambda_a Q_a$ and $Q_{\mu} \triangleq \sum^d_{a=1}\mu_a Q_a$.
    
    \noindent\textbf{Constraints:}
    % We consider the minimization
    % \begin{equation}
    %     \argmin_{\psE{\cdot}} \sum^d_{a,b,c=1} |T_{a,b,c} - \Tr(Q_a Q_b Q_c)|^2 + \sum^d_{a,b=1} |S_{a,b} - \Tr(Q_a Q_b)|^2
    % \end{equation}
    % subject to the following constraints:
    \begin{enumerate}[leftmargin=*,topsep=0pt]
        \setlength\itemsep{0em}
        \item (Symmetry): $Q_a = Q^{\top}_a$ for all $a\in[d]$. \label{constraint:sym}
        \item (Second moments match): $-\eta \le \Tr(Q_a Q_b) - S_{a,b} \le \eta$ for all $a,b\in[d]$. \label{constraint:second}
        \item (Third moments match): $-\eta \le \Tr(Q_a Q_b Q_c) - T_{a,b,c} \le \eta$ for all $a,b,c\in[d]$. \label{constraint:third}
        \item ($Q$'s bounded): $\norm{Q_a}^2_F \le \radius^2$ for all $a\in[d]$. \label{constraint:Qbound}
        \item (Left-inverse $L$): $L M = \Id$ \label{constraint:pinv} % if $i_1 \neq i_2$ and $\frac{1}{\sqrt{2}}(Q_a)_{i_1 i_2}$ otherwise.\label{constraint:pinv}
        \item ($L$ bounded): $\norm{L}^2_F \le r^2/\kappa^2$. \label{constraint:pbound}
        % entries
        % \begin{equation}
        %     M_{a,ij} \triangleq M_{a,(i-1)d + j} = (Q_a)_{ij}
        % \end{equation} for $a\in[d]$, $i,j\in[r]$.
        \item ($Q_{\lambda}$ diagonal): $(Q_\lambda)_{ij} = 0$ for all $i\neq j$. \label{constraint:diag}
        \item ($Q_{\lambda}$ sorted): $(Q_\lambda)_{jj} \ge (Q_\lambda)_{ii}$ for all $j > i$. \label{constraint:sorted}
        \item ($Q_{\mu}$'s first row): $(Q_\mu)_{1j} \ge 0$ for all $j\in[r]$. \label{constraint:firstrow}
    \end{enumerate}
\end{program}

\noindent We can easily verify that the ground truth is feasible.
\begin{lemma}\label{lem:feasible}
    When $d \ge \binom{r+1}{2}$, the pseudodistribution given by the point distribution supported on $(Q^*_1,\ldots,Q^*_d,L^*)$, where $L^*$ is the left inverse of $M^*$, is a feasible solution to Program~\ref{program:basic}.
\end{lemma}

\begin{proof}
    Note that $L^*$ is well-defined by Part~\ref{assume:condnumber} of Assumption~\ref{assume:tensorring}. It is immediate that Constraints~\ref{constraint:sym}-\ref{constraint:pinv} are satisfied, and Constraints~\ref{constraint:diag}-\ref{constraint:firstrow} are satisfied by \eqref{eq:Qdiag_gap} and \eqref{eq:Qstarfirstrow}. For Constraint~\ref{constraint:pbound}, note that $\norm{L^*}_{\op} \le 1/\kappa$ by Part~\ref{assume:condnumber} of Assumption~\ref{assume:tensorring}, so $\norm{L^*}^2_{F} \le \binom{r+1}{2}/\kappa^2 \le r^2/\kappa^2$.
\end{proof}

\noindent The main result we will show about this sum-of-squares program is the following:

\begin{theorem}\label{thm:main_pairwise}
    Suppose Assumption~\ref{assume:tensorring} holds, and for any $\lambda,\mu\in\S^{d-1}$ let $\psE{\cdot}$ be a degree-96 pseudo-expectation over the variables $Q_1,\ldots,Q_d,L$ satisfying the constraints of Program~\ref{program:basic}.

    Then if $\lambda,\mu$ are $\gap$-non-degenerate combinations of $Q^*_1,\ldots,Q^*_d$ for some $\gap > 0$, then $\norm{\psE{Q_a} - Q^*_a}_F \le \poly(d,r,\radius,1/\kappa,1/\gap)\cdot \eta^c$ for all $a\in[d]$ for some absolute constant $c > 0$.
    % \begin{equation}
    %     \abs*{\psE{(Q_a)_{ij} (Q_b)_{ij}} - (Q^*_a)_{ij} (Q^*_b)_{ij}} \le O(\epsnorm\radius^2 + \sqrt{\epsoffdiag} r^2 \radius^2 + \sqrt{\epsoffdiag} r\radius\alpha + \epsmap\alpha + \epsrel r^3\alpha)
    % \end{equation}
    % \begin{equation}
    %     \abs*{\psE{(Q_a)_{ij} (Q_a)_{jk} (Q_a)_{ik}} - (Q^*_a)_{ij} (Q^*_a)_{jk} (Q^*_a)_{ik}} \le O(\epsrel + \epsnorm + \epsoffdiag r^2)
    % \end{equation}for all $a,b\in[d], i,j,k\in[r]$. 
    % Additionally, $\norm{\psE{Q_\lambda} - Q^*_\lambda}_{\op} \le O(\epsmap + \epsrel r^3 + \sqrt{\epsoffdiag} r^2 \radius + \epsnorm)$. 
\end{theorem}

\subsection{Hidden Rotation Variable}
\label{sec:hiddenrot_tensorring}

In this section we use the SoS variables of Program~\ref{program:basic} to design an auxiliary ``rotation variable'' $U$ that will play the role of the unknown linear transformation sending every $Q^*_a$ to $Q_a$, after which the focus of our analysis in subsequent sections will be to show this transformation qualitatively behaves like $\Id^{\otimes 2}_r$.

First, define the $d\times r^2$ matrix $N^*$ (resp. $N$) to be the matrix whose $(a,(i_1,i_2))$-th entry is given by $(Q^*_a)_{i_1i_2}$ (resp. $(Q_a)_{i_1i_2}$) for all $a\in[d]$, $i,j\in[r]$. Note that $M,M^*$ are submatrices of $N,N^*$. Because $\Tr(Q^*_a Q^*_b) = (N^*{N^*}^{\top})_{ab}$ and $\Tr(Q_a Q_b) = (NN^{\top})_{ab}$, the first part of Eq.~\eqref{eq:assume_momentmatch} and Constraint~\ref{constraint:second} imply that $\norm{NN^{\top} - N^*{N^*}^{\top}}_{\max} \le \eta$.

A natural way to encode the unknown linear transformation from $Q^*_a$ to $Q_a$ as an auxiliary variable would be to consider something like $N^{-1}N^*$, because $(N^{-1}N^*){N^*}^{\top} \approx N^{\top}$, and the $a$-th column of this approximate equality between matrices implies that the transformation $N^{-1}N^*$ maps $Q^*_a$ to $Q_a$. By right multiplying this approximate equality by $(N^{-1})^{\top}$, we also see that $N^{-1}N^*$ is approximately orthogonal.

Of course, strictly speaking such a construction isn't well-defined: $N$ is an SoS variable, so there is no meaningful notion of a left inverse $N^{-1}$. In fact there isn't even a suitable left inverse for the \emph{scalar} matrix $N^*$, as $N^*$ has duplicate columns (because every $Q^*_a$ is symmetric). Nevertheless, we will use $L$ as a proxy for $N^{-1}$ and, with a few modifications, our construction of the ``rotation variable'' $U$ will behave like $N^{-1}N^*$.

Formally, to construct $U$, first define the $\binom{r+1}{2}\times r^2$ matrix $\wh{U}$ by
\begin{equation}
    \wh{U} \triangleq LN^*. \label{eq:Uhatdef}
\end{equation}
Then define the $r^2\times r^2$ matrix $U$ as follows. For any $i_1,i_2\in[r]$, the $(i_1,i_2)$-th row of $U$ is given by
\begin{equation}
    U_{i_1i_2} = \begin{cases}
        \wh{U}_{i_1i_2} & \text{if} \ i_1 = i_2 \\
        \frac{1}{2}\wh{U}_{i_1i_2} & \text{if} \ i_1 < i_2 \\
        \frac{1}{2}\wh{U}_{i_2i_1} & \text{if} \ i_1 > i_2
    \end{cases}
    \label{eq:Udef}
\end{equation}
When the context is clear, we will refer to $\mat(U_{i_1i_2})$ as simply $U_{i_1i_2}$, and similarly for any $j_1,j_2\in[r]$, we will refer to $\mat(U^{j_1j_2})$ as simply $U^{j_1j_2}$. Note that the entries of $U$ are (unknown) linear forms in the indeterminate entries of $L$.

\subsection{Basic Properties of \texorpdfstring{$U$}{U}}
\label{sec:Uproperties}

In this section we establish the following simple facts about $U$: 
\begin{enumerate}
    \item $U$ is ultra-symmetric in the sense of Definition~\ref{def:ultrasym} (Lemma~\ref{lem:symmetric})
    \item $U$ approximately maps every $\vec(Q^*_a)$ to $\vec(Q_a)$ (Lemma~\ref{lem:fakeUsendsQ})
    \item For any $j_1,j_2$, $\norm{U^{j_1j_2}}^2_F \approx (1/2)^{\bone{j_1\neq j_2}}$ (Lemma~\ref{lem:Unorm12}).
\end{enumerate}
Note that properties 2 and 3 are consistent with the heuristic that $U$ qualitatively behaves like ``$N^{-1}N^*$'' from the discussion in Section~\ref{sec:hiddenrot_tensorring}.

\begin{lemma}\label{lem:symmetric}
    $U$ is ultra-symmetric.
\end{lemma}

\begin{proof}
    The fact that $U_{i_{\tau(1)} i_{\tau(2)}} = U_{i_1i_2}$ follows from the definition of $U$ in \eqref{eq:Udef}. The fact that $U^{j_{\pi(1)} j_{\pi(2)}} = U^{j_1j_2}$ follows from the definition of $U$ and the fact that the matricization of any row of $\wh{U}$ is symmetric: $\wh{U}^{j_1j_2}_{i_1i_2} = \sum^d_{c=1} L^c_{i_1i_2} (Q^*_c)_{j_1j_2} = \sum^d_{c=1} L^c_{i_1i_2} (Q^*_c)_{j_2j_1} = \wh{U}^{j_2j_1}_{i_1i_2}$.
\end{proof}

\noindent To show the remaining two properties, first define the $\binom{r+1}{2} \times r^2$ matrix $I' \triangleq LN$ and observe that for any $1\le i_1\le i_2\le r$ and $j_1,j_2\in[r]$,
\begin{equation}
    {I'}^{j_1j_2}_{i_1i_2} = \bone{(i_1,i_2) = (j_1,j_2) \ \text{or} \ (i_1,i_2) = (j_2,j_1)}\label{eq:Iprime}
\end{equation}
by Constraint~\ref{constraint:pinv}. Also define the matrix-valued indeterminate
\begin{equation}
    \calE \triangleq N N^{\top} - N^*{N^*}^{\top}, \label{eq:Edef} %\ \text{and} \ \calE'\triangleq NX^{\top} - N^*{X^*}^{\top} \label{eq:Edef}
\end{equation}
noting that $\norm{\calE}_{\max} \le 2\eta$ by the first part of \eqref{eq:assume_momentmatch} and Constraint~\ref{constraint:second}.

Left-multiplying both sides of \eqref{eq:Edef} by $L$ and recalling the definition of $\wh{U}$ from \eqref{eq:Uhatdef}, we get that
\begin{equation}
    L\calE = I'N^{\top} - \wh{U}{N^*}^{\top}. \label{eq:starnostar2}
\end{equation}

\begin{lemma}\label{lem:almost_map}
    Define 
    \begin{equation}
        \epsmap \triangleq 2\eta rd/\kappa \label{eq:epsmap_Q_def}
    \end{equation}
    Then there is a degree-4 SoS proof using Constraints~\ref{constraint:second} and \ref{constraint:pbound} of Program~\ref{program:basic} that $\norm{L\calE}^2_F \le \epsmap^2$.
\end{lemma}

\begin{proof}
    We have
    \begin{equation}
        \norm{L\calE}^2_F = \sum_{1\le i\le j\le r, b\in[d]} \biggl(\sum_{a\in[d]} L_{ij,a} \calE_{ab}\biggr)^2 \le \sum_{i\le j,b} \biggl(\sum_a L^2_{ij,a}\biggr)\biggl(\sum_a \calE^2_{ab}\biggr) \le 4\eta^2 d^2 \sum_{i\le j,a} L^2_{ij,a} \le 4\eta^2r^2d^2/\kappa^2
    \end{equation}
    as claimed, where in the penultimate step we used the above observation that $\norm{\calE'}_{\max} \le \eta^2$ by Constraint~\ref{constraint:second}, and in the last step we used Constraint~\ref{constraint:pbound}.
    % The bound on $\norm{L\calE'}^2_F$ follows by an identical calculation.
\end{proof}

\noindent As a consequence, we can deduce property 2:

\begin{lemma}\label{lem:fakeUsendsQ}
    For any $a\in[d]$, there is a degree-6 SoS proof using Constraints~\ref{constraint:second} and \ref{constraint:pinv} and Lemma~\ref{lem:almost_map} that $F_U(Q^*_a) \approx_{\epsmap^2} Q_a$.
\end{lemma}

\begin{proof}
    For any $i\in[r]$, note that
    \begin{align}
        F_U(Q^*_a)_{ii} &= \sum_{j_1,j_2\in[r]} U^{j_1j_2}_{ii} (Q^*_a)_{j_1j_2} = \sum_{j_1,j_2} \wh{U}^{j_1j_2}_{ii} (Q^*_a)_{j_1j_2} \\
        &= (\wh{U}{N^*}^{\top})^a_{ii} = (I'N^{\top} - L\calE)^a_{ii} = (Q_a)_{ii} - (L\calE)^{a}_{ii},
    \end{align}
    where in the second step we used \eqref{eq:starnostar2} which is a degree-3 polynomial equality using Constraints~\ref{constraint:second} and \ref{constraint:pinv}. Similarly, given $1 \le i_1 < i_2 \le r$,
    \begin{align}
        F_U(Q^*_a)_{i_1i_2} &= \sum_{j_1,j_2\in[r]} U^{j_1j_2}_{i_1i_2} (Q^*_a)_{j_1j_2} = \frac{1}{2}\sum_{j_1,j_2} \wh{U}^{j_1j_2}_{i_1i_2} (Q^*_a)_{j_1j_2} \\
        &= \frac{1}{2}(\wh{U} {N^*}^{\top})^a_{i_1i_2} = \frac{1}{2}(I' N^{\top} - L\calE)^a_{i_1i_2} = \frac{1}{2}\sum_{j_1,j_2} {I'}^{j_1j_2}_{i_1i_2} (Q_a)_{j_1j_2} - \frac{1}{2}(L\calE)^a_{i_1i_2} \\
        &= \frac{1}{2}((Q_a)_{i_1i_2} + (Q_a)_{i_2i_1} - (L\calE)^a_{i_1i_2}) = (Q_a)_{i_1i_2} - \frac{1}{2}(L\calE)^a_{i_1i_2}.
    \end{align}
    The case of $i_1 > i_2$ holds analogously. By squaring these identities, we can thus conclude in degree-6 SoS from Lemma~\ref{lem:almost_map} that $\norm{F_U(Q^*_a) - Q_a}^2_F \le \norm{L\calE}^2_F \le \epsmap^2$.
\end{proof}

\noindent Finally, to show property 3, we need the following calculation:

\begin{lemma}\label{lem:almost_ortho}
    Define
    \begin{equation}
        \epsort \triangleq 2\eta r^2 d / \kappa^2. \label{eq:epsprimedef}
    \end{equation}
    Then there is a degree-8 SoS proof using Constraints~\ref{constraint:second}, \ref{constraint:pinv}, and \ref{constraint:pbound} that $\wh{U}\wh{U}^{\top} \approx_{\epsort^2} I'{I'}^{\top}$.
\end{lemma}

\begin{proof}
    Further right-multiplying the first equation in \eqref{eq:starnostar2} by $L^{\top}$, we find in degree-4 SoS using Constraint~\ref{constraint:pinv} that
    \begin{equation}
        \wh{U}\wh{U}^{\top} - I'{I'}^{\top} = L\calE L^{\top}.
    \end{equation}
    Taking the squared Frobenius norm on both sides in degree-8 SoS, we get the desired bound on $\norm{\wh{U}\wh{U}^{\top} - I'{I'}^{\top}}^2_F$ from
    \begin{align}
        \norm{L\calE L^{\top}}^2_F &= \sum_{i,j\in\left[\binom{r+1}{2}\right]} \biggl(\sum_{a,b\in[d]} \calE_{ab} L_{ia} L_{jb}\biggr)^2 \le \sum_{i,j\in\left[\binom{r+1}{2}\right]} \biggl(\sum_{a,b\in[d]} \calE_{ab} L^2_{ia}\biggr) \biggl(\sum_{a,b\in[d]} \calE_{ab} L^2_{jb}\biggr) \\
        &\le 4\eta^2 d^2 \biggl(\sum_{i,a} L^2_{ia}\biggr)\biggl(\sum_{j,b} L^2_{jb}\biggr) = 4\eta^2 d^2 \norm{L}^4_F \le 4\eta^2 r^4 d^2/\kappa^4. \qedhere
    \end{align}
\end{proof}

\noindent Lemma~\ref{lem:almost_ortho} approximately tells us that the \emph{rows} of $\wh{U}$ are orthogonal and what their norms are. It is then straightforward, using Lemma~\ref{lem:ortho_sos}, to deduce approximately what the norms of the \emph{columns} of $\wh{U}$ (and thus also of $U$) are.

\begin{lemma}\label{lem:Unorm12}
    For any $j_1,j_2\in[r]$, there is a degree-4 SoS proof using Lemma~\ref{lem:almost_ortho} that 
    \begin{equation}
        -O(\sqrt{\epsort} r^3) \le \norm{U^{j_1j_2}}^2_F - (1/2)^{\bone{j_1 \neq j_2}} \le O(\sqrt{\epsort} r^3).
    \end{equation}
\end{lemma}

\noindent We defer the formal proof of this to Section~\ref{app:defer_Unorm12}.

\subsection{Using Third Moments}
\label{subsec:exploitthird}

Thus far we haven't made use of the second part of \eqref{eq:assume_momentmatch} or Constraint~\ref{constraint:third}. Indeed, without them, we can't hope to say more about $U$ than that it qualitatively behaves like a scaled $r^2\times r^2$ rotation. In this section, we use them to show that the entries of $U$ satisfy a certain collection of quadratic relations that will be crucial to establishing that $U$ behaves like the Kronecker power of an $r\times r$ rotation.

First, analogous to the $N^*$ and $N$ matrices defined in the preceding section, define the $d^2\times r^2$ matrix $X^*$ (resp. $X$) to be the matrix whose $(a,b)$-th row is given by $\vec(Q^*_a Q^*_b)$ (resp. $\vec(Q_a Q_b)$) for all $a,b\in[d]$. By design, $(M^*{X^*}^{\top})^{bc}_a = \Tr(Q^*_a Q^*_b Q^*_c)$ and $(MX^{\top})^{bc}_a = \Tr(Q_a Q_b Q_c)$. Analogous to \eqref{eq:Edef}, we can define the matrix-valued indeterminate
\begin{equation}
    \calE'\triangleq NX^{\top} - N^*{X^*}^{\top}. \label{eq:Eprimedef}
\end{equation} Then by \eqref{eq:assume_momentmatch} and Constraint~\ref{constraint:third}, we conclude that $\norm{NX^{\top} - N^*{X^*}^{\top}}_{\max} \le 2\eta$.

Left-multiplying both sides of \eqref{eq:Eprimedef} by $L$, we have, analogously to \eqref{eq:starnostar2}, that
\begin{equation}
    L\calE' = I'X^{\top} - (LN^*){X^*}^{\top}.  \label{eq:starnostar3}
\end{equation}

\noindent Then by a completely analogous calculation to Lemma~\ref{lem:almost_map}, have the following bound on the left-hand side of \eqref{eq:starnostar3}:

\begin{lemma}\label{lem:LEprime_bound}
    There is a degree-4 SoS proof using Constraints~\ref{constraint:third} and \ref{constraint:pbound} that $\norm{L\calE'}^2_F \le \epsmap^2$.
\end{lemma}

\noindent We now establish the following analogue of Lemma~\ref{lem:fakeUsendsQ} from the previous section that, roughly speaking, says that $U$ approximately maps rows of $X^*$ to (symmetrizations of) rows of $X$.

\begin{lemma}\label{lem:fakeUsendsQaQb}
    For any $a,b\in[d]$, there is a degree-8 SoS proof using Constraints~\ref{constraint:third} and \eqref{constraint:pbound} and Lemma~\ref{lem:LEprime_bound} that $F_U(Q^*_a Q^*_b) \approx_{\epsmap^2} \frac{1}{2}(Q_a Q_b +Q_b Q_a)$.
\end{lemma}

\begin{proof}
    For $i\in[r]$, we have in degree-4 SoS that
    \begin{align}
        F_U(Q^*_a Q^*_b)_{ii} &= \sum_{j_1,j_2\in[r]} U^{j_1j_2}_{ii} (Q^*_a Q^*_b)_{ii} = \sum_{j_1,j_2} \wh{U}^{j_1j_2}_{ii} (Q^*_a Q^*_b)_{j_1j_2} = (I'X^{\top} - L\calE')^{ab}_{ii} \\
        &= (Q_a Q_b)_{ii} - (L\calE')^{ab}_{ii} = \frac{1}{2}(Q_a Q_b + Q_b Q_a)_{ii} - (L\calE')^{ab}_{ii}
    \end{align}
    where in the third step we used $I'X^{\top} = \wh{U}{X^*}^{\top} + L\calE'$ which is a degree-4 polynomial equality, in the fourth step we used \eqref{eq:Iprime}, and in the last step we used Constraint~\ref{constraint:sym} to conclude that $(Q_a Q_b)_{ii} = (Q_b Q_a)_{ii}$.
    
    Similarly, given $1 \le i_1 < i_2 \le r$,
    \begin{align}
        F_U(Q^*_a Q^*_b)_{i_1i_2} &= \sum_{j_1,j_2\in[r]} U^{j_1j_2}_{i_1i_2} (Q^*_a Q^*_b)_{j_1j_2} = \frac{1}{2} \sum_{j_1,j_2} \wh{U}^{j_1j_2}_{i_1i_2} (Q^*_a Q^*_b)_{j_1j_2} = \frac{1}{2}(I' X^{\top} - L\calE')^{ab}_{i_1i_2} \\
        &= \frac{1}{2}\sum_{j_1,j_2} {I'}^{j_1j_2}_{i_1i_2}(Q_a Q_b)_{j_1j_2} - \frac{1}{2}(L\calE')^{ab}_{i_1i_2} = \frac{1}{2}(Q_a Q_b +Q_b Q_a)_{i_1i_2} - \frac{1}{2}(L\calE')^{ab}_{i_1i_2},
    \end{align}
    where in the last step we used the fact that $(Q_a Q_b)_{i_2i_1} = (Q_b Q_a)_{i_1i_2}$ by Constraint~\ref{constraint:sym}. Taking the squared Frobenius norm on both sides in degree-8 SoS, we can thus conclude from Lemma~\ref{lem:LEprime_bound} that
    \begin{align}
        \norm{F_U(Q^*_a Q^*_b) - \frac{1}{2}(Q_a Q_b + Q_b Q_a)}^2_F &\le \norm{L\calE'}^2_F \le \epsmap^2. \qedhere
    \end{align}
\end{proof}

\noindent We can then combine Lemmas~\ref{lem:fakeUsendsQ} and \ref{lem:fakeUsendsQaQb} to establish that $U$ approximately maps rows of $X^*$ to (symmetrizations of) products of rows of $N^*$:

\begin{corollary}\label{cor:chain_frob}
    If $\eta \le \frac{\kappa\radius}{rd}$, then for any $a,b\in[d]$, there is a degree-8 SoS proof using the constraints of Program~\ref{program:basic} that $F_U(Q^*_a Q^*_b) \approx_{O(\radius^2\epsmap^2)} \frac{1}{2}(F_U(Q^*_a) F_U(Q^*_b) + F_U(Q^*_b) F_U(Q^*_a))$.
\end{corollary}

\noindent We defer the formal proof of this to Appendix~\ref{app:defer_chain_frob}.

% \begin{fact}\label{lem:symmetric}
%     For any $j_1,j_2\in[r]$, $U^{j_1j_2} = U^{j_2j_1}$.
% \end{fact}

% \begin{proof}
%     For any $i\in[r]$, note that \begin{equation}
%         U^{j_1j_2}_{ii} = \frac{1}{2}(\wh{U}^{j_1j_2}_{ii} + \wh{U}^{j_2j_1}_{ii}) = U^{j_2j_1}_{ii}.
%     \end{equation}
%     Similarly, for any $i_1\neq i_2$,
%     \begin{equation}
%         U^{j_1j_2}_{i_1i_2} = \frac{1}{4}(\wh{U}^{j_1j_2}_{i_1i_2} + \wh{U}^{j_2j_1}_{i_1i_2}) = U^{j_2j_1}_{i_1i_2}.
%     \end{equation}
% \end{proof}

We can now use Corollary~\ref{cor:chain_frob} to prove the following important collection of quadratic relations among entries of $U$.

\begin{lemma}\label{lem:main_identity}
    There is an absolute constant $C > 0$ such that for
    \begin{equation}
        \epsrel \triangleq C\cdot \radius\epsmap d r / \kappa = \Theta(\radius\beta\eta\cdot d^2r^2/\kappa^2) \label{eq:epsrel_def},
    \end{equation}
    the following holds. For any $j_1,j_2,k_1,k_2\in[r]$, there is a degree-8 SoS proof using the constraints of Program~\ref{program:basic} that
    \begin{equation}
         2 U^{j_1 j_2} U^{k_1k_2} + 2U^{k_1k_2} U^{j_1 j_2} \approx_{\epsrel^2} \sum_{\pi,\tau\in\calS_2} \bone{j_{\pi(1)} = k_{\pi(1)}} U^{j_{\pi(2)} k_{\pi(2)}}. \label{eq:matmul}
    \end{equation}
\end{lemma}

\begin{proof}
    % For any $a,b\in[d]$, define the matrix-valued indeterminate $\Delta_{a,b} \triangleq  \frac{1}{2}(F_U(Q^*_a) F_U(Q^*_b) + F_U(Q^*_b) F_U(Q^*_a)) - F_U(Q^*_a Q^*_b)$. 
    By Corollary~\ref{cor:chain_frob}, there is a degree-8 SoS proof that
    \begin{multline}
        \sum_{\ell_1,\ell_2,\ell_3\in[r]} (Q^*_a)_{\ell_1\ell_2} (Q^*_b)_{\ell_2\ell_3} U^{\ell_1\ell_3} \approx_{O(\radius^2\epsmap^2)}  \\
        \frac{1}{2}\sum_{j_1,j_2,k_1,k_2\in[r]} (Q^*_a)_{j_1j_2} (Q^*_b)_{k_1k_2} \left(U^{j_1j_2} U^{k_1k_2} +U^{k_1k_2} U^{j_1j_2}\right). \label{eq:rewrite_use_fakelems}
    \end{multline}
    % By appealing to symmetry of $Q^*_a, Q^*_b$, we will only consider summands for which $j_1\le j_2$ and $k_1\le k_2$.
    For the sum on the right-hand side of \eqref{eq:rewrite_use_fakelems}, we can first rewrite the contribution from the terms for which $j_1\neq j_2$ and $k_1\neq k_2$ as
    \begin{equation}
        \frac{1}{2}\sum_{\substack{j_1<j_2 \\ k_1 < k_2}} (Q^*_a)_{j_1j_2} (Q^*_b)_{k_1k_2} \sum_{\pi,\tau\in\calS_2} \left(U^{j_{\pi(1)} j_{\pi(2)}} U^{k_{\tau(1)} k_{\tau(2)}} + U^{k_{\pi(1)} k_{\pi(2)}} U^{j_{\pi(1)} j_{\pi(2)}}\right).
    \end{equation}
    We can rewrite the contribution from the terms for which $j_1 = j_2$ and $k_1 \neq k_2$ as
    \begin{equation}
        \frac{1}{2}\sum_{j; k_1<k_2} (Q^*_a)_{jj} (Q^*_b)_{k_1k_2} \sum_{\tau\in\calS_2} \left(U^{jj} U^{k_{\tau(1)} k_{\tau(2)}} + U^{k_{\tau(1)} k_{\tau(2)}} U^{jj}\right),
    \end{equation}
    the contribution from the terms for which $j_1\neq j_2$ and $k_1 = k_2$ as
    \begin{equation}
        \frac{1}{2}\sum_{j_1<j_2; k} (Q^*_a)_{j_1j_2} (Q^*_b)_{kk} \sum_{\pi\in\calS_2} \left(U^{j_{\pi(1)} j_{\pi(2)}} U^{kk} + U^{kk} U^{j_{\pi(1)} j_{\pi(2)}}\right),
    \end{equation}
    and the contribution from the terms for which $j_1 = j_2$ and $k_1 = k_2$ as
    \begin{equation}
        \frac{1}{2}\sum_{j,k} (Q^*_a)_{jj} (Q^*_b)_{kk} \left(U^{jj} U^{kk} + U^{kk} U^{jj}\right).
    \end{equation}
    For the sum on the left-hand side of \eqref{eq:rewrite_use_fakelems}, we can rewrite the contribution for the terms for which $\ell_1\neq \ell_2$ and $\ell_2 \neq \ell_3$ as follows. Let $j_1 = \min(\ell_1,\ell_2)$, $j_2 = \max(\ell_1,\ell_2)$, $k_1 = \min(\ell_2,\ell_3)$, $k_2 = \max(\ell_2,\ell_3)$. Then if $\ell_1, \ell_2, \ell_3$ are all distinct, we have 
    \begin{equation}
        (Q^*_a)_{\ell_1\ell_2} (Q^*_b)_{\ell_2\ell_3} U^{\ell_1\ell_3} = (Q^*_a)_{j_1j_2} (Q^*_b)_{k_1k_2} \biggl(\sum_{\pi,\tau\in\calS_2} \bone{j_{\pi(1)} = k_{\pi(1)}} U^{j_{\pi(2)} k_{\pi(2)}}\biggr).
    \end{equation}
    Otherwise, if $\ell_1 = \ell_3$, but $\ell_1 \neq \ell_2$ and $\ell_2\neq \ell_3$, then
    \begin{equation}
        (Q^*_a)_{\ell_1\ell_2} (Q^*_b)_{\ell_2\ell_1} + (Q^*_a)_{\ell_2 \ell_1} (Q^*_b)_{\ell_1 \ell_2} = (Q^*_a)_{j_1j_2} (Q^*_b)_{k_1k_2} \biggl(\sum_{\pi,\tau\in\calS_2} \bone{j_{\pi(1)} = k_{\pi(1)}} U^{j_{\pi(2)} k_{\pi(2)}}\biggr).
    \end{equation}
    We can therefore rewrite the contribution to the left-hand side of \eqref{eq:rewrite_use_fakelems} by $\ell_1,\ell_2,\ell_3$ satisfying $\ell_1 \neq \ell_2$ and $\ell_2\neq \ell_3$ as
    \begin{equation}
        \sum_{\substack{j_1<j_2 \\ k_1 < k_2}} (Q^*_a)_{j_1j_2} (Q^*_b)_{k_1k_2} \biggl(\sum_{\pi,\tau\in\calS_2} \bone{j_{\pi(1)} = k_{\tau(1)}} U^{j_{\pi(2)} k_{\tau(2)}}\biggr)
    \end{equation}
    We can similarly rewrite the contribution to the left-hand side of \eqref{eq:rewrite_use_fakelems} by $\ell_1,\ell_2,\ell_3$ satisfying $\ell_1 = \ell_2$ and $\ell_2 \neq \ell_3$ as
    \begin{equation}
        \sum_{j; k_1 < k_2} (Q^*_a)_{jj} (Q^*_b)_{k_1k_2} \biggl(\sum_{\tau\in\calS_2} \bone{j = k_{\tau(1)}} U^{jk_{\tau(2)}}\biggr),
    \end{equation}
    the contribution from the terms for which $\ell_1 \neq \ell_2$ and $\ell_2 = \ell_3$ as
    \begin{equation}
        \sum_{j_1<j_2; k} (Q^*_a)_{j_1j_2} (Q^*_b)_{kk} \biggl(\sum_{\pi\in\calS_2} \bone{k = j_{\pi(1)}} U^{j_{\pi(2)}k}\biggr),
    \end{equation}
    and the contribution from the terms for which $\ell_1 = \ell_2 = \ell_3$ as
    \begin{equation}
        \sum_{j, k} (Q^*_a)_{jj} (Q^*_b)_{kk} \cdot \bone{j = k} \cdot U^{jk}.
    \end{equation}
    Altogether, recalling the definition of $M^*$ in Assumption~\ref{assume:tensorring}, we get that \eqref{eq:rewrite_use_fakelems} can be rewritten as
    \begin{multline}
        0 \approx_{O(\radius^2\epsmap^2)} \sum_{\substack{j_1 \le j_2 \\ k_1 \le k_2}}  ({M^*}^{\otimes 2})^{j_1j_2,k_1k_2}_{ab} \cdot c_{j_1,j_2,k_1,k_2}  \\ \times \sum_{\pi,\tau\in\calS_2} \Bigl(\frac{1}{2} U^{j_{\pi(1)} j_{\pi(2)}} U^{k_{\tau(1)}k_{\tau(2)}} + \frac{1}{2} U^{k_{\tau(1)}k_{\tau(2)}} U^{j_{\pi(1)} j_{\pi(2)}} - \bone{j_{\pi(1)} = k_{\pi(1)}} U^{j_{\pi(2)} k_{\pi(2)}}\Bigr) , \label{eq:system_fullrank}
    \end{multline}
    where for any $j_1 \le j_2$ and $k_1 \le k_2$,
    \begin{equation}
        c_{j_1,j_2,k_1,k_2} = \begin{cases}
            1 & \text{if} \ j_1 < j_2 \ \text{and} \ k_1 < k_2 \\
            \frac{1}{4} & \text{if} \ j_1 = j_2 \ \text{and} \ k_1 = k_2 \\
            \frac{1}{2} & \text{otherwise}
        \end{cases}.
    \end{equation}
    By Part~\ref{assume:condnumber} of Assumption~\ref{assume:tensorring}, $(M^*)^{-1}$ satisfies $\norm{(M^*)^{-1}}_{\op} \le \kappa^{-1}$.
    Consider the matrix $A\triangleq ((M^*)^{-1})^{\otimes 2}\in \R^{\binom{r+1}{2}^2 \times d^2}$ and note that $A (M^*)^{\otimes 2} = \Id^{\otimes 2}_{\binom{r+1}{2}\times \binom{r+1}{2}}$. For any $j_1\le j_2$ and $k_1\le k_2$, we can take the linear combination of \eqref{eq:system_fullrank} across different choices of $a,b$ as specified by the $(j_1,j_2,k_1,k_2)$-th row of $A$. By Part~\ref{shorthand:lincombo} of Fact~\ref{fact:shorthand}, there is a degree-8 SoS proof that
    \begin{multline}
        0 \approx_{\norm{A_{j_1,j_2,k_1,k_2}}^2_2 \cdot d^2\radius^2\epsmap^2} \\ c_{j_1,j_2,k_1,k_2} \sum_{\pi,\tau\in\calS_2} \Bigl(\frac{1}{2} U^{j_{\pi(1)} j_{\pi(2)}} U^{k_{\tau(1)}k_{\tau(2)}} + \frac{1}{2} U^{k_{\tau(1)}k_{\tau(2)}} U^{j_{\pi(1)} j_{\pi(2)}} - \bone{j_{\pi(1)} = k_{\pi(1)}} U^{j_{\pi(2)} k_{\pi(2)}}\Bigr). \label{eq:Acond}
    \end{multline}
    Note that $\norm{A}^2_F \le \norm{(M^*)^{-1}}^4_F \le r^4 / \kappa^4$, so the lemma follows by Lemma~\ref{lem:symmetric}.
\end{proof}

\subsection{Auxiliary Matrix \texorpdfstring{$W$}{W}}
\label{sec:auxW}

The sum over permutations in \eqref{eq:matmul} is rather unwieldy and essentially a byproduct of the fact that we defined $U$ so that its $(i,j)$-th column is identical to its $(j,i)$-th column. Intuitively, this is because $U^{ij}$ behaves like the \emph{average} of the $(i,j)$-th and $(j,i)$-th columns of an $r^2\times r^2$ orthogonal matrix, rather than like the $(i,j)$-th column of such a matrix.

For this reason, we now define a new auxiliary variable $W$. For any $i,j\in [r]$, define the $r\times r$ matrix of indeterminates
\begin{equation}
    W^{ij} \triangleq \begin{cases}
        2U^{ii} U^{ij} U^{jj} & \text{if} \ i \neq j \\
        U^{ii} & \text{if} \ i = j
    \end{cases}
\end{equation}
By Lemma~\ref{lem:symmetric}, $W^{ii}$ is a symmetric matrix, and likewise for any $i\neq j$
\begin{equation}
    (W^{ij})^{\top} = (2U^{ii} U^{ij} U^{jj})^{\top} = 2 U^{jj} U^{ji} U^{ii} = W^{ji}. \label{eq:Wtranspose}
\end{equation}

\begin{remark}
    Our motivation for this choice of $W$ is that heuristically, if we expect $U$ to arise from an $r\times r$ rotation, we expect each column $U^{ij}$ of $U$ to resemble $\frac{1}{2}(V^i{V^j}^{\top} + V^j{V^i}^{\top})$ for a fixed $V\in O(r)$. We would like to extract from this a rank-1 matrix, so that the resulting matrix $W$ behaves like a Kronecker power. If we left- and right-multiplied $\frac{1}{2}(V^i{V^j}^{\top} + V^j{V^i}^{\top})$ by $2V^i$ and ${V^j}^{\top}$ as in the definition of $W^{ij}$ above, we would get $V^i{V^j}^{\top}$ as desired, by orthogonality of $V$.
\end{remark}

\noindent As we will see in this subsection:
\begin{enumerate}
    \item $W$ satisfies a simpler version of the relations \eqref{eq:matmul} (Lemma~\ref{lem:simpler_identity})
    \item Like $U$, $W$ approximately maps every $\vec(Q^*_a)$ to $\vec(Q_a)$ (Lemma~\ref{lem:WsendsQ})
    \item $W$ is approximately an orthogonal matrix (Lemma~\ref{lem:simpler_identity} and Lemma~\ref{lem:Wnorm1}).
\end{enumerate}

We begin by showing that $W$ satisfies a simpler version of \eqref{eq:matmul}:
\begin{lemma}\label{lem:simpler_identity}
    For any $i,j,j',k\in[r]$, there is a degree-24 SoS proof using the constraints of Program~\ref{program:basic} that
    \begin{equation}
        W^{ij} W^{j'k} \approx_{O(\epsrel^2)} W^{ik} \bone{j = j'}. \label{eq:simpler}
    \end{equation}
\end{lemma}

\noindent Note that Lemma~\ref{lem:simpler_identity} and \eqref{eq:Wtranspose} imply that the columns of $W$ are nearly orthogonal.

Before we prove Lemma~\ref{lem:simpler_identity}, we first record some useful special cases of Lemma~\ref{lem:main_identity}.

\begin{corollary}\label{cor:consequences}
    There is a degree-8 SoS proof using the constraints of Program~\ref{program:basic} for each of the following:
    \begin{enumerate}
        \item For any $i\in[r]$, $(U^{ii})^2 \approx_{O(\epsrel^2)} U^{ii}$. \label{item:Uii_idempotent}
        \item For any $i \neq k$, $U^{ii} U^{kk} \approx_{O(\epsrel^2)} 0$. \label{item:ijzero}
        \item For any $i\neq j$, $U^{ii} U^{ij} +U^{ij} U^{ii} \approx_{O(\epsrel^2)} U^{ij}$. \label{item:iiij}
        \item For any $i,j,k$ satisfying $i\neq j$ and $i\neq k$, $U^{ii} U^{jk} \approx_{O(\epsrel^2)} -U^{jk} U^{ii}$. \label{item:swap}
    \end{enumerate}
\end{corollary}

\begin{proof}
    Lemma~\ref{lem:main_identity} applied to $j_1 = j_2 = k_1 = k_2$ implies Part~\ref{item:Uii_idempotent}. Lemma~\ref{lem:main_identity} applied to $j_1 = j_2 \neq k_1 = k_2$ implies Part~\ref{item:ijzero}. Lemma~\ref{lem:main_identity} applied to $j_1 = j_2 = k_1 = i$ and $k_2 = j$ implies Part~\ref{item:iiij}. Lemma~\ref{lem:main_identity} applied to $j_1 = j_2 = i$, $k_1 = j$, and $k_2 = k$ implies Part~\ref{item:swap}.
\end{proof}
% note that Lemma~\ref{lem:main_identity} applied to $j_1 = i$, $j_2 = j$, and $k_1=k_2 =k$ implies
    % \begin{equation}
    %     U^{ij} U^{kk} = - U^{kk} U^{ij}.\label{eq:negate}
    % \end{equation}
    
\noindent We can now prove Lemma~\ref{lem:simpler_identity}. As the argument is rather technical, we defer the full proof to Appendix~\ref{app:defer_simpler_identity} and provide a sketch of the key ideas here.

\begin{proof}[Proof sketch of Lemma~\ref{lem:simpler_identity}]
    % We proceed by casework to show \eqref{eq:simpler}. 
    % First, for any $j_1,j_2,k_1,k_2$ define the $r\times r$ matrix of indeterminates
    % \begin{equation}
    %     D[j_1,j_2,k_1,k_2] \triangleq 2 U^{j_1 j_2} U^{k_1k_2} + 2U^{k_1k_2} U^{j_1 j_2} - \sum_{\pi,\tau\in\calS_2} \bone{j_{\pi(1)} = k_{\pi(1)}} U^{j_{\pi(2)} k_{\pi(2)}}
    % \end{equation}
    % so that by Lemma~\ref{lem:main_identity}, $\norm{D[j_1,j_2,k_1,k_2]}^2_F \le \epsrel^2$.
    Note that $W^{ij}W^{j'k}$ is equal, up to a positive integer factor, to the matrix $U^{ii}U^{ij}U^{jj} U^{j'j'}U^{j'k}U^{kk}$. So if $j\neq j'$, then $U^{jj}U^{j'j'} \approx 0$ by Part~\ref{item:ijzero} of Corollary~\ref{cor:consequences}, so $U^{ii}U^{ij}U^{jj} U^{j'j'}U^{j'k}U^{kk}\approx 0$.
    
    It remains to handle $j = j'$. For this proof sketch, we illustrate the argument in the special case where all of $i,j,k$ are distinct. In this case, Lemma~\ref{lem:main_identity} implies that
    \begin{equation}
        2U^{ij}U^{jk} + 2U^{jk}U^{ij} \approx U^{ik}.
    \end{equation}
    Left- and right- multiplying both sides by $U^{ii}$ and $U^{kk}$ gives
    \begin{equation}
        2U^{ii}U^{ij}U^{jk}U^{kk} + 2U^{ii}U^{jk}U^{ij}U^{kk} \approx U^{ii}U^{ik}U^{kk} = \frac{1}{2}W^{ik} \label{eq:sketch_simple}
    \end{equation}
    For the first term in \eqref{eq:sketch_simple}, we have
    \begin{align}
        2U^{ii}U^{ij}U^{jk}U^{kk} &\approx 2U^{ii}(U^{ii}U^{ij} + 2U^{ij}U^{ii})(U^{jj}U^{jk}+U^{jk}U^{jj})U^{kk} \\
        &\approx 2(U^{ii}U^{ij} + U^{ii}U^{ij}U^{ii})U^{jj}U^{jk}U^{kk} \\
        &\approx 2U^{ii}U^{ij}U^{jj}U^{jk}U^{kk} \approx 2U^{ii}U^{ij}(U^{jj})^2U^{jk}U^{kk} = \frac{1}{2}W^{ij}W^{jk},
    \end{align}
    where the first step follows by Part~\ref{item:iiij} of Corollary~\ref{cor:consequences}, the second by Part~\ref{item:Uii_idempotent} and \ref{item:ijzero}, the third by Part~\ref{item:iiij}, and the fourth by Part~\ref{item:Uii_idempotent}.
    
    It remains to show the second term in \eqref{eq:sketch_simple} approximately vanishes. By Part~\ref{item:swap} of Corollary~\ref{cor:consequences},
    \begin{align}
        U^{ii}U^{jk}U^{ij}U^{kk} &\approx (-U^{jk}U^{ii})(-U^{kk}U^{ij}) = U^{jk}(U^{ii}U^{kk})U^{ij} \approx 0. \qedhere
    \end{align}
\end{proof}

\noindent We can also show that $W$ approximately maps each $Q^*_a$ to $Q_a$.

\begin{lemma}\label{lem:WsendsQ}
    For any $a\in[d]$, there is a degree-12 SoS proof using the constraints of Program~\ref{program:basic} that $F_W(Q^*_a) \approx_{O(\epsmap^2 + r^6\epsrel^2)} Q_a$.
\end{lemma}

\noindent The idea is to use Corollary~\ref{cor:consequences} to show that $W^{ij} + W^{ji} \approx 2U^{ij}$ and then invoke symmetry of $Q^*_a$ and the fact that $U$ approximately maps each $Q^*_a$ to $Q_a$. We defer the formal proof to Appendix~\ref{app:defer_WsendsQ}.

\begin{corollary}\label{cor:WsendsQlam}
    For $c\in\brc{\lambda,\mu}$, there is a degree-12 SoS proof using the constraints of Program~\ref{program:basic} that $F_W(Q^*_c) \approx_{O(d\epsmap^2 + dr^6\epsrel^2)} Q_c$.
\end{corollary}

\begin{proof}
    $F_W(Q^*_c) - Q_c = \sum^d_{a=1} c_a(F_W(Q^*_a) - Q_a)$, so the claim follows by Part~\ref{shorthand:lincombo} of Fact~\ref{fact:shorthand}.
\end{proof}

\noindent Finally, we show that the columns of $W$ have approximately unit norm.

\begin{lemma}\label{lem:Wnorm1}
    Define 
    \begin{equation}
        \epsnorm \triangleq \sqrt{\epsort} r^3 + \epsrel \sqrt{r} = \Theta(\eta^{1/2}\beta d^{1/2} r^3 + \radius\beta\eta d^2r^{5/2}/\kappa^2) \label{eq:epsnorm_def}
    \end{equation}
    For any $j_1,j_2\in[r]$, there is a degree-24 SoS proof using the constraints of Program~\ref{program:basic} that $-O(\epsnorm) \le \norm{W^{j_1j_2}}^2_F - 1 \le O(\epsnorm)$.
\end{lemma}

\begin{proof}
    The case of $j_1 = j_2$ follows immediately from Lemma~\ref{lem:Unorm12} and the fact that $W^{j_1j_1} = U^{j_1j_1}$. Next, consider $j_1 < j_2$ and note that
    \begin{equation}
        \norm{W^{j_1j_2}}^2_F = \Tr(W^{j_1j_2}{W^{j_1j_2}}^{\top}) = \Tr(W^{j_1j_2} W^{j_2j_1}) \label{eq:normWsq}
    \end{equation}
    where in the second step we used \eqref{eq:Wtranspose}. As $W^{j_1j_2} W^{j_2j_1} \approx_{O(\epsrel^2)} (W^{j_1j_1})^2$ in degree-24 SoS by Lemma~\ref{lem:simpler_identity}, we have
    \begin{equation}
        \Tr(W^{j_1j_2} W^{j_2j_1}) = \Tr((W^{j_1j_1})^2) \pm O(\epsrel\sqrt{r}) = \norm{U^{j_1j_1}}^2_F \pm O(\epsrel\sqrt{r}) = 1 \pm O(\sqrt{\epsort}r^3 + \epsrel\sqrt{r}),
    \end{equation}
    where in the first step we used Fact~\ref{fact:frob_to_trace} and in the last step we used Lemma~\ref{lem:Unorm12}.
    % As $\Tr(W^{j_1j_2} W^{j_2j_1} - (W^{j_1j_1})^2 )^2 \le r\norm{W^{j_1j_2} W^{j_2j_1} - (W^{j_1j_1})^2 }^2_F \le O(\epsrel^2 r)$ by Cauchy-Schwarz and Lemma~\ref{lem:simpler_identity}, we conclude by \eqref{eq:normWsq} that
    % \begin{equation}
    %     \left(\norm{W^{j_1j_2}}^2_F - \Tr((W^{j_1j_1})^2)\right)^2 \le O(\epsrel^2 r).
    % \end{equation}
    % By Lemma~\ref{lem:Unorm12}, $\Tr((W^{j_1j_1})^2) = \norm{U^{j_1j_1}}^2_F = 1 \pm O(\sqrt{\epsort}r^3)$, concluding the proof of the lemma.
\end{proof}

\subsection{Breaking Gauge Symmetry for SoS Variables}
\label{sec:diagonal}

It turns out that Lemma~\ref{lem:simpler_identity} and Lemma~\ref{lem:Wnorm1} are already powerful enough to imply that the transformation $W$ mapping $Q^*_a$ to $Q_a$ for every $a\in[d]$ behaves like it approximately arises from an $r\times r$ rotation. While this isn't directly used in the subsequent analysis, we include a proof in Appendix~\ref{app:heuristic_rotation} to provide additional intuition.

Instead, by combining these lemmas with the strategy outlined in Section~\ref{sec:breakground} for breaking gauge symmetry, we show an even stronger statement in this section. We will prove that $W$ actually arises from the $r\times r$ \emph{identity} rotation. Specifically, we leverage Constraints~\ref{constraint:diag} and \ref{constraint:sorted}, together with \eqref{eq:Qdiag_gap}, in order to prove that $W$ is the Kronecker power of an $r\times r$ orthogonal matrix whose off-diagonal entries are close to zero. We then leverage Constraint~\ref{constraint:firstrow}, together with \eqref{eq:Qstarfirstrow}, in order to prove that $W$ in fact arises from an $r\times r$ rotation which is close to identity.

We begin by showing the following lemma that allows us to exchange $Q$ and $W$.

\begin{lemma}\label{lem:QWWQ}
    For any $a,b\in[r]$, let $W^{b:}_{a:}$ denote the matrix whose $(i,j)$-th entry is $W^{bj}_{ai}$ for any $i,j\in[r]$. There is a degree-48 SoS proof using the constraints of Program~\ref{program:basic} that
    \begin{equation}
        Q_c W^{b:}_{a:} \approx_{O(\radius^2 r^3 \epsrel^2 + d\epsmap^2 r + d\epsrel^2 r^7)} W^{b:}_{a:} Q^*_c  \label{eq:QWWQ}
    \end{equation} for any $c\in[d]\cup\brc{\lambda,\mu}$.
\end{lemma}

\begin{proof}
    Define $\Delta_c \triangleq Q_c - F_W(Q^*_c)$ and recall from Lemma~\ref{lem:WsendsQ} and Corollary~\ref{cor:WsendsQlam} that there is a degree-12 SoS proof that $\norm{\Delta_c}^2_F \le O(d\epsmap^2 + dr^6\epsrel^2)$. Right-multiplying both sides of $Q_c - \Delta_c = F_W(Q^*_c)$ by the matrix $W^{b:}_{a:}$ and considering the $(i,m)$-th entry for any $i,m\in[r]$, we get
    \begin{align}
        (Q_c W^{b:}_{a:} - \Delta_c W^{b:}_{a:})_{im} &= \sum_{j,k,\ell} W^{k\ell}_{ij} (Q^*_c)_{k\ell} W^{bm}_{aj} = \sum_{j,k,\ell} W^{bm}_{aj} W^{\ell k}_{ji}  (Q^*_c)_{k\ell} = \sum_{k,\ell} (W^{bm} W^{\ell k})_{ai} (Q^*_c)_{k\ell} \\
        \intertext{where the second step follows by \eqref{eq:Wtranspose}. Defining $D[b,m,\ell,k] \triangleq W^{bm} W^{\ell k} - \bone{m = \ell} W^{bk}$ so that by Lemma~\ref{lem:simpler_identity}, there is a degree-24 SoS proof that $\norm{D[b,m,\ell,k]}^2_F \le \epsrel^2$, we can express this as}
        &= \sum_{k} W^{bk}_{ai} (Q^*_c)_{km} + \sum_{k,\ell} D[b,m,\ell,k]_{ai} (Q^*_c)_{k\ell} \\
        &= (W^{b:}_{a:} Q^*_c)_{im} + \sum_{k,\ell} D[b,m,\ell,k]_{ai} (Q^*_c)_{k\ell}.
    \end{align}
    We can bound
    \begin{align}
        \sum_{i,m\in[r]} \biggl(\sum_{k,\ell} D[b,m,\ell,k]_{ai} (Q^*_c)_{k\ell}\biggr)^2 &\le \sum_{i,m} \biggl(\sum_{k,\ell} (D[b,m,\ell,k]_{ai})^2\biggr)\biggl(\sum_{k,\ell} ((Q^*_c)_{k\ell})^2\biggr) \\
        &\le \radius^2\sum_{m,\ell,k} \sum_i (D[b,m,\ell,k]_{ai})^2 \le \radius^2 r^3\epsrel^2
    \end{align}
    in degree-24 SoS, where in the penultimate step we used Part~\ref{assume:scale} of Assumption~\ref{assume:tensorring}. It follows that $\norm{Q_c W^{b:}_{a:} - \Delta_c W^{b:}_{a:} - W^{b:}_{a:} Q^*_c}^2_F \le O(\radius^2 r^3\epsrel^2)$. We get the lemma upon using the fact that $\norm{\Delta_c W^{b:}_{a:}}^2_F \le \norm{\Delta_c}^2_F \norm{W^{b:}_{a:}}^2_F \le O(d\epsmap^2 r + d\epsrel^2 r^7)$ in degree-48 SoS by Lemma~\ref{lem:Wnorm1}.
\end{proof}

\subsubsection{Using Diagonality of \texorpdfstring{$Q_{\lambda}$}{Qlambda}}

Our first main result of this subsection is to show that $W$ is the Kronecker power of a diagonal $r\times r$ rotation, using Constraints~\ref{constraint:diag} and \ref{constraint:sorted} along with \eqref{eq:Qdiag_gap}:

\begin{lemma}\label{lem:deg6}
    Define
    \begin{equation}
        \epsoffdiag \triangleq (d^2\radius^4 r^{11} \epsrel^2 \epsmap^4 r^9 + d^2\epsrel^4 r^{19}) / \gap^2 + \epsnorm \cdot r^2 \label{eq:epsoffdiag_Q_def}
    \end{equation}
    There is a degree-48 SoS proof using the constraints of Program~\ref{program:basic} that
    \begin{equation}
        (W^{jj'}_{ii'})^2 \le O(\epsoffdiag) \ \forall \ (i,i') \neq (j,j'). \label{eq:deg63}
    \end{equation}
    \begin{equation}
        (W^{jj'}_{jj'})^2 = 1 \pm O(\epsnorm + \epsoffdiag r^2) \ \forall \ j,j'\in[r] \label{eq:deg62}
    \end{equation}
    % \begin{equation}
    %     W^{jj}_{jj} = 1 \pm O(\epsnorm + \epsrel + \epsoffdiag r^2) \ \forall \ j\in[r] \label{eq:deg61}
    % \end{equation}
\end{lemma}

\noindent Before we prove this, we record several useful claims exploiting Constraint~\ref{constraint:sorted} and \eqref{eq:Qdiag_gap}.

\begin{lemma}\label{lem:downleft}
    For any $i,j,k,\ell\in[r]$ for which $k\ge i$ and $\ell < j$, there is a degree-2 SoS proof using Constraint~\ref{constraint:sorted} that
    \begin{equation}
        ((Q_\lambda)_{ii} - (Q^*_\lambda)_{jj})^2 + ((Q_\lambda)_{kk} - (Q^*_\lambda)_{\ell\ell})^2 \ge \gap^2/2. \label{eq:Qpos}
    \end{equation}
\end{lemma}

\begin{proof}
    For convenience, denote $Q = Q_\lambda$ and $Q^*= Q^*_\lambda$. Note that
    \begin{align}
        \MoveEqLeft (Q_{ii} - Q^*_{jj})^2 + (Q_{kk} - Q^*_{\ell\ell})^2 \\
        &= (Q_{ii} - Q^*_{jj})^2 + \left(Q_{ii} - Q^*_{jj} + (Q_{kk} - Q_{ii}) +(Q^*_{jj} - Q^*_{\ell\ell})\right)^2 \\
        &= 2\Bigl(Q_{ii} - Q^*_{jj} + \frac{Q_{kk} - Q_{ii} +Q^*_{jj} - Q^*_{\ell\ell}}{2}\Bigr)^2 + \frac{(Q_{kk} - Q_{ii} +Q^*_{jj} - Q^*_{\ell\ell})^2}{2} \\
        &\ge \frac{(Q_{kk} - Q_{ii} +Q^*_{jj} - Q^*_{\ell\ell})^2}{2}.
    \end{align}
    By our assumptions on $i,j,k,\ell$, we know $k\ge i$ and $j> \ell$, so $Q_{kk} - Q_{ii} + Q^*_{jj} - Q^*_{\ell\ell} \ge \gap$ by Constraint~\ref{constraint:sorted} and \eqref{eq:Qdiag_gap}, and the claim follows.
\end{proof}

\begin{lemma}\label{lem:downleft_Wentry}
    For any $i,j,k,\ell\in[r]$ for which $k\ge i$ and $\ell < j$, and any $a,b,a',b'\in[r]$, there is a degree-96 SoS proof using the constraints of Program~\ref{program:basic} that
    \begin{equation}
        W^{bj}_{ai} W^{b'\ell}_{a'k} = \pm O( (d\epsmap^2 r^2 + d\epsrel^2 r^7) / \gap). \label{eq:downleft_pair_zero}
    \end{equation}
\end{lemma}

\begin{proof}
    In this proof we will refer to $Q^*_\lambda$ and $Q_{\lambda}$ as $Q^*$ and $Q$ respectively. By Lemma~\ref{lem:QWWQ}, together with Constraint~\ref{constraint:diag} and diagonality of $Q^*_{\lambda}$, there is a degree-48 SoS proof that
    \begin{equation}
        \sum_{i,j\in[r]} \bigl((Q_{ii} - Q^*_{jj}) W^{bj}_{ai}\bigr)^2 \le O(d\epsmap^2 r + d\epsrel^2 r^7)
    \end{equation} for all $a,b\in[r]$.
    In particular, by upper bounding any summand on the left-hand side by the sum, we find that for any $i,j,k,\ell\in[r]$, there is a degree-96 SoS proof that
    \begin{equation}
        (W^{bj}_{ai})^2 (W^{b'\ell}_{a'k})^2 \left((Q_{ii} - Q^*_{jj})^2 + (Q_{kk} - Q^*_{\ell\ell})^2\right) \le O(d^2\epsmap^4 r^2 + d^2\epsrel^4 r^{14}).
    \end{equation}
    But if $i,j,k,\ell$ satisfy the hypotheses of Lemma~\ref{lem:downleft}, \eqref{eq:downleft_pair_zero} follows by Lemma~\ref{lem:downleft} and Part~\ref{fact:divide_both_sides} of Fact~\ref{fact:division}.
\end{proof}

\noindent Lastly, we need the following simple helper lemma that follows from Lemma~\ref{lem:simpler_identity}.

\begin{lemma}\label{lem:helper_userel}
    For any $i_2,j_1,j_2\in[r]$, there is a degree-24 SoS proof using the constraints of Program~\ref{program:basic} that $\sum^r_{i_1 = 1} (W^{j_1j_2}_{i_1i_2})^2 = W^{j_2j_2}_{i_2i_2} \pm O(\epsrel)$.
\end{lemma}

\begin{proof}
    By \eqref{eq:Wtranspose}, $\sum^r_{i_1=1} (W^{j_1j_2}_{i_1i_2})^2 = \sum^r_{i_1=1} W^{j_2j_1}_{i_2i_1} W^{j_1j_2}_{i_1i_2} = (W^{j_2j_1} W^{j_2j_1})_{i_2i_2} = W^{j_2j_2}_{i_2i_2} \pm O(\epsrel)$, where in the last step we used Lemma~\ref{lem:simpler_identity}.
\end{proof}

\noindent We are now ready to prove Lemma~\ref{lem:deg6}:

\begin{proof}[Lemma~\ref{lem:deg6}]
    We first prove \eqref{eq:deg63}. Henceforth, take an arbitrary $b\in[\ell]$ which will be fixed throughout this proof.
    
    For any $i,j,\ell\in[r]$ we have
    \begin{equation}
        \sum_{a,a',k\in[r]} (W^{bj}_{ai})^2 \cdot (W^{b\ell}_{a'k})^2 = \norm{W^{b\ell}}^2_F\cdot \sum_a (W^{bj}_{ai})^2 = \sum_a (W^{bj}_{ai})^2 \pm O(\epsnorm) \label{eq:expandnorm}
    \end{equation}
    in degree-48 SoS, where in the second step we used Lemma~\ref{lem:Wnorm1}.
    
    Define
    \begin{equation}
        \epsilon' \triangleq (\radius^3 r^3 \epsrel^2 + d\epsmap^2 r^2 + d\epsrel^2 r^7) / \gap.
    \end{equation}
    If $j > \ell$, then we can upper bound the terms on the left-hand side of \eqref{eq:expandnorm} for which $k \ge i$ by Lemma~\ref{lem:downleft_Wentry}. We conclude in degree-48 SoS that
    \begin{align}
        \sum_a (W^{bj}_{ai})^2 &\le \biggl(\sum_{a\in[r]} (W^{bj}_{ai})^2\biggr) \cdot \biggl(\sum_{a\in[r], k\in[i-1]} (W^{b\ell}_{ak})^2 \biggr) + O(\epsilon'^2 \cdot r^3 + \epsnorm) \\
        &\le (W^{jj}_{ii} + O(\epsrel))\cdot\biggl(\sum_{a\in[r], k\in[i-1]} (W^{b\ell}_{ak})^2\biggr) + O(\epsilon'^2\cdot r^3 + \epsnorm) \\
        &\le W^{jj}_{ii} \cdot\biggl(\sum_{a\in[r], k\in[i-1]} (W^{b\ell}_{ak})^2\biggr) + O(\epsilon'^2\cdot r^3 + \epsnorm) 
        \label{eq:pre_use_for_induct}
    \end{align}
    where in the second step we used Lemma~\ref{lem:helper_userel} and in the third step we used Lemma~\ref{lem:Wnorm1} as well as the fact that $O(\epsrel) \ll O(\epsilon'^2\cdot r^3 + \epsnorm)$.
    
    Now sum Eq.~\eqref{eq:pre_use_for_induct} over $1\le i\le i^*$ for any $i^*\in[r-1]$ to get
    \begin{align}
        \sum^{i^*}_{i=1} \sum_a (W^{bj}_{ai})^2 &\le \sum^{i^*}_{i=1} W^{jj}_{ii} \cdot\biggl(\sum_{a\in[r], k\in[i-1]} (W^{b\ell}_{ak})^2\biggr) + O(\epsilon'^2\cdot r^4 + \epsnorm r) \\
        &= \sum^{i^*-1}_{i=1} \sum_a (W^{b\ell}_{ai})^2 \cdot \biggl(\sum^{i^*}_{k={i+1}} W^{jj}_{kk}\biggr) + O(\epsilon'^2\cdot r^4 + \epsnorm r), \\
        \intertext{where in the second step we swapped the summation over $i\in[i^*]$ and the summation over $k\in[i-1]$ and also swapped the names of the corresponding indices $i$ and $k$. Note that by Lemma~\ref{lem:helper_userel}, $W^{jj}_{kk} \ge -O(\epsrel)$ in degree-24 SoS for all $k\in[r]$, so $\sum^{i^*}_{k={i+1}} W^{jj}_{kk} \le \Tr(W^{jj})$. This combined with Lemma~\ref{lem:Wnorm1} implies in degree-48 SoS that we can further upper bound the above display by}
        &\le \sum^{i^*-1}_{i=1} \sum_a (W^{b\ell}_{ai})^2 \cdot \Tr(W^{jj}) + O(\epsilon'^2\cdot r^4 + \epsnorm r + \epsrel r) \\
        &\le \sum^{i^*-1}_{i=1} \sum_a (W^{b\ell}_{ai})^2 + O(\epsrel'^2\cdot r^4 + \epsnorm r + \epsrel r) \label{eq:use_for_induct}
    \end{align}
    in degree-48 SoS, where in the last step we combined Lemma~\ref{lem:Wnorm1} with the fact that in degree-24 SoS,
    \begin{equation}
        \Tr(W^{jj}) = \Tr((W^{jj})^2) \pm O(\epsrel\sqrt{r}) = \norm{W^{jj}}^2_F \pm O(\epsrel\sqrt{r}) = 1 \pm O(\epsnorm + \epsrel\sqrt{r}). 
    \end{equation}
    We want to use \eqref{eq:use_for_induct} inductively. Take any $j > i^*$ and take $\ell = j - 1$. Then by \eqref{eq:use_for_induct},
    \begin{equation}
        \sum^{i^*}_{i=1}\sum_a (W^{bj}_{ai})^2 - \sum^{i^*-1}_{i=1}\sum_a (W^{b(j-1)}_{ai})^2 \le O(\epsrel'^2\cdot r^4 + \epsnorm r + \epsrel r).
    \end{equation}
    As $j - c > i^* - c$ for any $c\in\mathbb{Z}$, we have more generally that for this choice of $j$,
    \begin{equation}
        \sum^{i^*-c}_{i=1}\sum_a (W^{b(j-c)}_{ai})^2 - \sum^{i^*-c-1}_{i=1}\sum_a (W^{b(j-c-1)}_{ai})^2 \le O(\epsrel'^2\cdot r^4 + \epsnorm r + \epsrel r). \label{eq:forc}
    \end{equation}
    Summing \eqref{eq:forc} over $c$ from $0$ to $i^*-1$, we get a degree-48 SoS proof that
    \begin{equation}
        \sum^{i^*}_{i=1}\sum_a (W^{bj}_{ai})^2 \le O(\epsrel'^2\cdot r^5 + \epsnorm r^2 + \epsrel r^2).
    \end{equation}
    As the left-hand side is lower bounded by any individual summand, we conclude that for all $i,j\in[r]$ satisfying $j > i$, and for all $a,b\in[r]$,
    \begin{equation}
        (W^{bj}_{ai})^2 \le O(\epsrel'^2\cdot r^5 + \epsnorm r^2 + \epsrel r^2) \le O(\epsoffdiag). \label{eq:final_wsmall}
    \end{equation}
    By symmetry, we can also show \eqref{eq:final_wsmall} for $j < i$ in an analogous fashion. This together with \eqref{eq:Wtranspose} completes the proof of \eqref{eq:deg63}.
    
    We next prove \eqref{eq:deg62}. By Lemma~\ref{lem:Wnorm1},
    \begin{equation}
        1\pm O(\epsnorm) = \norm{W^{jj'}}^2_F = \sum_{i_1i_2\in[r]} (W^{jj'}_{i_1i_2})^2 = (W^{jj'}_{jj'})^2 + O(\epsoffdiag r^2) \label{eq:entry0}
    \end{equation}
    in degree-48 SoS, where the third step follows by applying \eqref{eq:final_wsmall} to all $(i_1,i_2) \neq (j,j')$, completing the proof of \eqref{eq:deg62}.     
    % Finally, for \eqref{eq:deg61}, we can specialize \eqref{eq:entry0} to $j = j'$ to get
    % \begin{equation}
    %     (W^{jj}_{jj})^2 = 1 \pm O(\epsnorm + \epsoffdiag r^2). \label{eq:entry1}
    % \end{equation}
    % Additionally, by Lemma~\ref{lem:simpler_identity},
    % \begin{equation}
    %     W^{jj}_{jj} = ((W^{jj})^2)_{jj} \pm O(\epsrel) = \sum^r_{i=1} W^{jj}_{ji} W^{jj}_{ij} + O(\epsrel) = (W^{jj}_{jj})^2 + O(\epsrel + \epsoffdiag r), \label{eq:entry2}
    % \end{equation}
    % where the third step follows by applying \eqref{eq:final_wsmall} to all $i\neq j$. By adding \eqref{eq:entry1} and \eqref{eq:entry2} and rearranging, we conclude that
    % \begin{equation}
    %     W^{jj}_{jj} = 1 \pm O(\epsnorm + \epsrel + \epsoffdiag r^2).
    % \end{equation}
\end{proof}

\noindent Because the ``off-diagonal'' entries of $W$ are small, we can show that for all $c\in[d]\cup\brc{\lambda,\mu}$, $Q_c$ and $Q^*_c$ are equal up to rotation by a diagonal matrix with $\pm 1$ entries:

\begin{lemma}\label{lem:QQstar}
    There is a degree-48 SoS proof using the constraints of Program~\ref{program:basic} that 
    \begin{equation}
        (Q_c)_{ij} = W^{ij}_{ij} (Q^*_c)_{ij} \pm O(\sqrt{d}\epsmap + \epsrel r^3\sqrt{d} + \sqrt{\epsoffdiag} r\cdot\radius) \ \ \forall \ c\in[d]\cup\brc{\lambda,\mu}, i,j\in[r]. \label{eq:QQstar}
    \end{equation}
\end{lemma}

\noindent We defer the proof of this to Appendix~\ref{app:defer_QQstar}

\subsubsection{Using Positivity of \texorpdfstring{$(Q_{\mu})_{1j}$}{First Row of Qmu}}

We now use Constraint~\ref{constraint:firstrow} along with \eqref{eq:Qstarfirstrow} to refine Lemma~\ref{lem:deg6} and show that $W$ essentially arises from the $r\times r$ identity rotation:

\begin{lemma}\label{lem:usemu}
    There is a degree-96 SoS proof using the constraints of Program~\ref{program:basic} that
    \begin{equation}
        W^{ij}_{ij} = 1\pm O(\epsnorm^2 + \epsoffdiag^2 r^4 + (\epsmap + \epsrel r^3 + \sqrt{\epsoffdiag})/\gap)^{1/4} \ \ \forall \ i,j\in[r].
    \end{equation}
\end{lemma}

\begin{proof}
    Recall that by \eqref{eq:Qstarfirstrow}, $(Q^*_\mu)_{1j} \ge \gap$ for all $j\in[r]$, and by Constraint~\ref{constraint:firstrow}, $(Q_\mu)_{1j} \ge 0$. Dividing by the scalar quantity $(Q^*_\mu)_{1j}$ on both sides of \eqref{eq:QQstar} from Lemma~\ref{lem:QQstar} for $c = \mu$ and $i = 1$ and rearranging, we have in degree-48 SoS that
    \begin{equation}
        W^{1j}_{1j} \ge -O((\epsmap + \epsrel r^3 + \sqrt{\epsoffdiag}r\cdot \radius)/\gap). \label{eq:W1j}
    \end{equation}
    This implies that
    \begin{equation}
        (W^{1j}_{1j} - 1)^2 \le (W^{1j}_{1j} +1)^2 + O((\epsmap + \epsrel r^3 + \sqrt{\epsoffdiag}r\cdot \radius)/\gap),
    \end{equation}
    so multiplying both sides by $(W^{1j}_{1j} - 1)^2$ and noting that by \eqref{eq:deg62}, there is a degree-96 SoS proof that $((W^{1j}_{1j})^2 - 1)^2 \le O(\epsnorm^2 + \epsoffdiag^2 r^4)$ and $(W^{1j}_{1j} - 1)^2 \le 2 + O(\epsnorm^2 + \epsoffdiag^2 r^4 + (\epsmap + \epsrel r^3 + \sqrt{\epsoffdiag} r\cdot \radius)/\gap) = O(1)$, we get
    \begin{equation}
        (W^{1j}_{1j}-1)^4 \le O(\epsnorm^2 + \epsoffdiag^2 r^4 + (\epsmap + \epsrel r^3 + \sqrt{\epsoffdiag} r \cdot \radius)/\gap)
    \end{equation}
    in degree-96 SoS.
    By Fact~\ref{fact:nonneg_power_of_two} we conclude that 
    \begin{equation}
        W^{1j}_{1j} = 1\pm O(\epsnorm^2 + \epsoffdiag^2 r^4 + (\epsmap + \epsrel r^3 + \sqrt{\epsoffdiag}r\cdot \radius)/\gap)^{1/4} \ \ \forall \ j\in[r].  \label{eq:1jbound}
    \end{equation}
    
    Finally, we use \eqref{eq:1jbound} to show that $W^{ij}_{ij}$ is close to 1 for all $i,j\in[r]$. For this, observe that
    \begin{equation}
        W^{ij}_{ij} = (W^{i1} W^{1j})_{ij} \pm O(\epsrel) = \sum_{\ell} W^{i1}_{i\ell} W^{1j}_{\ell j} + O(\epsrel) \label{eq:Wikik}
    \end{equation}
    in degree-24 SoS, where the first step follows by Lemma~\ref{lem:simpler_identity}. By \eqref{eq:deg63} from Lemma~\ref{lem:deg6}, there is a degree-96 SoS proof that
    \begin{equation}
        \biggl(\sum_{\ell: \ell\neq 1} W^{i1}_{i\ell} W^{1j}_{\ell j}\biggr)^2 \le \biggl(\sum_{\ell: \ell\neq 1} (W^{i1}_{i\ell})^2\biggr) \biggl(\sum_{\ell: \ell\neq 1} (W^{1j}_{\ell j})^2\biggr) \le \epsoffdiag^2 r^2, \label{eq:Wikikoff}
    \end{equation}
    so combining \eqref{eq:Wikik} and \eqref{eq:Wikikoff}, we conclude that 
    \begin{equation}
        W^{ij}_{ij} = W^{i1}_{i1} W^{1j}_{1j} \pm O(\epsrel + \epsoffdiag r) = 1\pm O(\epsnorm^2 + \epsoffdiag^2 r^4 + (\epsmap + \epsrel r^3 + \sqrt{\epsoffdiag}r\cdot\radius)/\gap)^{1/4},
    \end{equation}
    where the second step follows by \eqref{eq:1jbound}.
\end{proof}

\subsection{Proof of Theorem~\ref{thm:main_pairwise} and Rounding}
\label{sec:puttogether}

We are now ready to prove Theorem~\ref{thm:main_pairwise} and establish our main algorithmic guarantee for tensor ring decomposition.

\begin{proof}[Proof of Theorem~\ref{thm:main_pairwise}]
    Recall \eqref{eq:QQstar} from Lemma~\ref{lem:QQstar}. By applying the result of Lemma~\ref{lem:usemu} to \eqref{eq:QQstar}, we conclude that there is a degree-96 SoS proof that
    \begin{equation}
        (Q_c)_{ij} - (Q^*_c)_{ij} = \pm O(\epsnorm^2 + \epsoffdiag^2 r^4 + (\epsmap + \epsrel r^3 + \sqrt{\epsoffdiag})/\gap)^{1/4}\cdot \radius. \label{eq:finalQbound}
    \end{equation}
    for all $c\in[d]$ and $i,j\in[r]$.
    By taking pseudoexpectations on both sides of \eqref{eq:finalQbound}, the same bound holds for $\psE{(Q_c)_{ij}} - (Q^*_c)_{ij}$, so the theorem follows upon passing from max-norm to Frobenius norm and recalling the definitions of the error terms in \eqref{eq:finalQbound} from \eqref{eq:epsmap_Q_def}, \eqref{eq:epsrel_def}, \eqref{eq:epsnorm_def}, and \eqref{eq:epsoffdiag_Q_def}.
\end{proof}

\noindent We give the full specification of our algorithm in Algorithm~\ref{alg:tensorring} below. We can now complete the proof of Theorem~\ref{thm:main_tensorring}.

\begin{algorithm2e}
\DontPrintSemicolon
\caption{\textsc{TensorRingDecompose}($S,T$)}
\label{alg:tensorring}
    \KwIn{Second- and third-order moments $\brc{S_{a,b}}, \brc{T_{a,b,c}}$}
    \KwOut{Components $\brc{\wh{Q}_a}$}
        $(\lambda,\mu)\gets$ {\sc FindCombo}($S$).\;
        Let $\psE{\cdot}$ be a degree-96 pseudo-expectation satisfying the constraints of Program~\ref{program:basic} run with vectors $\lambda,\mu$.\;    
        $\wh{Q}_a \gets \psE{Q_a}$ for all $a\in[d]$.\;
        \Return $\brc{\wh{Q}_a}$.
\end{algorithm2e}

\begin{proof}[Proof of Theorem~\ref{thm:main_tensorring}]
    By Lemma~\ref{lem:forcegap} and the assumed bound on $\eta$, $\lambda,\mu$ are $\gap$-non-degenerate combinations of $Q^*_1,\ldots,Q^*_d$ for $\gap = 1/(\kappa\poly(r))$, so the theorem follows from Theorem~\ref{thm:main_pairwise}.
\end{proof}

\subsection{Dependence on \texorpdfstring{$d$}{d}}
\label{sec:fpt}

In this section we observe that for $d$ sufficiently large, one can actually decouple the dependence on $d$ from all other parameters and obtain run in time \emph{linear} in $d$.
\begin{corollary}\label{cor:fpt}
    Suppose that for some $\binom{r+1}{2} \le d'\le d$, Assumption~\ref{assume:tensorring} holds for the first $d'$ units of the polynomial network (i.e. $Q^*_1,\ldots,Q^*_{d'}$) and $\eta \le O(\frac{\kappa^2}{r{d'}^{3/2}})$, and we are given query access to $S\in\R^{d\times d}$ and $T\in\R^{d\times d\times d}$ satisfying \eqref{eq:assume_momentmatch}.
    
    Then there is an algorithm which runs in time $d\cdot\poly(d',r)$ and outputs $\wh{Q}_1,\ldots,\wh{Q}_d$ for which $\gaugedist(\brc{Q^*_a},\brc{\wh{Q}_a}) \le \poly(d',r,\radius,1/\kappa)\cdot \eta^c$ for some absolute constant $c > 0$, with high probability.
\end{corollary}

Note that if $Q^*_1,\ldots,Q^*_{d'}$ are fully-smoothed in the sense of Definition~\ref{def:smoothed_1}, then as we show in Lemma~\ref{lem:fullysmooth} in Section~\ref{sec:conditions}, this holds for $d' = \wt{\Theta}(r^2)$, and we thus obtain a runtime which is linear in $d$ as claimed. 

\begin{proof}
    We can run Algorithm~\ref{alg:tensorring} on the parts of $S$ and $T$ corresponding to the first $d'$ units of the polynomial network to produce $\wh{Q}_1,\ldots,\wh{Q}_{d'}$ satisfying $\gaugedist(\brc{Q^*_1,\ldots,Q^*_{d'}},\brc{\wh{Q}_1,\ldots,\wh{Q}_{d'}}) \le \eta'$ for $\eta' \triangleq \poly(r,\radius,1/\kappa)\cdot \eta^c$ as in Theorem~\ref{thm:main_tensorring}. Note that this takes time $\poly(d',r)$. At this point we can assume without loss of generality that $\norm{Q^*_a - \wh{Q}_a}_F \le \eta'$ for all $1\le a\le d'$.

    To recover $Q^*_{d'+1},\ldots,Q^*_d$, we can then use our estimates $S_{a,b}$ of $\iprod{Q^*_a,Q^*_b}$ for all $1 \le a \le d'$ and $b > d'$ to set up linear systems in the unknowns $Q^*_{d'+1},\ldots,Q^*_d$ (note that this only requires reading at most $dd'$ entries of $S$). That is, for every $b > d'$, we define 
    \begin{equation}
        \wh{Q}_b \triangleq \arg\min_{\wh{Q}} \sum^{d'}_{a=1} \left(S_{a,b} - \iprod{\wh{Q}_a,\wh{Q}} \right)^2.
    \end{equation}
    Because $|S_{a,b} - \iprod{\wh{Q}_a, Q^*_b}| \le |\iprod{\wh{Q}_a - Q^*_a,Q^*_b}| \le \eta'\radius$, we conclude by Part~\ref{assume:condnumber} that $\norm{\wh{Q}_b - Q^*_b}_F \le \eta'\radius\sqrt{d'} / \kappa = \poly(r,\radius,1/\kappa)\cdot \eta^c$ for all $b > d'$. This part of the algorithm only runs in time $d\cdot \poly(r)$ because it only needs to solve a $d'\times \binom{r+1}{2}$-dimensional least-squares problem for every $b > d'$.
\end{proof}

\section{Low-Rank Factorization}
\label{sec:push}

Recall that in low-rank factorization (Definition~\ref{def:lowrankfactorization}), we are given $S\in\R^{d\times d}$ such that there exist unknown symmetric tensors $T^*_1,\ldots,T^*_d\in\R^{r\times r}$ satisfying
\begin{equation}
    |S_{a,b} - \iprod{\vec(T^*_a),\vec(T^*_b)}_{\Sigma}| \le \eta \ \ \forall \ a,b\in[d]. \label{eq:push_moment}
\end{equation} for all $a,b\in[d]$. Additionally, we assume that $T^*_1,\ldots,T^*_d$ are of symmetric rank $\ell < r$, that is, for every $a\in[d]$ there exist vectors $v^*_{a,1},\ldots,v^*_{a,\ell} \in \R^r$ for which 
\begin{equation}
    T^*_a = \sum^{\ell}_{t=1} (v^*_{a,t})^{\otimes \omega}.
\end{equation}
In this section we will focus on $\Sigma\in\R^{r^{\omega}\times r^{\omega}}$ given by
\begin{equation}
    \Sigma \triangleq \E[g\sim\calN(0,\Id_r)]{\vec(g^{\otimes\omega})\vec(g^{\otimes\omega})^{\top}}, \label{eq:Sigmadef}
\end{equation}
though in Section~\ref{sec:rotationinvariant} we describe how our analysis extends easily to general rotation-invariant distributions.

% \begin{lemma}\label{lem:gaussian_reasonable}
%     $\calN(0,\Id)$ is $(1,\omega^{-\omega/2},(2\omega)^{\omega},(2\omega\sqrt{r})^{\omega})$-reasonable. 
% \end{lemma}

% \begin{proof}
%     Condition~\ref{reasonable:normmoment} follows trivially. For Condition~\ref{reasonable:sigmamin}, let $\Sigma = \E[g\sim\calN(0,\Id)]{x^{\otimes\omega}(x^{\otimes\omega})^{\top}}$. For any $p\in\S^{\rchoose-1}$ regarded as an $r$-variate homogeneous polynomial of degree $\omega$, we have $p^{\top} \Sigma p = \E[g\sim\calN(0,\Id)]{p(g)^2}$. Condition~\ref{reasonable:sigmamin} then follows from Lemma~\ref{lem:variance_shift} applied to $a = 1$ and $b = 0$. Condition~\ref{reasonable:unit} follows by the fact that $\E[g\sim\calN(0,\Id_r)]{\iprod{g,v}^{2\omega}} = \E[g\sim\calN(0,1)]{g^{2\omega}} = (2\omega-1)!!$. Finally, for Condition~\ref{reasonable:frob}, $\norm{\Sigma}^2_F = \sum_{i_1,\ldots,i_\omega\in[r]} \E[g]{g_{i_1}\cdots g_{i_\omega}}^2 \le r^{\omega}\cdot (2\omega-1)!!^2 \le (4\omega^2 r)^{\omega}$.
% \end{proof}

For such $\Sigma$, we give a polynomial-time algorithm for recovering $T^*_1,\ldots,T^*_d$ from $S$ for odd $\omega$ under the following extra assumptions:

\begin{assumption}\label{assume:push}
    For parameters $\radius \ge 1$,  $\kappa, \theta, \psi > 0$,% there exist symmetric tensors $T^*_1,\ldots,T^*_d\in(\R^r)^{\otimes \omega}$ of symmetric rank $\ell < r$ for which 
    \begin{enumerate}
        \item (Scaling) $\norm{T^*_a}_F \le \radius$ for all $a\in[d]$. \label{item:push_scale}
        \item (Condition number bound) $\sigma_{\min}(M^*) \ge \kappa$, where $M^*\in\R^{d\times \rchoose}$ is the matrix whose $(a,\sort{i})$-th entry, for $a\in[d]$ and sorted tuple $\sort{i}\in\strings$, is given by $(T^*_a)_{\sort{i}}$. \label{item:push_condnumber}
        \item For any vectors $v_1,\ldots,v_\ell\in\R^r$, let $q(v_1,\ldots,v_\ell)$ denote the vector such that for any sorted tuples $\mb{j}^1,\ldots,\mb{j}^{\ell+1}\in\strings$ and $t_1,\ldots,t_{\ell+1}\in[\ell]$, its $(\mb{j}^1,\ldots,\mb{j}^{\ell+1},t_1,\ldots,t_{\ell+1})$-th entry is given by
        \begin{equation}
            \prod^{\ell+1}_{s=1} (v^{\otimes \omega}_{t_s})_{\mb{j}^s}.
        \end{equation}
        Then for any $v_1,\ldots,v_\ell\in\R^r$, there exists a vector $\lambda\in\R^d$ for which
        \begin{equation}
            q(v_1,\ldots,v_\ell) = \sum^d_{a=1} \lambda_a \cdot q(v^*_{a,1},\ldots,v^*_{a,\ell})
        \end{equation} and $\norm{\lambda}^2_2 \le \theta^2 (\sum^{\ell}_{t=1}\norm{v_t}^2)^{\omega(\ell+1)}$.
        \label{item:push_condnumber2}
        \item $\sigma_{\min}(H)\ge \psi$, where $H\in\R^{d\times\binom{r+1}{2}}$ is the matrix whose $(a,(i,j))$-th entry, for $a\in[d]$ and $1\le i\le j \le r$, is given by $(f^*_a)_i (f^*_a)_j$, where $f^*_a\in\R^r$ is the vector given by
        \begin{equation}
            f^*_a \triangleq \sum^r_{j_1,\ldots,j_{\floor{\omega/2}}=1} (T^*_a)_{j_1j_1\cdots j_{\floor{\omega/2}}j_{\floor{\omega/2}}:}.
        \end{equation}
        \label{item:push_condnumber3}
    \end{enumerate}
\end{assumption}

\begin{remark}
    Parts~\ref{item:push_scale} and \ref{item:push_condnumber} are analogous to those of Assumption~\ref{assume:tensorring}. As for parts~\ref{item:push_condnumber2} and \ref{item:push_condnumber3}, the reader can think of them as consequences of the following stronger assumption: if $w_a$ denotes the concatenation of $v^*_{a,1},\ldots,v^*_{a,\ell}$, then no low-degree $r\ell$-variate polynomial can nearly vanish simultaneously on each of $w_1,\ldots,w_d$. As we show in Section~\ref{sec:compsmooth}, this stronger condition is satisfied by componentwise-smoothed polynomial networks.
\end{remark}

\noindent One can readily check that Assumption~\ref{assume:push} is gauge-invariant (see Appendix~\ref{app:push_gauge} for the proof):

\begin{lemma}\label{lem:push_gauge_invariant}
    If $\brc{T^*_a}$ satisfy \eqref{eq:push_moment} and Assumption~\ref{assume:push} with parameters $\radius, \kappa, \theta, \psi$, then $\brc{F_{V^{\otimes\omega}}(T^*_a)}$ also satisfy \eqref{eq:push_moment} and Assumption~\ref{assume:push} with the same parameters for any $V\in O(r)$.
\end{lemma}

\noindent Under Assumption~\ref{assume:push}, we give an algorithm for low-rank factorization that runs in time polynomial in $d$ when $\omega,\ell = O(1)$:

\begin{theorem}\label{thm:main_push}
    For $d\ge\binom{r+\omega-1}{\omega}$ and $\ell < r$, suppose $T^*_1,\ldots,T^*_d\in (\R^r)^{\otimes\omega}$ satisfy Assumption~\ref{assume:push} and $\eta \le \poly(r,\omega,d,\radius,1/\kappa)^{-\poly(\omega,\ell)}$, and we are given $S\in\R^{d\times d}$ satisfying \eqref{eq:push_moment} for $\Sigma$ given by \eqref{eq:Sigmadef}.
    
    Then there is an algorithm {\sc LowRankFactorize}($S$) (see Algorithm~\ref{alg:push}) which runs in time $(dr)^{\poly(\omega,\ell)}$ and outputs $\wh{T}_1,\ldots,\wh{T}_d$ for which
    \begin{equation}
        \gaugedist(\brc{T^*_a},\brc{\wh{T}_a}) \le \poly(r,\omega,d,\radius,1/\kappa)^{\omega^3} \cdot \left((d\eta/\kappa)^{O(1/\omega)} + \poly(r^{\omega}, \omega^{\ell},\ell^{\ell},d,\radius,1/\kappa)^{\ell}\cdot \sqrt[8]{\theta\eta/\psi^2} \right)
    \end{equation}
    with high probability.
\end{theorem}

\paragraph{Section overview.} The high-level strategy is the same as that of Section~\ref{sec:tensorring}: exhibit a low-degree sum-of-squares proof that the ground truth $\brc{T^*_a}$ is identifiable from $S$. That is, we introduce SoS variables $\brc{T_a}$ which are constrained to have low symmetric rank and approximately the same inner products as $\brc{T^*_a}$, and we want to prove in SoS that the $\brc{T_a}$ are close to $\brc{T^*_a}$ in Frobenius norm by showing that the $r^{\omega}\times r^{\omega}$ linear transformation mapping every $\vec(T^*_a)$ to $\vec(T_a)$ behaves like the Kronecker power $\Id_r^{\otimes\omega}$. In Section~\ref{sec:firstsos} we present a sum-of-squares program along these lines. In Section~\ref{sec:hiddenrotation} we construct an auxiliary $r^{\omega}\times r^{\omega}$ matrix variable $U$ as a proxy for the $r^{\omega}\times r^{\omega}$ transformation and establish basic properties of $U$ in Section~\ref{sec:Uproperties_push}. 

In Section~\ref{sec:preservelowrank} comes the first departure from the techniques of Section~\ref{sec:tensorring}: we leverage the low-rank structure of $\brc{Q^*_a}$ and $\brc{Q_a}$ to prove, roughly speaking, that the transformation $U$ maps any rank-1 tensor to a rank-1 tensor (Corollary~\ref{cor:rank1relation}). In Section~\ref{sec:Uouter_hard} we use this to prove that $U$ has a nice outer product structure (Lemma~\ref{lem:main_userank1_odd} and Lemma~\ref{lem:tensorouter}) that implies that $U$ approximately arises from an $r\times r$ rotation $\wt{U}$, and moreover this $\wt{U}$ can be expressed as a certain linear combination of slices of $U$ (see \eqref{eq:wtU_def}).

It remains to break gauge symmetry and prove that $\wt{U}$ is close to $\Id_r$. In Section~\ref{sec:breaksym} we outline our strategy for breaking symmetry, which requires a number of modifications to the analogous strategy in the tensor ring decomposition setting. In particular, it requires running a \emph{second} sum-of-squares relaxation, which we present in Section~\ref{sec:newprogram}, whose constraints are a strict superset of those of the first relaxation. In Section~\ref{sec:breaksos_push} we analyze this second SoS program and show that $\wt{U}$ is approximately (a multiple of) $\Id_r$. In Section~\ref{sec:push_put_together} we put everything together to give our main algorithm {\sc LowRankFactorize} and prove Theorem~\ref{thm:main_push}.

In Section~\ref{sec:rotationinvariant}, we describe how our analysis extends to more general $\Sigma$, e.g. $\Sigma$ given by $\E[x\sim D]{\vec(x^{\otimes \omega})\vec(x^{\otimes\omega})^{\top}}$ for any reasonable rotation-invariant distribution $D$ over $\R^r$. Finally, in Section~\ref{sec:fpt2}, we show how to improve the runtime of Theorem~\ref{thm:main_push} to only depend \emph{linearly} on $d$.

\subsection{First Sum-of-Squares Relaxation}
\label{sec:firstsos}

To define the first program, we introduce the following notation. Let $\Sigma_{\sym}$ denote the symmetrization of $\Sigma$ (see Definition~\ref{def:ultrasym}), and let $\Sigma^{1/2}_{\sym}$ denote the symmetric square root of $\Sigma_{\sym}$. Let $D\in\R^{\rchoose\times\rchoose}$ denote the diagonal matrix given by
\begin{equation}
    D_{\mb{i}\mb{i}} = \num{i}.\label{eq:Ddef}
\end{equation}
The first SoS program we run is the following:

\begin{program}\label{program:sos2}
    \begin{center}
        \textsc{(Low-Rank Factorization\--- First Part)} \\
    \end{center}
    \noindent\textbf{Parameters:} $S\in\R^{d\times d}$, $\radius\ge 1$, $\kappa > 0$.
    
    \noindent\textbf{Variables:} Let $T_1,\ldots,T_d$ be $r$-dimensional order-$\omega$ tensor-valued indeterminates, let $L, P$ be $\rchoose\times d$ matrix-valued indeterminates, and for every $a\in[d]$, let $v_{a,1},\ldots,v_{a,\ell}$ be $r$-dimensional vector-valued indeterminates. Let $M$ be the $d\times \rchoose$ matrix of indeterminates whose $(a,\sort{i})$-th entry, for $a\in[d]$ and sorted tuple $\sort{i}\in\strings$, is given by $(T_a)_{\sort{i}}$.
    
    \noindent\textbf{Constraints:}
    \begin{enumerate}[leftmargin=*,topsep=0pt]
        \setlength\itemsep{0em}
        \item (Symmetry): $(T_a)_{i_1,\ldots,i_\omega} = (T_a)_{i_{\pi(1)},\ldots, i_{\pi(\omega)}}$ for any $\pi\in\calS_\omega$, $i_1,\ldots,i_\omega\in[r]$. \label{constraint:push_sym}
        \item (Second moments match): $|S_{a,b} - \iprod{\vec(T_a),\vec(T_b)}_{\Sigma}| \le \epsilon$ for all $a,b\in[d]$. \label{constraint:push_moment}
        \item (Low rank): $T_a = \sum^{\ell}_{i=1} (v_{a,i})^{\otimes \omega}$ for all $a\in[d]$. \label{constraint:push_lowrank}
        \item ($T$'s bounded): $\norm{T_a}^2_F \le \radius^2$ for all $a\in[d]$. \label{constraint:push_Tfrob}
        \item (Left-inverse $L$): $LM = \Id$. \label{constraint:push_L}
        \item (Inverse $P$): $PMD\Sigma^{1/2}_{\sym} = \Id$. \label{constraint:push_P}        
        \item ($L$ bounded): $\norm{L}^2_F \le r^{\omega}/\kappa^2$. \label{constraint:push_cond_L}
        \item ($P$ bounded): $\norm{P}^2_F \le r^{\omega}\omega^{\omega/2}/\kappa^2$. \label{constraint:push_cond}
    \end{enumerate}
\end{program}

\noindent The role of the variable $P$ will become apparent when we construct our auxiliary variable for the transformation mapping each $\vec(Q^*_a)$ to $\vec(Q_a)$ in Section~\ref{sec:hiddenrotation}.

We can easily verify that the ground truth is feasible.

\begin{lemma}\label{lem:push_feasible}
    When $d\ge \binom{r+\omega-1}{\omega}$, the pseudodistribution given by the point distribution supported on $(\brc{T^*_a},L^*,P^*,\brc{v^*_{a,t}})$, where $L^*$ is the left inverse of $M^*$, and $P^*$ is the left inverse of $M^*D\Sigma^{1/2}_{\sym}$, is a feasible solution to Program~\ref{program:sos2}.
\end{lemma}

\noindent To prove this, we will use the following condition number bound:

\begin{lemma}\label{lem:sigma_cond}
    $\norm{\Sigma^{-1}_{\sym}}_{\op} \le \omega^{\omega/2}$.
\end{lemma}

\begin{proof}
    For any $p\in\S^{\rchoose-1}$ regarded as an $r$-variate homogeneous polynomial of degree $\omega$, we have $p^{\top} \Sigma p = \E[g\sim\calN(0,\Id)]{p(g)^2}$. The lemma immediately follows from Lemma~\ref{lem:variance_shift} applied to $a = 1$ and $b = 0$.
\end{proof}

\begin{proof}[Proof of Lemma~\ref{lem:push_feasible}]
    It is immediate that Constraints~\ref{constraint:push_sym}-\ref{constraint:push_P} are satisfied. For Constraint~\ref{constraint:push_cond_L}, note that $\norm{L^*}_{\op} \le 1/\kappa$ by Part~\ref{item:push_condnumber} of Assumption~\ref{assume:push}, so $\norm{L^*}^2_F \le \rchoose/\kappa^2 \le r^{\omega}/\kappa^2$. For Constraint~\ref{constraint:push_cond}, note that $\sigma_{\min}(MD\Sigma^{1/2}_{\sym}) \ge  \kappa\sigma_{\min}(\Sigma^{1/2}_{\sym}) \ge \kappa/\omega^{\omega/4}$, where in the last step we used Lemma~\ref{lem:sigma_cond} above.
\end{proof}

\noindent Program~\ref{program:sos2} will only be used to recover partial information about $\brc{T^*_a}$, which we will later use to construct some additional constraints to introduce into Program~\ref{program:sos2}. The resulting modified program, which we give in Section~\ref{sec:newprogram}, will allow us to fully recover $\brc{T^*_a}$.

\subsection{Hidden Rotation Variable}
\label{sec:hiddenrotation}

In this section we use the SoS variables of Program~\ref{program:sos2} to design an auxiliary ``rotation variable'' $U$ that will play the role of the unknown linear transformation sending every $T^*_a$ to $T_a$, after which the focus of our analysis in subsequent sections will be to show this transformation qualitatively behaves like it arises from an $r\times r$ rotation.

First, define the $d\times r^{\omega}$ matrix $N^*$ (resp. $N$) to be the matrix whose $(a,\sort{i})$-th entry for any $a\in[d]$ and sorted tuple $\sort{i}\in\strings$ is given by $(T^*_a)_{\sort{i}}$ (resp. $(T_a)_{\sort{i}}$). Note that $M,M^*$ are sub-tensors of $N,N^*$. Define the $d\times d$ matrix of indeterminates
\begin{equation}
    \calE \triangleq N\Sigma N^{\top} - N^*\Sigma {N^*}^{\top}. \label{eq:push_Edef}
\end{equation}
Because $\iprod{\vec(T^*_a), \vec(T^*_b)}_{\Sigma} = (N^*\Sigma {N^*}^{\top})_{ab}$ and $\iprod{\vec(T_a), \vec(T_b)}_{\Sigma} = (N\Sigma N^{\top})_{ab}$, we conclude by \eqref{eq:push_moment} and Constraint~\ref{constraint:push_moment} of Program~\ref{program:sos2} that
\begin{equation}
    \norm{\calE}_{\max} \le 2\eta. \label{eq:calE_bound}
\end{equation}
A natural way to encode the unknown linear transformation from $\vec(T^*_a)$ to $\vec(T_a)$ as an auxiliary variable would be to consider something like $\Sigma^{-1}N^{-1}N^*\Sigma$, as $(\Sigma^{-1}N^{-1}N^*\Sigma){N^*}^{\top} \approx N^{\top}$, and the $a$-th column of this approximate equality between matrices implies that $\Sigma^{-1}N^{-1}N^*\Sigma$ maps $\vec(T^*_a)$ to $\vec(T_a)$.

As in the discussion in Section~\ref{sec:hiddenrot_tensorring} however, such a construction isn't well-defined: $N$ is an SoS variable, so there is no meaningful notion of a left inverse $N^{-1}$. And because $N^*$ has duplicate columns (because every $T^*_a$ is symmetric), not even $N^*$ has a suitable left inverse. 

To circumvent this issue of duplicate columns, our starting point is to express \eqref{eq:push_Edef} in terms of $M,M^*$ instead of $N,N^*$. Observe that for the matrices $D$ and $\Sigma^{1/2}_{\sym}$ defined in Section~\ref{sec:firstsos},
\begin{equation}
    \left(M^*D\Sigma^{1/2}_{\sym}\right)\left(M^*D\Sigma^{1/2}_{\sym}\right)^{\top} = N^*\Sigma {N^*}^{\top} \qquad\text{and}\qquad \left(MD\Sigma^{1/2}_{\sym}\right)\left(MD\Sigma^{1/2}_{\sym}\right)^{\top} = N\Sigma N^{\top},
\end{equation}
so we can rewrite \eqref{eq:push_Edef} as
\begin{equation}
    \calE = \left(MD\Sigma^{1/2}_{\sym}\right)\left(MD\Sigma^{1/2}_{\sym}\right)^{\top}  - \left(M^*D\Sigma^{1/2}_{\sym}\right)\left(M^*D\Sigma^{1/2}_{\sym}\right)^{\top}. \label{eq:calE_rewrite}
\end{equation}
Now define $\wh{U}$ to be the $\rchoose\times\rchoose$ matrix of indeterminates
\begin{equation}
    \wh{U} \triangleq D^{-1}\Sigma^{-1/2}_{\sym}(PM^* D\Sigma^{1/2}_{\sym})\Sigma^{1/2}_{\sym} D.
\end{equation}
While this expression appears rather cumbersome, the motivation being $\wh{U}$ is simply that if we left-multiply both sides of \eqref{eq:calE_rewrite} by $D^{-1}\Sigma^{-1/2}_{\sym}P$, we get
\begin{equation}
    D^{-1}\Sigma^{-1/2}_{\sym}P\calE = M^{\top} - \wh{U}{M^*}^{\top}. \label{eq:leftmultiplyDsigP}
\end{equation}
As the left-hand side of \eqref{eq:leftmultiplyDsigP} is small, qualitatively this means that $\wh{U}$ is an $\rchoose\times\rchoose$ linear transformation that approximately maps the rows of $M^*$, which correspond to $\vec((T^*_a)_{\sym})$ for all $a\in[d]$, to the rows of $M$, which correspond to $\vec((T_a)_{\sym})$.

To get an $r^{\omega}\times r^{\omega}$ transformation, we define $U$ to be the $r^{\omega}\times r^{\omega}$ matrix of indeterminates given by
\begin{equation}
    U^{\mb{j}}_{\mb{i}} \triangleq \frac{1}{\num{j}}\wh{U}^{\sort{j}}_{\sort{i}} \ \ \forall \ \mb{i},\mb{j}\in\strings. 
\end{equation}
Note that the entries of $U$ are (unknown) linear forms in the indeterminate entries of $P$. The bulk of our analysis will be dedicated to showing that $U$ behaves like the $\omega$-th Kronecker power of an $r\times r$ rotation.

\subsection{Basic Properties of \texorpdfstring{$U$}{U}}
\label{sec:Uproperties_push}

In this section we establish the following simple facts about $U$:
\begin{enumerate}
    \item $U$ is ultra-symmetric in the sense of Definition~\ref{def:ultrasym} (Lemma~\ref{lem:push_Usym})
    \item $U$ approximately maps every $\vec(T^*_a)$ to $\vec(T_a)$ (Lemma~\ref{lem:push_almost_map})
    \item The columns of $\Sigma^{-1/2}U\Sigma^{1/2}$ are approximately orthonormal (Lemma~\ref{lem:push_almost_ortho}).
    \item The Frobenius norm of $U$ can be (loosely) upper bounded (Lemma~\ref{lem:crude_Ubound})\--- we will bootstrap this into a more refined bound later.
    \item The Frobenius norm of the image of any rank-1 tensor under $U$ can be (loosely) upper and lower bounded (Lemma~\ref{lem:bootstrap})\--- we will also bootstrap these later.
    \item $U_{\sym}$ has a right-inverse whose entries are linear forms in the entries of $L$ (Lemma~\ref{lem:UBId}).
\end{enumerate}

\begin{lemma}\label{lem:push_Usym}
    $U$ is ultra-symmetric.
\end{lemma}

\begin{proof}
    This is immediate from the definition of $U$.
\end{proof}

\begin{lemma}\label{lem:push_almost_map} Define
    \begin{equation}
        % \epsmap \triangleq 2(r\omega^{3/2})^{\omega/2}\cdot \epsilon\normbound d.
        \epsmap \triangleq 2(r\omega)^{\omega}\cdot \eta d/\kappa. \label{eq:epsmap_push_def}
    \end{equation}
    Then there is a degree-$O(1)$ SoS proof using the constraints of Program~\ref{program:sos2} that $U{N^*}^{\top} \approx_{\epsmap^2} N^{\top}$. In particular, for any $a\in[d]$, $F_U(T^*_a) \approx_{\epsmap^2} T_a$.
\end{lemma}

\begin{proof}
    We can upper bound the norm of the left-hand side of \eqref{eq:leftmultiplyDsigP} as follows:
    \begin{align}
        \norm{\Sigma^{-1/2}_{\sym}P\calE}^2_F &\le \norm{\Sigma^{-1/2}}^2_F \cdot \sum_{i\in[\rchoose],b\in[d]} \Big(\sum_{a\in[d]}P_{ia} \calE_{ab}\Big)^2 \\
        &\le \Tr(\Sigma^{-1})\cdot \sum_{i,b} \Big(\sum_a P^2_{ia}\Big)\Big(\sum_a \calE^2_{ab}\Big)\\
        % &\le 4\omega^{\omega/2}\cdot r^{\omega}\cdot\eta^2\cdot d^2\sum_{i,a}P^2_{ia} \le 4(\omega^{1/2}r)^{\omega}\cdot\eta^2\normbound^2 d^2,
        &\le 4\omega^{\omega/2}\cdot r^{\omega}\cdot\eta^2\cdot d^2\sum_{i,a}P^2_{ia} \le 4(\omega r^2)^{\omega}\cdot\eta^2d^2/\kappa^2,
    \end{align}
    where in the penultimate step we used Lemma~\ref{lem:sigma_cond}.
    Next, note that for any $a\in[d]$ and $\mb{i}\in\strings$,
    \begin{equation}
        (U{N^*}^{\top})^a_{\mb{i}} = \sum_{\mb{j}\in\strings} U^{\mb{j}}_{\mb{i}} (T^*_a)_{\mb{j}}
        = \sum_{\mb{j}\in\strings} \frac{1}{\num{j}} \wh{U}^{\sort{j}}_{\sort{i}} (T^*_a)_{\sort{j}}
        = (\wh{U}{M^*}^{\top})^a_{\sort{i}}. \label{eq:NtoM}
    \end{equation}
    So
    \begin{align}
        % \norm{N^{\top} - U{N^*}^{\top}}^2_F \le \omega! \cdot \norm{M^{\top} - \wh{U}{M^*}^{\top}}^2_F \le 4(r\omega^{3/2})^{\omega}\cdot\eta^2\normbound^2 d^2
        \norm{N^{\top} - U{N^*}^{\top}}^2_F &\le \omega! \cdot \norm{M^{\top} - \wh{U}{M^*}^{\top}}^2_F \le 4(r^2\omega^2)^{\omega}\cdot\eta^2 d^2/\kappa^2. \qedhere
    \end{align}
\end{proof}

\begin{lemma}\label{lem:push_almost_ortho}
    Define
    \begin{equation}
        % \epsort \triangleq O(\omega r)^{O(\omega)}\cdot \eta^{1/2}\normbound d^{1/2}
        \epsort \triangleq O(\omega r)^{O(\omega)}\cdot \eta^{1/2} d^{1/2} / \kappa^2 \label{eq:epsort_push_def}
    \end{equation}
    Then there is a degree-$O(1)$ SoS proof using the constraints of Program~\ref{program:sos2} that $U^{\top}\Sigma U \approx_{\epsort^2} \Sigma$.
\end{lemma}

\begin{proof}
    We have
    \begin{align}
        (U^{\top}\Sigma U)^{\mb{j}}_{\mb{i}} &= \frac{1}{\num{i}\cdot\num{j}}\sum_{\mb{k},\mb{\ell}\in\strings} \wh{U}^{\sort{i}}_{\sort{k}} \Sigma^{\sort{\ell}}_{\sort{k}} \wh{U}^{\sort{j}}_{\sort{\ell}} = \frac{1}{\num{i}\cdot\num{j}} \sum_{\sort{k},\sort{\ell} \ \text{sorted}} \wh{U}^{\sort{i}}_{\sort{k}} D_{\sort{k}\sort{k}}\Sigma^{\sort{\ell}}_{\sort{k}} D_{\sort{\ell}\sort{\ell}} \wh{U}^{\sort{j}}_{\sort{\ell}} \\
        &= \frac{1}{\num{i}\cdot\num{j}} (\wh{U}^{\top} D \Sigma_{\sym} D \wh{U})^{\sort{j}}_{\sort{i}} = (D^{-1} \wh{U}^{\top} D\Sigma_{\sym} D \wh{U} D^{-1})^{\sort{j}}_{\sort{i}}. \label{eq:USigU}
    \end{align}
    Expanding the definition of $\wh{U}$,
    \begin{align}
        \MoveEqLeft D^{-1} \wh{U}^{\top} D\Sigma_{\sym} D \wh{U} D^{-1} \\
        &= D^{-1} \left(D\Sigma^{1/2}_{\sym} (PM^*D\Sigma^{1/2}_{\sym})^{\top} \Sigma^{-1/2}_{\sym} D^{-1}\right) D\Sigma_{\sym} D \left(D^{-1}\Sigma^{-1/2}_{\sym} (PM^*D\Sigma^{1/2}_{\sym}) \Sigma^{1/2}_{\sym} D\right) D^{-1} \\
        &= \Sigma^{1/2}_{\sym} (PM^*D\Sigma^{1/2}_{\sym})^{\top}(PM^*D\Sigma^{1/2}_{\sym}) \Sigma^{1/2}_{\sym} \label{eq:SPMD}
    \end{align}
    We will argue that $(PM^*D\Sigma^{1/2}_{\sym})^{\top}(PM^*D\Sigma^{1/2}_{\sym})$ is close to identity. Indeed, by left- and right-multiplying \eqref{eq:calE_rewrite} by $P$ and $P^{\top}$, we find that
    \begin{equation}
        P\calE P^{\top} = \Id - (PM^*D\Sigma^{1/2})(PM^*D\Sigma^{1/2}_{\sym})^{\top}.
    \end{equation}
    By a sequence of steps essentially identical to in the proof of Lemma~\ref{lem:almost_ortho}, we can show that 
    % $\norm{P\calE P^{\top}}^2_F \le 4\eta^2\normbound^4 d^2$. 
    $\norm{P\calE P^{\top}}^2_F \le 4\eta^2 r^{2\omega} \omega^{\omega} d^2 / \kappa^4$. 
    So, denoting $\wt{M}^*\triangleq PM^*D\Sigma^{1/2}$ for convenience, we find that
    \begin{equation}
        % \iprod{(\wt{M}^*)_{\mb{i}}, (\wt{M}^*)_{\mb{j}}} = \bone{\mb{i} = \mb{j}} \pm 2\eta\normbound^2 d. \ \ \forall \ \text{sorted} \ \mb{i},\mb{j}\in\strings.
        \iprod{(\wt{M}^*)_{\mb{i}}, (\wt{M}^*)_{\mb{j}}} = \bone{\mb{i} = \mb{j}} \pm 2\eta r^{\omega} \omega^{\omega/2} d / \kappa^2 \ \ \forall \ \text{sorted} \ \mb{i},\mb{j}\in\strings.
    \end{equation}
    By Lemma~\ref{lem:ortho_sos}, we conclude that
    \begin{equation}
        % \iprod{(\wt{M}^*)^{\mb{i}}, (\wt{M}^*)^{\mb{j}}} = \bone{\mb{i} = \mb{j}} \pm O(\eta^{1/2}\normbound d^{1/2} r^{3\omega/2}). \ \ \forall \ \text{sorted} \ \mb{i},\mb{j}\in\strings.
        \iprod{(\wt{M}^*)^{\mb{i}}, (\wt{M}^*)^{\mb{j}}} = \bone{\mb{i} = \mb{j}} \pm O(\eta^{1/2}\omega^{\omega/4} d^{1/2} r^{2\omega} / \kappa) \ \ \forall \ \text{sorted} \ \mb{i},\mb{j}\in\strings.
    \end{equation}
    That is, there is a symmetric matrix $\calE'$ with 
    % $\norm{\calE'}^2_F \le O(\eta\normbound^2 d\cdot r^{O(\omega)})$
    $\norm{\calE'}^2_F \le O(\eta\omega^{\omega/2} d\cdot r^{O(\omega)} / \kappa^2)$
    for which $(\wt{M}^*)^{\top} \wt{M}^* = \Id + \calE'$.
    So we can further rewrite \eqref{eq:SPMD} as $\Sigma_{\sym} + \Sigma^{1/2}_{\sym}\calE'\Sigma^{1/2}_{\sym}$. To bound its norm, note that 
    \begin{equation}
        \norm{\Sigma^{1/2}_{\sym}}^2_F = \Tr(\Sigma_{\sym}) = \sum_{1\le i_1 \le \cdots \le i_\omega \le r} \E[g\sim\calN(0,\Id)]{g^2_{i_1}\cdots g^2_{i_\omega}} \le O(\omega r)^{\omega}.
        % \norm{\Sigma^{1/2}_{\sym}}^2_F = \Tr(\Sigma_{\sym}) \le r^{\omega/2}\tau_2
    \end{equation}
    We have thus concluded that
    \begin{equation}
        \norm{\Sigma_{\sym} - D^{-1}\wh{U}^{\top}_{\sym}D\Sigma_{\sym} D \wh{U} D^{-1}}^2_F \le 
        % O(\omega r)^{O(\omega)}\cdot \eta\normbound^2 d.
        O(\omega r)^{O(\omega)}\cdot \tau_2\eta d / \kappa^2.
    \end{equation}
    The lemma then follows from \eqref{eq:USigU} and the fact that $\norm{U^{\top}\Sigma U - \Sigma}^2_F \le \omega!^2 \norm{(U^{\top}\Sigma U)_{\sym} - \Sigma_{\sym}}^2_F$.
\end{proof}

\begin{lemma}\label{lem:crude_Ubound}
    % $\norm{U}^2_F \le (r\omega)^{O(\omega)} \cdot d\normbound^2\radius^2$. 
    There is a degree-$O(1)$ SoS proof using the constraints of Program~\ref{program:sos2} that $\norm{U}^2_F \le (r\omega)^{O(\omega)} \cdot d \radius^2 / \kappa^2$. 
    % In particular, for any $\mb{i}\in\strings$, 
    % % $\norm{U^{\mb{i}}}^2 \le (r\omega)^{O(\omega)} \cdot d\normbound^2\radius^2$.
    % $\norm{U^{\mb{i}}}^2 \le (r\omega)^{O(\omega)} \cdot d\radius^2 / \kappa^2$.
\end{lemma}

\begin{proof}
    As $\norm{U}^2_F \le r^{2\omega}\cdot\norm{\wh{U}}^2_F$, it is enough to bound $\norm{\wh{U}}^2_F$. Note that
    \begin{align}
        \norm{\wh{U}}^2_F &\le \omega!^2\norm{\Sigma^{-1/2}_{\sym}(PM^*D\Sigma^{1/2})\Sigma^{1/2}_{\sym}}^2_F \\
        &\le \omega!^2 \cdot \Tr(\Sigma_{\sym})\norm{\Sigma^{-1}_{\sym}}^2_F \cdot \norm{PM^*D}^2_F \\
        &\le \omega!^4 \Tr(\Sigma_{\sym})\norm{\Sigma^{-1}_{\sym}}^2_F \cdot \norm{P}^2_F \norm{M^*}^2_F. \label{eq:whUbound}
    \end{align}
    By Lemma~\ref{lem:sigma_cond}, $\norm{\Sigma^{-1}_{\sym}}^2_F \le (r\omega)^{\omega}$. Recall from \eqref{eq:sigmaupperbound} that $\norm{\Sigma}^2_F \le r^{\omega}\cdot (2\omega-1)!!^2$, so $\Tr(\Sigma) \le r^{\omega}\cdot (2\omega-1)!!$. 
    % Additionally, $\norm{P}^2_F \le \normbound^2$ 
    Additionally, $\norm{P}^2_F \le r^{\omega}\omega^{\omega/2}/\kappa^2$ 
    by Constraint~\ref{constraint:push_cond} and $\norm{M^*}^2_F \le d\radius^2$ by Part~\ref{assume:scale} of Assumption~\ref{assume:push}. So we can upper bound \eqref{eq:whUbound} by $(r\omega)^{O(\omega)} \cdot d\radius^2 / \kappa^2$ as claimed.
\end{proof}

\begin{lemma}\label{lem:bootstrap}
    For any $v\in\mathbb{S}^{r-1}$, there is a degree-$O(1)$ SoS proof using the constraints of Program~\ref{program:sos2} that $r^{-\omega/2}/2 \le  \norm{F_U(v^{\otimes\omega})}^2_F \le \omega^{O(\omega)}$.
\end{lemma}

\begin{proof}
    By Lemma~\ref{lem:push_almost_ortho}, 
    \begin{equation}
        \iprod{F_U(v^{\otimes\omega}), F_U(v^{\otimes\omega})}_{\Sigma} = \E[g\sim\calN(0,\Id)]{\iprod{v,g}^{2\omega}} \pm \epsort = (2\omega-1)!! \pm \epsort. \label{eq:sigmanorm}
    \end{equation} For the upper bound, recall by Lemma~\ref{lem:sigma_cond} we that $\Sigma\succeq \omega^{-\omega/2}\Id$, so by \eqref{eq:sigmanorm}, $\norm{F_U(v^{\otimes\omega})}^2_F \le \omega^{O(\omega)}$.
    
    For the lower bound, we must upper bound the spectral norm of $\Sigma$. We will do this by giving a (crude) upper bound on the Frobenius norm of $\Sigma$:
    \begin{equation}
        \norm{\Sigma}^2_F = \sum_{i_1,\ldots,i_\omega\in[r]} \E[g]{g_{i_1}\cdots g_{i_\omega}}^2 \le r^{\omega}\cdot (2\omega-1)!!^2, \label{eq:sigmaupperbound}
    \end{equation}
    from which we conclude that $\Sigma \preceq r^{\omega/2}\cdot (2\omega-1)!!\cdot \Id$. Combining this with \eqref{eq:sigmanorm} gives the desired bound of $\norm{F_U(v^{\otimes\omega})}^2_F \ge r^{-\omega/2}/2$.
\end{proof}

\begin{lemma}\label{lem:UBId}
    There exists a $\rchoose\times\rchoose$ matrix $B$ of indeterminates whose entries are linear forms in the entries of $L$ such that there is a degree-$O(1)$ SoS proof using the constraints of Program~\ref{program:sos2} that 
    % $\norm{U_{\sym}B - \Id}^2_F \le 4(r\omega^{1/2})^{\omega}\eta^2\normbound^4 d^2$. 
    $\norm{U_{\sym}B - \Id}^2_F \le 4(r^3\omega^{3/2})^{\omega}\eta^2 d^2 / \kappa^4$. 
    Furthermore, $\norm{B}^2_F \le r^{\omega}\omega^{2\omega} \radius^2 d / \kappa^2$.
\end{lemma}

\begin{proof}
    Consider $B\triangleq D{M^*}^{\top} L^{\top}$. For any sorted $\sort{i},\sort{j}\in\strings$,
    \begin{equation}
        (U_{\sym}B)^{\sort{j}}_{\sort{i}} = \sum_{\sort{k} \ \text{sorted}, a\in[d]} \num{k} \cdot U^{\sort{k}}_{\sort{i}}\cdot (T^*_a)_{\sort{k}}\cdot (L^{\top})^{\sort{j}}_a = \sum_{a\in[d]} (U{N^*}^{\top})^a_{\sort{i}}\cdot (L^{\top})^{\sort{j}}_a = (\wh{U} {M^*}^{\top} L^{\top})^{\mb{j}}_{\mb{i}},
    \end{equation}
    where in the last step we used \eqref{eq:NtoM}. So by \eqref{eq:leftmultiplyDsigP} and Constraint~\ref{constraint:push_L},
    \begin{equation}
        U_{\sym}B = \wh{U}{M^*}^{\top} L^{\top} = \Id - D^{-1}\Sigma^{-1/2}_{\sym}P\calE L^{\top}.
    \end{equation}
    Note that
    \begin{align}
        \norm{P\calE L^{\top}}^2_F &\le \sum_{\mb{i},\mb{j} \ \text{sorted}} \biggl(\sum^d_{a,b=1} P^a_{\mb{i}}\calE_{ab} L^b_{\mb{j}}\biggr)^2 \le \sum_{\mb{i},\mb{j} \ \text{sorted}} \biggl(\sum_{a,b} \calE_{ab} (P^a_{\mb{i}})^2\biggr)\biggl(\sum_{a,b} \calE_{ab} (L^b_{\mb{j}})^2\biggr) \\
        % &= 4\eta^2d^2 \norm{P}^2_F \norm{L}^2_F \le 4\eta^2\normbound^2 r^{\omega} d^2 / \kappa^2, 
        &= 4\eta^2d^2 \norm{P}^2_F \norm{L}^2_F \le 4\eta^2 r^{2\omega} \omega^{\omega/2} d^2 / \kappa^4, 
    \end{align}
    % so $\norm{D^{-1}\Sigma^{-1/2}_{\sym}P\calE L^{\top}}^2_F \le \Tr(\Sigma^{-1})\cdot 4\eta^2\normbound^4 d^2 \le 4(r\omega^{1/2})^{\omega}\eta^2\normbound^4 d^2$.
    so $\norm{D^{-1}\Sigma^{-1/2}_{\sym}P\calE L^{\top}}^2_F \le \Tr(\Sigma^{-1})\cdot 4\eta^2r^{2\omega}\omega^{\omega} d^2 / \kappa^4 \le 4(r^3\omega^{3/2})^{\omega}\eta^2 d^2 / \kappa^4$, concluding the proof of the first part. For the second part, $\norm{B}^2_F \le (r^{\omega}/\kappa^2)\cdot \norm{D{M^*}^{\top}}^2_F \le r^{\omega}\omega^{2\omega}\radius^2 d / \kappa^2$.
\end{proof}

\subsection{Preservation of Low Rank}
\label{sec:preservelowrank}

In this section we first show that $U$ maps \emph{any} rank-$\ell$ tensor to an approximately rank-$\ell$ tensor (Lemma~\ref{lem:genericflatrank}). We then show that this implies that $U$ maps any rank-$(\ell-1)$ tensor to an approximately rank-$(\ell-1)$ tensor (Lemma~\ref{lem:rankrtorankr}). Continuing inductively, we conclude that $U$ maps any rank-1 tensor to an approximately rank-1 tensor (Corollary~\ref{cor:rank1relation}). We then use this fact to deduce useful algebraic identities on the entries of $U$, in particular Lemma~\ref{lem:perm}, and show that the $(i,\cdots,i)$-th columns of $U$ for $i\in[r]$ are approximately orthonormal (Lemma~\ref{lem:diagonal_orth}).

As it is difficult to reason about tensor rank, we will work with the following proxy which will be straightforward to encode using simple polynomial constraints.

\begin{definition}\label{def:Vrank}
    Given $V\in(\R^r)^{\otimes \omega-2}$, we say that an $r$-dimensional order-$\omega$ tensor $T$ of indeterminates is \emph{$\xi$-approximately $V$-rank at most $m$} if for all $\brc{a_1,\ldots,a_{m+1}},\brc{b_1,\ldots,b_{m+1}}\subset[r]$, 
    \begin{equation}
        -\xi \le \sum_{\pi\in\calS_{m+1}} \sgn(\pi) \prod^{m+1}_{s=1} T(V,:,:)_{a_s,b_{\pi(s)}} \le \xi \label{eq:minor}
    \end{equation}
    Note that if $T$ is $\xi$-approximately $V$-rank at most $m$, then for any $c,c'\in\R$, $cT$ is $(cc')^{m+1}$-approximately $c'V$-rank at most $m$.
\end{definition}

\noindent In other words, a tensor is approximately $V$-rank at most $m$ if its contraction according to $V$ is approximately rank-$m$, in the sense that its $(m+1)\times(m+1)$ minors are all small. While this is a strictly weaker than the usual notion of symmetric tensor rank, we show in Lemma~\ref{lem:Vrankimplies1rank} below that if a tensor is approximately $V$-rank 1 for many ``random'' choices of $V$, then it effectively behaves like a symmetric rank 1 tensor.

We first verify that $U$ maps $T^*_1,\ldots,T^*_d$ to tensors of low approximate $V$-rank for any $V$. This is a simple consequence of Lemma~\ref{lem:push_almost_map} and Constraint~\ref{constraint:push_lowrank} of Program~\ref{program:sos2}, and we defer its proof to Appendix~\ref{app:defer_approxVrank}:

\begin{lemma}\label{lem:approxVrank}
    For every $c\in[d]$ and $V\in(\R^r)^{\otimes\omega-2}$ satisfying $\norm{V}^2_F \le 1$, there is an SoS proof which is degree-$O(\ell)$ in the indeterminate $T_c$ and degree-$O(\omega\ell)$ in the indeterminates $\brc{v_{c,t}}$ that $F_U(T^*_c)$ is $(3\radius\ell)^{\ell+1}\cdot \epsmap$-approximately $V$-rank at most $\ell$.
    % \begin{equation}
    %     \xi \triangleq \max(3\ell,3\radius\ell)^{\ell+1}\cdot \epsmap. \label{eq:xidef}
    % \end{equation}
\end{lemma}

\noindent Next we show that, because $U$ maps the symmetric rank-$\ell$ tensors $T^*_1,\ldots,T^*_d$ to tensors with low approximate $V$-rank and because $T^*_1,\ldots,T^*_d$ are sufficiently ``generic'' by Part~\ref{item:push_condnumber} of Assumption~\ref{assume:push}, $U$ sends \emph{all} symmetric rank-$\ell$ tensors $S$ to tensors that are approximately $V$-rank at most $\ell$.

\begin{lemma}\label{lem:genericflatrank}
    Take any $V\in(\R^r)^{\otimes\omega-2}$ for which $\norm{V}^2_F = 1$. Let $S\in(\R^r)^{\otimes\omega}$ be any symmetric tensor for which there exist $v_1,\ldots,v_{\ell}$ satisfying $S = \sum^{\ell}_{t=1} v^{\otimes\omega}_t$ and $\sum^{\ell}_{t=1}\norm{v_t}^2 = 1$. There is an SoS proof which is degree-$O(\ell)$ in the indeterminates $\brc{T_a}$ and degree-$O(\omega\ell)$ in the indeterminates $\brc{v_{a,t}}$ that $F_U(S)$ is $\theta\sqrt{d} (3\radius\ell)^{\ell+1}\epsmap$-approximately $V$-rank at most $\ell$.
\end{lemma}

\begin{proof}
    Recall from Definition~\ref{def:Vrank} that for any symmetric tensor $S\in(\R^r)^{\otimes\omega}$, the condition that $F_U(S)$ is $\xi$-approximately $V$-rank at most $\ell$ is equivalent to the condition that for all $\brc{a_1,\ldots,a_{\ell+1}}$,  $\brc{b_1,\ldots,b_{\ell+1}}\subset[r]$ of size $\ell+1$,
    \begin{equation}
        -\xi \le \sum_{\pi\in\calS_{\ell+1}} \sgn(\pi)\prod^{\ell+1}_{s=1}F_U(S)(V,:,:)_{a_s,b_{\pi(s)}} \le \xi. \label{eq:rank1minorvanish}
    \end{equation}
    Because
    \begin{equation}
        F_U(S)(V,:,:)_{a_s,b_{\pi(s)}} = \sum_{\mb{j}\in\strings} S_{\mb{j}} \cdot \biggl(\sum_{\mb{i}\in[r]^{\omega-2}} U^{\mb{j}}_{\mb{i}a_s b_{\pi(s)}} V_{\mb{i}}\biggr),
    \end{equation}
    we can use symmetry of $S$ to express the left-hand side of \eqref{eq:rank1minorvanish} as 
    \begin{equation}
        \sum_{\sort{j}^1,\ldots,\sort{j}^{\ell+1}\in\strings} S_{\sort{j}^1}\cdots S_{\sort{j}^{\ell+1}}\cdot Z_{\sort{j}^1\cdots\sort{j}^{\ell+1}} \label{eq:pU_ugly}
    \end{equation}
    for some terms $\brc{Z_{\mb{j}^1,\ldots,\mb{j}^{\ell+1}}}$ which are degree-$O(\ell)$ polynomials in the entries of $U,V$. Here the subscripts $\sort{j}^1,\ldots,\sort{j}^{\ell+1}$ are sorted tuples from $\strings$.
    
    Because $S = \sum^{\ell}_{t=1} v^{\otimes\omega}_t$, we can rewrite \eqref{eq:pU_ugly} as
    \begin{equation}
        -\xi \le \sum_{\substack{\sort{j}^1,\ldots,\sort{j}^{\ell+1}\in\strings \\ t_1,\ldots,t_{\ell+1}\in[\ell]}} \left(\prod^{\ell+1}_{s=1}(v^{\otimes \omega}_{t_s})_{\sort{j}^s}\right)\cdot Z_{\sort{j}^1\cdots\sort{j}^{\ell+1}} \sum^{\ell}_{t_1,\ldots,t_{\ell+1} = 1} \le \xi \label{eq:final_ugly}
    \end{equation}
    By Lemma~\ref{lem:approxVrank}, \eqref{eq:final_ugly} holds for the components $\brc{v_t}\triangleq \brc{v^*_{a,t}}$ of $T^*_a$ and $\xi = (3\radius\ell)^{\ell+1}\cdot \epsmap$.
    
    From \eqref{eq:final_ugly} it becomes apparent why we require Part~\ref{item:push_condnumber2} of Assumption~\ref{assume:push}: for any $v_1,\ldots,v_{\ell}\in\R^r$, recall that $q(v_1,\ldots,v_{\ell})$ denotes the vector such that for any sorted $\mb{j}^1,\ldots,\mb{j}^{\ell+1}\in\strings$ and $t_1,\ldots,t_{\ell+1}\in[\ell]$, its $(\mb{j}^1,\ldots,\mb{j}^{\ell+1},t_1,\ldots,t_{\ell+1})$-th entry is given by $\prod^{\ell+1}_{s=1}(v^{\otimes \omega}_{t_s})_{\mb{j}^s}$. We can thus rewrite \eqref{eq:final_ugly} for any choice of $\brc{v_t} = \brc{v^*_{a,t}}$ as 
    \begin{equation}
        -\xi \le \iprod{q(v^*_{a,1},\ldots,v^*_{a,\ell}),Z} \le \xi. \label{eq:final_ugly2}
    \end{equation}
    By Part~\ref{item:push_condnumber2} of Assumption~\ref{assume:push}, for any $v_1,\ldots,v_{\ell}\in\R^r$ there exists $\lambda\in\R^d$ satisfying $\norm{\lambda}^2_2 \le \theta^2$ and for which $\sum^d_{a=1} \lambda_a q(v^*_{a,1},\ldots,v^*_{a,\ell}) = q(v_1,\ldots,v_{\ell})$, so we conclude that
    \begin{align}
        \iprod{q(v_1,\ldots,v_{\ell}),Z}^2 &= \biggl(\sum^d_{a=1} \lambda_a \iprod{q(v^*_{a,1},\ldots,v^*_{a,\ell}),Z}\biggr)^2 \le \norm{\lambda}^2_2 \cdot \biggl(\sum^d_{a=1}\iprod{q(v^*_{a,1}\ldots,v^*_{a,\ell}),Z}^2 \biggr) \le \xi^2 d\theta^2. \qedhere
    \end{align}
\end{proof}

\noindent Next comes the key inductive step. We show that if $U$ sends all tensors of symmetric rank $\ell$ to tensors that are approximately $V$-rank $\ell$, then $U$ sends all tensors of symmetric rank $\ell-1$ to tensors of $V$-rank $\ell - 1$. The high-level idea in the proof is that any $(\ell+1)\times(\ell+1)$ minor in the definition of having $V$-rank $\ell$ can be expanded as a linear combination of $\ell\times\ell$ minors. So because the $(\ell+1)\times(\ell+1)$ minors of the image under $U$ of an arbitrary symmetric rank $\ell - 1$ tensor $T$ plus an arbitrary rank-1 perturbation $z^{\otimes\omega}$ are small, we can show that the $\ell\times\ell$ minors of the image of $T$ are also small:

\begin{lemma}\label{lem:rankrtorankr}
    Take any $V\in (\R^r)^{\otimes \omega-2}$ for which $\norm{V}_F = 1$ and, for some $\gamma > 0$, $|V_{x\cdots x}| \ge \gamma$ for all $x\in[r]$. Let $T\in(\R^r)^{\otimes\omega}$ be any symmetric tensor of symmetric rank at most $\ell - 1$, given by $T = \sum^{\ell-1}_{t=1} v_t^{\otimes\omega}$ for $\sum^{\ell-1}_{t=1} \norm{v_t}^2  = 1$.
    
    Using the constraints that for any $W = \sum^{\ell}_{t=1} w_t^{\otimes\omega}$ for which $\sum^{\ell}_{t=1}\norm{w_t}^2 = 1$, $F_U(W)$ is $\xi$-approximately $V$-rank at most $\ell$, there is a degree-$O(\ell)$ SoS proof that $F_U(T)$ is $\xi^*$-approximately $V$-rank at most $\ell-1$ for
    \begin{equation}
        \xi^*\triangleq \frac{1}{\gamma}\cdot r^{\omega/2}\cdot O(\ell)^{O(\ell)} \cdot 
        % \left(\xi\cdot \omega^{\omega} r^{\omega/2}\radius\sqrt{d}/\kappa^2 + O(\omega)^{O(\omega\ell)} \eta\normbound^2 d\right).
        \left(\xi \omega^{\omega} r^{\omega/2}\radius\sqrt{d}/\kappa^2 + O(\omega)^{O(\omega\ell)}\cdot r^{\omega} \eta d / \kappa^2\right).
    \end{equation}
\end{lemma}

\begin{proof}
    Take any subsets $I,J\subset[r]$ of size $\ell + 1$ and define the linear map $F^{I,J}_{U;V}: (\R^r)^{\otimes \omega} \to \R^{(\ell+1)\times(\ell+1)}$ by
    \begin{equation}
        F^{I,J}_{U;V}(\cdot) \triangleq F_U(\cdot)(V,:,:)^J_I.
    \end{equation}
    Take any $z\in\S^{r-1}$ and define the $(\ell+1)\times(\ell+1)$ matrices
    \begin{equation}
        T'\triangleq F^{I,J}_{U;V}(T) \qquad S_z \triangleq F^{I,J}_{U;V}(z^{\otimes\omega})
    \end{equation}
    % Note that for any $a,b\in I\cap J$,
    % \begin{equation}
    %     (S_z)_{ab} = (S_z)_{ba} \ \text{and} \ T'_{ab} = T'_{ba} \label{eq:STpartialsym}
    % \end{equation}
    % because $U$ is ultra-symmetric.
    We will need some basic bounds on the entries of $T'$. For any $a\in I,b\in J$,
    \begin{equation}
        (T'_{a,b})^2 = \biggl(\sum_{\mb{i}\in[r]^{\omega-2}} F_U(T)_{\mb{i}ab}\cdot V_{\mb{i}}\biggr)^2 \le \norm{F_U(T)}^2_F \label{eq:Tprimeentry}
    \end{equation}
    by Cauchy-Schwarz and our assumption that $\norm{V}_F = 1$. Note that by Lemma~\ref{lem:sigma_cond} and Lemma~\ref{lem:push_almost_ortho},
    \begin{equation}
        \norm{F_U(T)}^2_F \le \omega^{\omega/2} \iprod{F_U(T),F_U(T)}_{\Sigma} = \omega^{\omega/2} \Bigl( \E[g\sim\calN(0,\Id)]{\iprod{T,g^{\otimes\omega}}^2} \pm \epsort\Bigr). \label{eq:FUT_norm}
    \end{equation}
    We have
    \begin{equation}
        \E[g]{\iprod{T,g^{\otimes\omega}}^2} = \E[g]*{\biggl(\sum^{\ell-1}_{t=1} \iprod{v_t,g}^{\omega}\biggr)^2} \le \ell (2\omega-1)!! \le \ell\cdot (2\omega)^{\omega},
    \end{equation}
    so combining this with \eqref{eq:Tprimeentry} and \eqref{eq:FUT_norm}, we find that
    \begin{equation}
        (T'_{a,b})^2 \le \ell\cdot O(\omega)^{3\omega/4} \ \ \forall \ a\in I, b\in J.
    \end{equation}
    In particular, for any $a_1,\ldots,a_\ell\in I$, $b_1,\ldots,b_{\ell}\in J$,
    \begin{equation}
        \prod^\ell_{s=1} T'_{a_s,b_s} = \pm \ell^{\ell/2}\cdot O(\omega)^{3\omega\ell/4}. \label{eq:Tprimepower}
    \end{equation}

    %%%%%%%%%%% end of preliminary parts of the proof %%%%%%%%%%% 
    With these preliminary estimates in place, we proceed to the core of the argument. Given $A\subseteq[\ell+1]$, let $R_A\in\R^{(\ell+1)\times(\ell+1)}$ denote the matrix whose columns indexed by $A$ are given by the corresponding columns in $S_z$, and whose remaining columns are given by the corresponding columns in $T'$. 
    Then for any $c > 0$,
    \begin{equation}
        \det(T'+c S_z) = \sum^{\ell+1}_{t = 0} c^{t} \sum_{A\subseteq[\ell+1]:|A|=t} \det(R_A). \label{eq:polarization}
        %  = \sum^{\ell+1}_{t=1} c^{t} \sum_{A\subseteq[\ell+1]:|A|=t} \sgn(A)\det(R_A),
    \end{equation}
    Note that because $T + c\cdot z^{\otimes\omega} = \sum^{\ell-1}_{t=1} v^{\otimes\omega}_t + c\cdot z^{\otimes \omega}$ has symmetric rank at most $\ell$, by assumption on $U$ we have
    \begin{equation}
        \det(T'+c S_z) = \pm\xi\cdot (1 + c^{2/\omega})^{\omega(\ell+1)/2}. \label{eq:det_bound}
    \end{equation}
    Consider taking $c = \frac{1}{\ell+2},\frac{2}{\ell+2},\ldots,1$. By \eqref{eq:det_bound} and Corollary~\ref{cor:general_vandermonde} applied with $D = 1$ and $e = \ell+1$, 
    % there is $\lambda\in\R^{\ell+2}$ for which, for the Vandermonde matrix
    % \begin{equation}
    %     V\triangleq \begin{pmatrix}
    %         1 & \frac{1}{\ell+2} & \cdots & \left(\frac{1}{\ell+2}\right)^{\ell+1} \\
    %         1 & \frac{2}{\ell+2} & \cdots & \left(\frac{2}{\ell+2}\right)^{\ell+1} \\
    %         \vdots & \vdots & \ddots & \vdots \\
    %         1 & 1 & \cdots & 1
    %     \end{pmatrix},
    % \end{equation}
    % $\lambda^{\top} V = (0,1,\ldots,0)$ and furthermore $\norm{\lambda} \le O(\ell)^{2\ell+3}$. 
    there is a linear combination of the equations \eqref{eq:polarization} for these different choices of $c$ that yields a bound on the $t = 1$ summand of \eqref{eq:polarization}. Rewriting that summand as $\sum_{a\in I, b\in J} \sigma_{a,b} (S_z)_{ab} \cdot M_{(a,b)}$, where each $\sigma_{a,b}\in\brc{\pm 1}$ is some sign, and $M_{(a,b)}$ is the $(a,b)$-th minor of $T'$, we conclude that
    \begin{equation}
        \sum_{a\in I, b\in J} \sigma_{a,b} (S_z)_{ab} \cdot M_{(a,b)} = \pm \xi\cdot  O(\ell)^{O(\ell)}\cdot 2^{\omega(\ell+1)/2}. \label{eq:smallminors}
    \end{equation}
    Note that by definition of $S_z$, for any $a\in I, b\in J$
    \begin{equation}
        (S_z)_{ab} = \sum_{\mb{i}\in[r]^{\omega-2},\mb{j}\in\strings} U^{\mb{j}}_{\mb{i}ab}(z^{\otimes\omega})_{\mb{j}} V_{\mb{j}}. \label{eq:Sz_expression}
    \end{equation}
    By Lemma~\ref{lem:symmetric_inspan}, for any $\mb{j}^* = (j^*_1,\ldots,j^*_\omega)\in\strings$, there exist $z^{(1)},\ldots,z^{(s)}\in\S^{r-1}$ and $w\in\R^s$ for which $\norm{w}_1\le\symbound$ (where $\symbound$ is defined in \eqref{eq:gamdef}) and
    \begin{equation}
        \sum^s_{i=1} w_i \sum_{\mb{j}\in\strings} U^{\mb{j}}({z^{(i)}}^{\otimes\omega})_{\mb{j}} = \frac{1}{\omega!}\sum_{\pi\in\calS_{\omega}} U^{j^*_{\pi(1)}\cdots j^*_{\pi(\omega)}} = U^{\mb{j}^*}, \label{eq:symmetrize}
    \end{equation}
    where in the last step we used the fact that $U$ is ultra-symmetric.
    
    We conclude from \eqref{eq:Sz_expression} and \eqref{eq:symmetrize} that for any $\mb{j}^*\in\strings$, there are $z^{(1)},\ldots,z^{(s)}\in\S^{r-1}$ and $w\in\R^s$ for which $\norm{w}_1 \le\symbound$ and
    \begin{equation}
        \sum^s_{i=1} w_i S'_{z^{(i)}} = \sum_{\mb{i}\in[r]^{\omega-2}} U^{\mb{j}^*}_{\mb{i}::}V_{\mb{i}}. \label{eq:UjVi}
    \end{equation}
    Incorporating this into \eqref{eq:smallminors}, we get
    \begin{equation}
        \sum_{a\in I,b\in J}\sigma_{a,b}M_{(a,b)} \sum_{\mb{i}\in[r]^{\omega-2}} U^{\mb{j}^*}_{\mb{i}ab} V_{\mb{i}} = \pm \xi\cdot \symbound\cdot O(\ell)^{O(\ell)}\cdot 2^{\omega(\ell+1)/2} \ \ \forall \mb{j}^*\in\strings. \label{eq:use_sym_inspan}
    \end{equation}
    
    Finally, recall from Lemma~\ref{lem:UBId} that there is an $\rchoose\times\rchoose$ matrix $B$ of indeterminates such that
    % $\norm{U_{\sym}B - \Id}^2_F \le 4(r\omega^{1/2})^{\omega}\eta^2\normbound^4 d^2$ 
    $\norm{U_{\sym}B - \Id}^2_F \le 4(r^3\omega^{3/2})^{\omega}\eta^2 d^2 / \kappa^4$ 
    and $\norm{B}^2_F \le r^{\omega}\omega^{2\omega}\radius^2 d / \kappa^2$. 
    Let $\delta$ denote the $r^{\omega}$-dimensional vector given by
    \begin{equation}
        \delta_{\mb{i}} = U_{\sym}B^{\sort{i}} - \bone{\sort{i} = (x,\cdots,x)} \ \ \forall \ \mb{i}\in\strings.
    \end{equation} By the bound on $\norm{U_{\sym}B - \Id}^2_F$, we know \begin{equation}
        \norm{\delta}^2 \le
        % 4\omega!(r\omega^{1/2})^{\omega}\eta^2\normbound^4 d^2. \label{eq:delta_bound}
        4\omega!(r^3\omega^{3/2})^{\omega}\eta^2 d^2 / \kappa^4. \label{eq:delta_bound}
    \end{equation}
    Then taking the linear combination of \eqref{eq:UjVi} for different choices of sorted $\mb{j}^*$ according to the entries of $B^{x\cdots x}$, we obtain a matrix whose $(a,b)$ entry for any $a\in I, b\in J$ is given by
    \begin{align}
        \sum_{\mb{j}^* \ \text{sorted}} B^{x\cdots x}_{\mb{j}^*} \sum_{\mb{i}\in[r]^{\omega-2}} U^{\mb{j}^*}_{\mb{i}ab}V_{\mb{i}} &= \sum_{\mb{i}\in[r]^{\omega-2}} \left(\bone{\overline{\mb{i}ab} = x\cdots x} +\delta_{\mb{i}ab}\right)V_{\mb{i}} \\
        &= V_{x\cdots x}\cdot \bone{a = x, b = x} + \sum_{\mb{i}\in[r]^{\omega-2}} \delta_{\mb{i}ab}\cdot V_{\mb{i}} \\
        \intertext{Defining $\epsilon_{a,b} \triangleq  \sum_{\mb{i}\in[r]^{\omega-2}} \delta_{\mb{i}ab}\cdot V_{\mb{i}}$, we can further rewrite this as}
        &= V_{x\cdots x}\cdot \bone{a=x,b=x} + \epsilon_{a,b}. \label{eq:BUV}
    \end{align}
    Note that
    \begin{equation}
        \epsilon_{a,b} = 
        % \pm 2\omega!^{1/2}(r\omega^{1/2})^{\omega/2}\eta\normbound^2 d
        \pm 2\omega!^{1/2}(r\omega^{3/4})^{\omega}\eta d / \kappa^2
        \ \ \forall \ a\in I, b\in J
    \end{equation}
    by Cauchy-Schwarz and \eqref{eq:delta_bound}.
    
    Combining \eqref{eq:use_sym_inspan} and the fact that $\norm{B^{x\cdots x}}^2 \le \norm{B}^2_F \le r^{\omega}\omega^{2\omega}\radius^2 d / \kappa^2$, note that
    \begin{equation}
        \sum_{\mb{j}^* \ \text{sorted}} B^{x\cdots x}_{\mb{j}^*} \sum_{a\in I,b\in J}\sigma_{a,b}M_{(a,b)} \sum_{\mb{i}\in[r]^{\omega-2}} U^{\mb{j}^*}_{\mb{i}ab} V_{\mb{i}} = \pm\xi'  \label{eq:final_combo}
    \end{equation}
    for
    \begin{equation}
        \xi' \triangleq \xi\cdot \symbound\cdot O(\ell)^{O(\ell)} \omega^{\omega}\radius \sqrt{d} r^{\omega}\cdot 2^{\omega(\ell+1)/2} / \kappa.
    \end{equation}
    We can combine \eqref{eq:BUV} and \eqref{eq:final_combo} to find that
    \begin{equation}
        \sum_{a\in I, b\in J} \sigma_{a,b} M_{(a,b)} \left(V_{x\cdots x} \cdot \bone{a=x,b=x} + \epsilon_{a,b}\right) \le \xi'. \label{eq:sigmaMV}
    \end{equation}
    Note that
    \begin{align}
        \biggl(\sum_{a\in I, b\in J} \sigma_{a,b} M_{(a,b)}\epsilon_{a,b}\biggr)^2 &\le \biggl(\sum_{a,b} M^2_{(a,b)}\biggr)\biggl(\sum_{a,b} \epsilon_{a,b}^2\biggr) \\
        &\le (\ell+1)^4 \cdot (\ell+1)!^2\cdot \ell^{\ell}\cdot O(\omega)^{3\omega\ell/2}\cdot
        % \left(4\omega!(r^2\omega)^{\omega/2}\eta^2\normbound^4 d^2\right) \\
        \biggl(4\omega!(r^3\omega^{3/2})^{\omega}\eta^2 d^2 / \kappa^4\biggr) \\
        % \le O(\ell)^{O(\ell)}\cdot O(\omega)^{O(\omega\ell)} r^{\omega} \eta^2\normbound^4 d^2 \triangleq (\xi'')^2,
        &\le O(\ell)^{O(\ell)}\cdot O(\omega)^{O(\omega\ell)} r^{3\omega} \eta^2 d^2 / \kappa^4 \triangleq (\xi'')^2,
    \end{align}
    so substituting this into \eqref{eq:sigmaMV}, we have
    \begin{equation}
        \sigma_{x,x} M_{(a,b)} V_{x\cdots x} = \pm(\xi' + \xi'').\label{eq:xixi}
    \end{equation}
    We know $|C_{x,x}| = \omega!\cdot |V_{x\cdots x}| \ge \omega!\gamma$ by assumption on $V$, so this implies that $M_{(a,b)} \le \pm (\xi'+\xi'')/(\omega!\gamma)$. Unpacking the definitions of $\xi'$ and $\xi''$, this concludes the proof of the lemma.
\end{proof}

\noindent Lemmas~\ref{lem:genericflatrank} and \ref{lem:rankrtorankr} already imply that for any $V\in(\R^r)^{\otimes\omega-2}$, the transformation $U$ maps all symmetric rank-1 tensors to tensors which are approximately $V$-rank 1. It remains to relate $V$-rank back to the usual notion of symmetric tensor rank. To do this, we next show that there exists a collection of $V$ such that any tensor which is approximately $V$-rank 1 for all such $V$ behave approximately like a tensor with symmetric rank 1.
% SoS proof that V-rank 1 for several V's implies rank-1
\begin{lemma}\label{lem:Vrankimplies1rank}
    Let $T$ be an order-$\omega$, $r$-dimensional symmetric tensor of indeterminates. For $m = r^{2\omega}$, there exist tensors $V^{(1)},\ldots,V^{(N)}\in(\R^r)^{\otimes \omega-2}$ of unit Frobenius norm, whose entries are all lower bounded in magnitude by $1/r^{\Theta(\omega)}$, and such that, using the constraints that $T$ is $\xi$-approximately $V_i$-rank 1 for all $i\in[m]$, there is a degree-$O(1)$ SoS proof that
    \begin{equation}
        T_{\mb{i}ab} \cdot T_{\mb{j}cd} - T_{\mb{i}ad} \cdot T_{\mb{j}cb} + T_{\mb{j}ab} \cdot T_{\mb{i}cd} - T_{\mb{j}ad} \cdot T_{\mb{i}cb} = \xi\cdot r^{O(\omega)} \label{eq:2x2vanish_pre}
    \end{equation}
    for all $\mb{i},\mb{j}\in[r]^{\omega-2}$ and $a,b,c,d\in[r]$.
\end{lemma}

\noindent To prove this, we use the following helper lemma which supplies the desired collection of tensors $V^{(1)},\ldots,V^{(N)}$.

\begin{lemma}\label{lem:basis_sym}
    For any $m\in\mathbb{N}$, there is a basis $\calB = \brc{M_1,\ldots,M_{\binom{m+1}{2}}}$ for the space of symmetric matrices in $\R^{m\times m}$ consisting of rank-$1$ matrices of unit Frobenius norm and with all entries lower bounded in magnitude by $1/\poly(m)$. Furthermore, for any $i,j\in[m]$, there exist coefficients $\lambda\in\R^{\binom{m+1}{2}}$ such that $\sum_t \lambda_t M_t = e_ie_j^{\top} + e_je_i^{\top}$ for which $\norm{\lambda} = O(1)$.
\end{lemma}

We defer the proof of Lemma~\ref{lem:basis_sym} to Appendix~\ref{app:basis_sym}.

\begin{proof}[Proof of Lemma~\ref{lem:Vrankimplies1rank}]
    By Lemma~\ref{lem:basis_sym}, there exist $V^{(1)},\ldots,V^{(N)}\in(\R^r)^{\otimes \omega-2}$ for $N = \binom{r^{\omega-2}+1}{2} = r^{\Theta(\omega)}$ of unit Frobenius norm, whose entries are all lower bounded in magnitude by $1/r^{\Theta(\omega)}$, and for which for any $\mb{i}^*,\mb{j}^*\in[r]^{\omega-2}$, there exists $\lambda\in\R^N$ with norm $O(1)$ such that $e_{\mb{i}^*}e_{\mb{j}^*}^{\top} + e_{\mb{j}^*} e_{\mb{i}^*}^{\top} = \sum^N_{t=1} \lambda_t \vec(V^{(t)})\vec(V^{(t)})^{\top}$.
    
    For any $V\in(\R^r)^{\otimes \omega-2}$, as $T(V,:,:)_{ab} = \sum_{\mb{i}\in[r]^{\omega-2}} T_{\mb{i}ab} V_{\mb{i}}$, the constraint that $T$ is $\xi$-approximately $V$-rank 1 is equivalent to the constraint that for all $a,b\in[r]$,
    \begin{equation}
        \sum_{\mb{i},\mb{j}\in[r]^{\omega-2}} V_{\mb{i}} V_{\mb{j}} \left(T_{\mb{i}ab} \cdot T_{\mb{j}cd} - T_{\mb{i}ad} \cdot T_{\mb{j}cb}\right) = \pm\xi. \label{eq:mult_Vranks}
    \end{equation}
    By taking $V$ above to be $V^{(t)}$ for $t = 1,\ldots,N$ and taking the linear combination of the resulting constraints specified by $\lambda$, we obtain \eqref{eq:2x2vanish_pre}.
\end{proof}

\noindent We can now combine Lemmas~\ref{lem:genericflatrank}, \ref{lem:rankrtorankr}, and \ref{lem:Vrankimplies1rank} to conclude that $U$ sends any symmetric rank-1 tensor to a tensor $T$ satisfying \eqref{eq:2x2vanish_pre} (see Appendix~\ref{app:defer_rank1relation} for a formal proof).

\begin{corollary}\label{cor:rank1relation}
    Define
    \begin{equation}
        % \epsrel \triangleq O(\ell^{\ell}\omega^{\omega}r^{\omega}\radius d)^{O(\ell)}\cdot\normbound^2(\theta\epsmap+\eta)
        \epsrel\triangleq O(\ell^{\ell}\omega^{\omega}r^{2\omega}\radius d/\kappa)^{O(\ell)}\cdot(\theta\epsmap+\eta)
    \end{equation}
    Let $S\in(\R^r)^{\otimes\omega}$ be any symmetric rank-1 tensor with Frobenius norm 1. Under Assumption~\ref{assume:push}, there is an SoS proof which is degree-$O(\ell)$ in the indeterminates $\brc{T_a}$ and degree-$O(\omega\ell)$ in the indeterminates $\brc{v_{a,t}}$ that $T\triangleq F_U(S)$ satisfies
    \begin{equation}
        T_{\mb{i}ab} \cdot T_{\mb{j}cd} - T_{\mb{i}ad} \cdot T_{\mb{j}cb} + T_{\mb{j}ab} \cdot T_{\mb{i}cd} - T_{\mb{j}ad} \cdot T_{\mb{i}cb} = \pm\epsrel \label{eq:2x2vanish}
    \end{equation}
\end{corollary}

\subsubsection{Consequences of Corollary~\ref{cor:rank1relation}}

We can bootstrap Corollary~\ref{cor:rank1relation} to prove the following strengthening:

\begin{lemma}\label{lem:perm}
    Let $S,T$ be as in Corollary~\ref{cor:rank1relation}. 
    % Let $T$ be an order-$\omega$, $r$-dimensional symmetric tensor of indeterminates satisfying $\norm{T}^2_F \ge c$ for some absolute constant $c > 0$. 
    Define
    \begin{equation}
        \epsrel^* \triangleq O(r\omega)^{O(\omega^2)}\epsrel = O(r\omega)^{O(\omega^2)}\cdot O(\ell^{\ell}\omega^{\ell}r^{\omega}\radius d/\kappa)^{O(\ell)}\cdot \eta
    \end{equation}
    For any $m\in\mathbb{N}$, any collection of indices $\brc{i_a}_{a\in[\omega m]}$, and any permutation $\pi\in\calS_{\omega m}$, there is a degree-$O(\omega m)$ SoS proof using the constraints \eqref{eq:2x2vanish} that
    \begin{equation}
        \prod^{m-1}_{a=0} T_{i_{a\omega+1} \cdots i_{(a+1)\omega}} = \prod^{m-1}_{a=0} T_{i_{\pi(a\omega+1)} \cdots i_{\pi((a+1)\omega)}} \pm \epsrel^* \label{eq:rearrange_all}
    \end{equation}
\end{lemma}

\noindent To show this, we need the following helper lemma showing that any $T$ satisfying \eqref{eq:2x2vanish} from Corollary~\ref{cor:rank1relation} has the following outer product structure:

% outer product structure
\begin{lemma}\label{lem:outerprodstructure}
    Let $S,T$ be as in Corollary~\ref{cor:rank1relation}. There is a degree-$O(\omega)$ SoS proof using the constraints \eqref{eq:2x2vanish} that for any $\mb{i},\mb{a}\in\strings$, $T \triangleq F_U(S)$ satisfies
    \begin{equation}
        (T_{a_{1:\omega}})^{\omega-1} T_{i_1\cdots i_{\omega}} = \prod^{\omega}_{t=1} T_{a_{1:t-1} i_t a_{t+1:\omega}} \pm \omega^{O(\omega^2)}\cdot\epsrel, \label{eq:outerproducttensor}
    \end{equation} where we use the notation $a_{s:t}$ introduced in Section~\ref{sec:prelims} to denote the string $a_s a_{s+1}\cdots a_t$.
\end{lemma}

\noindent We defer the proof of this to Appendix~\ref{app:defer_outerprodstructure}. Here we give a simple proof sketch for a special case.

\begin{proof}[Proof sketch for $\omega = 3$]
    In place of $a_1,a_2,a_3$ and $i_1,i_2,i_3$, we will use the letters $a,b,c$ and $i,j,k$ to make the notation clearer. When $\omega = 3$, the desired identity \eqref{eq:outerproducttensor} takes the form
    \begin{equation}
        (T_{abc})^2 T_{ijk} \approx T_{ibc} T_{ajc} T_{abk} \label{eq:sketchomega3}
    \end{equation}
    In this proof sketch, we will pretend that this and \eqref{eq:2x2vanish} from Corollary~\ref{cor:rank1relation} are exact equalities. The latter tells us that
    \begin{equation}
        T_{abc}T_{ijk} + T_{ibc}T_{ajk} = T_{abk}T_{ijc} + T_{ajc}T_{ibk}.
    \end{equation}
    Right multiplying both sides by $T_{abc}$, we get
    \begin{equation}
        (T_{abc})^2T_{ijk} + T_{ibc}(T_{abc} T_{ajk}) = T_{abk}(T_{abc}T_{ijc}) + T_{ajc}(T_{abc}T_{ibk}). \label{eq:sketch}
    \end{equation}
    By another application of Corollary~\ref{cor:rank1relation}, $T_{abc} T_{ajk} = T_{abk} T_{ajc}$, $T_{abc}T_{ijc} = T_{ajc} T_{ibc}$, and $T_{abc}T_{ibk} = T_{abk} T_{ibc}$. So the second, third, and fourth terms in \eqref{eq:sketch} are all equal to $T_{ibc}T_{ajc}T_{abk}$, and rearranging yields the desired identity \eqref{eq:sketchomega3}.
\end{proof}

\begin{proof}[Proof of Lemma~\ref{lem:perm}]
    We first show that for any $i_1,\ldots,i_{\omega-2},j_1,\ldots,j_{\omega-2},a,b,c,d\in[r]$,
    \begin{equation}
        T_{i_1\cdots i_{\omega-2}ab} T_{j_1\cdots j_{\omega-2}cd} = T_{i_1\cdots i_{\omega}ad} T_{j_1\cdots j_{\omega-2}cb} \pm O(\norm{T}_{\max})^{O(\omega)}\epsrel/ (c^{\omega-1}r^{\omega/2-1}) \label{eq:true2x2}
    \end{equation} 
    By Lemma~\ref{lem:outerprodstructure}, for any $x_1,\ldots,x_{\omega}\in[r]$ we have
    \begin{multline}
        \biggl(\prod^{\omega-2}_{t=1} T_{x_{1:t-1}i_t x_{t+1:\omega}} T_{x_{1:t-1}j_t x_{t+1:\omega}}\biggr) T_{x_{1:\omega-2}a x_{\omega}} T_{x_{1:\omega-2}x_{\omega-1}b} T_{x_{1:\omega-2}c x_{\omega}} T_{x_{1:\omega-2}x_{\omega-1}d} \\
        = (T_{x_{1:\omega}})^{2\omega-2} T_{i_1\cdots i_{\omega-2}ab} T_{j_1\cdots j_{\omega-2}cd} \pm O(\norm{T}_{\max})^{O(\omega)}\epsrel
    \end{multline}
    Similarly, we have
    \begin{multline}
        \biggl(\prod^{\omega-2}_{t=1} T_{x_{1:t-1}i_t x_{t+1:\omega}} T_{x_{1:t-1}j_t x_{t+1:\omega}}\biggr) T_{x_{1:\omega-2}a x_{\omega}} T_{x_{1:\omega-2}x_{\omega-1}d} T_{x_{1:\omega-2}c x_{\omega}} T_{x_{1:\omega-2}x_{\omega-1}b} \\
        = (T_{x_{1:\omega}})^{2\omega-2} T_{i_1\cdots i_{\omega-2}ad} T_{j_1\cdots j_{\omega-2}cb} \pm O(\norm{T}_{\max})^{O(\omega)}\epsrel.
        % &= \left(\prod^{\omega-2}_{t=1} T_{x_{1:t-1}i_t x_{t+1:\omega}} T_{x_{1:t-1}j_t x_{t+1:\omega}}\right)\cdot T_{x_{1:\omega-2}a x_{\omega-1}} T_{x_{1:\omega-2}x_{\omega-1}b} T_{x_{1:\omega-2}c x_{\omega}} T_{x_{1:\omega-2}x_{\omega-1}d}
    \end{multline}
    Note that the two expressions on the left-hand side are equal, so in particular
    \begin{equation}
        (T_{x_{1:\omega}})^{2\omega-2} \left(T_{i_1\cdots i_{\omega-2}ab} T_{j_1\cdots j_{\omega-2}cd} - T_{i_1\cdots i_{\omega-2}ad} T_{j_1\cdots j_{\omega-2}cb}\right) = \pm O(\norm{T}_{\max})^{O(\omega)}\epsrel \label{eq:almost2x2}
    \end{equation}
    Recall that $\norm{T}^2_F \ge r^{-\omega/2}/2$ by Lemma~\ref{lem:bootstrap}. By degree-$(4\omega-4)$ SoS Holder's,
    \begin{equation}
        \sum_{x_1,\ldots,x_{\omega}} (T_{x_{1:\omega}})^{4\omega-4} \ge r^{2\omega-2} \biggl(\sum_{x_1,\ldots,x_{\omega}} (T_{x_{1:\omega}})^2\biggr)^{2\omega-2} \ge r^{-O(\omega^2)}. \label{eq:holder}
    \end{equation} Squaring both sides of \eqref{eq:almost2x2}, summing over $x_1,\ldots,x_\omega$, and applying \eqref{eq:holder}, we get
    \begin{equation}
        \left(T_{i_1\cdots i_{\omega-2}ab} T_{j_1\cdots j_{\omega-2}cd} - T_{i_1\cdots i_{\omega-2}ad} T_{j_1\cdots j_{\omega-2}cb}\right)^2 \le O(\norm{T}_{\max})^{O(\omega)}\cdot r^{O(\omega^2)}\epsrel^2,
    \end{equation}
    so \eqref{eq:true2x2} follows.
    
    To complete the proof of the lemma, it suffices to establish the special case where $\pi$ is given by any transposition of two elements in $[\omega m]$. If $\pi$ is a transposition of two elements from $\brc{a\omega+1,\ldots,(a+1)\omega}$ for some $0\le a < m$, then this just follows by symmetry of $T$. If $\pi$ is a transposition of two elements which lie in $\brc{a\omega+1,\ldots,(a+1)\omega}$ and $\brc{a'\omega+1,\ldots,(a'+1)\omega}$ respectively for some distinct $a,a'$, we can invoke \eqref{eq:true2x2} and symmetry of $T$.
\end{proof}

\noindent We can now use Lemma~\ref{lem:perm} to show that certain columns of $U$ are orthogonal and have unit norm.

\begin{lemma}\label{lem:diagonal_orth}
    For any odd $\omega\ge 3$, using Lemma~\ref{lem:push_almost_ortho} and the constraints that $T = F_U(e_i^{\otimes\omega})$ satisfies \eqref{eq:2x2vanish} for any $i\in[r]$, there is a degree-$O(\omega^2)$ SoS proof that 
    \begin{equation}
        \iprod{U^{i\cdots i}, U^{j\cdots j}} = \bone{i = j} \pm \left(\epsort/\omega! + O(r\omega)^{O(\omega^2)}\cdot {\epsrel^*}^{1/2}\right) \label{eq:ort}
    \end{equation}
    for all $i,j\in[r]$.
    % provided that $\epsrel^* \le r^{-\Omega(\omega^3)}$.
\end{lemma}

\begin{proof}
    Define the matrix $E \triangleq U^{\top}\Sigma U - \Sigma$. By Lemma~\ref{lem:push_almost_ortho}, $\norm{E}^2_F \le \epsort^2$. We have
    \begin{equation}
        \sum_{\mb{s},\mb{t}\in\strings} U^{i\cdots i}_{\mb{s}} U^{j\cdots j}_{\mb{t}} \E[g\sim\calN(0,\Id)]{g_{s_1}\cdots g_{s_\omega} g_{t_1}\cdots g_{t_\omega}} = \E[g\sim\calN(0,\Id)]{g^{\omega}_i g^{\omega}_j} + E^{i\cdots i}_{j\cdots j}. \label{eq:tautology}
    \end{equation}
    Take any $x,y\in[r]$ and multiply both sides by $(U^{i\cdots i}_{x\cdots x} U^{j\cdots j}_{y\cdots y})^{\omega-1}$ to get 
    \begin{equation}
        (U^{i\cdots i}_{x\cdots x} U^{j\cdots j}_{y\cdots y})^{\omega-1} \sum_{\mb{s},\mb{t}\in\strings} U^{i\cdots i}_{\mb{s}} U^{j\cdots j}_{\mb{t}} \E[g]{g_{s_1}\cdots g_{s_\omega} g_{t_1}\cdots g_{t_\omega}} = (U^{i\cdots i}_{x\cdots x} U^{j\cdots j}_{y\cdots y})^{\omega-1}\biggl(\E[g]{g^{\omega}_i g^{\omega}_j} + E^{i\cdots i}_{j\cdots j}\biggr). \label{eq:usigu_entry}
    \end{equation}
    For any $\mb{s},\mb{t}\in\strings$, note that by Lemma~\ref{lem:perm} applied to $T = F_U(e_i^{\otimes\omega})$,
    \begin{equation}
        (U^{i\cdots i}_{x\cdots x} U^{j\cdots j}_{y\cdots y})^{\omega-1} U^{i\cdots i}_{\mb{s}} U^{j\cdots j}_{\mb{t}} = \prod^{\omega}_{z=1} U^{i\cdots i}_{x\cdots x s_z} U^{j\cdots j}_{y\cdots y t_z} \pm \epsrel^*. \label{eq:apply_perm}
    \end{equation} Define the $\omega$-dimensional vectors of indeterminates $v \triangleq U^{i\cdots i}_{x\cdots x :}$ and $w\triangleq U^{j\cdots j}_{y\cdots y :}$. Then substituting \eqref{eq:apply_perm} into \eqref{eq:usigu_entry}, we have
    \begin{align}
        \text{LHS of \eqref{eq:usigu_entry}} &= \sum_{\mb{s},\mb{t}\in\strings} \E[g]*{\prod^{\omega}_{z=1} g_{s_z} v_{s_z} \cdot g_{t_z} v_{t_z}} \pm (2\omega-1)!!\cdot\epsrel^* \\
        &= \E[g]*{\iprod{v,g}^{\omega}\cdot \iprod{w,g}^{\omega}} \pm (2\omega-1)!!\cdot\epsrel^*\\
        &= \omega!\sum^{\floor{\omega/2}}_{m=0} \binom{\omega}{m, m, \omega - 2m} \frac{1}{2^{2m}} \iprod{v,w}^{\omega - 2m} \norm{v}^{2m} \norm{w}^{2m} \pm (2\omega-1)!!\cdot\epsrel^*, \label{eq:apply_hermite}
    \end{align}
    where in the second step we used Lemma~\ref{lem:hermite}.
    
    We now consider the cases of $i = j$ and $i\neq j$ separately.
    
    \noindent\textbf{Case 1:} $i = j$. Take $x = y$. Then we can rewrite \eqref{eq:apply_hermite} minus the error term as
    \begin{align}
        \norm{v}^{2\omega}\cdot \omega!\sum^{\floor{\omega/2}}_{m=0} \binom{\omega}{m, m, \omega - 2m} \frac{1}{2^{2m}} &= (2\omega - 1)!!\norm{v}^{2\omega} = (2\omega-1)!!\sum_{u_1,\ldots,u_{\omega}\in[r]} v^2_{u_1}\cdots v^2_{u_{\omega}} \\
        &= (2\omega-1)!!\sum_{u_1,\ldots,u_{\omega}\in[r]} (U^{i\cdots i}_{x\cdots x u_1})^2 \cdots (U^{i\cdots i}_{x\cdots x u_\omega})^2 \\
        &= (2\omega-1)!!\sum_{u_1,\ldots,u_{\omega}\in[r]} (U^{i\cdots i}_{x\cdots x})^{2\omega-2} (U^{i\cdots i}_{u_1\cdots u_{\omega}})^2 \pm r^{\omega}\epsrel^* \\
        &= (2\omega-1)!!\cdot (U^{i\cdots i}_{x\cdots x})^{2\omega-2} \cdot \norm{U^{i\cdots i}}^2 \pm r^{\omega}\epsrel^*, \label{eq:rewrite_v2w}
    \end{align}
    where in the fourth step we applied Lemma~\ref{lem:perm}. From \eqref{eq:usigu_entry} with this choice of $i = j$ and $x = y$ and \eqref{eq:apply_hermite}, we conclude that
    \begin{equation}
        (2\omega-1)!!\cdot (U^{i\cdots i}_{x\cdots x})^{2\omega-2} \cdot \norm{U^{i\cdots i}}^2 \pm O(r^{\omega}\cdot\epsrel^*) = (U^{i\cdots i}_{x\cdots x})^{2\omega - 2}\cdot ((2\omega-1)!! + E^{i\cdots i}_{i\cdots i}).
    \end{equation}
    Rearranging and summing over $x\in[r]$ yields
    \begin{equation}
        \biggl(\sum^r_{x=1} (U^{i\cdots i}_{x\cdots x})^{2\omega - 2}\biggr)\cdot \Bigl(\norm{U^{i\cdots i}}^2 - 1 - \frac{1}{(2\omega-1)!!}E^{i\cdots i}_{i\cdots i}\Bigr) = O(r^{\omega+1}\epsrel^*/(2\omega-1)!!).
    \end{equation}
    Lemma~\ref{lem:sumpowersT} below (which we can apply by the bound we assumed on $\eta$ to begin with in Theorem~\ref{thm:push_main_pairwise}, which ensures that $\epsrel^* \le r^{-\Omega(\omega^3)}$) and Part~\ref{fact:divide_both_sides} of Fact~\ref{fact:division} imply that
    \begin{equation}
        \norm{U^{i\cdots i}}^2 = 1 + \frac{1}{(2\omega-1)!!} E^{i\cdots i}_{i\cdots i} \pm 2r^{\omega^2}\epsrel^* / \norm{U^{i\cdots i}}^{\omega} = 1\pm (\epsort/(2\omega-1)!! + r^{O(\omega^2)}\epsrel^*),
    \end{equation}
    where in the last step we used Lemma~\ref{lem:bootstrap}.
    
    \noindent\textbf{Case 2:} $i\neq j$. Unlike in Case 1, we no longer insist that $x = y$. Note that $\E{g^{\omega}_i g^{\omega}_j} = 0$, so substituting \eqref{eq:apply_hermite} into \eqref{eq:usigu_entry} yields
    \begin{equation}
        \omega!\sum^{\floor{\omega/2}}_{m=0} \binom{\omega}{m, m, \omega - 2m} \frac{1}{2^{2m}} \iprod{v,w}^{\omega - 2m} \norm{v}^{2m} \norm{w}^{2m} = (U^{i\cdots i}_{x\cdots x})^{\omega-1} (U^{j\cdots j}_{y\cdots y})^{\omega-1} E^{i\cdots i}_{j\cdots j} \pm (2\omega-1)!!\cdot\epsrel^*.
    \end{equation}
    Squaring both sides of this, we get 
    \begin{multline}
        \omega!^2 \sum_{m,m'=0}^{\floor{\omega/2}} c_m c_{m'} \iprod{v,w}^{2(\omega-m-m')}\norm{v}^{2(m+m')} \norm{w}^{2(m+m')} \\
        = (U^{i\cdots i}_{x\cdots x})^{2\omega-2} (U^{j\cdots j}_{y\cdots y})^{2\omega-2} (E^{i\cdots i}_{j\cdots j})^2 \pm O(\norm{U}^{2\omega-2}_{\max}\epsort\cdot (2\omega-1)!!\epsrel^*),
    \end{multline}
    where $\brc{c_m}$ are positive scalars and $c_0 = 1$. Lower bounding the left-hand side by the $(m,m') = (0,0)$ summand, we conclude that
    \begin{equation}
        \omega!^2 \cdot \iprod{v,w}^{2\omega} \le (U^{i\cdots i}_{x\cdots x})^{2\omega-2} (U^{j\cdots j}_{y\cdots y})^{2\omega-2} (E^{i\cdots i}_{j\cdots j})^2 + \omega^{O(\omega)} \epsort\epsrel^*, \label{eq:zeroth}
    \end{equation}
    where we used Lemma~\ref{lem:bootstrap} to naively upper bound the entries of $U^{i\cdots i}$ and $U^{j\cdots j}$ in magnitude.
    
    Note that
    \begin{align}
        \iprod{v,w}^{2\omega} &= \sum_{u_1,\ldots,u_{2\omega}\in[r]} v_{u_1}\cdots v_{u_{2\omega}} w_{u_1}\cdots w_{u_{2\omega}} \\
        &= \sum_{u_1,\ldots,u_{2\omega}\in[r]} U^{i\cdots i}_{x\cdots x u_1} \cdots U^{i\cdots i}_{x\cdots x u_{2\omega}} U^{j\cdots j}_{y\cdots y u_1} \cdots U^{j\cdots j}_{y\cdots y u_{2\omega}} \\
        &= \sum_{u_1,\ldots,u_{2\omega}\in[r]} (U^{i\cdots i}_{x\cdots x})^{2\omega-2} (U^{j\cdots j}_{y\cdots y})^{2\omega-2} U^{i\cdots i}_{u_1\cdots u_\omega} U^{i\cdots i}_{u_{\omega+1}\cdots u_{2\omega}} U^{j\cdots j}_{u_1\cdots u_\omega} U^{j\cdots j}_{u_{\omega+1}\cdots u_{2\omega}} \pm r^{2\omega}\epsrel^*\\
        &= (U^{i\cdots i}_{x\cdots x})^{2\omega-2} (U^{j\cdots j}_{y\cdots y})^{2\omega - 2} \iprod{U^{i\cdots i}, U^{j\cdots j}}^2 \pm r^{2\omega}\epsrel^*,
    \end{align}
    where we used Lemma~\ref{lem:perm} in the third step.
    Substituting this into \eqref{eq:zeroth}, summing over $x,y\in[r]$, and rearranging, we conclude that
    \begin{multline}
        \biggl(\sum^r_{x=1} (U^{i\cdots i}_{x\cdots x})^{2\omega - 2}\biggr) \biggl(\sum^r_{y=1} (U^{j\cdots j}_{y\cdots y})^{2\omega-2}\biggl) \left(\omega!^2 \iprod{U^{i\cdots i}, U^{j\cdots j}}^2 - (E^{i\cdots i}_{j\cdots j})^2\right) \\
        \le \omega^{O(\omega)} \epsort\epsrel^* + \omega!^2\cdot r^{2\omega}\epsrel^* \le O(r\omega)^{O(\omega)}\cdot \epsrel^*.
    \end{multline}
    By Lemma~\ref{lem:sumpowersT} and Part~\ref{fact:divide_both_sides} of Fact~\ref{fact:division}, we get $\iprod{U^{i\cdots i}, U^{j\cdots j}} = \pm \left(\epsort/\omega! + O(r\omega)^{O(\omega^2)}\cdot {\epsrel^*}^{1/2}\right)$.
\end{proof}

In the proof of the lemma above, we used the following helper lemma whose proof we defer to Appendix~\ref{app:sumpowersT}:

\begin{lemma}\label{lem:sumpowersT}
    Let $e$ be any even positive integer. Take any order-$d$, $r$-dimensional tensor $T$ of indeterminates satisfying \eqref{eq:rearrange_all} and $\norm{T}^2_F \ge c$ for some $c > 0$. If $\epsrel^* \le \frac{c^{e\omega/2}}{2r^{e\omega^2/2}}$, then there is a degree-$(e\omega)$ SoS proof that
    \begin{equation}
        \sum_{x\in[r]} (T_{x\cdots x})^e \ge \frac{c^{e/2}}{2r^{e\omega/2 - 1}}
    \end{equation}
\end{lemma}

\subsection{Rank-1 Structure of \texorpdfstring{$U$}{U}}
\label{sec:Uouter_hard}

For convenience, define 
\begin{equation}
    \omega'\triangleq \floor{\omega/2} \label{eq:omegaprime}
\end{equation} 
For any $i\in[r]$, define the $r$-dimensional vector-valued indeterminate
\begin{equation}
    \wt{U}^i \triangleq \sum_{j_1,\ldots, j_{\omega'}} U^{i \cdots  i}_{j_1j_1\cdots j_{\omega'}j_{\omega'}:} \label{eq:wtU_def}
\end{equation}
Similarly, for any $x\in[r]$, define $\wt{U}_x$ to be the $r$-dimensional vector-valued indeterminate whose $i$-th entry is equal to $(\wt{U}^i)_x$. Let $\wt{U}$ denote the $r\times r$ matrix-valued indeterminate whose $j$-th column is $\wt{U}^j$ for any $j\in[r]$. This matrix will play an important role in the latter stages of the proof.

The main result of this section is to show that $U$ admits the following decomposition in terms of $\wt{U}$:

\begin{lemma}[Main lemma]\label{lem:tensorouter}
    Define
    \begin{align}
        \epstrueouter &\triangleq r^{O(\omega)}(\epsouter + \omega^{O(\omega)}\sqrt{d}\radius\epstrueort/\kappa) \label{eq:epstrueouter_def} \\
        &= O(r\omega)^{O(\omega^3)}(d\radius/\kappa)^{O(\omega^2)}\cdot \left((d\eta/\kappa)^{1/\omega} + r^{O(\omega\ell)}(\omega\ell)^{O(\ell^2)} (d\radius/\kappa)^{O(\ell)}\theta\eta\right).
    \end{align}
    Then for any $\mb{i} = (i_1,\ldots,i_\omega) \in\strings$, there is a degree-$\poly(\omega,\ell)$ SoS proof using the constraints of Program~\ref{program:sos2} that
    \begin{equation}
        U^{\mb{i}} \approx_{O({\epstrueouter}^2)} \frac{1}{\num{i}}\sum_{\mb{j}\in\strings: \sort{j} = \sort{i}} \wt{U}^{j_1}\otimes \cdots \otimes \wt{U}^{j_\omega}. \label{eq:columnU}
    \end{equation}
\end{lemma}

\noindent Before proving this, we pause to interpret the implications of Lemma~\ref{lem:tensorouter}. Note that because the tensors $\brc{T^*_a}$ are symmetric, the action of the $r^{\omega}\times r^{\omega}$ transformation whose $\mb{i}$-th column is given by the right-hand side of \eqref{eq:columnU} on each $T^*_a$ is \emph{identical} to the action of the transformation $(\wt{U})^{\otimes \omega}$. In Lemma~\ref{lem:tensorouter}, we establish that this action is well-approximated by the action of $U$. So moving forward, instead of proving that $U$ behaves like $\Id_r^{\otimes\omega}$, it suffices to prove that $\wt{U}$ behaves like $\Id_r$! In fact, in the course of proving Lemma~\ref{lem:tensorouter}, we will already show that $\wt{U}$ is an approximately orthogonal matrix.

We now proceed to establish Lemma~\ref{lem:tensorouter}. First, note that taking $S = e_i^{\otimes \omega}$ for any $i\in[r]$ in Corollary~\ref{cor:rank1relation} and Lemma~\ref{lem:perm} already suggests that $U^{i\cdots i}$ behaves like a rank-1 tensor. We can thus easily deduce Lemma~\ref{lem:tensorouter} for the columns $U^{i\cdots i}$ (see Appendix~\ref{app:defer_nice_outerprod_odd} for a formal proof):

\begin{lemma}\label{lem:nice_outerprod_odd}
    For any $i\in[r]$, there is a degree-$O(\omega)$ SoS proof, using \eqref{eq:rearrange_all} from Lemma~\ref{lem:perm} and \eqref{eq:ort} from Lemma~\ref{lem:diagonal_orth}, that
    \begin{equation}
        \bigl\|U^{i\cdots i} - (\wt{U}^i)^{\otimes\omega}\bigr\|_{\max} \le O\bigl(\epsort/2^{\omega} + O(r\omega)^{O(\omega^2)}\cdot {\epsrel^*}^{1/2}\bigr). \label{eq:diagonal_outerprod}
    \end{equation}
\end{lemma}

\noindent As a simple but important consequence of Lemma~\ref{lem:diagonal_orth} and \ref{lem:nice_outerprod_odd}, we conclude that $\brc{\wt{U}^i}$ are nearly orthonormal (see Appendix~\ref{app:defer_dd_orth_odd} for a formal proof):

\begin{corollary}\label{cor:dd_orth_odd}
    Define
    \begin{align}
        \epstrueort &\triangleq \epsort^{1/2\omega}\cdot r^{7/4} + O(r\omega)^{O(\omega)}\cdot {\epsrel^*}^{1/2\omega} \cdot r^{3/2} \\
        &= \poly(\omega r) \cdot (\eta d / \kappa)^{O(1/\omega)} + r^{O(\omega^2+\omega\ell)} \omega^{O(\omega^2+\ell^2)}\ell^{O(\ell^2)}\cdot (d\radius/\kappa)^{O(\ell)}\theta\eta \label{eq:epstrueort_def}
    \end{align}
    There is a degree-$O(\omega)$ SoS proof, using \eqref{eq:diagonal_outerprod} from Lemma~\ref{lem:nice_outerprod_odd} and \eqref{eq:ort} from Lemma~\ref{lem:diagonal_orth}, that 
    \begin{equation}
        \iprod{\wt{U}^i,\wt{U}^j} = \bone{i = j} \pm O(\epstrueort) \ \ \forall \ i,j\in[r]. \label{eq:UiUj_ort}
    \end{equation} 
    \begin{equation}
        \iprod{\wt{U}_x, \wt{U}_y} = \bone{x = y} \pm O(\epstrueort) \ \ \forall \ x,y\in[r]. \label{eq:UxUy_ort}
    \end{equation}
\end{corollary}

\noindent Our goal is to extend Lemma~\ref{lem:nice_outerprod_odd} to the remaining columns of $U$. For this, we introduce some notation. Define
\begin{equation}
    C^{\mb{i}} \triangleq \sum_{\mb{i}': \sort{i}' = \mb{i}} U^{\mb{i}'} = \num{i}\cdot U^{\mb{i}}
\end{equation}
for any sorted tuple $\mb{i}\in\strings$, recalling Lemma~\ref{lem:push_Usym}. Given a tensor $S$ and tuple $\mb{a} = (a_1,\ldots,a_\omega) \in[r]^\omega$, we will use the shorthand $S(\mb{a})$ or $S(a_1\cdots a_\omega)$ to denote $S(\wt{U}^{a_1},\ldots,\wt{U}^{a_\omega})$ for any $a_1,\ldots,a_\omega\in[r]$. Note that if $S = T$ or $S = C^{\mb{i}}$, then because $T$ and $C^{\mb{i}}$ are symmetric tensors, $S(\mb{a}) = S(\sort{a})$.

We will use the following basic bound extensively in the sequel:
\begin{lemma}\label{lem:crude_C}
    Define
    \begin{equation}
        % \crude \triangleq (r\omega)^{O(\omega)} \cdot \sqrt{d}\normbound\radius.
        \crude \triangleq (r\omega)^{O(\omega)} \cdot \sqrt{d}\radius / \kappa.
    \end{equation}
    For any $\mb{i}, \mb{i}'\in\strings$, there is a degree-$O(\omega)$ SoS proof using \eqref{eq:UiUj_ort} from Corollary~\ref{cor:dd_orth_odd} that
    \begin{equation}
        -O(\crude) \le C^{\mb{i}}(\mb{i}') \le O(\crude)
    \end{equation}
\end{lemma}

\begin{proof}
    We can write $C^{\mb{i}}(\mb{i}')^2$ as
    \begin{equation}
        \biggl(\sum_{\mb{j}} C^{\mb{i}}_{\mb{j}} \wt{U}^{i'_1}_{j_1}\cdots \wt{U}^{i'_\omega}_{j_\omega}\biggr)^2 \le \norm{C^{\mb{i}}}^2_F \cdot \prod^\omega_{s=1} \norm{\wt{U}^{i'_s}}^2_2 \le \num{i}^2\cdot (r\omega)^{O(\omega)} \cdot 
        % d\normbound^2\radius 
        d/\kappa^2
        \cdot (1 + O(\omega \epstrueort)) = 
        % (r\omega)^{O(\omega)} \cdot d\normbound^2\radius^2,
        (r\omega)^{O(\omega)} \cdot d\radius^2/\kappa^2,
    \end{equation}
    where in the penultimate step we used Corollary~\ref{cor:dd_orth_odd} and bounded $\norm{C^{\mb{i}}}^2_F$ by Lemma~\ref{lem:crude_Ubound} and the definition of $C$.
\end{proof}

\noindent The following lemma is the main step in the proof of Lemma~\ref{lem:tensorouter}:

\begin{lemma}\label{lem:main_userank1_odd}
    Define
    \begin{align}
        \epsouter &\triangleq O(\omega\crude)^{O(\omega^2)}\cdot (\omega\epstrueort + \epsort r^{\omega/2}/2^{\omega} + O(r\omega)^{O(\omega^2)}\cdot \epsrel^*) \\
        &= O(r\omega)^{O(\omega^3)}(d\radius/\kappa)^{O(\omega^2)}\cdot \left((d\eta/\kappa)^{1/\omega} + r^{O(\omega\ell)}(\omega\ell)^{O(\ell^2)} (d\radius/\kappa)^{O(\ell)}\theta\eta\right).
    \end{align}
    There is a degree-$\poly(\omega,\ell)$ SoS proof, using the constraints of Program~\ref{program:sos2}, that for any sorted $\mb{i}\in\strings$ and any $\mb{i}'\in[r]^\omega$, 
    \begin{equation}
        C^{\mb{i}}(\mb{i}') = \bone{\sort{i}' = \mb{i}} \pm \epsouter.
    \end{equation}
\end{lemma}

\noindent The proof of this is considerably involved, and we defer the complete argument to Appendix~\ref{app:main_userank1_odd}. Here, we sketch it in a special case to convey some key aspects of the argument.

\begin{proof}[Sketch for $r = 2$]
    In this proof sketch we will pretend all approximate equalities are exact, which is the case when $\eta = 0$. By Lemma~\ref{lem:nice_outerprod_odd} and Corollary~\ref{cor:dd_orth_odd}, if $\mb{i} = (1,\ldots,1)$, then
    \begin{equation}
        C^{1\cdots 1}(\mb{i}') = U^{1\cdots 1}(\mb{i}') = \prod^{\omega}_{s=1} \iprod{\wt{U}^1,\wt{U}^{i'_s}} = \bone{\mb{i}' = (1,\ldots, 1)}
    \end{equation} as desired. We can prove $C^{2\cdots 2}(\mb{i}') = \bone{\mb{i}' = (2,\ldots,2)}$ in an identical fashion.
    
    It remains to handle $\mb{i}$ which contain both 1 and 2. For any $0 \le m \le \omega$, let $C^{[m]}$ denote $C^{1\cdots 12\cdots 2}$ where there are $m$ 2's in the superscript, and let $a^{[m]}$ denote the string $1\cdots 12\cdots 2$ containing $m$ 2's.
    
    We first show that $C^{[m]}(a^{[m']}) = \bone{m = m'}$ for all $0\le m < m'\le \omega$; we proceed inductively in $m$. By Lemma~\ref{lem:perm} applied to the tensor
    \begin{equation}
        T\triangleq F_U((e_1 + \epsilon\cdot e_2)^{\otimes\omega}) = C^{[0]} + \epsilon C^{[1]} + \cdots + \epsilon^{\omega}C^{[\omega]},
    \end{equation}
    where $\epsilon$ is a parameter we will vary,
    \begin{equation}
        T(a_1\cdots a_\omega)T(a'_1\cdots a'_\omega) = T(a_1\cdots a_{\omega-1}a'_\omega)T(a'_1\cdots a'_{\omega-1}a_\omega) \label{eq:sketch_2x2}
    \end{equation} for any $a_1,a'_1,\ldots,a_\omega,a'_\omega$. We may regard \eqref{eq:sketch_2x2} as a polynomial identity \emph{in the variable $\epsilon$}. So because \eqref{eq:sketch_2x2} holds for all $\epsilon$, the coefficient of $\epsilon^m$ in the monomial expansion of \eqref{eq:sketch_2x2} must vanish for every $0 \le m \le \omega$. For any $m' \ge m$, take $a_1\cdots a_\omega = a^{[0]}$ and $a'_1\cdots a'_\omega = a^{[m']}$. One can check that the vanishing of the coefficient of $\epsilon^m$ then yields
    \begin{multline}
        C^{[0]}(a^{[0]}) C^{[m]}(a^{[m']}) + \sum^m_{i=1} C^{[i]}(a^{[0]})C^{[m-i]}(a^{[m']}) \\ = C^{[0]}(a^{[1]})C^{[m]}(a^{[m'-1]}) + \sum^m_{i=1} C^{[i]}(a^{[1]})C^{[m-i]}(a^{[m'-1]}). \label{eq:sketch_coeff}
    \end{multline}
    If $m' > m$, then by the inductive hypothesis, the terms in the summations on either side vanish, as well as the first term on the right-hand side. Finally, recall that $C^{[0]}(a^{[0]}) = C^{1\cdots 1}(1\cdots 1) = 1$, so we conclude that $C^{[m]}(a^{[m']}) = 0$ for all $m' > m$. By symmetry, we also conclude that $C^{[m]}(a^{[m']}) = 0$ for all $m > m'$.
    
    All that is left is to verify that $C^{[m]}(a^{[m]}) = 1$ for all $m$. If we take $m = m'$ in \eqref{eq:sketch_coeff}, note that the summation on the left-hand side, the first term on the right-hand side, and all but the first summand in the summation on the right-hand side vanish by what we have just shown. Because $C^{[0]}(a^{[0]}) = 1$, we are left with the identity
    \begin{equation}
        C^{[m]}(a^{[m]}) = C^{[1]}(a^{[1]})C^{[m-1]}(a^{[m-1]})
    \end{equation} for all $m$. This implies (in degree-$\omega$ SoS) that $\left(C^{[1]}(a^{[1]})\right)^{\omega} = C^{[\omega]}(a^{[\omega]}) = 1$, so $C^{[1]}(a^{[1]}) = 1$ and therefore $C^{[m]}(a^{[m]}) = 1$ for all $m$.
\end{proof}

\noindent Roughly speaking, Lemma~\ref{lem:main_userank1_odd} shows that $C$, when placed in the basis specified by $\wt{U}$, looks like the $r^{\omega}\times r^{\omega}$ identity matrix.\footnote{More precisely, in this basis $C$ looks like several copies of the $\rchoose\times\rchoose$ identity matrix, as $C^{\mb{i}}(\mb{i}'_1) = C^{\mb{i}}(\mb{i}'_2)$ for any tuples $\mb{i}'_1, \mb{i}'_2$ which are equal after sorting.} As $\wt{U}$ is approximately orthogonal, we can use this to deduce Lemma~\ref{lem:tensorouter}.

\begin{proof}[Proof of Lemma~\ref{lem:tensorouter}]
    For convenience, define
    \begin{equation}
        \delta_i \triangleq e_i - \sum^r_{j=1} (\wt{U}^j)_i \wt{U}^j \ \ \ \text{and} \ \ \
        \epsilon_{\mb{i},\mb{j}} \triangleq C^{\mb{i}}(\wt{U}^{j_1},\ldots,\wt{U}^{j_\omega}) - \bone{\sort{i} = \sort{j}}.
    \end{equation}
    By the second part of Corollary~\ref{cor:dd_orth_odd}, $\norm{\wt{U}\wt{U}^{\top} - \Id}_{\max} \le O(\epstrueort)$. So for any $i\in[r]$,
    \begin{equation}
        \norm{(\wt{U}\wt{U}^{\top} - \Id)e_i}^2 \le O({\epstrueort}^2 r),
    \end{equation}
    and thus $\norm{\delta_i}^2 = \norm{\wt{U}\wt{U}^{\top}e_i - e_i}^2 \le O({\epstrueort}^2 r)$.
    
    For any $\mb{i},\mb{k}\in\strings$,
    \begin{align}
        C^{\mb{i}}_{\mb{k}} &= C^{\mb{i}}(e_{k_1},\ldots,e_{k_\omega}) = \sum_{\mb{D}\in\strings} C^{\mb{i}}_{\mb{\ell}}\cdot \prod^{\omega}_{s=1} (e_{k_s})_{\ell_s} = \sum_{\mb{\ell}\in\strings} C^{\mb{i}}_{\mb{\ell}}\cdot \prod^{\omega}_{s=1} \biggl(\sum^r_{j=1} (\wt{U}^j)_{k_s} \cdot (\wt{U}^j)_{\ell_s} +(\delta_{k_s})_{\ell_s}\biggr) \label{eq:binomial}
        % &= \sum_{\mb{j}\in\strings} C^{\mb{i}}(U^{j_1},\ldots,U^{j_\omega}) (e_1,\ldots,e_{\omega}) = 
    \end{align}
    Note that 
    \begin{align}
        \sum_{\mb{\ell}\in\strings} C^{\mb{i}}_{\mb{\ell}}\cdot \sum_{\mb{j}\in\strings}\prod^{\omega}_{s=1} (\wt{U}^{j_s})_{k_s} \cdot (\wt{U}^{j_s})_{\ell_s} &= \sum_{\mb{j}\in\strings} (\wt{U}^{j_1}\otimes\cdots\otimes \wt{U}^{j_\omega})_{\mb{k}}\cdot C^{\mb{i}}(\wt{U}^{j_1},\ldots,\wt{U}^{j_\omega}) \\
        &= \sum_{\mb{j}\in\strings} (\wt{U}^{j_1}\otimes\cdots\otimes \wt{U}^{j_\omega})_{\mb{k}}\cdot (\bone{\sort{i} = \sort{j}} + \epsilon_{\mb{i},\mb{j}}) \\
        &= \sum_{\mb{j}:\sort{j} = \sort{i}} (\wt{U}^{j_1}\otimes \cdots\otimes \wt{U}^{j_\omega})_{\mb{k}} + \sum_{\mb{j}\in\strings} \epsilon_{\mb{i},\mb{j}}\cdot (\wt{U}^{j_1}\otimes \cdots\otimes \wt{U}^{j_\omega})_{\mb{k}}. \label{eq:mainterm}
    \end{align}
    Using Lemma~\ref{lem:main_userank1_odd}, we can bound the error term via
    \begin{align}
        \biggl(\sum_{\mb{j}\in\strings} \epsilon_{\mb{i},\mb{j}}\cdot (\wt{U}^{j_1}\otimes \cdots\otimes \wt{U}^{j_\omega})_{\mb{k}}\biggr)^2 &\le \biggl(\sum_{\mb{j}} \epsilon_{\mb{i},\mb{j}}^2\biggr)\biggl(\sum_{\mb{j}} (\wt{U}^{j_1}\otimes \cdots\otimes \wt{U}^{j_\omega})^2_{\mb{k}}\biggr) \\
        &\le r^{\omega}\cdot\epsouter^2\cdot(1+O(\omega\epstrueort)) = O(r^{\omega}\epsouter^2). \label{eq:error1}
    \end{align}
    % provided $\epsouter \le r^{-\omega/2}$.
    
    Finally, we need to bound the contribution from the $\delta$'s in the expansion of \eqref{eq:binomial}. For any nonempty $S\subseteq[\omega]$,
    \begin{align}
        \biggl(\sum_{\mb{\ell}\in\strings} C^{\mb{i}}_{\mb{\ell}} \cdot \prod_{s\in S} (\delta_{k_s})_{\ell_s} \cdot \prod_{t\not\in S} \biggl(\sum^r_{j=1} (\wt{U}^j)_{k_t} \cdot (\wt{U}^j)_{\ell_t}\biggr)\biggr)^2 &\le r^{2\omega}\crude^2 \cdot (1 + O(\omega\epstrueort))\cdot\max_{s\in[\omega]}\norm{\delta_{k_s}}^2 \\
        &\le O(r^{2\omega+1}\crude^2 {\epstrueort}^2). \label{eq:error2}
    \end{align}
    Applying \eqref{eq:mainterm}, \eqref{eq:error1}, and \eqref{eq:error2} to \eqref{eq:binomial} yields the desired bound.
\end{proof}

\subsection{Breaking Gauge Symmetry for the Ground Truth}
\label{sec:breaksym}

Having exhibited an SoS proof using the constraints of Program~\ref{program:sos2} that $U$ behaves like it arises from an $r\times r$ rotation $\wt{U}$ in the sense of Lemma~\ref{lem:tensorouter}, we still need to pin down what this rotation is. We now show how to modify Program~\ref{program:sos2} to break symmetry and force $\wt{U}$ to be close to a diagonal orthogonal matrix.

Recall that in Section~\ref{sec:tensorring}, this was done by forming a suitable linear combination of $Q^*_1,\ldots,Q^*_d$ that had eigengaps and then imposing an additional constraint in the sum-of-squares program that the same linear combination of the SoS variables $Q_1,\ldots,Q_d$ was \emph{diagonal}. The linear combination was constructed based on estimates for the inner products $\brc{\iprod{Q^*_a,Q^*_b}}_{a,b\in[d]}$ which were supplied as input to the algorithm.

In our setting there are several differences. First, we are working with general tensors $\brc{T^*_a}$, for which there is no notion of eigengap. To circumvent this, we will consider the matrices
\begin{equation}
    F^*_a \triangleq f^*_a(f^*_a)^{\top},
\end{equation}
where recall from Part~\ref{item:push_condnumber3} of Assumption~\ref{assume:push} that $f^*_a \triangleq \sum^r_{j_1,\ldots,j_{\omega'}=1} (T^*_a)_{j_1j_1\cdots j_{\omega'}j_{\omega'}:}$. We will try to form non-degenerate linear combinations of $F^*_1,\ldots,F^*_d$ in the sense of Definition~\ref{def:nondeg}.

To do this, we would like to invoke Lemma~\ref{lem:forcegap}, for which we need estimates of $\brc{\iprod{F^*_a,F^*_b}}_{a,b\in[d]}$. But unlike in the setting of Section~\ref{sec:tensorring}, we are not given these estimates to start with. Instead, we will show that we can form such estimates using a pseudoexpectation satisfying the constraints of Program~\ref{program:sos2}.

For every $a\in[d]$, let $F_a$ denote the $r\times r$ matrix-valued indeterminate
\begin{equation}
    F_a \triangleq f_a(f_a)^{\top} \ \ \text{for} \ \ f_a \triangleq \sum^r_{j_1,\ldots,j_{\omega'}=1} (T_a)_{j_1j_1\cdots j_{\omega'}j_{\omega'}:}. \label{eq:Fdef}
\end{equation}

As it is cumbersome to write the subscript ``$j_1j_1,\ldots,j_{\omega'}j_{\omega'}{:}$,'' we will abuse notation slightly and write it as $\mb{j}\mb{j}{:}$, where $\mb{j} = (j_1,\ldots,j_{\omega'})$ (note that this abuse of notation is not too bad as $T_a$ is symmetric by Constraint~\ref{constraint:sym}).

In Appendix~\ref{app:defer_UsendsF}, we show that $\wt{U}$ approximately maps every $F^*_a$ to $F_a$:

\begin{lemma}\label{lem:UsendsF}
    There is a degree-$\poly(\omega,\ell)$ SoS proof using the constraints of Program~\ref{program:sos2} that $\wt{U}F^*_a\wt{U}^{\top} \approx_{O(r)^{4\omega}\cdot \radius^4({\epstrueort}^2 + {\epstrueouter}^2 + \epsmap^2)} F_a$.
\end{lemma}

\begin{proof}[Proof sketch]
    By Lemma~\ref{lem:push_almost_map}, we have that 
    \begin{equation}
        F_a \approx \biggl(\sum_{\mb{j}\in[r]^{\omega'}} F_U(T^*_a)_{\mb{j}\mb{j}:}\biggr)\biggl(\sum_{\mb{j}\in[r]^{\omega'}} F_U(T^*_a)_{\mb{j}\mb{j}:}\biggr)^{\top}
    \end{equation} 
    By Lemma~\ref{lem:tensorouter} and symmetry of $T^*_a$, we have that $F_U(T^*_a)_{\mb{j}\mb{j}:} \approx F_{\wt{U}^{\otimes\omega}}(T^*_a)_{\mb{j}\mb{j}:}$. Then for any $x\in[r]$,
    \begin{align}
        \sum_{\mb{j}} F_{\wt{U}^{\otimes\omega}}(T^*_a)_{\mb{j}\mb{j}x} &= \sum_{\mb{j}} \sum_{\mb{i}\in\strings} \wt{U}_{j_1}^{i_1}\wt{U}_{j_1}^{i_2} \cdots \wt{U}_{j_{\omega'}}^{i_{\omega-2}}\wt{U}_{j_{\omega'}}^{i_{\omega-1}} \wt{U}^{i_{\omega}}_x (T^*_a)_{\mb{i}} \\
        &= \sum_{\mb{i}\in\strings} \iprod{\wt{U}^{i_1},\wt{U}^{i_2}}\cdots \iprod{\wt{U}^{i_{\omega-2}},\wt{U}^{i_{\omega-1}}} \wt{U}^{i_{\omega}}_x (T^*_a)_{\mb{i}} \\
        &\approx \sum_{\mb{i}\in\strings} \bone{i_1=i_2,i_3=i_4,\ldots,i_{\omega-2}=i_{\omega-1}}\wt{U}^{i_\omega}_x (T^*_a)_{\mb{i}} \\
        &= \wt{U}_x f^*_a,
    \end{align}
    where in the third step we used Corollary~\ref{cor:dd_orth_odd}. The outer product of $\sum_{\mb{j}} F_{\wt{U}^{\otimes\omega}}(T^*_a)_{\mb{j}\mb{j}:}$ with itself is thus approximated by $\wt{U}F^*_a\wt{U}^{\top}$ as claimed.
\end{proof}

\noindent As $\wt{U}$ is approximately orthogonal, Lemma~\ref{lem:UsendsF} implies that we can read off estimates for $\iprod{F^*_a,F^*_b}$ from the pseudoexpectation:

\begin{lemma}\label{lem:push_gram}
    For any $a,b\in[d]$, $\abs*{\psE*{\iprod{F_a,F_b}} - \iprod{F^*_a,F^*_b}} \le r^{O(\omega)}\radius^4(\epstrueouter + \epstrueort + \epsmap)$.
\end{lemma}

\noindent We defer a formal proof of this to Appendix~\ref{app:defer_push_gram}.

With estimates for $\iprod{F^*_a,F^*_b}$ in hand, we can invoke Lemma~\ref{lem:forcegap} to obtain $\lambda,\mu$ which are $\gap$-non-degenerate combinations of $F^*_1,\ldots,F^*_d$. Recall that this means that for
\begin{equation}
    F^*_{\lambda}\triangleq \sum_{a\in[d]} \lambda_a F^*_a \qquad \text{and} \qquad F^*_{\mu}\triangleq \sum_{a\in[d]}\mu_a F^*_a, \label{eq:Fstarlam}
\end{equation}
$F^*_{\lambda}$ has minimum eigengap at least $\gap$, and if $V^{\top}\Lambda V$ is the eigendecomposition of $Q^*_{\mu}$, then every entry of $VF^*_{\mu} V^{\top}$ has magnitude at least $\gap$.

We can use $F^*_{\lambda}, F^*_{\mu}$ to break gauge symmetry for the ground truth in the same way that we used $Q^*_{\lambda}, Q^*_{\mu}$ in Section~\ref{sec:breakground} to break gauge symmetry in the tensor ring decomposition setting.

Indeed, because Assumption~\ref{assume:push} is gauge-invariant by Lemma~\ref{lem:push_gauge_invariant}, we can assume without loss of generality that $F^*_{\lambda}$ is diagonal with entries sorted in nondecreasing order. As $F^*_{\lambda}$ has minimum eigengap at least $\gap$, this yields
\begin{equation}
    (F^*_{\lambda})_{jj} \ge (F^*_{\lambda})_{ii} + \gap \ \ \forall \ j > i. \label{eq:Fdiag_gap}
\end{equation}
After diagonalizing $F^*_{\lambda}$, the condition on $F^*_{\mu}$ implies that
\begin{equation}
    |(F^*_{\mu})_{ij}| \ge \gap \ \ \forall \ i,j\in[r]. \label{eq:Fmubig}
\end{equation}
By applying one more joint rotation to $F^*_1,\ldots,F^*_d$ given by a diagonal matrix of $\pm 1$ entries, we can additionally assume that the first row of $F^*_{\mu}$ consists of all nonnegative entries satisfying \eqref{eq:Fmubig}:
\begin{equation}\label{eq:Fstarfirstrow}
    (F^*_{\mu})_{1j} \ge \gap \ \ \forall \ j\in[r].
\end{equation}

\subsection{Second Sum-of-Squares Relaxation}
\label{sec:newprogram}

With these in hand, we can finally consider the second of our two sum-of-squares programs. This program is given by adding several additional constraints involving $\lambda,\mu$:

\begin{program}\label{program:lastprogram}
    \begin{center}
        \textsc{(Low-Rank Factorization\--- Second Part)} \\
    \end{center}
    \noindent\textbf{Parameters:} $\lambda,\mu\in\S^{d-1}$, in addition to parameters of Program~\ref{program:sos2} ($S\in\R^{d\times d}$, $\radius\ge 1$, $\kappa > 0$).
    
    \noindent\textbf{Variables:} Same as those of Program~\ref{program:sos2} ($T_1,\ldots,T_d,L,P,\brc{v_{a,t}}$). Also define 
    \begin{equation}
        F_{\lambda} \triangleq \sum^d_{a=1}\lambda_a F_a \qquad \text{and} \qquad F_{\mu} \triangleq \sum^d_{a=1}\mu_a F_a. \label{eq:Flammu}
    \end{equation}
    
    \noindent\textbf{Constraints:} In addition to the constraints of Program~\ref{program:sos2},
    \begin{enumerate}[leftmargin=*,topsep=0em]
        \setcounter{enumi}{8}
        \item ($F_{\lambda}$ diagonal): $(F_{\lambda})_{ij} = 0$ for all $i\neq j$. \label{constraint:push_diag}
        \item ($F_{\lambda}$ sorted): $(F_{\lambda})_{jj} \ge (F_{\lambda})_{ii}$ for all $j > i$. \label{constraint:push_sorted}
        \item ($F_{\mu}$'s first row): $(F_{\mu})_{1j} \ge 0$ for all $j\in[r]$. \label{constraint:push_firstrow}
    \end{enumerate}
\end{program}

\noindent We can again verify that the ground truth is feasible.
\begin{lemma}\label{lem:push_feasible2}
    When $d\ge \binom{r+\omega-1}{\omega}$, the pseudodistribution given by the point distribution supported on $(\brc{T^*_a},L^*,P^*,\brc{v^*_{a,t}})$, where $L^*$ is the left inverse of $M^*$, and $P^*$ is the left inverse of $M^*D\Sigma^{1/2}_{\sym}$, is a feasible solution to Program~\ref{program:sos2}.
\end{lemma}

\begin{proof}
    By Lemma~\ref{lem:feasible}, we already know that the point distribution satisfies the constraints from Program~\ref{program:sos2}. And the new constraints are satisfied by the discussion at the end of Section~\ref{sec:breaksym}.
\end{proof}

\noindent As the constraints in this program are a strict superset of those from Program~\ref{program:sos2}, all of the SoS proofs from previous sections still apply, and we will continue to use the auxiliary variables like $\wt{U}$ and the results that we have proven about them in the sequel when analyzing Program~\ref{program:lastprogram}.

We are now ready to describe the main technical claim which will be the focus of the rest of the proof:

\begin{theorem}\label{thm:push_main_pairwise}
    Suppose Assumption~\ref{assume:push} holds and $\eta \le \Theta(r\omega)^{-\Omega(\omega^3)}\cdot O(\ell^{\ell}\omega^{\ell}r^{\omega}\radius d/\kappa)^{O(\ell)}$. For any $\lambda,\mu\in\S^{d-1}$ let $\psE{\cdot}$ be any degree-$\poly(\omega,\ell)$ pseudo-expectation satisfying the constraints of Program~\ref{program:lastprogram} with parameters $\lambda,\mu$. 
    
    If $\lambda,\mu$ are $\gap$-non-degenerate combinations of $F^*_1,\ldots,F^*_d$ for some $\gap > 0$, then for
    \begin{equation}
        \epsilon^* \triangleq \poly(r,\omega,d,\radius,1/\kappa)^{\omega^3} \cdot \left((d\eta/\kappa)^{O(1/\omega)} + \poly(r^{\omega}, \omega^{\ell},\ell^{\ell},d,\radius,1/\kappa)^{\ell}\cdot \theta^{1/4}\eta^{1/4}\right)/ \gap^{1/2}, \label{eq:final_pair_eps}
        % O(r\omega)^{O(\omega^3)}(d\radius/\kappa)^{O(\omega^2)}\left((d\eta/\kappa)^{O(1/\omega)} + r^{O(\omega\ell)}(\omega\ell)^{O(\ell^2)}(d\radius/\kappa)^{O(\ell)}\theta^{1/4}\eta^{1/4}\right)/\gap^{1/2} <--- more detailed version of the above
    \end{equation}
    we have that $|\psE{(T_a)_{\mb{i}} (T_b)_{\mb{j}}} - (T^*_a)_{\mb{i}} (T^*_b)_{\mb{j}}| \le \epsilon^*$ for all $a,b\in[d], \mb{i},\mb{j}\in\strings$.
\end{theorem}

\noindent The upshot of Theorem~\ref{thm:push_main_pairwise} is that we can accurately estimate the magnitude of every entry of every $T^*_a$ using the pseudoexpectation, as well as the sign relationship between any pair of entries of sufficiently large magnitude:

\begin{corollary}\label{cor:push_pairwise}
    For $\psE{\cdot}$ from Theorem~\ref{thm:push_main_pairwise}, let $\wh{T}_a \in (\R^d)^{\otimes\omega}$ denote the tensor whose $\mb{i}$-th entry is $\psE{((T^*_a)_{\mb{i}})^2}^{1/2}$ for all $\mb{i}\in\strings$.
    \begin{enumerate}
        \item $\abs*{(\wh{T}_a)_{\mb{i}} - |(T^*_a)_{\mb{i}}|} \le \sqrt{\epsilon^*}$ for all $a\in[d], \mb{i}\in\strings$.
        \item $\sgn(\psE{(T_a)_{\mb{i}} (T_b)_{\mb{j}}}) = \sgn((T^*_a)_{\mb{i}}(T^*_b)_{\mb{j}})$ for all $a,b\in[d], \mb{i},\mb{j}\in\strings$ satisfying $|(\wh{T}_a)_{\mb{i}} (\wh{T}_b)_{\mb{j}}| > \epsilon^*$.
    \end{enumerate}
\end{corollary}

\begin{proof}
    The first part follows by taking $a = b$, $\mb{i} = \mb{j}$ in \eqref{eq:final_pair_eps} from Theorem~\ref{thm:push_main_pairwise} and using the elementary inequality $(|x|-|y|)^2 \le |x^2 - y^2|$. The second part follows by triangle inequality.
\end{proof}

As we show in Section~\ref{sec:push_put_together}, with a few additional steps this will be enough to solve the low-rank factorization problem. We now focus on proving Theorem~\ref{thm:push_main_pairwise}.

\subsection{Breaking Gauge Symmetry for SoS Variables}
\label{sec:breaksos_push}

To prove Theorem~\ref{thm:push_main_pairwise}, we leverage Constraints~\ref{constraint:push_diag} and \ref{constraint:push_sorted}, together with \eqref{eq:Fdiag_gap}, in order to prove that $\wt{U}$ behaves like an $r\times r$ rotation whose off-diagonal entries are close to zero (Lemma~\ref{lem:Udiag}), and then we leverage Constraint~\ref{constraint:push_firstrow}, together with \eqref{eq:Fstarfirstrow}, in order to prove that $\wt{U}$ is close to a multiple of the identity matrix (Lemma~\ref{lem:push_usemu}).

The argument is reminiscent of the proof in Section~\ref{sec:diagonal}, except in place of $Q_{\lambda}, Q_{\mu}, Q^*_{\lambda}, Q^*_{\mu}$, we use $F_{\lambda}, F_{\mu}, F^*_{\lambda}, F^*_{\mu}$ from  \eqref{eq:Fstarlam} and \eqref{eq:Flammu} to break gauge symmetry, and in place of the $r^2\times r^2$ matrix $W$ which behaves like the Kronecker product of an $r\times r$ rotation with itself, we directly consider the $r\times r$ matrix $\wt{U}$ which behaves like an $r\times r$ rotation.

We begin by noting the following simple consequence of Lemma~\ref{lem:UsendsF}, in analogy to Corollary~\ref{cor:WsendsQlam}:

\begin{corollary}\label{cor:UsendsFlam}
    For $c\in\brc{\lambda,\mu}$, there is a degree-$O(1)$ SoS proof using the result of Lemma~\ref{lem:UsendsF} that 
    \begin{equation}
        \wt{U}F^*_c\wt{U}^{\top} \approx_{O(r)^{4\omega}\cdot \radius^4\cdot d({\epstrueort}^2 + {\epstrueouter}^2 + \epsmap^2)} F_c.
    \end{equation}
\end{corollary}

\begin{proof}
    $\wt{U}F^*_c\wt{U}^{\top} - F_c = \sum^d_{a=1} c_a(\wt{U}F^*_c\wt{U}^{\top} - F_c)$, so the claim follows by Part~\ref{shorthand:lincombo} of Fact~\ref{fact:shorthand}.
\end{proof}

\noindent Next we use Corollary~\ref{cor:UsendsFlam} to establish the analogue of Lemma~\ref{lem:QWWQ}.

\begin{corollary}\label{cor:UFFU}
    For $c\in[d]\cup\brc{\lambda,\mu}$, there is a degree-$O(1)$ SoS proof using Corollary~\ref{cor:dd_orth_odd} and Lemma~\ref{lem:UsendsF} that 
    \begin{equation}
        \wt{U}F^*_c \approx_{r^{O(\omega)}\cdot \radius^4\cdot d({\epstrueort}^2 + \epsouter^2 + \epsmap^2)} F_c\wt{U}.
    \end{equation}
\end{corollary}

\noindent We defer the proof to Appendix~\ref{app:defer_UFFU}, noting that it follows by right-multiplying both sides of the approximate equality in Corollary~\ref{cor:UsendsFlam} by $\wt{U}$ and recalling from Corollary~\ref{cor:dd_orth_odd} that $\wt{U}^{\top}\wt{U} \approx \Id$.

\subsubsection{Using Diagonality of \texorpdfstring{$F_{\lambda}$}{Flambda}}

Our first main result of this subsection is to show that $\wt{U}$ is close to a diagonal $r\times r$ rotation, using Constraints~\ref{constraint:push_diag} and \ref{constraint:push_sorted} along with \eqref{eq:Fdiag_gap}:

\begin{lemma}\label{lem:Udiag}
    Define
    \begin{align}
        \epsoffdiag &\triangleq O(r)^{4\omega}\cdot \radius^4\cdot d({\epstrueort}^2 + \epsouter^2 + \epsmap^2) / \upsilon + \epstrueort\cdot r^3 \label{eq:epsoffdiag_push_def} \\
        &= O(r\omega)^{O(\omega^3)}(d\radius/\kappa)^{O(\omega^2)}\left((d\eta/\kappa)^{O(1/\omega)} + r^{O(\omega\ell)}(\omega\ell)^{O(\ell^2)}(d\radius/\kappa)^{O(\ell)}\theta\eta\right)/\gap
    \end{align}
    There is a degree-$\poly(\omega,\ell)$ SoS proof using the constraints of Program~\ref{program:lastprogram} that
    \begin{equation}
        \left(U^j_k\right)^2 = \bone{j = k} \pm O(\epsoffdiag).
    \end{equation}
\end{lemma}

\noindent This is the analogue of Lemma~\ref{lem:deg6} from our analysis for tensor ring decomposition. To prove this, we first establish an analogue of Lemma~\ref{lem:downleft} that follows from Constraint~\ref{constraint:push_sorted} and \eqref{eq:Fdiag_gap}.

\begin{lemma}\label{lem:push_downleft}
    For any $i,j,k,\ell\in[r]$ for which $k\ge i$, $\ell < j$, there is a degree-2 SoS proof using Constraint~\ref{constraint:push_sorted} that
    \begin{equation}
        ((F_{\lambda})_{ii} - (F^*_{\lambda})_{jj})^2 + ((F_\lambda)_{kk} - (F^*_\lambda)_{\ell\ell})^2 \ge \upsilon^2/2. \label{eq:Fpos}
    \end{equation}
\end{lemma}

\begin{proof}
    The proof is identical to that of Lemma~\ref{lem:downleft} with $Q_\lambda$ and $Q^*_\lambda$ replaced by $F_\lambda$ and $F^*_\lambda$.
\end{proof}

\noindent As a consequence, we obtain the following analogue of Lemma~\ref{lem:downleft_Wentry}.

\begin{lemma}\label{lem:push_downleft_Wentry}
    Define
    \begin{equation}
        \epspair \triangleq O(r)^{4\omega}\cdot \radius^4\cdot d({\epstrueort}^2 + \epsouter^2 + \epsmap^2) / \upsilon.
    \end{equation}
    For any $i,j,k,\ell\in[r]$ for which $k\ge i$, $\ell < j$, there is a degree-$O(1)$ SoS proof using Corollary~\ref{cor:UFFU}, \eqref{eq:Fpos}, and Constraint~\ref{constraint:push_diag} that \begin{equation}
        -\epspair \le \wt{U}^j_i \wt{U}^{\ell}_k \le \epspair. \label{eq:push_downleft_pair_zero}
    \end{equation}
\end{lemma}

\noindent The proof of this is very similar to that of Lemma~\ref{lem:downleft_Wentry}, so we defer it to Appendix~\ref{app:defer_push_downleft_Wentry}.

\begin{proof}[Proof of Lemma~\ref{lem:Udiag}]
    We will use Lemma~\ref{lem:push_downleft_Wentry} in an inductive fashion to prove Lemma~\ref{lem:Udiag}. As the argument closely mirrors that of Lemma~\ref{lem:deg6}, we defer the formal proof to Appendix~\ref{app:defer_Udiag}. We remark that because we directly work with an approximately orthogonal $r\times r$ matrix $\wt{U}$ rather than an $r^2\times r^2$ matrix $W$ which behaves like the Kronecker power of an orthogonal $r\times r$ matrix, the proof is actually somewhat simpler in the present setting.
\end{proof}

\noindent Now that we know by Lemma~\ref{lem:Udiag} that the off-diagonal entries of $\wt{U}$ are small, we can show the following analogue of Lemma~\ref{lem:QQstar} stating that for all $c\in[d]\cup\brc{\lambda,\mu}$, $Q_c$ and $Q^*_c$ are equal up to rotation by a diagonal matrix with $\pm 1$ entries.

\begin{lemma}\label{lem:FFstar}
    Using Lemma~\ref{lem:UsendsF}, Corollary~\ref{cor:UsendsFlam}, and Lemma~\ref{lem:Udiag}, there is a degree-$O(1)$ SoS proof that
    \begin{equation}
        (F_c)_{ij} = \wt{U}_{ii} \wt{U}_{jj} (F^*_c)_{ij} \pm O(r)^{2\omega}\cdot (\radius^2\cdot \sqrt{d}({\epstrueort} + {\epstrueouter} + \epsmap) + \radius^4\epsoffdiag). \label{eq:FFstar}
    \end{equation}
\end{lemma}

\noindent The proof closely parallels the corresponding one for Lemma~\ref{lem:QQstar}. We defer the formal proof to Appendix~\ref{app:defer_FFstar}.

\subsubsection{Using Positivity of \texorpdfstring{$(F_{\mu})_{1j}$}{First Row of Fmu}}

We now use Constraint~\ref{constraint:push_firstrow} along with \eqref{eq:Fstarfirstrow} to refine Lemma~\ref{lem:Udiag} and show that $\wt{U}$ is close to either $\Id_r$ or $-\Id_r$:

\begin{lemma}\label{lem:push_usemu}
    Provided $O(r)^{2\omega}\cdot (\radius^2({\epstrueort} + {\epstrueouter} + \epsmap) + \radius^4\epsoffdiag) / \gap$ is upper bounded by a sufficiently small constant, there is a degree-$O(1)$ SoS proof using Lemma~\ref{lem:Udiag}, Lemma~\ref{lem:FFstar}, and Constraint~\ref{constraint:firstrow} of Program~\ref{program:lastprogram} that
    \begin{equation}
        \wt{U}_{ii} \wt{U}_{jj} = 1 \pm O(r\omega)^{O(\omega^3)}(d\radius/\kappa)^{O(\omega^2)}\left((d\eta/\kappa)^{O(1/\omega)} + r^{O(\omega\ell)}(\omega\ell)^{O(\ell^2)}(d\radius/\kappa)^{O(\ell)}\theta^{1/4}\eta^{1/4}\right)/\gap^{1/2}
    \end{equation}
    for all $i,j\in[r]$.
\end{lemma}

\noindent The proof is analogous to that of Lemma~\ref{lem:usemu} in our analysis for tensor ring decomposition, so we defer it to Appendix~\ref{app:defer_push_usemu}.

There is one last symmetry that we must resolve: is $\wt{U}$ close to $\Id_r$ or $-\Id_r$? Note that this cannot be resolved merely by using $F^*_{\lambda}, F^*_{\mu}, F_{\lambda}, F_{\mu}$, as Constraints~\ref{constraint:push_diag}-\ref{constraint:push_firstrow} and \eqref{eq:Fdiag_gap}-\eqref{eq:Fstarfirstrow} are consistent with both $\wt{U}\approx\Id$ and $\wt{U}\approx-\Id$. The reason for this is simply that for any matrix $A$, $\Id\cdot A\cdot \Id = (-\Id)\cdot A \cdot (-\Id)$.

Note that this was not a problem in our tensor ring decomposition analysis because there our goal was to show the $r^2\times r^2$ auxiliary variable $W$ was close to the identity, and $W$ essentially played the role of the \emph{Kronecker square} of the $r\times r$ rotation sending $Q^*_a$ to $Q_a$. So regardless of whether that rotation was approximately $\Id$ or $-\Id$, $W$ would still approximately be the identity.

On the other hand, because we are working with \emph{odd-order} tensors, there is a distinction between $\Id^{\otimes\omega}$ and $(-\Id)^{\otimes\omega}$. In the next section, we break this last symmetry, but in the \emph{rounding step} rather than the analysis of the SoS program.

\subsection{Proof of Theorem~\ref{thm:push_main_pairwise} and Rounding}
\label{sec:push_put_together}

We are now ready to prove Theorem~\ref{thm:push_main_pairwise} and establish our main algorithmic guarantee for learning low-rank polynomial transformations.

\begin{proof}[Proof of Theorem~\ref{thm:push_main_pairwise}]
    Recall the hidden rotation variable $U$ defined in Section~\ref{sec:hiddenrotation} and define $\Delta\triangleq T_a - F_U(T^*_a)$ and $\calE^{\mb{i}} \triangleq U^{\mb{i}} - \frac{1}{\num{i}}\sum_{\mb{j}\in\strings: \sort{j} = \sort{i}} \wt{U}^{j_1}\otimes \cdots \otimes \wt{U}^{j_\omega}$ for any $\mb{i}\in\strings$. By Lemma~\ref{lem:push_almost_map}, $\norm{\Delta}^2_F \le \epsmap^2$ for every $a\in[d]$, and by Lemma~\ref{lem:tensorouter}, $\norm{\calE^{\mb{i}}}_{\max} \le O(\epstrueouter)$. So for every $\mb{i}\in\strings, a\in[d]$,
    \begin{align}
        (T_a)_{\mb{i}} &= \Delta_{\mb{i}} + \sum_{\mb{j}\in\strings} U^{\mb{j}}_{\mb{i}} (T^*_a)_{\mb{j}} = \Delta_{\mb{i}} + \sum_{\mb{j}\in\strings}\biggl(\calE^{\mb{j}}_{\mb{i}} + \frac{1}{\num{j}}\sum_{\mb{j}'\in\strings: \sort{j}' = \sort{j}} \wt{U}^{j'_1}_{i_1}\cdots\wt{U}^{j'_\omega}_{i_\omega}\biggr) (T^*_a)_{\mb{j}} \\
        &= \Delta_{\mb{i}} + \sum_{\mb{j}\in\strings} \left(\calE^{\mb{j}}_{\mb{i}} + \wt{U}^{j_1}_{i_1}\cdots\wt{U}^{j_\omega}_{i_\omega}\right) (T^*_a)_{\mb{j}} = F_{\wt{U}^{\otimes\omega}}(T^*_a)_{\mb{i}} + \Delta_{\mb{i}} + \sum_{\mb{j}\in\strings} \calE^{\mb{j}}_{\mb{i}},
    \end{align}
    % so $T_a - F_{\wt{U}^{\otimes\omega}}(T^*_a) = \Delta + \sum_{\mb{j}\in\strings} \calE^{\mb{j}}$, so
    so for any $\mb{i},\mb{i}'\in\strings$ and $a,b\in[d]$,
    \begin{multline}
        (T_a)_{\mb{i}} (T_b)_{\mb{i}'} - F_{\wt{U}^{\otimes\omega}}(T^*_a)_{\mb{i}} F_{\wt{U}^{\otimes\omega}}(T^*_b)_{\mb{i}'} \\
        = \bigl(\Delta_{\mb{i}} +\sum_{\mb{j}} \calE^{\mb{j}}_{\mb{i}}\bigr)\bigl(\Delta_{\mb{i}'} +\sum_{\mb{j}} \calE^{\mb{j}}_{\mb{i}'}\bigr) + F_{\wt{U}^{\otimes\omega}}(T^*_a)_{\mb{i}}\bigl(\Delta_{\mb{i}'} +\sum_{\mb{j}} \calE^{\mb{j}}_{\mb{i}'}\bigr) + F_{\wt{U}^{\otimes\omega}}(T^*_b)_{\mb{i}'}\bigl(\Delta_{\mb{i}} +\sum_{\mb{j}} \calE^{\mb{j}}_{\mb{i}}\bigr). \label{eq:splitthree}
    \end{multline}
    We can bound the first of the three terms in \eqref{eq:splitthree} by $(\epsmap + r^{\omega}\epstrueouter)^2$. For the remaining two terms, we can use the fact that $-2 \le \wt{U}^j_i \le 2$ for all $i,j$ to obtain the (very loose) upper bound $\left(F_{\wt{U}^{\otimes\omega}}(T^*_a)_{\mb{i}}\right)^2 \le O(r)^{\omega}\radius^2$, so
    \begin{equation}
        \bigl(F_{\wt{U}^{\otimes\omega}}(T^*_a)_{\mb{i}}\bigr)^2\bigl(\Delta_{\mb{i}'} +\sum_{\mb{j}} \calE^{\mb{j}}_{\mb{i}'}\bigr)^2 \le O(r)^{\omega}\radius^2 \cdot (\epsmap + r^{\omega}\epstrueouter)^2.
    \end{equation}
    and similarly for $\left(F_{\wt{U}^{\otimes\omega}}(T^*_b)_{\mb{i}'}\right)^2\left(\Delta_{\mb{i}} +\sum_{\mb{j}} \calE^{\mb{j}}_{\mb{i}}\right)^2$.
    So we conclude that
    \begin{equation}
        (T_a)_{\mb{i}} (T_b)_{\mb{i}'} = F_{\wt{U}^{\otimes\omega}}(T^*_a)_{\mb{i}} F_{\wt{U}^{\otimes\omega}}(T^*_b)_{\mb{i}'} \pm O(r)^{\omega/2}\radius \cdot (\epsmap + r^{\omega}\epstrueouter).\label{eq:TTversusFTFT}
    \end{equation}
    
    Finally, we turn to showing that $F_{\wt{U}^{\otimes\omega}}(T^*_a)_{\mb{i}}\cdot F_{\wt{U}^{\otimes\omega}}(T^*_b)_{\mb{i}'}$ is close to $(T^*_a)_{\mb{i}} (T^*_b)_{\mb{i}'}$:
    \begin{align}
        \MoveEqLeft F_{\wt{U}^{\otimes\omega}}(T^*_a)_{\mb{i}}\cdot F_{\wt{U}^{\otimes\omega}}(T^*_b)_{\mb{i}'} = \sum_{\mb{j},\mb{j}'\in\strings} \wt{U}^{j_1}_{i_1}\wt{U}^{j'_1}_{i'_1}\cdots \wt{U}^{j_\omega}_{i_\omega} \wt{U}^{j'_\omega}_{i'_\omega} (T^*_a)_{\mb{j}} (T^*_b)_{\mb{j}'} \\
        &= \wt{U}^{i_1}_{i_1}\wt{U}^{i'_1}_{i'_1}\cdots \wt{U}^{i_\omega}_{i_\omega}\wt{U}^{i'_\omega}_{i'_\omega} (T^*_a)_{\mb{i}} (T^*_b)_{\mb{i}'} + \sum_{(\mb{j},\mb{j}')\neq (\mb{i},\mb{i}')} \wt{U}^{j_1}_{i_1}\wt{U}^{j'_1}_{i'_1}\cdots \wt{U}^{j_\omega}_{i_\omega} \wt{U}^{j'_\omega}_{i'_\omega} (T^*_a)_{\mb{j}} (T^*_b)_{\mb{j}'} \label{eq:FiFi}
    \end{align}
    To bound the first term in \eqref{eq:FiFi}, denote $\delta_{ii'}\triangleq \wt{U}^i_i \wt{U}^{i'}_{i'} - 1$, noting that by Lemma~\ref{lem:push_usemu}, $\delta_{ii'} = \pm \epsilon'$ for $\epsilon'\triangleq O(r\omega)^{O(\omega^3)}(d\radius/\kappa)^{O(\omega^2)}\left((d\eta/\kappa)^{O(1/\omega)} + r^{O(\omega\ell)}(\omega\ell)^{O(\ell^2)}(d\radius/\kappa)^{O(\ell)}\theta^{1/4}\eta^{1/4}\right)/\gap^{1/2}$. So
    \begin{equation}
        \wt{U}^{i_1}_{i_1}\wt{U}^{i'_1}_{i'_1}\cdots \wt{U}^{i_\omega}_{i_\omega}\wt{U}^{i'_\omega}_{i'_\omega} (T^*_a)_{\mb{i}} (T^*_b)_{\mb{i}'}  = (T^*_a)_{\mb{i}} (T^*_b)_{\mb{i}'}\sum_{S\subseteq[\omega]} \prod_{s\in S} \delta_{i_si'_s} = (T^*_a)_{\mb{i}} (T^*_b)_{\mb{i}'} \pm 2^{\omega}\cdot\epsilon'.\label{eq:FiFidiag}
    \end{equation}
    To bound the sum in \eqref{eq:FiFi}, let $\xi_{ij} \triangleq \wt{U}^j_i - 1$, noting that by Lemma~\ref{lem:Udiag}, $\xi_{ij} \le O(\epsoffdiag^{1/2})$. Then
    \begin{align}
        \biggl(\sum_{(\mb{j},\mb{j}')\neq(\mb{i},\mb{i}')} \wt{U}^{j_1}_{i_1}\wt{U}^{j'_1}_{i'_1}\cdots \wt{U}^{j_\omega}_{i_\omega} \wt{U}^{j'_\omega}_{i'_\omega} (T^*_a)_{\mb{j}} (T^*_b)_{\mb{j}'}\biggr)^2 &\le \sum_{(\mb{j},\mb{j}')\neq(\mb{i},\mb{i}')} \left(\wt{U}^{j_1}_{i_1}\wt{U}^{j'_1}_{i'_1}\cdots \wt{U}^{j_\omega}_{i_\omega} \wt{U}^{j'_\omega}_{i'_\omega}\right)^2 \cdot \norm{T^*_a}^2_F \norm{T^*_b}^2_F\\
        &\le O((r^2\epsoffdiag)^{\omega}\cdot\radius^4), \label{eq:FiFioffdiag}
    \end{align}
    so the summation in \eqref{eq:FiFi} is upper bounded by $O((r^2\epsoffdiag)^{\omega/2}\cdot\radius^2)$. Combining \eqref{eq:TTversusFTFT}, \eqref{eq:FiFi}, \eqref{eq:FiFidiag}, and \eqref{eq:FiFioffdiag}, we conclude that
    \begin{equation}
        (T_a)_{\mb{i}} (T_b)_{\mb{i}'} = (T^*_a)_{\mb{i}}(T^*_b)_{\mb{i}'} \pm \bigl(O(r)^{\omega/2}\radius\cdot(\epsmap + r^{\omega}\epstrueouter) + 2^{\omega}\epsilon' + O((r^2\epsoffdiag)^{\omega/2}\cdot \radius^2)\bigr). \label{eq:final_pair_compare}
    \end{equation}
    The theorem follows upon taking pseudo-expectations on both sides, and recalling the definition of $\epsilon'$ above, along with the definitions of $\epsmap, \epstrueort, \epstrueouter, \epsoffdiag$ from \eqref{eq:epsmap_push_def}, \eqref{eq:epstrueort_def}, \eqref{eq:epstrueouter_def}, \eqref{eq:epsoffdiag_push_def}, noting that the dominant term in the error term of \eqref{eq:final_pair_compare}, is $2^{\omega}\epsilon'$.
\end{proof}

\noindent Note that Theorem~\ref{thm:push_main_pairwise} allows us to accurately estimate the \emph{magnitude} of every entry of every $T^*_a$ using a pseudoexpectation satisfying Program~\ref{program:lastprogram}. To break the last remaining symmetry of whether the underlying rotation is approximately $\Id$ or $-\Id$, we observe that Theorem~\ref{thm:push_main_pairwise} also implies that the pseudoexpectation tells us the sign of $(T^*_a)_{\mb{i}} (T^*_b)_{\mb{j}}$ for any $a,b\in[d]$, $\mb{i},\mb{j}\in\strings$ (as long as $(T^*_a)_{\mb{i}}$ and $(T^*_b)_{\mb{j}}$ are not too small relative to $\epsilon^*$). 

It therefore suffices to arbitrarily fix the sign of our estimate for $(T^*_a)_{\mb{i}}$ for \emph{some} $a,\mb{i}$ and read off the signs of the remaining entries of the ground truth using the pseudoexpectation. We give a full description of the resulting algorithm in {\sc LowRankFactorize} (Algorithm~\ref{alg:push}) below.

\begin{algorithm2e}
\DontPrintSemicolon
\caption{\textsc{LowRankFactorize}($S$)}
\label{alg:push}
    \KwIn{Second-order moments $\brc{S_{a,b}}$}
    \KwOut{Components $\brc{\wh{T}_a}$}
        Let $\wt{\mathbb{E}}_1[\cdot]$ be a degree-$\poly(\omega,\ell)$ pseudo-expectation satisfying the constraints of Program~\ref{program:sos2}.\;
        Define $\wh{G}\in\R^{d\times d}$ by $\wh{G}_{ab} \gets \wt{\mathbb{E}}_1[\iprod{F_a,F_b}]$ for all $a,b\in[d]$.\label{step:formgram} \;
        $(\lambda,\mu)\gets$ {\sc FindCombo}($\wh{G}$). \label{step:push_lammu}\;
        Let $\wt{\mathbb{E}}_2[\cdot]$ be a degree-$\poly(\omega,\ell)$ pseudo-expectation satisfying the constraints of Program~\ref{program:lastprogram} run with vectors $\lambda,\mu$.\label{step:secondpseudo}\;
        \For{$1\le a \le d$}{
            Initialize $\wh{T}_a\in(\R^d)^{\otimes\omega}$ by $(\wh{T}_a)_{\mb{i}} \gets \wt{\mathbb{E}}_2[((T_a)_{\mb{i}})^2]^{1/2}$ for all $\mb{i}\in\strings$.\;
        }
        $(a^*,\mb{i}^*)\gets \arg\max_{a\in[d],\mb{i}\in\strings} (\wh{T}_a)_{\mb{i}}$.\label{step:aistar}\;
        \For{$1 \le a \le d$ and $\mb{i}\in\strings$}{
            $s_{a,\mb{i}}\gets$ sign of $\wt{\mathbb{E}}_2[(T_a)_{\mb{i}}\cdot (T_{a^*})_{\mb{i}^*}]$. \label{step:getsign}\;
            $(\wh{T}_a)_{\mb{i}}\gets s_{a,\mb{i}} \cdot (\wh{T}_a)_{\mb{i}}$.\;
        }
        \Return $\brc{\wh{T}_a}$.
\end{algorithm2e}

We now complete the proof of Theorem~\ref{thm:main_push}.

\begin{proof}
    By Lemma~\ref{lem:push_gram}, $\wh{G}$ computed in Step~\ref{step:formgram} of {\sc LowRankFactorize} satisfies 
    \begin{align}
        |\wh{G}_{ab} - \iprod{F^*_a,F^*_b}| &\le r^{O(\omega)}\radius^4(\epstrueouter + \epstrueort + \epsmap) \\
        &= \poly(r,\omega,d,\radius,1/\kappa)^{\omega^3}\cdot \left((d\eta/\kappa)^{O(1/\omega)} + \poly(r^{\omega},\omega^{\ell},\ell^{\ell},d,\radius,1/\kappa)^{\ell}\cdot \theta^2\eta^2\right).
    \end{align}
    By taking this latter quantity to be $\epsgram$ in Lemma~\ref{lem:forcegap} (note that our assumed bound on $\eta$ in Theorem~\ref{thm:main_push} easily ensures that $\epsgram$ is sufficiently small to apply Lemma~\ref{lem:forcegap}), we find that $\lambda,\mu$ in Step~\ref{step:push_lammu} is $\gap$-non-degenerate for $\gap = \sigma_{\min}(H)/\poly(r) \ge \psi/\poly(r)$ by Part~\ref{item:push_condnumber3} of Assumption~\ref{assume:push}. We can thus apply Corollary~\ref{cor:push_pairwise} with this choice of $\gap$ to the pseudoexpectation $\wt{\mathbb{E}}_2[\cdot]$ in Step~\ref{step:secondpseudo}.
    % satisfies $|\wt{\mathbb{E}}_2[(T^*_a)_{\mb{i}}(T^*_b)_{\mb{j}}] - (T^*_a)_{\mb{i}}(T^*_b)_{\mb{j}}| \le \epsilon^*$ for any $a,b\in[d]$, $\mb{i},\mb{j}\in\strings$.
    
    Now consider $(a^*,\mb{i}^*)$ from Step~\ref{step:aistar}. If at that step $\wh{T}_{a^*}$ satisfies $(\wh{T}_{a^*})_{\mb{i}^*} \le \sqrt{\epsilon^*}$, then by the first part of Corollary~\ref{cor:push_pairwise} we have that $|(T^*_a)_{\mb{i}}| \le 2\sqrt{\epsilon^*}$ for all $a,\mb{i}$, in which case for the final $\brc{\wh{T}_a}$ output by the algorithm, $\norm{T^*_a - \wh{T}_a}^2_F \le O(r^{\omega}\epsilon^*)$ for all $a$.
    
    On the other hand, suppose $\wh{T}_{a^*}$ in Step~\ref{step:aistar} satisfies $(\wh{T}_{a^*})_{\mb{i}^*} > \sqrt{\epsilon^*}$. We can assume without loss of generality that $(T^*_{a^*})_{\mb{i}^*} > 0$. So by the second part of Corollary~\ref{cor:push_pairwise}, the sign $s_{a,\mb{i}}$ computed in Step~\ref{step:getsign} satisfies $s_{a,\mb{i}} = \sgn((T^*_a)_{\mb{i}})$ for all $a,\mb{i}$ satisfying $|(\wh{T}_a)_{\mb{i}}| \ge \sqrt{\epsilon^*}$. So for all such $a,\mb{i}$, the final $\brc{\wh{T}_a}$ output by the algorithm satisfies $|(\wh{T}_a)_{\mb{i}} - (T^*_a)_{\mb{i}}| \le \sqrt{\epsilon^*}$ by the first part of Corollary~\ref{cor:push_pairwise}. And for all remaining $a,\mb{i}$, by the first part of Corollary~\ref{cor:push_pairwise}, $|(T^*_a)_{\mb{i}}| \le 2\sqrt{\epsilon^*}$, so $|(\wh{T}_a)_{\mb{i}} - (T^*_a)_{\mb{i}}| \le 3\sqrt{\epsilon^*}$ by triangle inequality. Thus, the output $\brc{\wh{T}_a}$ satisfies $\norm{T^*_a - \wh{T}_a}^2_F \le O(r^{\omega}\epsilon^*)$ for all $a$.
\end{proof}

\subsection{Other Choices of \texorpdfstring{$\Sigma$}{Sigma}}
\label{sec:rotationinvariant}

Here we briefly note that our guarantees easily carry over to $\Sigma$ of the form
\begin{equation}
    \Sigma = \E[x\sim D]{\vec(x^{\otimes\omega})\vec(x^{\otimes\omega})^{\top}} \label{eq:general_Sig}
\end{equation} for any rotation-invariant distribution $D$ over $\R^r$ for which $\Sigma$ is reasonably bounded.

The reason is that the entire argument above made very limited use of the structure of $\Sigma = \E[g\sim\calN(0,\Id)]{\vec(g^{\otimes\omega})\vec(g^{\otimes\omega})^{\top}}$ beyond the fact that 
\begin{enumerate}
    \item $\Sigma$ is ultra-symmetric
    \item The inner product induced by $\Sigma$ is gauge-invariant over the space of symmetric tensors, that is, $\iprod{T,T'}_{\Sigma} = \iprod{F_{V^{\otimes\omega}}(T), F_{V^{\otimes\omega}}(T')}_{\Sigma}$ for any $V\in O(r)$ and symmetric tensors $T,T'$
    \item The bottom and top eigenvalues $\Sigma_{\sym}$ and $\Sigma$ respectively are bounded (Lemma~\ref{lem:sigma_cond} and Lemma~\ref{lem:crude_Ubound})
    \item By Lemma~\ref{lem:hermite}, we have an explicit expression for the moments $\E[g]{\iprod{v,g}^{\omega}\iprod{w,g}^{\omega}}$ (this is used in Lemma~\ref{lem:diagonal_orth} to establish that $\brc{U^{i\cdots i}}$ are approximately orthonormal).
\end{enumerate}
Note that the first three properties hold for any $\Sigma$ of the form \eqref{eq:general_Sig} for which $D$ is rotation-invariant and reasonably anti-concentrated and bounded. And while it would seem that property 4 makes essential use of $\calN(0,\Id)$, recall that Lemma~\ref{lem:rot_to_gaussian} tells us that any rotation invariant distribution $D$ has the same moments, up to a fixed constant factor depending on $D$ and the degree of the moment.

As a result, the proof of Lemma~\ref{lem:diagonal_orth} immediately carries over to the setting where $\calN(0,\Id)$ is replaced with any rotation-invariant distribution for which $C_D$ is reasonably bounded. As this was the only place where Lemma~\ref{lem:hermite} was used, we conclude that Theorem~\ref{thm:main_push} extends to $\Sigma$ of the form \eqref{eq:general_Sig} for rotation-invariant $D$, and the final error bound will at worst have an additional factor of $\poly(C_D,\sigma_{\min}(\Sigma_{\sym}),\sigma_{\max}(\Sigma))^{\poly(\omega,\ell)}$.

Finally, we argue that Theorem~\ref{thm:main_push} also extends to $\Sigma = \Id_{r^{\omega}}$. While the identity matrix is not ultra-symmetric, we can replace it by
$\Sigma$ whose $(\mb{i},\mb{j})$-th entry is $\frac{1}{\num{i}}\bone{\sort{i} = \sort{j}}$. This new $\Sigma$ is ultra-symmetric and satisfies $\iprod{T^*_a,T^*_b}_{\Sigma} = \iprod{T^*_a,T^*_b}_{\Id}$. The inner product is clearly gauge-invariant over symmetric tensors as the Euclidean inner product is, and $\Sigma_{\sym}$ and $\Sigma$ clearly have bounded bottom and top eigenvalues respectively. As for property 4, it is true that we no longer have any reasonable analogue of Lemma~\ref{lem:hermite}, but to prove that $U^{i\cdots i}$, we can simply use the fact that $U^{\top}\Sigma U \approx \Sigma$ from Lemma~\ref{lem:push_almost_ortho}. Observe that the $(i\cdots i, j\cdots j)$-th entry of this approximate matrix equality yields
\begin{equation}
    \bone{i=j} \approx \iprod{U^{i\cdots i},U^{j\cdots j}}_{\Sigma} = \sum_{\mb{k},\mb{k}' \in \strings} \frac{1}{\num{k}}\bone{\sort{k} = \sort{k}'}\cdot U^{i\cdots i}_{\mb{k}} U^{j\cdots j}_{\mb{k}'} = \iprod{U^{i\cdots i}, U^{j\cdots j}},
\end{equation}
thus proving Lemma~\ref{lem:diagonal_orth} directly.

\begin{remark}\label{remark:hermite}
    Besides representing the simplest possible setting of low-rank factorization, the special case of $\Sigma = \Id$ also has the following implication for learning \emph{inhomogeneous} polynomial transformations, specifically where the network is a one hidden layer network with \emph{Hermite activations}. Suppose $\calD$ is the transformation of $\calN(0,\Id)$ under the map that sends input $x\in\R^r$ to $(p_1(x),\ldots,p_d(x))$ for
    \begin{equation}
        p_a(x) \triangleq \sum^\ell_{t=1} \lambda_{a,t} \phi_{\omega}(\iprod{v^*_{a,t},x}) \ \ \forall \ a\in[d],
    \end{equation} where $v^*_{a,t}$ are unit vectors and $\lambda_{a,t}$ are scalars, and $\phi_{\omega}$ corresponds to the degree-$\omega$ (normalized) probabilist's Hermite polynomial. In this case, the pairwise moments of $\calD$ are given by
    \begin{equation}
        \E[g\sim\calN(0,\Id)]{p_a(g) p_b(g)} = \sum^{\ell}_{t,t'=1} \lambda_{a,t}\lambda_{b,t'} \E[g]{\phi_{\omega}(\iprod{v^*_{a,t},g})\phi_{\omega}(\iprod{v^*_{b,t'},g})} = \sum^{\ell}_{t,t'=1} \lambda_{a,t}\lambda_{b,t'}\iprod{v^*_{a,t},v^*_{b,t'}}^{\omega},
    \end{equation}
    which we can express as $\iprod{T^*_a,T^*_b}$, where $T^*_a \triangleq \sum^{\ell}_{t=1} \lambda_{a,t} (v^*_{a,t})^{\otimes \omega}$ for every $a\in[d]$. Thus, our algorithm for low-rank factorization when $\Sigma = \Id$ yields a learning algorithm for this family of inhomogeneous polynomial transformations.
\end{remark}

% The reason is that the only place in the proof where the structure of $\Sigma$ as an expectation over $\vec(x^{\otimes\omega})\vec(x^{\otimes\omega})^{\top}$ is crucially used is Lemma~\ref{lem:diagonal_orth}. But when $\Sigma = \Id$, by Lemma~\ref{lem:push_almost_ortho} we already know that $U^{\top}U \approx \Id$, so \emph

\subsection{Dependence on \texorpdfstring{$d$}{d}}
\label{sec:fpt2}

In this section we observe, analogously to Section~\ref{sec:fpt}, that for $d$ sufficiently large, one can decouple the dependence on $d$ from all other parameters and obtain run in time \emph{linear} in $d$.

\begin{corollary}\label{cor:fpt2}
    Suppose that for some $\binom{r+\omega-1}{\omega} \le d'\le d$, Assumption~\ref{assume:tensorring} holds for the first $d'$ units of the polynomial network (i.e. $T^*_1,\ldots,T^*_{d'}$) and $\ell < r$ and $\eta \le \poly(r,\omega,d',\radius,1/\kappa)^{-\poly(\omega,\ell)}$, and we are given query access to $S\in\R^{d\times d}$ satisfying \eqref{eq:push_moment} for $\Sigma$ given by \eqref{eq:Sigmadef}.
    
    Then there is an algorithm which runs in time $\poly(d'r)^{\poly(\omega,\ell)} + d\cdot\poly(d',r^\omega)$ and outputs $\wh{T}_1,\ldots,\wh{T}_d$ for which \begin{equation}
        \gaugedist(\brc{T^*_a},\brc{\wh{T}_a}) \le \poly(r,\omega,d',\radius,1/\kappa)^{\omega^3} \cdot \left((d'\eta/\kappa)^{O(1/\omega)} + \poly(r^{\omega}, \omega^{\ell},\ell^{\ell},d',\radius,1/\kappa)^{\ell}\cdot \sqrt[8]{\theta\eta/\psi^2} \right)
    \end{equation}
    with high probability.
\end{corollary}

Note that if $T^*_1,\ldots,T^*_{d'}$ are componentwise-smoothed in the sense of Definition~\ref{def:smoothed_2}, then as we show in Lemma~\ref{lem:compsmooth} in Section~\ref{sec:conditions}, this holds for $d' = \wt{\Theta}((r + \omega)^{\omega\ell})$, and we thus obtain a runtime which is linear in $d$ as claimed. 

\begin{proof}
    We can run Algorithm~\ref{alg:push} on the parts of $S$ corresponding to the first $d'$ units of the polynomial network to produce $\wh{T}_1,\ldots,\wh{T}_{d'}$ satisfying $\gaugedist(\brc{T^*_1,\ldots,T^*_{d'}},\brc{\wh{T}_1,\ldots,\wh{T}_{d'}}) \le \eta'$ for $\eta'$ given by Theorem~\ref{thm:main_push}. Note that this takes time $\poly(d'r)^{\poly(\omega,\ell)}$. At this point we can assume without loss of generality that $\norm{T^*_a - \wh{T}_a}_F \le \eta'$ for all $1\le a\le d'$.

    To recover $T^*_{d'+1},\ldots,T^*_d$, we can then use our estimates $S_{a,b}$ of $\iprod{T^*_a,T^*_b}_\Sigma$ for all $1 \le a \le d'$ and $b > d'$ to set up linear systems in the unknowns $T^*_{d'+1},\ldots,T^*_d$. That is, for every $b > d'$, we define 
    \begin{equation}
        \wh{T}_b \triangleq \arg\min_{\wh{T}} \sum^{d'}_{a=1} \left(S_{a,b} - \iprod{\wh{T}_a,\wh{T}}_\Sigma \right)^2.
    \end{equation}
    Because $|S_{a,b} - \iprod{\wh{T}_a, T^*_b}_\Sigma| \le |\iprod{\wh{T}_a - T^*_a,T^*_b}| \le \eta'\radius r^{\omega/2} (2\omega-1)!!$, where in the last step we used Cauchy-Schwarz and \eqref{eq:sigmaupperbound}, we conclude by Part~\ref{assume:condnumber} and Lemma~\ref{lem:sigma_cond} that $\norm{\wh{T}_b - T^*_b}_F \le \eta'\radius r^{\omega/2} (2\omega-1)!! \sqrt{d'} / (\kappa\omega^{\omega/2})$ for all $b > d'$. The factors next to $\eta'$ can be absorbed into the asymptotic form of $\eta'$. This part of the algorithm only runs in time $d\cdot \poly(d',r^\omega)$ because it only needs to solve an $d'\times \binom{r+\omega-1}{\omega}$-dimensional least-squares problem for every $b > d'$.
\end{proof}
 
% As in Section~\ref{sec:fpt}, here the key takeaway is that the proposed algorithm only runs in time $d\cdot ((r+\omega)^{\poly(\omega\ell)})$ because it only needs to solve a $d'$-dimensional least-squares problem for every $b > d'$.

\section{Smoothed Networks Satisfy Deterministic Conditions}
\label{sec:conditions}

In this section we verify that polynomial networks which are smoothed in the sense of Definitions~\ref{def:smoothed_1} and \ref{def:smoothed_2} satisfy Assumptions~\ref{assume:tensorring} and Assumptions~\ref{assume:push} respectively. We then use this to deduce our main algorithmic guarantees for learning smoothed polynomial transformations.

\subsection{Fully-Smoothed Quadratic Networks}

\begin{lemma}\label{lem:fullysmooth}
    Suppose $d \ge \Omega(r^2\log(drR/\rho))$. Let $R\ge 1$ and $\rho \le 1$.
    
    If $Q^*_1,\ldots,Q^*_d$ are $\rho$-fully-smoothed relative to base network $\overline{Q}_1,\ldots,\overline{Q}_d$ and $\norm{\overline{Q}_a}_{\op} \le R$ for all $a\in[d]$, then with probability at least $1 - \exp(-\Omega(d))$ over the smoothing, Assumption~\ref{assume:tensorring} holds with parameters $\radius = R\sqrt{r} + O(\rho\sqrt{d})$ and $\kappa = \Theta(\rho\sqrt{d/r})$.
\end{lemma}

\begin{proof}
    \noindent\textbf{Part~\ref{assume:scale}:} By Lemma~\ref{lem:goe}, we have with probability at least $1 - 2\exp(-\Omega(d))$ that $\norm{G_a}_{\op} \ge \Omega(\sqrt{d})$. If this happens, then $\norm{\overline{Q}_a + \frac{\rho}{\sqrt{r}}\cdot G_a}_{\op} \le R + O(\rho\sqrt{d/r})$ for all $a\in[d]$, so we can take $\radius = R\sqrt{r} + O(\rho\sqrt{d})$.

    % \noindent\textbf{Part~\ref{assume:breaksym}:} By applying Theorem~\ref{thm:gaps} to $\frac{\sqrt{r}}{\rho}\left(\overline{Q}_a + \frac{\rho}{\sqrt{r}}\cdot G_a\right)$ for each $a\in[d]$, we see that for any $c > 0$ and $\alpha = 3(c+1)(1+2\gamma)+5$, Part~\ref{assume:breaksym} of Assumption~\ref{assume:tensorring} holds with $\gap = \rho\cdot r^{-\alpha-1/2}$ with probability at least $1 - r^{1-c}$. In particular, this happens with probability at least $1 - \delta/3$ if we take $c = \log(3r/\delta)/\log(r)$, in which case
    % \begin{equation}
    %     r^{-\alpha-1/2} \ge r^{-6\gamma-8}\cdot r^{-3c(1+2\gamma)} = r^{-6\gamma-8} \cdot (3r/\delta)^{-6\gamma-3} \ge (3r^2/\delta)^{-6\gamma-8},
    % \end{equation}
    % so we can take $\gap = (3r^2/\delta)^{-6\gamma-8}$.
    
    \noindent\textbf{Part~\ref{assume:condnumber}:} Let $\overline{M}\in\R^{d\times \binom{r+1}{2}}$ denote the matrix whose $(a,(i_1,i_2))$-th entry, for $a\in[d]$ and $1 \le i_1 \le i_2 \le r$, is given by $(\overline{Q}_a)_{i_1 i_2}$. Note that $M^* = \overline{M} + \frac{\rho}{\sqrt{r}}\cdot G$ for $G\in\R^{d\times \binom{r+1}{2}}$ whose entries are independent draws from $\calN(0,1)$. 

    For any $v\in\S^{\binom{r+1}{2}-1}$ we have for any $a\in[m]$ that $(M^* v)_a = \iprod{\overline{M}_a + \frac{\rho}{\sqrt{r}}G_a,v}$. Because the rows of $G$ are independent, each $(M^* v)_a$ is an independent draw from $\calN(\iprod{\overline{M}_a,v}, \rho^2/r)$, so by standard Gaussian anticoncentration, there is an absolute constant $c > 0$ such that $\Pr*{|(M^* v)_a| \le 2\rho/3\sqrt{r}} \le 1/2$. We conclude that $\norm{M^* v}^2 \ge \Omega(d\rho^2/r)$ with probability at least $1 - 2^{-\Omega(d)}$. For $\epsilon\triangleq \Theta(\rho\sqrt{d}/(rR))$, take an $\epsilon$-net $\calN$ of $\mathbb{S}^{\binom{r+1}{2}-1}$. For any $v\in\S^{\binom{r+1}{2}-1}$, if $\norm{v - \wt{v}} \le \epsilon$ then $\norm{M^*(v - \wt{v})}^2 \le \norm{M^*}^2_{\op}\epsilon^2 \le O(\epsilon^2d r R^2)$. So if $\norm{M^* \wt{v}}^2 \ge \Omega(d\rho^2/r)$ for all $\wt{v}\in\calN$, then $\norm{M^* v}^2 \ge \Omega(d\rho^2/r)$ for all $v\in\S^{\binom{r+1}{2}}$. This happens with probability at least $1 - |\calN|\cdot 2^{-\Omega(d)} \ge 1 - \exp(r^2\log(1/\epsilon) - \Omega(d)) \ge 1 - \exp(-\Omega(d))$, where the last step follows by the assume bound of $d \ge \Omega(r^2\log(rR/\rho))$.
\end{proof}

\subsection{Componentwise-Smoothed Polynomial Networks}
\label{sec:compsmooth}

\begin{lemma}\label{lem:compsmooth}
    Suppose $d\ge \max\brc*{(C\ell(r+\omega))^{C'\omega\ell}\cdot \log(R/\rho), C''r\log(\ell r R^{2\omega}/\rho)}$ for sufficiently large absolute constants $C,C',C''>0$, and suppose $\ell \le r - r^{0.9}$. Let $R \ge 1$ and $\rho \le 1$.
    
    If $T^*_1,\ldots,T^*_d$ are $\rho$-componentwise-smoothed relative to base network $\overline{T}_1,\ldots,\overline{T}_d$ such that for each $a\in[d]$, there exist vectors $\overline{v}_{a,1},\ldots,\overline{v}_{a,\ell}$ for which $\overline{T}_a = \sum^{\ell}_{t=1} \overline{v}_{a,\ell}^{\otimes\omega}$ and $\norm{\overline{v}_{a,t}}^2\le R$ for all $t$, then with probability at least $1 - \exp(-\Omega(d))$ over the smoothing, Assumption~\ref{assume:push} holds with parameters $\radius = \ell\cdot \Theta(R)^\omega$, $\kappa = \Theta(\sqrt{d\ell}(\rho^2\omega/r)^{\omega/2})$, $\theta = \Theta(Rr\omega\ell)^{O(\omega\ell)}$, and $\psi = \Theta(\rho/(r\omega))^{\Theta(\omega)}$.
\end{lemma}

To prove Lemma~\ref{lem:compsmooth}, we will need the following condition number bound:

\begin{lemma}\label{lem:giant_mat}
    Let $\brc{v^*_{a,t}}$ be as in Lemma~\ref{lem:compsmooth}. For any $a\in[d]$, let $w_a\in\R^{r\ell}$ denote the concatenation of $v^*_{a,1},\ldots,v^*_{a,\ell}$. For any $e\in\mathbb{N}$, define $N\triangleq \binom{r\ell + e - 1}{e}$ and let $K^{(e)}\in\R^{d\times N}$ denote the matrix whose rows consist of vectorizations of $(w^{\otimes e}_a)_{\sym}$. 
    
    Then if $d \ge (C(r\ell + e))^{C'e}\cdot \log(R/\rho)$ for sufficiently large absolute constants $C,C'>0$, then with probability at least $1 - \exp(-\Omega(d))$ over the randomness of $\brc{v^*_{a,t}}$, we have that $\sigma_{\min}(K^{(e)}) \ge \sqrt{d}\cdot \Theta(\rho/(re))^{\Theta(e)}$.
\end{lemma}

\begin{proof}
    For any $p\in\S^{N-1}$ regarded as an $r\ell$-variate homogeneous polynomial of degree $e$, we have for any $a\in[d]$ that
    \begin{equation}
        (K^{(e)}p)_a = p(w_a)^2.
    \end{equation}
    Because $T^*_1,\ldots,T^*_d$ are $\rho$-componentwise-smoothed, every $w_a$ is an independent sample from the distribution $\calN(\overline{w}_a,\frac{\rho^2}{r}\Id)$. So for any $a\in[d]$, consider the degree-$e$ polynomial $p'(x) \triangleq p(\overline{w}_a + \frac{\rho}{\sqrt{r}}x)$. By Lemma~\ref{lem:variance_shift}, 
    \begin{equation}
        \Var[g\sim\calN(0,\Id)]{p'(g)} \ge (\rho^2/r)^{e}/e^{e/2} = O(\rho/(re))^{O(e)}.
    \end{equation}
    By Carbery-Wright,
    \begin{equation}
        \Pr{|p'(g)| \le O(\rho/(re))^{O(e)}} \le 1/2.
    \end{equation}
    As the randomness for each of $w_1,\ldots,w_d$ is independent, we conclude that for fixed $p$,
    % \begin{equation}
    %     \Pr{p(w_a)^2 \le O(\rho/r\omega\ell)^{O(\omega\ell)} \ \text{for all} \ a\in[m]} \le 2^{-d/2},
    % \end{equation}
    % which implies that
    \begin{equation}
        \Pr*{\sum^d_{a=1} p(w_a)^2 \le d\cdot O(\rho/(re))^{O(e)}} \le 2^{-\Omega(d)}.
    \end{equation}
    
    We will now net over $p$'s. For any $p,\wt{p}\in\S^{N-1}$, note that
    \begin{equation}
        \sum^d_{i=1} (p - \wt{p})(w_a)^2 \le \norm{p - \wt{p}}^2_2 \cdot \sum^d_{a=1} \norm{(w_a)^{\otimes e}_{\sym}} \le d(R\sqrt{\ell})^{e}\cdot \norm{p-\wt{p}}^2_2. \label{eq:diff_net}
    \end{equation}
    So for $\epsilon\triangleq d\cdot \Theta(\rho/(Rre))^{\Theta(e)}/ d$, take an $\epsilon$-net $\calN$ of $\S^{N-1}$. If $\wt{p}(w_a)^2 > d\cdot \Theta(\rho/(re))^{\Theta(e)}$ for all $\wt{p}\in\calN$, then by \eqref{eq:diff_net} and squared triangle inequality,
    \begin{equation}
        \sum^d_{a=1} p(w_a)^2 > d\cdot \Theta(\rho/(re))^{\Theta(e)} \ \ \forall \ p\in\S^{N-1}. \label{eq:allp}
    \end{equation}
    This happens with probability at least
    \begin{equation}
        1 - |\calN|\cdot 2^{-\Omega(d)} = 1 - \exp\left(N\log(1/\epsilon) - \Omega(d)\right) = 1 - \exp\left(O(r\ell+e)^{O(e)}\cdot \log(Rr/\rho) + \log(d) - \Omega(d)\right),
    \end{equation}
    so if $d \ge (C(\ell r + e))^{C'e}\cdot \log(R/\rho)$ for sufficiently large absolute constants $C,C'>0$, then \eqref{eq:allp} holds with probability $1 - \exp(-\Omega(d))$. In this case,
    $\sigma_{\min}(K^{(e)}) \ge \sqrt{d}\cdot \Theta(\rho/(re))^{\Theta(e)}$.
\end{proof}

We are now ready to prove Lemma~\ref{lem:compsmooth}:

\begin{proof}[Proof of Lemma~\ref{lem:compsmooth}]
    \noindent\textbf{Part~\ref{item:push_scale}:} As every $v^*_{a,t}$ is an independent draw from $\calN(\overline{v}_{a,t},\frac{\rho^2}{r}\Id)$ and $\norm{\overline{v}_{a,t}}\le R$ by assumption, we conclude by Fact~\ref{fact:shell} that $\norm{v^*_{a,t}} \le R + O(\rho) \le O(R)$ for all $a,t$ with probability at least $1 - d\ell\exp(-\Omega(r))$. For the rest of the proof, we will condition on this event. This implies that for any $a\in[d]$, $\norm{T^*_a}_F \le \sum^\ell_{t=1} \norm{(v^*_{a,t})^{\otimes\omega}}_F \le \ell\cdot O(R)^\omega$ as claimed.
    
    \noindent\textbf{Part~\ref{item:push_condnumber}:} 
    % Let $\overline{M}\in\R^{d\times\rchoose}$ denote the matrix whose $(a,\sort{i})$-th entry, for $a\in[d]$ and sorted $\sort{i}\in\strings$, is given by $(\overline{T}_a)_{\mb{i}}$. Note that $M^* = \overline{M} + \frac{\rho}{\sqrt{r}}\cdot G$ for $G\in\R^{d\times\rchoose}$ whose entries are independent draws from $\calN(0,1)$.
    Take any $p\in\S^{\rchoose-1}$, regarded as an $r$-variate homogeneous polynomial of degree $\omega$. Note that
    \begin{equation}
        \norm{M^*p}^2 = \sum^d_{a=1} \biggl(\sum^{\ell}_{t=1} p(v^*_{a,t})\biggr)^2.
    \end{equation}
    For any $a\in[d]$, consider the polynomial
    \begin{equation}
        p'_a: (g_1,\ldots,g_t) \mapsto \sum^{\ell}_{t=1} p(\overline{v}_{a,t} + \frac{\rho}{\sqrt{r}} g_t).
    \end{equation}
    As $g_1,\ldots,g_t$ are independent, we conclude by Lemma~\ref{lem:variance_shift} that
    \begin{equation}
        \Var{p'_a(g_1,\ldots,g_t)} = \sum^\ell_{t=1} \Var{p(\overline{v}_{a,t} + \frac{\rho}{\sqrt{r}}g_t)} \ge \sum^{\ell}_{t=1} \ell(\rho^4\omega^2/r^2)^{\omega/2}
    \end{equation}
    By Carbery-Wright, for any $\nu > 0$ we have
    \begin{equation}
        \Pr{|p'_a(g_1,\ldots,g_t)| \le O(\sqrt{\ell}(\rho^2\omega/r)^{\omega/2})} \le 1/2.
    \end{equation}
    Let $\calN$ be an $\epsilon$-net of $\S^{\rchoose-1}$ for $\epsilon = O(\sqrt{d\ell}(\rho^2\omega/r)^{\omega/2} / (d\ell^2\cdot O(R)^{2\omega}))$. As the randomness for each of $T^*_1,\ldots,T^*_d$ is independent and $|\calN| \le O(1/\epsilon)^{\rchoose}$,
    \begin{equation}
        \Pr{\norm{M^*p}^2 \ge \Omega(d\ell(\rho^2\omega/r)^{\omega}) \ \forall \ p\in\calN} \ge 1 - 2^{-\Omega(d)}\cdot O(1/\epsilon)^{\rchoose}.
    \end{equation}
    Provided $d$ exceeds $\Theta(\rchoose\cdot\log(d\ell^2\cdot  \Theta(R)^{2\omega} / \sqrt{\ell}(\rho^2\omega/r)^{\omega/2}))$, the failure probability here is $\exp(-\Omega(d))$, so it suffices for $d \ge \Omega(r^{\omega}\log(R/\rho))$, which is clearly satisfied by the assumed bound on $d$.
    
    Now note that 
    \begin{equation}
        \norm{M^*}^2_{\op} \le \norm{M^*}^2_F \le \sum^d_{a=1} \norm{T^*_a}^2_F \le d\ell^2\cdot O(R)^{2\omega}.
    \end{equation}
    So for any $p\in\S^{\rchoose-1}$, if $p'\in\calN$ satisfies $\norm{p - p'}_2 \le \epsilon$, then 
    \begin{equation}
        \norm{M^*p} \ge \Omega(\sqrt{d \ell}(\rho^2\omega/r)^{\omega/2}) - \norm{M^*}_{\op}\norm{p - p'}_2 \ge \Omega(\sqrt{d \ell}(\rho^2\omega/r)^{\omega/2})
    \end{equation}
    as desired.
    
    \noindent\textbf{Part~\ref{item:push_condnumber2}:} Let $K$ denote the matrix whose rows consist of the vectors $q(v^*_{a,1},\ldots,v^*_{a,\ell})$. Our goal is to lower bound $\sigma_{\min}(K)$.
    
    For $N\triangleq \binom{r\ell + \omega(\ell+1)-1}{\omega(\ell+1)}$, let $K'\in\R^{d\times N}$ denote the matrix $K^{(e)}$ from Lemma~\ref{lem:giant_mat} for $e = \omega(\ell+1)$. Note that the columns of $K$ are a subset of those of $K'$. By Lemma~\ref{lem:giant_mat} and the assumed bound on $d$, $\sigma_{\min}(K) \ge \sigma_{\min}(K') \ge \sqrt{d}\cdot \Theta(\rho/r\omega\ell)^{\Theta(\omega\ell)}$ with probability $1 - \exp(-\Omega(d))$, where the first step is because the columns of $K$ are a subset of those of $K$.
    
    To conclude the proof of this part of the lemma, we wish to apply Fact~\ref{fact:minnorm}. To do this, we need some bound on $\norm{K'}_{\op}$:
    \begin{equation}
        \norm{K'}^2_F = \sum^d_{a=1} \norm{w_a}^{2\omega(\ell+1)}_2 = \sum^d_{a=1} \biggl(\sum^{\ell}_{t=1} \norm{v^*_{a,t}}^2\biggr)^{\omega(\ell+1)} \le d\cdot(\ell R^2)^{\omega(\ell+1)}.
    \end{equation}
    So by Fact~\ref{fact:minnorm}, for any vector $v\in\R^N$, there is a $\lambda\in\R^d$ for which $\lambda^{\top} K' = v$ and \begin{equation}
        \norm{\lambda}_2 \le O(r\omega\ell/\rho)^{O(\omega\ell)} \cdot (\ell R^2)^{\omega(\ell+1)/2} \cdot\norm{v}_2 = O(Rr\omega\ell)^{O(\omega\ell)}\norm{v}_2. \label{eq:lambdanorm}
    \end{equation}
    In particular, take $v = (w^{\otimes\omega(\ell+1)})_{\sym}$ for $w$ given by the concatenation of $v_1,\ldots,v_\ell$. Then
    \begin{equation}
        \norm{v}^2_2 = \norm{w}^{2\omega(\ell+1)}_2 = \biggl(\sum^{\ell}_{t=1}\norm{v_t}^2\biggr)^{\omega(\ell+1)}.
    \end{equation}
    The proof is completed upon noting that the columns of $K$ are a subset of $K'$ and furthermore the entries of $q(v_1,\ldots,v_\ell)$ corresponding to columns of $K$ are precisely given by the entries of this choice of $v$.
    
    \noindent\textbf{Part~\ref{item:push_condnumber3}:} We will again use Lemma~\ref{lem:giant_mat}. Define $N\triangleq \binom{r\ell+2\omega\ell-1}{2\omega\ell}$. As
    \begin{equation}
        (F^*_a)_{ij} = \sum_{t,t'\in[\ell]} \norm{v^*_{a,t}}^{\omega-1}\norm{v^*_{a,t'}}^{\omega-1} v^*_{a,t} (v^*_{a,t'})^{\top},
    \end{equation}
    we can express the $a$-th row of $H$ as $B\vec((w^{\otimes 2\omega\ell}_a)_{\sym})$ for the following matrix $B\in\R^{\binom{r+1}{2}\times N}$. The rows of $B$ are indexed by $(i,j)$ for $1\le i \le j \le r$, and the columns are indexed by $\mb{k}\in([r]\times[\ell])^{2\omega\ell}$. In the $(i,j)$-th row, the $\mb{k}$-th entry is $1$ if 
    \begin{equation}
        \mb{k} = ((t,k_1),(t,k_1),\ldots,(t,k_{\omega'}),(t,k_{\omega'}),(t',k'_1),(t',k'_1),\ldots,(t',k'_{\omega'}),(t',k'_{\omega'}),(t,i),(t',j))
    \end{equation}
    for any $t,t'\in[\ell]$, $k_1,k'_1,\ldots,k_{\omega'},k'_{\omega'}\in[r]$.
    Consider the matrix $K^{(e)}$ from Lemma~\ref{lem:giant_mat} for $e = 2\omega$, we note that $H = K^{(2\omega)} B^{\top}$. For any $p\in\S^{\binom{r+1}{2}-1}$,
    \begin{equation}
        \norm{B^{\top}p}^2 = \sum_{i\le j} \sum_{\mb{k}} B^2_{ij,\mb{k}} p^2_{ij} = \ell^2\cdot r^{\omega-1}, \label{eq:Bbound}
    \end{equation}
    where we used the fact that the nonzero entries of each row of $B$ are supported in disjoint columns, and every row has exactly $\ell^2 \cdot r^{\omega-1}$ distinct nonzero entries.
    So by \eqref{eq:Bbound}, Lemma~\ref{lem:giant_mat}, and the assumed lower bound on $d$, with probability at least $1 - \exp(-\Omega(d))$ we have that for any $p\in\S^{\binom{r+1}{2}-1}$, 
    \begin{equation}
        \norm{Hp}_2 \ge \ell^2\cdot r^{\omega-1} \cdot \sigma_{\min}(K^{(2\omega)}) \ge \Theta(\rho/(r\omega))^{\Theta(\omega)}
    \end{equation}
    as claimed.
\end{proof}

% Define
% \begin{equation}
%     F^*_a \triangleq f^*_a(f^*_a)^{\top} \ \ \text{for} \ \ f^*_a \triangleq \sum_{j_1,\ldots,j_{\omega'}\in[r]} (T^*_a)_{j_1\cdots j_{\omega'}:}.
% \end{equation}
% Observe that
% \begin{equation}
%     f^*_a = \sum^\ell_{t=1} \norm{v^*_{a,t}}^{\omega-1}\cdot v^*_{a,t}.
% \end{equation}

% Let $\phi_a\in\R^{\binom{r+1}{2}}$ denote the vector whose $(i,j)$-th entry is $(F^*_a)_{ij}$ for any $1\le i \le j\le r$.
 
% \begin{lemma}\label{lem:Hcond}
%     For any $d\ge \Theta(\ell r + \omega)^{\Theta(\omega)}$, define the matrix $H\in\R^{d\times\binom{r+1}{2}}$ whose $s$-th row is $\phi_{s}$. Then with probability at least $1 - \exp(-\Omega(d))$ over the randomness of $T^*_1,\ldots,T^*_d$, $\sigma_{\min}(H) \ge \Theta(\rho/(r\omega))^{\Theta(\omega)}$.
% \end{lemma}

\section{Putting Everything Together: Learning Smoothed Networks}
\label{sec:final}

We can now prove our main algorithmic guarantees about learning smoothed quadratic and low-rank polynomial transformations by plugging the algorithms from Sections~\ref{sec:tensorring} and \ref{sec:push}, which we can apply to smoothed networks by virtue of Section~\ref{sec:conditions}, into the reduction from Section~\ref{sec:connect}.

\begin{theorem}\label{thm:main_quadratic}
    Suppose $d \ge \Omega(r^2\log(drR/\rho))$. Let $R\ge 1$ and $\rho \le 1$. If $Q^*_1,\ldots,Q^*_d$ are $\rho$-fully-smoothed relative to base network $\overline{Q}_1,\ldots,\overline{Q}_d$ and $\norm{\overline{Q}_a}_{\op} \le R$ for all $a\in[d]$, then with probability at least $1 - \exp(-\Omega(d))$ over the smoothing, the following holds for the transformation specified by $Q^*_1,\ldots,Q^*_d$:
    
    For any $\epsilon > 0$, given $\poly(r,R,1/\rho,1/\epsilon,\log(1/\delta))$ samples from $\calD$, there is an algorithm that runs in $d\cdot \poly(r)$ additional time and parameter learns $\calD$ to error $\epsilon$ (and also solves proper density estimation to Wasserstein error $\epsilon r\sqrt{d}$) with high probability.
\end{theorem}

\begin{theorem}\label{thm:main_lowrank}
    Suppose $d \ge (\ell(r+\omega))^{\Omega(\omega\ell)}\cdot\log(R/\rho)$ and $\ell < r$. Let $R\ge 1$ and $\rho \le 1$. If $T^*_1,\ldots,T^*_d$ are $\rho$-componentwise-smoothed relative to base network $\overline{T}_1,\ldots,\overline{T}_d$ such that for each $a\in[d]$, there exist vectors $\overline{v}_{a,1},\ldots,\overline{v}_{a,\ell}$ for which $\overline{T}_a = \sum^{\ell}_{t=1}\overline{v}^{\otimes\omega}_{a,\ell}$ and $\norm{\overline{v}_{a,t}}^2\le R$ for all $t$, then with probability at least $1 - \exp(-\Omega(d))$ over the smoothing, the following holds for the transformation specified by $T^*_1,\ldots,T^*_d$:
    
    For any $\epsilon > 0$, there is an algorithm that takes $\poly(r,\omega,\radius,1/\kappa,1/\rho)^{\poly(\omega,\ell)}\cdot (1/\epsilon)^{O(\omega)}$ samples and runs in $\poly(r,\omega,\log R, \log 1/\rho)^{\poly(\omega,\ell)} + d\cdot (r + \omega)^{O(\omega\ell)}\cdot \text{\emph{polylog}}(R/\rho)$ additional time and parameter learns $\calD$ to error $\epsilon$ (and also solves proper density estimation to Wasserstein error $\epsilon r\sqrt{d}$) with high probability.
\end{theorem}

\begin{proof}[Proof of Theorem~\ref{thm:main_quadratic}]
    By Theorem~\ref{thm:quadratic_to_tensorring}, one can reduce learning the first $d' = \wt{\Theta}(r^2\log(rR/\rho))$ units of the polynomial network to the problem of tensor ring decomposition with parameter $\eta$ using $\poly(r,\radius,\log(d'/\delta),1/\eta)$ samples from the transformation. According to Lemma~\ref{lem:fullysmooth}, Assumption~\ref{assume:tensorring} holds for these first $d'$ units with parameters $\radius = R\sqrt{r} + \rho\sqrt{d'}$ and $\kappa = \Theta(\rho\sqrt{d'/r})$ with probability at least $1 - \exp(-\Omega(d)) = 1 - \exp(-\Omega(r^2))$. If this happens, then by Corollary~\ref{cor:fpt} of Theorem~\ref{thm:main_tensorring}, provided $\eta \le O(\kappa^2/(r{d'}^{3/2})) = O(\rho^2/(r^2{d'}^{1/2}))$, we can recover $\wh{Q}_1,\ldots,\wh{Q}_d$ satisfying $\gaugedist(\brc{Q^*_a},\brc{\wh{Q}_a}) \le \poly(r,R,1/\rho)\cdot \eta^c$ with high probability. Taking $\eta$ to be of order $\poly(1/r,1/R,\rho,\epsilon)$ completes the argument.
\end{proof}

\begin{proof}[Proof of Theorem~\ref{thm:main_lowrank}]
    By Theorem~\ref{thm:lowrank_to_factorization}, one can reduce learning the first $d' = \Theta(\ell(r+\omega))^{\Omega(\omega\ell)}\cdot \log(R/\rho)$ units of the polynomial network to the problem of low-rank factorization with parameter $\eta$ using $\poly(\omega^{\omega},r^{\omega},\radius,\log^{\omega}(d'/\delta),1/\eta)$ samples from the transformation. By Lemma~\ref{lem:compsmooth} applied to the first $d'$ units, Assumption~\ref{assume:tensorring} holds for $\radius = \ell\cdot \Theta(R)^\omega$, $\kappa = \Theta(\sqrt{d'\ell}(\rho^2\omega/r)^{\omega/2})$, $\theta = \Theta(Rr\omega\ell)^{O(\omega\ell)}$, and $\psi = \Theta(\rho/(r\omega))^{\Theta(\omega)}$ with probability at least $1 - \exp(-\Omega(d')) = 1 - \exp(-\Omega(r^{\omega\ell}))$. If this happens, then for $\eta \le \poly(r,\omega,d',\radius,1/\kappa)^{-\poly(\omega,\ell)}$, by Corollary~\ref{cor:fpt2} of Theorem~\ref{thm:main_push} we can recover $\wh{T}_1,\ldots,\wh{T}_d$ for which
    \begin{multline}
        \gaugedist(\brc{T^*_a},\brc{\wh{T}_a}) \le \poly(r^{\omega},d',R^{\omega},1/\rho^{\omega})^{\omega^3} \cdot \left((\poly(d',r^{\omega},1/\rho^{\omega})\cdot\eta)^{O(1/\omega)} + \right. \\
        \left. \poly(r^{\omega},\omega^{\ell},\ell^{\ell},d',R^{\omega},1/\rho^{\omega})^{\ell}\cdot (Rr\omega\ell)^{O(\omega\ell)}\cdot \eta^{1/8}\right)
    \end{multline} with high probability. Taking $\eta = \poly(r,\omega,d',\radius,1/\kappa,1/\rho)^{-\poly(\omega,\ell)}\cdot \epsilon^{\omega}$ completes the argument.
\end{proof}

Note that in the proofs above, it was only necessary for the first $d'$ units of the polynomial network to be smoothed.

\paragraph{Acknowledgments.} Part of this work was done while SC, JL, and AZ were visiting the Simons Institute for the Theory of Computing. The authors would like to thank Sebastien Bubeck and Raghu Meka for illuminating discussions in the early stages of this work. SC was supported in part by NSF Award 2103300. AZ was supported in part by NSF Award CAREER 2203741.

\bibliographystyle{alpha}
\bibliography{biblio}

\appendix

\section{Moments of Transformations and Tensor Decomposition}
\label{app:diagonal}

In this section we elaborate on the connection between tensor decomposition and learning \emph{diagonal} quadratic transformations using method of moments.
% We will argue that this connection suggests that a polynomial-time algorithm for learning worst-case quadratic transformations, even in the diagonal case,  would necessitate making progress on long-standing open problems in learning latent variable models.

Instead of moments, it will be cleaner to work with cumulants. For a collection of random variables $X_1,\ldots,X_m$, define their \emph{joint cumulant} by
\begin{equation}
    \kappa_{\beta}(X_1,\ldots,X_d) \triangleq \sum_{\pi\in\calS_m} (|\pi| - 1)!(-1)^{|\pi|-1} \prod_{C\in\pi} \E*{\prod_{i\in C} X_i},
\end{equation}
where $\pi$ ranges over partitions of $[m]$, $|\pi|$ denotes the number of parts of the partition $\pi$, and the product over $C\in\pi$ ranges over the parts of the partition.

Given a tuple $\mb{\beta}\in\mathbb{Z}^d$ and random variables $Y_1,\ldots,Y_d$, we will let $\kappa_{\beta}(Y_1,\ldots,Y_d)$ denote the joint cumulant of the random variables $Y_1,\ldots,Y_1,Y_2,\ldots,Y_2,\ldots,Y_d,\ldots,Y_d$, where each $Y_a$ appears $\beta_a$ times. It is a standard fact that these are given by coefficients of the \emph{cumulant generating function}, that is, for formal variables $t_1,\ldots,t_d$,
\begin{equation}
    \sum_{\mb{\beta}\in\mathbb{Z}^d} \kappa_{\beta}(Y_1,\ldots,Y_d) t^{\beta_1}_1\cdots t^{\beta_d}_d = \log\E*{\exp\biggl(\sum^d_{a=1} t_a Y_a\biggr)} \label{eq:cgf}
\end{equation}

\begin{lemma}\label{lem:cumulants}
    Let $\calD$ be a quadratic transformation arising from polynomial network $Q^*_1,\ldots,Q^*_d\in\R^{r\times r}$ consisting of diagonal matrices. For every $i\in[r]$, define $v_i\in\R^d$ to be the vector whose $a$-th entry is $(Q^*_a)_{ii}$. Then for any $\mb{\beta}\in\mathbb{Z}^d$, if $z_1,\ldots,z_d$ are random variables corresponding to the coordinates of a sample from $\calD$, then
    \begin{equation}
        \kappa_{\mb{\beta}}(z_1,\ldots,z_d) = \binom{\beta_1+\cdots+\beta_d}{\beta_1 \cdots \beta_d} \frac{2^{\beta_1+\cdots+\beta_d-1}}{\beta_1+\cdots+\beta_d} \biggl(\sum^r_{i=1} v^{\otimes (\beta_1+\cdots+\beta_d)}\biggr)_{\mb{\beta}}.
    \end{equation}
\end{lemma}

\begin{proof}
    We can express the cumulant generating function as
    \begin{equation}
        \log\brk*{\frac{1}{(2\pi)^{r/2}} \int_{\R^r} \exp\biggl(-\frac{1}{2} g^{\top}\biggl(\Id - 2\sum^d_{a=1} t_a Q^*_a\biggr)g\biggr) \, dg} = -\frac{1}{2}\log\det\biggl(\Id - 2\sum^d_{a=1} t_a Q^*_a\biggr).
    \end{equation}
    As $(Q^*_a)_{ii} = (v_i)_a$, we can rewrite the above as
    \begin{align}
        -\frac{1}{2}\sum^r_{i=1} \log\biggl(1 - 2\sum_a t_a (v_i)_a\biggr) &= \sum^r_{i=1} \sum^{\infty}_{\ell = 1} \frac{2^{\ell-1}}{\ell} \biggl(\sum_a t_a (v_i)_a\biggr)^{\ell} \\
        &= \sum^{\infty}_{\ell = 1} \frac{2^{\ell-1}}{\ell} \sum_{\beta_1+\cdots+\beta_d = \ell} \binom{\ell}{\beta_1 \cdots \beta_d} t^{\beta_1}_1\cdots t^{\beta_d}_d \sum_i (v_i)^{\beta_1}_1 \cdots (v_i)^{\beta_d}_d.
    \end{align}
    By \eqref{eq:cgf}, we conclude that for $\mb{\beta} = (\beta_1,\ldots,\beta_d)$,
    \begin{equation}
        \kappa_{\mb{\beta}}(z_1,\ldots,z_d) = \binom{\beta_1+\cdots+\beta_d}{\beta_1 \cdots \beta_d} \frac{2^{\beta_1+\cdots+\beta_d-1}}{\beta_1+\cdots+\beta_d} \sum_i (v_i)^{\beta_1}_1\cdots (v_i)^{\beta_d}_d
    \end{equation}
    as claimed.
\end{proof}

As the cumulants of a joint distribution are an alternative basis for the moments of that distribution, any algorithm for learning diagonal quadratic transformations using moments of $\calD$ up to some degree $m$ (where naively the runtime will scale with $d^m$) must solve the following inverse problem:

\begin{definition}\label{def:tensor}
    Let $v_1,\ldots,v_r$ be unknown vectors in $\R^d$. Given noisy estimates of the tensors $\sum^r_{i=1} v^{\otimes\ell}_a$ for $\ell = 1,\ldots,m$, recover $v_1,\ldots,v_r$ up to error $\epsilon$.\footnote{Here one can also ask for other weaker notions of recovering $v_1,\ldots,v_r$ that correspond to weaker notions of learning for transformations, e.g. improper density estimation.}
\end{definition}

This is essentially tensor decomposition except where the learner gets access to $\sum^r_{i=1} v^{\otimes\ell}_a$ for many choices of $\ell$. There are a number of efficient algorithms for tensor decomposition under various separation assumptions on $v_1,\ldots,v_r$ \cite{anandkumar2014tensor,bhaskara2014smoothed,barak2015dictionary,ma2016polynomial,hopkins2019robust,bhaskara2019smoothed}, but applied in a black-box, these will fail to solve the above inverse problem in the worst case in polynomial time simply because one can construct vectors $\brc{v_1,\ldots,v_r}$ and $\brc{v'_1,\ldots,v'_r}$ (even in one dimension, see e.g. \cite[Lemma A.1]{chen2020learning}) for which $\sum_i v^{\otimes\ell}_i = \sum_i {v'}^{\otimes\ell}_i$ for all $\ell = O(r)$. 

We also note that the recent work of \cite{diakonikolas2020small} gives a highly sophisticated algorithm for constructing a small \emph{cover} over solutions to the above inverse problem\footnote{Technically their work pertains to a version of Definition~\ref{def:tensor} for which $\ell$ can only be even.} which runs in time \emph{quasipolynomial} in $r$ as one must take moments up to degree $m = \mathrm{polylog}(r)$. This suggests that in the worst case, even learning diagonal polynomial pushforwards in $\poly(r)$ time can be quite challenging.

\section{Information-Theoretic Lower Bound}
\label{app:lbd}

In this section we show an exponential lower bound for parameter learning polynomial transformations in the worst case, even when the polynomial network is a single quadratic form!

\begin{theorem}\label{thm:info}
    For $r\in\mathbb{N}$, let $\calC_r$ denote the family of $1$-dimensional degree-2 transformations with seed length $r$ specified by a polynomial network with polynomially bounded operator norm. Any algorithm for parameter learning any distribution from $\calC_r$ to error $O(1)$ requires $\exp(\Omega(r))$ samples.
\end{theorem}

% \begin{fact}\label{fact:chisq}
%     For any $c > 0$ and $m\in\mathbb{N}$, if $g\sim\calN(0,\Id_m)$, then the random variable $c\norm{g}^2$ has probability density function $h_m(x) = \frac{1}{c\cdot2^{m/2}\Gamma(m/2)} (x/c)^{m/2-1}e^{-x/2c}$.
% \end{fact}

Our construction and analysis is reminiscent of existing lower bounds for Gaussian mixture models \cite{moitra2010settling,hardt2015tight,regev2017learning}. We design a pair of transformations $p_1,p_2$ whose moments up to degree $\Theta(r)$ match exactly but which are far in parameter distance. We then convolve $p_1,p_2$ by a reasonably smooth kernel (which can be simulated using a slightly larger polynomial network) and argue that the resulting convolutions $q_1,q_2$ are close in total variation distance.

\subsection{Lower Bound Instance}

We begin by describing our moment-matching construction, which is a straightforward consequence of Borsuk-Ulam:

\begin{lemma}[Lemma A.1 from \cite{chen2020learning}]\label{lem:momentmatch}
    For any $r\ge 2$, there exists $v\in\S^{r-1}$ such that for $a_i \triangleq i + v_i/4$ and $b_i \triangleq i - v_i/4$, $\sum^r_{i=1} a^\ell_i = \sum^r_{i=1} {b}^\ell_i$ for all $1 \le \ell < 2r$.
\end{lemma}

For $a_1,\ldots,a_r, b_1,\ldots, b_r$ from Lemma~\ref{lem:momentmatch}, consider the quadratic forms $Q_1, Q_2\in\R^{(2r+2)\times (2r+2)}$ given by
\begin{equation}
    Q_1 \triangleq \diag(a_1,a_1,\ldots,a_r,a_r,1,1,1,-1,-1,-1) \ \text{and} \ Q_2 \triangleq \diag(b_1,b_1,\ldots,b_r,b_r,1,1,1,-1,-1,-1).
\end{equation}
Let $\calD_1$ and $\calD_2$ denote the 1-dimensional polynomial transformations specified by $Q_1$ and $Q_2$ respectively, and denote their pdf's by $q_1$ and $q_2$ respectively. 

We can think of $q_1$ and $q_2$ as convolutions as follows. Let $p_1$ and $p_2$ denote the densities of the 1-dimensional transformations specified by 
\begin{equation}
    \diag(a_1,a_1,\ldots,a_k,a_k) \  \text{and} \  \diag(b_1,b_1,\ldots,b_k,b_k),
\end{equation}
and let $\nu$ denote the density of the 1-dimensional transformation specified by the quadratic form $\diag(1,1,1,-1,-1,-1)$.
Then we have that
\begin{equation}
    q_1 = p_1\star \nu \ \ \text{and} \ \ q_2 = p_2\star \nu.
\end{equation}
Our goal is to show that $\norm{q_1 - q_2}_1$ is exponentially small.

\subsection{Regularity of \texorpdfstring{$q_i, \nu$}{qi, nu}}

In this section we collect some basic properties of $q_i$ and $\nu$ that we will use to bound $\norm{q_1 - q_2}_1$. First, we observe all quadratic transformations have subexponential tails:

\begin{lemma}\label{lem:chitail}
    For any $t > 2r\norm{Q}^2_{\op}$, $\Pr[g\sim\calN(0,\Id_r)]{g^{\top}Qg \ge t} \le \exp(-\Omega(t / \norm{Q}_{\op} - 2r))$.
\end{lemma}

\begin{proof}
    As $g^{\top}Qg \le \norm{Q}_{\op}\norm{g}^2$, it suffices to bound $\Pr{\norm{g}^2 \ge t / \norm{Q}_{\op}}$. By Fact~\ref{fact:shell}, we have the tail bound $\Pr*{\norm{g}^2 \ge 2r + 2s^2} \le \Pr*{\norm{g}^2 \ge (\sqrt{r} + s)^2} \le \exp(-\Omega(s^2))$ for any $s > 0$, from which the lemma follows by taking $s = (t / (2\norm{Q}_{\op}) - r)^{1/2}$
\end{proof}

We next verify that $q_1$ and $q_2$ have bounded derivatives. For this, we will need the following form for $\nu'$:

\begin{fact}\label{fact:deriv}
    $|\nu'(z)| \le \frac{1}{8\sqrt{\pi}}\cdot |z|^{1/2} e^{-|z|/2}$ for all $z\in\R$.
\end{fact}

\begin{proof}
    As the random variable corresponding to $\nu$ is simply the difference between two independent chi-squared random variables each with 3 degrees of freedom. As the moment generating function for chi-squared random variable with 3 degrees of freedom is given by $M(t) = (1 - 2t)^{-3/2}$, the moment generating function of $\nu$ is given by $(1 - 4t^2)^{-3/2}$. Note that this is precisely the moment generating function of a variance-gamma distribution with parameters $\lambda=3/2, \alpha=1/2, \beta=0, \mu = 0$. We conclude that the pdf of $\nu$ is given by the differentiable function
    \begin{equation}
        \nu(z) = \begin{cases}
            \frac{1}{4\pi} |z|\cdot K_1(|z|/2) & \text{if} \ z \neq 0 \\
            \frac{1}{2\pi} & \text{otherwise}
        \end{cases},
    \end{equation} where $K_1$ denotes the modified Bessel function of the second kind. We can differentiate this function to find that $\nu'(z) = -\frac{1}{8\pi}\cdot |z|\cdot K_0(|z|/2)$ for $z\neq 0$, and $\nu'(0) = 0$. By Eq.~(6.30) from \cite{luke1972inequalities},
    \begin{equation}
        K_0(z) < \frac{16z+7}{16z+9}\cdot e^{-z}(2z/\pi)^{-1/2} \ \ \forall \ z > 0.
    \end{equation}
    The claimed upper bound on $|\nu'(z)|$ follows.
\end{proof}

\begin{lemma}\label{lem:qderiv}
    $|(q_1-q_2)'(x)| \le O(1)$ for all $x\in\R$.
\end{lemma}

\begin{proof}
    Note that $q'_1 = p_1\star \nu'$. By Fact~\ref{fact:deriv}, $|\nu'(z)| \le O(1)$ for all $z$. So
    \begin{equation}
        |q'_1(x)| = \abs*{\int^{\infty}_{-\infty} p_1(x-z) \nu'(z) \, dz} \le O(1)
    \end{equation} as desired. The same bound applies to $|q'_2(x)|$.
\end{proof}

\subsection{Bounding \texorpdfstring{$\norm{q_1 - q_2}_2$}{Norm of q1-q2}}

The main step in bounding $\norm{q_1 - q_2}_1$ is to first bound the $L_2$ distance between $q_1$ and $q_2$.

\begin{lemma}\label{lem:l2}
    $\norm{q_1 - q_2}_2 \le \exp(-\Omega(r))$.
\end{lemma}

\begin{proof}
    By Plancherel, it suffices to bound the $L_2$ distance between the Fourier transforms $\wh{q}_1, \wh{q}_2$. In fact we will even bound the $L_{\infty}$ distance. 
    
    By the expression for the characteristic function of a generalized chi-squared distribution, the Fourier transforms of $q_1, q_2$ are given by
    \begin{equation}
        \wh{q}_1[t] = \frac{1}{(1 + 4t^2)^3}\cdot \prod^r_{j=1} \frac{1}{1 + 4a_j^2 t^2}
        \qquad
        \text{and}
        \qquad
        \wh{q}_2[t] = \frac{1}{(1 + 4t^2)^3}\cdot \prod^r_{j=1} \frac{1}{1 + 4b_j^2 t^2}.
    \end{equation}
    Note that
    \begin{equation}
        \prod^r_{j=1} (1 + 4a_j^2 t^2) - \prod^r_{j=1} (1 + 4b_j^2 t^2) = \sum^r_{\ell=1} (2t)^{2\ell} \left(e_\ell(a^2_1,\ldots,a^2_r) - e_\ell(b^2_1,\ldots,b^2_r)\right) = (2t)^{2r} \biggl(\prod^r_{j=1} a^2_j - \prod^r_{j=1} b^2_j\biggr),
    \end{equation}
    where $e_\ell$ denotes the elementary symmetric polynomial of degree $\ell$ in $r$ variables, and in the second step we used the condition that $\sum^r_{j=1} a^{2\ell}_j = \sum^r_{j=1} a^{2\ell}_j$ for all $\ell < k$, together with the fact that every $e_\ell$ can be expressed as a polynomial in power sum polynomials of degree at most $\ell$.
    
    We also have
    \begin{equation}
        \prod^r_{j=1} (1 + 4a_j^2 t^2) (1 + 4b_j^2 t^2) = \prod^r_{j=1} (1 + 4a_j^2 t^2)(1+4b^2_{r+1-j} t^2) \ge \prod^r_{j=1} (2a_j t + 2b_{r+1-j} t)^2,
    \end{equation}
    where in the second step we the elementary inequality $(1+x)(1+y) \ge (\sqrt{x} +\sqrt{y})^2$ for $x,y\ge 0$.
    
    So for all $t\in\R$,
    \begin{align}
        |\wh{q}_1[t] - \wh{q}_2[t]| &\le \biggl|\prod^r_{j=1} \frac{1}{1 + 4a_j^2 t^2} - \prod^r_{j=1} \frac{1}{1 + 4b_j^2 t^2}\biggr| \\
        &\le \biggl|\frac{(2t)^{2r}\left(\prod^r_{j=1} a^2_j - \prod^r_{j=1} b^2_j\right)}{\prod^r_{j=1} (2a_j t + 2b_{r+1-j} t)^2}\biggr| = \frac{\abs*{\prod^r_{j=1} a^2_j - \prod^r_{j=1} b^2_j}}{\prod^r_{j=1} (a_j + b_{r+1-j})^2} \label{eq:numdenom}
    \end{align}
    Note that by the definition of $\brc{a_j,b_j}$ in Lemma~\ref{lem:momentmatch}, $a_j + b_{r+1-j} > r$ for all $j\in[r]$, so the denominator above is at least $r^r$. On the other hand, because $|a_j| < j+1$ for all $j\in[r]$, $\prod^r_{j=1} a^2_j \le (r+1)!^2$, and similarly for $\prod^r_{j=1} b^2_j$, so the numerator in \eqref{eq:numdenom} is at most $2(r+1)!^2$. The lemma is complete upon noting that $\frac{2(r+1)!^2}{r^r} \le \exp(-\Omega(r))$.
\end{proof}

\subsection{Bounding Total Variation Distance}

Here we finally upper bound $d_{\mathrm{TV}}(\calD_1,\calD_2)$ and complete the proof of Theorem~\ref{thm:info}.

\begin{lemma}\label{lem:tvbound}
    $d_{\mathrm{TV}}(\calD_1,\calD_2) \le \exp(-\Omega(r))$.
\end{lemma}

\begin{proof}
    We reduce upper bounding $\norm{q_1-q_2}_1$ to upper bounding $\norm{q_1 - q_2}_2$ as follows. Let $T = \Omega(r^2)$ so that $\Pr[x\sim \calD_i]*{|x| \ge T} \le \exp(-\Omega(r))$ for $i = 1,2$. Let $f(x) \triangleq q_1(x) - q_2(x)$. By the fundamental theorem of calculus, for all $z > -T$ we have
    \begin{equation}
        \frac{1}{3}|f(z)^3| = \left|\frac{1}{3}f(-T)^3 + \int^z_{-T} f(x)^2 f'(x) \, dx\right| \le \frac{1}{3}|f(-T)^3| + O(1)\cdot \int^T_{-T} f(x)^2 \, dx,
    \end{equation}
    where in the last step we used Lemma~\ref{lem:qderiv} and triangle inequality. To control $f(-T)^3$, observe that
    \begin{align}
        |f(-T)| &= \abs*{\left((q_1 - q_2)\star\nu\right)(-T)} = \abs*{\int^{\infty}_{-\infty} (p_1-p_2)(z) \cdot \nu(-T-z)} \\
        &\le \frac{1}{8\sqrt{\pi}}\int^\infty_{-T/2} |(p_1 - p_2)(z)| \cdot (T+z)^{1/2} e^{-(T+z)/2}\, dz + O(1)\cdot \int^{-T/2}_{-\infty} (p_1(z) + p_2(z)) \, dz \\
        &\le \exp(-\Omega(T)) + \exp(-\Omega(r)) \le \exp(-\Omega(r)).
    \end{align}
    We conclude by Lemma~\ref{lem:chitail} and \ref{lem:l2} that
    \begin{align}
        \int^{\infty}_{\infty} |q_1(x) - q_2(x)| &\le \exp(-\Omega(r)) + \int^T_{-T} |q_1(x) - q_2(x)| \\
        &\le \exp(-\Omega(r)) + 2T\max_{z\in[-T,T]} |f(z)| \\
        &\le \exp(-\Omega(r)) + T\left(|f(-T)|^3 + O(\norm{q_1-q_2}^2_2)\right)^{1/3} \le \exp(-\Omega(r))
    \end{align}
    as claimed.
\end{proof}

\begin{proof}[Proof of Theorem~\ref{thm:info}]
    We first compute the parameter distance between $\calD_1$ and $\calD_2$, which is given by
    \begin{equation}
        \min_{\pi\in\calS_r} 2|a_j - b_{\pi(j)}|.
    \end{equation}
    As $a_j,b_j\in [j-1/4,j+1/4]$ for every $j\in[r]$, the minimizing choice of $\pi$ is given by the identity permutation, so the parameter distance is $2\sum^r_{j=1} 2(|v_i|/4) = \norm{v}_1 \ge 1$ for $v\in\S^{r-1}$ from Lemma~\ref{lem:momentmatch}. As $d_{\mathrm{TV}}(\calD_1,\calD_2) = \exp(-\Omega(r))$ by Lemma~\ref{lem:tvbound}, we conclude that $\exp(\Omega(r))$ samples are needed to distinguish between $\calD_1$ and $\calD_2$.
\end{proof}

\section{Finding Non-Degenerate Combinations}
\label{app:forcegap}

In this section we give a randomized algorithm for the following task: given an estimate of the Gram matrix of a collection of unknown symmetric matrices $S^*_1,\ldots,S^*_d\in\R^{r\times r}$, construct a pair of $\lambda,\mu\in\S^{d-1}$ which are non-degenerate combinations of $S^*_1,\ldots,S^*_d$ in the sense of Definition~\ref{def:nondeg}. This subroutine is essential to our algorithms in Sections~\ref{sec:tensorring} and \ref{sec:push}.
% \begin{enumerate}
%     \item Construct a linear combination $S^*_{\lambda} \triangleq \sum^d_{a=1} \lambda_a S^*_a$ which has non-negligible minimum eigengap.
%     \item Construct another linear combination $S^*_{\mu} \triangleq \sum^d_{a=1} \mu_a S^*_a$ which, when put in the eigenbasis of $S^*_{\lambda}$, has all entries bounded away from zero.
% \end{enumerate}
We summarize the main guarantee of this subroutine as follows:

\begin{lemma}\label{lem:forcegap}
    Let $d\ge \binom{r+1}{2}$ and let $S^*_1,\ldots,S^*_d\in\R^{r\times r}$ be symmetric matrices. Let $G\in\R^{d\times d}$ denote the symmetric matrix whose $(a,b)$-th entry is $\iprod{S^*_a,S^*_b}$. Let $\wh{G}\in\R^{d\times d}$ denote a rank-$\binom{r+1}{2}$ symmetric matrix satisfying
    \begin{equation}
        \norm*{\wh{G} - G}_{\max} \le \epsgram \ \ \forall \ a,b\in[d] \label{eq:hatGdef}
    \end{equation}
    for some $\epsgram > 0$. Let $H\in\R^{d\times\binom{r+1}{2}}$ denote the matrix satisfying
    \begin{equation}
        H_{a,(i,j)} = (S^*_a)_{ij} \ \ \forall \ a\in[d], 1\le i \le j\le r. \label{eq:Hdef}
    \end{equation}
    
    If $\epsgram\le c \sigma_{\min}(H)^2/(rd^{3/2})$ for sufficiently small constant $c > 0$, then given as input a symmetric matrix $\wh{G}\in\R^{d\times d}$ satisfying \eqref{eq:hatGdef}, with probability at least $2/3$ {\sc GapCombo} (Algorithm~\ref{alg:eigengap}) outputs $\lambda,\mu\in\S^{d-1}$ which are $\gap$-non-degenerate combinations of $S^*_1,\ldots,S^*_d$ for $\gap = \sigma_{\min}(H)/\poly(r)$.
\end{lemma}

\noindent To prove Lemma~\ref{lem:forcegap}, our strategy will be to use $\wh{G}$ to construct an approximately orthonormal basis for the space of symmetric $r\times r$ matrices using linear combinations of $S^*_1,\ldots,S^*_d$. Once we have this, we can simply take a Gaussian linear combination of these basis elements, and this will be a linear combination of $S^*_1,\ldots,S^*_d$ which is approximately distributed as a Gaussian symmetric matrix. The minimum eigengap of such a matrix is known to be lower bounded with high probability (Theorem~\ref{thm:gaps}), yielding the desired linear combination $S^*_{\lambda}$. For the other linear combination $S^*_{\mu}$, the property that its entries all have non-negligible magnitude will follow by standard anticoncentration.

To produce an approximately orthonormal basis, we first need to spectrally bound $\wh{G}$ to ensure that this can be done in a well-conditioned fashion.

\begin{lemma}\label{lem:Gpsd}
    $\sigma_{\binom{r+1}{2}}(\wh{G}) \ge \sigma_{\min}(H)^2 - \epsgram\cdot d$.
\end{lemma}

\begin{proof}
    Let $D\in\R^{\binom{r+1}{2}\times\binom{r+1}{2}}$ denote the diagonal matrix with diagonal entries indexed by $1 \le i\le j \le r$ whose $(i,j)$-th diagonal entry is $1$ if $i = j$ and $2$ otherwise. Observe that $HDH^{\top} = G$. So $\sigma_{\binom{r+1}{2}}(G) \ge \sigma_{\binom{r+1}{2}}(H)^2 = \sigma_{\min}(H)^2$, whereas $\norm{\wh{G} - G}_{\op} \le \norm{\wh{G} - G}_F < \epsgram\cdot d$. By Weyl's inequality, we conclude that $\sigma_{\binom{r+1}{2}}(\wh{G}) \ge \sigma^2 - \epsgram\cdot d$.
\end{proof}

As $G$ is rank-$\binom{r+1}{2}$, it will be convenient to work with the best rank-$\binom{r+1}{2}$ approximation of $\wh{G}$:

\begin{lemma}\label{lem:lowrank}
    Let $\wt{G}$ denote the best rank-$\binom{r+1}{2}$ approximation of $\wh{G}$. Then $\norm{\wt{G} - G}_F \le O(\epsgram\cdot d^{3/2})$.
\end{lemma}

\begin{proof}
    As $\sigma_{\binom{r+1}{2}+1}(G),\ldots,\sigma_d(G) = 0$, by Weyl's inequality we conclude that $\sigma_{\binom{r+1}{2}+1}(\wh{G}),\ldots,$ $\sigma_d(\wh{G}) \le \norm{\wh{G} - G}_{\op} \le \epsgram\cdot d$. So upon projecting out the corresponding eigenvectors of $\wh{G}$ to get $\wt{G}$, we conclude that $\norm{\wt{G} - \wh{G}}_F \le \epsgram\cdot d^{3/2}$, and the lemma follows by triangle inequality.
\end{proof}

Denote the eigendecomposition of $\wt{G}$ by $U\Sigma U^{\top}$ for $U,\Sigma\in\R^{d\times d}$, where the diagonal entries of $\Sigma$ are sorted in nondecreasing order. Let $\Sigma'\in\R^{d\times\binom{r+1}{2}}$ denote the first $\binom{r+1}{2}$ columns of $\Sigma$. We now show how to use the bounds from Lemma~\ref{lem:Gpsd} and Lemma~\ref{lem:lowrank} to construct an approximately orthonormal basis for the space of $r\times r$ symmetric matrices using linear combinations of rows of $H$.

\begin{lemma}\label{lem:almostisotropic}
    Suppose $\epsgram < \sigma_{\min}(H)^2/d$. Define $\wt{H} \triangleq U\Sigma''\in\R^{d\times \binom{r+1}{2}}$, where $\Sigma''$ denotes the entrywise square root of $\Sigma'$. For every $1\le i \le j \le r$, define 
    \begin{equation}
        w^{(ij)} \triangleq \wt{H}(\wt{H}^{\top}\wt{H})^{-1}e_{ij}\in \R^{d},
    \end{equation}
    where $e_{ij}$ denotes the $(i,j)$-th standard basis vector in $\R^{\binom{r+1}{2}}$. Then for all sorted $1\le i\le j\le r$ and $1\le i'\le j'\le r$,
    \begin{equation}
        \iprod*{H^{\top}w^{(ij)}, H^{\top}w^{(i'j')}} = \bone{i = j} \pm  (\sigma_{\min}(H)^2 - \epsgram\cdot d)^{-1}\cdot\epsgram\cdot d^{3/2}.
    \end{equation}
\end{lemma}

\begin{proof}
    First note that $\Sigma''$ is well-defined: by Lemma~\ref{lem:Gpsd}, $\Sigma'_{ii} \ge \sigma^2 - \epsgram\cdot d$, so by the assumed bound on $\epsgram$ we know $\Sigma'_{ii} > 0$ for all $i$.
    
    Because
     \begin{equation}
        H^{\top}w^{(ij)} = H^{\top}\wt{H}(\wt{H}^{\top}\wt{H})^{-1}e_{ij},
    \end{equation}
    we get that
    \begin{equation}
        \iprod*{H^{\top} w^{(ij)}, H^{\top} w^{(i'j')}} = e^{\top}_{ij} \left((\wt{H}^{\top}\wt{H})^{-1}\wt{H}^{\top}G \wt{H}(\wt{H}^{\top}\wt{H})^{-1}\right) e_{i'j'}.
    \end{equation}
    To establish the lemma, we will upper bound the operator norm of
    \begin{equation}
        (\wt{H}^{\top}\wt{H})^{-1}\wt{H}^{\top} G \wt{H}(\wt{H}^{\top}\wt{H})^{-1} - \Id. \label{eq:HHHdiff}
    \end{equation}
    Note that $\wt{G} = \wt{H}\wt{H}^{\top}$, so
    \begin{equation}
        (\wt{H}^{\top}\wt{H})^{-1}\wt{H}^{\top}\wt{G}\wt{H}(\wt{H}^{\top}\wt{H})^{-1} = \Id.
    \end{equation}
    Also, note that
    \begin{equation}
        \wt{H}(\wt{H}^{\top}\wt{H})^{-1} = U \Sigma'' (\Sigma'')^{-2}.
    \end{equation}
    So by Lemma~\ref{lem:lowrank},
    \begin{equation}
        \norm*{(\wt{H}^{\top}\wt{H})^{-1}\wt{H}^{\top}(G - \wt{G}) \wt{H}(\wt{H}^{\top}\wt{H})^{-1}}_{\op} \le \norm{\Sigma'' (\Sigma'')^{-2}}^2_{\op}\cdot \epsgram \cdot d^{3/2}. \label{eq:bigproduct}
    \end{equation}
    By Weyl's inequality and Lemma~\ref{lem:Gpsd},
    \begin{equation}
        \norm{\Sigma''(\Sigma'')^{-2}}^2_{\op} = \sigma_{\binom{r+1}{2}}(\wt{G}) \le (\sigma_{\min}(H)^2 - \epsgram\cdot d)^{-1}. \label{eq:sigmaprimecond}% \le \Theta(r\omega/\rho)^{\Theta(\omega)}. 
    \end{equation}
    % We also have
    % \begin{equation}
    %     \norm{\Sigma''}^2_{\op} = \norm{\wt{G}}_{\op} \le \norm{G}_{\op} + \epsgram \cdot d^{3/2} \le 2\norm{H}^2_{\op} + \epsgram \cdot d^{3/2}, \label{eq:sigmaop} % \le O\left(\sum_{a\in S} \norm{F^*_a}^2_F\right) \le |S| \ell^4 O(R)^{4\omega}, \label{eq:sigmaop}
    % \end{equation}
    % where we again used the fact that $G = HDH^{\top}$ for $D$ defined in the proof of Lemma~\ref{lem:Gpsd}.
    % where in the last step we used that by Lemma~\ref{lem:compsmooth}, $\norm{v^*_{a,t}} \le O(R)$ for all $a,t$ with high probability. 
    Substituting \eqref{eq:sigmaprimecond} into \eqref{eq:bigproduct} yields the desired bound on the operator norm of \eqref{eq:HHHdiff}.
\end{proof}

We now verify that a Gaussian linear combination of the basis elements that were constructed in Lemma~\ref{lem:almostisotropic} is distributed approximately as a Gaussian symmetric matrix and therefore has non-negligible minimum eigengap.

\begin{lemma}\label{lem:findcombo}
    Suppose $\epsgram\le c \sigma_{\min}(H)^2/(rd^{3/2})$ for sufficiently small constant $c > 0$. Let $\wt{H}$, $\brc{w^{(ij)}}$ be as defined in Lemma~\ref{lem:almostisotropic}. Then for Gaussians $\brc{g_{ij}, g'_{ij}}_{1\le i \le j \le r}$ sampled independently from $\calN(0,1)$, with probability $7/10$ over the randomness of $\brc{g_{ij}, g'_{ij}}$, the following holds:
    \begin{enumerate}
        \item The matrix \begin{equation}
            \sum_{a\in S}\biggl(\sum_{1\le i \le j \le r} g_{ij} w^{(ij)}_a\biggr) S^*_a \label{eq:glam}
        \end{equation} 
        has minimum eigengap at least $\Omega(r^{-12})$.
        \item If the matrix in \eqref{eq:glam} has eigendecomposition $V^{\top}\Lambda V$, then every entry of matrix \begin{equation}
            \sum_{a\in S}\biggl(\sum_{1\le i \le j \le r} g'_{ij} w^{(ij)}_a\biggr) VS^*_aV^{\top} \label{eq:glamprime}
        \end{equation}
        is lower bounded in magnitude by $\Omega(r^{-2})$.
    \end{enumerate}
\end{lemma}

\begin{proof}
    Let $\calD$ denote the distribution of the matrix in \eqref{eq:glam} with respect to the randomness of $\brc{g_{ij}}$. Note that by Lemma~\ref{lem:almostisotropic} and the assumed bound on $\epsgram$, if $C$ denotes the covariance matrix for the distribution over the diagonal and upper-triangular entries of \eqref{eq:glam}, then $\norm{C - \Id}_F \le \epsilon'$ for arbitrarily small constant $c'$ depending on $c$. Therefore, by Theorem~\ref{thm:tvgaussian}, $d_{\mathrm{TV}}(\calN(0,C),\calN(0,\Id)) \le O(\epsilon')$, so in particular, $d_{\mathrm{TV}}(\calD,\calG_r) \le O(\epsilon')$, where $\calG_r$ denotes the distribution over $r\times r$ symmetric matrices whose diagonal and upper-triangular entries are all independent draws from $\calN(0,1)$. But by Theorem~\ref{thm:gaps} applied to $M = 0$, a matrix sampled from $\calG_r$ has minimum eigengap at least $r^{-12}$ with probability at least $3/4$. By our bound on $d_{\mathrm{TV}}(\calD,\calG_r)$, \eqref{eq:glam} therefore has such an eigengap with probability at least $3/4 - \epsilon' > 0.74$, establishing the first part of the lemma.
    
    For the second part of the lemma, note that the marginal distribution on the matrix in \eqref{eq:glamprime} is given by $\calN(0,VCV^{\top})$. Because $\norm{VCV^{\top} - \Id}_F = \norm{C - \Id}_F \le \epsilon'$, this distribution also has total variation distance $O(\epsilon')$ from $\calG_r$. So by standard Gaussian anticoncentration, with high probability every entry's magnitude is lower bounded in magnitude by $\Omega(1/r^2)$ with probability $0.99 - \epsilon'$. The lemma follows by a union bound.
\end{proof}

\begin{algorithm2e}
\DontPrintSemicolon
\caption{\textsc{FindCombo}($\wh{G}$)}
\label{alg:eigengap}
    \KwIn{Approximate Gram matrix $\wh{G}\in\R^{d\times d}$ for $S^*_1,\ldots,S^*_d\in\R^{r\times r}$}
    \KwOut{$\lambda,\mu\in\S^{d-1}$ for which $S^*_{\lambda}$ and $S^*_{\mu}$ satisfy Lemma~\ref{lem:forcegap}}
        $\wt{G}\gets$ best rank-$\binom{r+1}{2}$ approximation of $\wh{G}$.\;
        Compute eigendecomposition $U\Sigma U^{\top}$ of $\wt{G}$, with $\Sigma$'s entries sorted in nondecreasing order.\;
        Let $\Sigma'\in\R^{d\times\binom{r+1}{2}}$ denote first $\binom{r+1}{2}$ columns of $\Sigma$.\;
        Let $\Sigma''$ denote the entrywise square root of $\Sigma'$.\;
        $\wt{H}\gets U\Sigma''$.\;
        $w^{(ij)}\gets \wt{H}(\wt{H}^{\top}\wt{H})^{-1}e_{ij}$ for every $1 \le i \le j \le r$.\;
        Sample $\brc{g_{ij}, g'_{ij}}_{1 \le i \le j \le r}$ independently from $\calN(0,1)$. \;
        $h_a \gets \sum_{1\le i \le j \le r} g_{ij} w^{(ij)}_a$ for every $a\in[d]$. \label{step:defineh}\;
        $h'_a \gets \sum_{1\le i \le j \le r} g'_{ij} w^{(ij)}_a$ for every $a\in[d]$. \label{step:definehprime}
        \Return{$h / \norm{h}, h' / \norm{h'}$}.\;
\end{algorithm2e}

Lemma~\ref{lem:forcegap} follows easily from the preceding ingredients.

\begin{proof}[Proof of Lemma~\ref{lem:forcegap}]
    By Lemma~\ref{lem:findcombo}, for the vectors $h,h'$ constructed in Step~\ref{step:defineh} and Step~\ref{step:definehprime} of Algorithm~\ref{alg:eigengap}, with probability at least $7/10$ we have that $\sum_a h_a S^*_a$ has minimum eigengap at least $\Omega(r^{-12})$ and $\sum_a h'_a VS^*_aV^{\top}$ has entries lower bounded in magnitude by $\Omega(r^{-2})$, where $V^{\top}\Lambda V$ is the eigendecomposition of $\sum_a h_a S^*_a$. 
    
    With probability $1 - o(1)$, $\sum_{i\le j} g^2_{ij} \le r^2$ and $\sum_{i\le j} {g'}^2_{ij} \le r^2$; henceforth condition on this event. By Cauchy-Schwarz,
    \begin{equation}
        \norm{h}^2 \le r^2 \sum_{a,i,j} \norm{w^{(ij)}_a}^2 = r^2\norm{\wt{H}(\wt{H}^{\top}\wt{H})^{-1}}^2_F = r^2\norm{\Sigma''^{-1}}^2_F \ge r^3(\sigma_{\min}(H)^2 - \epsgram\cdot d)^{-1},
    \end{equation}
    where in the last step we used \eqref{eq:sigmaprimecond}, and the same bound holds for $\norm{h'}^2$. By the assumed bound on $\epsgram$, this is at least $\Omega(r^3\sigma_{\min}(H)^{-1})$. So for $\lambda\triangleq h/\norm{h}$ and $\mu \triangleq g/\norm{g}$ we conclude that $\sum_a \lambda_a S^*_a$ has minimum eigengap at least $\sigma_{\min}(H)/\poly(r)$ and $\sum_a \mu_a VS^*_a V^{\top}$ has entries lower bounded in magnitude by $\sigma_{\min}(H)/\poly(r)$.
\end{proof}

\section{\texorpdfstring{$W$}{W} Comes From an \texorpdfstring{$r\times r$}{r by r} Rotation}
\label{app:heuristic_rotation}

In this section, we use Lemma~\ref{lem:simpler_identity} from our SoS analysis for tensor ring decomposition to give an SoS proof that the auxiliary matrix $W$ constructed in Section~\ref{sec:auxW} from the analysis of Program~\ref{program:basic} behaves like the Kronecker power of some $r\times r$ orthogonal matrix. While we do not explicitly use this in the analysis in Section~\ref{sec:tensorring}, it may be helpful to the reader for understanding why the identity from Lemma~\ref{lem:simpler_identity} is crucial to ensuring identifiability in tensor ring decomposition.

We begin by showing that the $2\times 2$ minors of any column of $W$ approximately vanish.

\begin{lemma}\label{lem:2by2}
    For any $i,j\in[r]$, there is a degree-48 SoS proof using the constraints of Program~\ref{program:basic} that
    \begin{equation}
        \Tr(W^{ij} {W^{ij}}^{\top})^2 \le \norm{W^{ij} {W^{ij}}^{\top}}^2_F + O(\epsrel\sqrt{r} + \epsnorm).
    \end{equation}
    In particular, this implies that for all $s,t,s',t'\in[r]$,
    \begin{equation}
        W^{ij}_{st}W^{ij}_{s't'} = W^{ij}_{st'} W^{ij}_{s't} \pm O(\sqrt{\epsrel} \cdot \sqrt[4]{r} + \sqrt{\epsnorm}) \label{eq:minorvanish}
    \end{equation}
\end{lemma}

\begin{proof}
    As $W^{ij} {W^{ij}}^{\top} \approx_{O(\epsrel^2)} W^{ii} W^{ij} {W^{ij}}^{\top}$ in degree-48 SoS by Lemma~\ref{lem:simpler_identity}, Part~\ref{shorthand:bothsides} of Fact~\ref{fact:shorthand}, and Lemma~\ref{lem:Wnorm1}, we get that $\Tr(W^{ij} {W^{ij}}^{\top}) = \Tr(W^{ii} W^{ij} {W^{ij}}^{\top}) \pm O(\epsrel\sqrt{r})$. Squaring both sides of this and noting that $\Tr(W^{ii} W^{ij} {W^{ij}}^{\top})^2 \le \norm{W^{ii}}^2_F \norm{W^{ij} {W^{ij}}^{\top}}^2_F \le O(1)$ in degree-48 SoS, we conclude that $\Tr(W^{ij} {W^{ij}}^{\top})^2 = \Tr(W^{ii} W^{ij} {W^{ij}}^{\top})^2 \pm O(\epsrel\sqrt{r})$.
    
    We thus have
    \begin{align}
        \Tr(W^{ij} {W^{ij}}^{\top})^2 &\le \Tr(W^{ii} W^{ij} {W^{ij}}^{\top})^2 + O(\epsrel) \\
        &\le \norm{W^{ii}}^2_F  \norm{W^{ij}{W^{ij}}^{\top}}^2_F + O(\epsrel\sqrt{r}) \\
        &\le \norm{W^{ij}{W^{ij}}^{\top}}^2_F + O(\epsrel\sqrt{r} + \epsnorm),
    \end{align} where the first step follows by Lemma~\ref{lem:simpler_identity}, the second by Cauchy-Schwarz, and the third by Lemma~\ref{lem:Wnorm1}, completing the proof of the first part of the lemma. 
    
    For the second part of the lemma, apply Fact~\ref{fact:cauchy_schwarz} with $x$'s and $y$'s both given by the entries of $W^{ij}$ to get
    \begin{align}
        \Tr(W^{ij} {W^{ij}}^{\top})^2 - \norm{W^{ij} {W^{ij}}^{\top}}^2_F &= \sum_{s,u} \biggl(\sum_t W^{ij}_{st} W^{ij}_{ut}\biggr)^2 - \biggl(\sum_{s,t} (W^{ij}_{s,t})^2\biggr)^2
        \\
        % &= \frac{1}{2}\sum_{(r,s) \neq (r',s')} (W^{ij}_{rs} W^{ij}_{r's'} - W^{ij}_{rs'} W^{ij}_{r's})^2 \\
        &= 2\sum_{s<s'; t<t'} (W^{ij}_{st} W^{ij}_{s't'} - W^{ij}_{st'} W^{ij}_{s't})^2.
    \end{align}
    This implies that for all $s<s'$ and $t<t'$, $W^{ij}_{st}W^{ij}_{s't'} = W^{ij}_{st'} W^{ij}_{s't} \pm O(\sqrt{\epsrel}\cdot\sqrt[4]{r} + \sqrt{\epsnorm})$ as claimed.
\end{proof}

\noindent We can use that the $2\times 2$ minors of $W^{ij}$ approximately vanish to conclude that $W^{ij}$ can be written as a certain outer product.

\begin{lemma}\label{lem:outerproduct}
   For $\eta \le \poly(d,r,\beta,\radius,1/\kappa)^{-1}$ sufficiently small, for any $i,j,\ell\in[r]$ there is a degree-96 SoS proof using the constraints of Program~\ref{program:basic} that
    \begin{align}
        \norm{W^{ij}_{:,\ell}}^2 W^{ij} {W^{ij}}^{\top} & \approx_{O((\epsrel\sqrt{r} +\epsnorm)r^5 + \epsnorm^2)} W^{ij}_{:,\ell} {W^{ij}_{:,\ell}}^{\top} \label{eq:exchange1} \\
        \norm{W^{ij}_{\ell,:}}^2 {W^{ij}}^{\top} W^{ij} &\approx_{O((\epsrel\sqrt{r} +\epsnorm)r^5 + \epsnorm^2)} W^{ij}_{\ell,:} {W^{ij}_{\ell,:}}^{\top} \label{eq:exchange2}
    \end{align}
\end{lemma}

\begin{proof}
    For any $a,b\in[r]$,
    \begin{equation}
        \left(\norm{W^{ij}_{:,\ell}}^2 W^{ij} {W^{ij}}^{\top} - \norm{W^{ij}}^2_F W^{ij}_{:,\ell} (W^{ij}_{:,\ell})^{\top} \right)_{a,b} = \sum_{k,c\in[r]} \left(W^{ij}_{k\ell}\right)^2 W^{ij}_{ac} W^{ij}_{bc} - \left(W^{ij}_{kc}\right)^2W^{ij}_{a\ell}W^{ij}_{b\ell}. \label{eq:rewrite_diff}
    \end{equation}
    Defining $\delta_{akc} \triangleq W^{ij}_{k\ell} W^{ij}_{ac} - W^{ij}_{kc} W^{ij}_{a\ell}$ so that by Lemma~\ref{lem:2by2} there is a degree-48 SoS proof that $\delta^2_{akc} \le O(\epsrel + \epsnorm\sqrt{r})$, we can rewrite each summand on the right-hand side of \eqref{eq:rewrite_diff} as
    \begin{equation}
        \delta_{bkc} W^{ij}_{kc} W^{ij}_{a\ell} + \delta_{akc} W^{ij}_{kc} W^{ij}_{b\ell} + \delta_{akc} \delta_{bkc}.
    \end{equation}
    We have
    \begin{multline}
        \sum_{a,b}\biggl(\sum_{k,c\in[r]} \delta_{bkc} W^{ij}_{kc} W^{ij}_{a\ell} + \delta_{akc} W^{ij}_{kc} W^{ij}_{b\ell} + \delta_{akc} \delta_{bkc} \biggr)^2 \\ 
        \le 3\sum_{a,b} \biggl(\sum_{k,c} \delta_{bkc} W^{ij}_{kc} W^{ij}_{a\ell}\biggr)^2 + 3\sum_{a,b} \biggl(\sum_{k,c} \delta_{akc} W^{ij}_{kc} W^{ij}_{b\ell}\biggr)^2 + 3\sum_{a,b} \biggl(\sum_{k,c} \delta_{akc} \delta_{bkc}\biggr)^2, \label{eq:errorterms_ugly}
    \end{multline}
    In degree-96 SoS, we can bound
    \begin{align}
        \sum_{a,b} \biggl(\sum_{k,c} \delta_{bkc} W^{ij}_{kc} W^{ij}_{a\ell}\biggr)^2 &\le \sum_{a,b} (W^{ij}_{a\ell})^2 \biggl(\sum_{k,c} \delta^2_{bkc}\biggr)\biggl(\sum_{k,c} (W^{ij}_{a\ell})^2\biggr) \\
        &\le O((\epsrel\sqrt{r} + \epsnorm)r^4)\cdot \sum_{a,b} (W^{ij}_{a\ell})^2 \\
        &\le O((\epsrel\sqrt{r} + \epsnorm)r^5) \label{eq:deltacross}
    \end{align}
    using Lemma~\ref{lem:Wnorm1} and our bound on $\delta^2_{bkc}$, and we can bound the second term on the right-hand side of \eqref{eq:errorterms_ugly} in an identical fashion. In degree-96 we can also bound
    \begin{equation}
        \sum_{a,b} \biggl(\sum_{k,c} \delta_{akc} \delta_{bkc}\biggr)^2 \le \sum_{a,b} \biggl(\sum_{k,c} \delta^2_{akc}\biggr) \biggl(\sum_{k,c} \delta^2_{bkc}\biggr) \le O((\epsrel^2 r + \epsnorm^2)r^6) \ll O((\epsrel\sqrt{r} + \epsnorm)r^5), \label{eq:delta2}
    \end{equation}
    where in the last step we used the assumed bound on $\eta$.
    Putting everything together, we conclude that
    \begin{equation}
        \norm*{\norm{W^{ij}_{:,\ell}}^2 W^{ij} {W^{ij}}^{\top} - \norm{W^{ij}}^2_F W^{ij}_{:,\ell} (W^{ij}_{:,\ell})^{\top}}^2_F \le O((\epsrel\sqrt{r} +\epsnorm)r^5).
    \end{equation}
    Finally, because $\norm{W^{ij}}^2_F = 1 \pm O(\epsnorm)$, we also know $\norm{W^{ij}}^2_F W^{ij}_{:,\ell} (W^{ij}_{:,\ell})^{\top} \approx_{O(\epsnorm^2)} W^{ij}_{:,\ell}(W^{ij}_{:,\ell})^{\top}$, completing the proof of the first part of the lemma. The second part follows analogously.
    % \begin{align}
    %     \left(\norm{W^{ij}_{:,\ell}}^2 W^{ij} {W^{ij}}^{\top}\right)_{a,b} &= \sum_{k,c\in[r]} \left(W^{ij}_{k\ell}\right)^2 W^{ij}_{ac} W^{ij}_{bc} \\
    %     &= \sum_{k,c\in[r]} \left(W^{ij}_{kc} W^{ij}_{a\ell} + O(\sqrt{\epsrel} + \sqrt{\epsnorm})\right) \left(W^{ij}_{kc} W^{ij}_{b\ell} + O(\sqrt{\epsrel} + \sqrt{\epsnorm})\right) \\
    %     &= \sum_{k,c\in[r]} \left(W^{ij}_{kc}\right)^2W^{ij}_{a\ell}W^{ij}_{b\ell} \pm O(\sqrt{\epsrel}\\
    %     &= \norm{W^{ij}}^2_F \cdot W^{ij}_{a\ell} W^{ij}_{b\ell},
    % \end{align}
    % where in the second step we used \eqref{eq:minorvanish} to conclude that $W^{ij}_{k\ell} W^{ij}_{ac} = W^{ij}_{kc} W^{ij}_{a\ell}$ and $W^{ij}_{k\ell} W^{ij}_{bc} = W^{ij}_{kc} W^{ij}_{b\ell}$. 
    % The proof of \eqref{eq:exchange1} is complete upon recalling that $\norm{W^{ij}}^2_F = 1$ by Lemma~\ref{lem:Wnorm1}. The proof of \eqref{eq:exchange2} follows completely analogously.
\end{proof}

Finally, we show how to leverage the outer product structure of $W^{ij}$, together with Lemma~\ref{lem:simpler_identity}, to deduce that the vectors in the outer product satisfy certain orthonormality relations.

\begin{lemma}\label{lem:swap}
    Define
    \begin{equation}
        \epsswap \triangleq \epsrel^{1/4}r^{11/8} + \epsnorm^{1/4}r^{5/4}.
        % (\sqrt{(\epsrel\sqrt{r}+\epsnorm)r^5 + \epsnorm^2} + \epsnorm + \epsrel\sqrt{r})^{1/2}
    \end{equation}
    There is a degree-96 SoS proof using the constraints of Program~\ref{program:basic} that
    \begin{equation}
        \iprod{W^{ij}_{\ell,:},W^{j'k}_{:,\ell}}^2 = \norm{W^{ij}_{\ell,:}}^2 \cdot \norm{W^{j'k}_{:,\ell}}^2 \cdot \bone{j = j'} \pm \epsswap^2.
    \end{equation}
    for any $i,j,j',k,\ell\in[r]$. This implies that
    \begin{equation}
        W^{ij}_{\ell a} W^{jk}_{b \ell} = W^{ij}_{\ell b} W^{jk}_{a \ell} \pm \epsswap \label{eq:exchange_scalar}
    \end{equation} for all $i,j,k,\ell,a,b\in[r]$.
\end{lemma}

\begin{proof}
    By Lemma~\ref{lem:simpler_identity}, there is a degree-24 SoS proof that $W^{ij} W^{j'k} {W^{j'k}}^{\top} {W^{ij}}^{\top} \approx_{O(\epsrel^2)} \bone{j = j'}\cdot W^{ik} {W^{ik}}^{\top}$. So by Fact~\ref{fact:frob_to_trace} we have
    \begin{equation}
        \Tr\left(W^{ij} W^{j'k} {W^{j'k}}^{\top} {W^{ij}}^{\top}\right) = \bone{j = j'}\cdot \Tr(W^{ik} {W^{ik}}^{\top}) \pm O(\epsrel \sqrt{r}) = \bone{j = j'} \pm O(\epsnorm + \epsrel \sqrt{r}). \label{eq:intermediate}
    \end{equation}
    Multiplying both sides of \eqref{eq:intermediate} by $\norm{W^{ij}_{\ell,:}}^2_F \cdot \norm{W^{j'k}_{:,\ell}}^2_F$
    and rewriting $\Tr(W^{ij} W^{j'k} {W^{j'k}}^{\top} {W^{ij}}^{\top}) = \iprod{W^{j'k} {W^{j'k}}^{\top}, {W^{ij}}^{\top} W^{ij}}$, we get
    \begin{equation}
        \iprod*{\norm{W^{j'k}_{:,\ell}}^2_F W^{j'k}{W^{j'k}}^{\top}, \norm{W^{ij}_{\ell,:}}^2_F {W^{ij}}^{\top} W^{ij}} = \norm{W^{ij}_{\ell,:}}^2_F \cdot \norm{W^{j'k}_{:,\ell}}^2_F\cdot \bone{j = j'} \pm O(\epsnorm +\epsrel\sqrt{r}). \label{eq:bothsides_norm}
    \end{equation}
    By Lemma~\ref{lem:outerproduct}, there is a degree-96 SoS proof that
    \begin{equation}
        \norm{W^{j'k}_{:,\ell}}^2_F W^{j'k}{W^{j'k}}^{\top}\cdot \norm{W^{ij}_{\ell,:}}^2_F {W^{ij}}^{\top} W^{ij} \approx_{O((\epsrel\sqrt{r} +\epsnorm)r^5 + \epsnorm^2)} W^{j'k}_{:,\ell}{W^{j'k}_{:\ell}}^{\top} \cdot W^{ij}_{\ell,:} {W^{ij}_{\ell,:}}^{\top}, \label{eq:apply_outerproduct}
    \end{equation}
    so by Fact~\ref{fact:frob_to_trace} we conclude that
    \begin{equation}
        \iprod*{\norm{W^{j'k}_{:,\ell}}^2_F W^{j'k}{W^{j'k}}^{\top}, \norm{W^{ij}_{\ell,:}}^2_F {W^{ij}}^{\top} W^{ij}} = \iprod{W^{ij}_{\ell,:} W^{j'k}_{:,\ell}}^2 \pm O\left(\sqrt{(\epsrel\sqrt{r}+\epsnorm)r^5 + \epsnorm^2}\right). \label{eq:rewrite_trace}
    \end{equation}
    Combining \eqref{eq:bothsides_norm} and \eqref{eq:rewrite_trace} yields the first part of the lemma. For the second part, we can apply Fact~\ref{fact:cauchy_schwarz} and the first part of the lemma with $j = j'$ to conclude that
    \begin{equation}
        \sum_{a,b} \left(W^{ij}_{\ell a} W^{jk}_{b\ell} - W^{ij}_{\ell b} W^{jk}_{a\ell}\right)^2 \le \epsswap^2,
    \end{equation}
    from which we conclude that each summand is at most $\epsswap^2$ as desired.
\end{proof}

Lemmas~\ref{lem:outerproduct} and \ref{lem:swap} are an SoS formulation of the fact that there exist orthonormal unit vectors $\brc{v_i}_{i\in[r]}$ such that $W^{ij} \approx v_iv_j^{\top}$ for all $i,j\in[r]$, e.g. for any $a,b\in[r]$, define $v_i \triangleq W^{ai}_{b:} / \norm{W^{ai}_{b:}}$ for all $i\in[r]$. We get that $W = V^{\otimes 2}$ for the matrix $V$ whose columns consist of $\brc{v_i}$ as desired. Of course this isn't strictly speaking well-defined as $W$ is an SoS variable, but it is helpful as a heuristic for understanding our proof of identifiability.

\section{Deferred Proofs from Section~\ref{sec:tensorring}}

\subsection{Proof of Lemma~\ref{lem:tensorring_gauge_invariant}}
\label{app:tensorring_gauge}

\begin{proof}
    Define $Q'_a \triangleq VQ^*_aV^{\top}$ for all $a\in[d]$. 
    
    \noindent\textbf{Eq.~\eqref{eq:assume_momentmatch}:} This immediately follows from the fact that $\Tr(Q'_aQ'_b) = \Tr(VQ^*_aQ^*_bV^{\top}) = \Tr(Q^*_aQ^*_b)$ and $\Tr(Q'_aQ'_bQ'_c) = \Tr(VQ^*_aQ^*_bQ^*_cV^{\top})$.
    
    \noindent\textbf{Part~\ref{assume:scale}:} This immediately follows from the fact that rotations preserve Frobenius norm.
    
    \noindent\textbf{Part~\ref{assume:condnumber}:} Let $M'\in\R^{d\times\binom{r+1}{2}}$ denote the matrix whose $(a,(i_1,i_2))$-th entry, for $a\in[d]$ and $1 \le i_1 \le i_2 \le r$, is given by $(Q'_a)_{i_1i_2}$. We will show that $\sigma_{\min}(M') \ge \sigma_{\min}(M^*)$.
    
    Note that for any $i_1 \le i_2$, 
    \begin{equation}
        (Q'_a)_{i_1i_2} = \sum^r_{j_1,j_2=1} V_{i_1j_1} V_{i_2j_2} (Q^*_a)_{j_1j_2} = \sum^r_{j=1} V_{i_1j} V_{i_2j} (Q^*_a)_{jj} + \sum_{1\le j_1<j_2\le r} (V_{i_1j_1} V_{i_2j_2} + V_{i_1j_2} V_{i_2j_1}) (Q^*_a)_{j_1j_2} \label{eq:VQVentry}
    \end{equation}
    So consider the $\binom{r+1}{2} \times \binom{r+1}{2}$ matrix $\overline{V}$ given by
    \begin{equation}
        \overline{V}^{i_1i_2}_{j_1j_2} = \begin{cases}
            V_{i_1j_1} V_{i_2j_2} & \text{if} \ j_1 = j_2 \\
            V_{i_1j_1} V_{i_2j_2} + V_{i_1j_2} V_{i_2j_1} & \text{if} \ j_1 \neq j_2,
        \end{cases}
    \end{equation} noting that \eqref{eq:VQVentry} implies that $M' = M^*\cdot \overline{V}$. Note that the rows of $\overline{V}$ are orthogonal: for any $j_1\le j_2$ and $j'_1\le j'_2$, it is straightforward to check that $\iprod{\overline{V}_{j_1j_2},\overline{V}_{j'_1j'_2}}$ is zero if $(j_1,j_2)\neq(j'_1,j'_2)$ and a positive constant from $\brc{1,2}$ otherwise, implying that $\sigma_{\min}(\overline{V}) = 1$ and thus that $\sigma_{\min}(M') \ge \sigma_{\min}(M^*)$ as desired.
\end{proof}

\subsection{Proof of Lemma~\ref{lem:Unorm12}}
\label{app:defer_Unorm12}

\begin{proof}
    Because $\wh{U}\wh{U}^{\top} \approx_{\epsort^2} I'I'^{\top}$ by Lemma~\ref{lem:almost_ortho}, we know \begin{equation}
        \iprod{\wh{U}_{i_1i_2}, \wh{U}_{i'_1i'_2}} =
        \begin{cases}
            2 \pm \epsort & \text{if} \ (i_1,i_2) = (i'_1,i'_2) \ \text{and} \ i_1 < i_2 \\
            1 \pm \epsort & \text{if} \ (i_1,i_2) = (i'_1,i'_2) \ \text{and} \ i_1 = i_2 \\
            \pm \epsort & \text{otherwise}
        \end{cases}.
    \end{equation}
    so 
    \begin{equation}
        \iprod{U_{i_1i_2}, U_{i'_1i'_2}} =
        \begin{cases}
            1/2 \pm \epsort/4 & \text{if} \ (i_1,i_2) = (i'_1,i'_2) \ \text{and} \ i_1 < i_2 \\
            1 \pm \epsort & \text{if} \ (i_1,i_2) = (i'_1,i'_2) \ \text{and} \ i_1 = i_2 \\
            \pm O(\epsort) & \text{otherwise}
        \end{cases}.\label{eq:Uiprod}
    \end{equation}
    Consider the $\binom{r+1}{2} \times \binom{r+1}{2}$ matrix of indeterminates $V$ defined by
    \begin{equation}
        V^{j_1j_2}_{i_1i_2} = \sqrt{2}^{\bone{i_1 \neq i_2} + \bone{j_1 \neq j_2}} \cdot U^{j_1j_2}_{i_1i_2}.
    \end{equation}
    Observe that for any $i_1,i_2,i'_1,i'_2\in[r]$ satisfying $i_1 \le i_2$ and $i'_1 \le i'_2$,
    \begin{align}
        \iprod{V_{i_1i_2}, V_{i'_1i'_2}} &= \sum_j V^{jj}_{i_1i_2} V^{jj}_{i'_1i'_2} + \sum_{j_1 < j_2} V^{j_1j_2}_{i_1i_2} V^{j_1j_2}_{i'_1i'_2}  \\
        &= \sqrt{2}^{\bone{i_1 \neq i_2} + \bone{i'_1 \neq i'_2}}\biggl(\sum_j U^{jj}_{i_1i_2} U^{jj}_{i'_1i'_2} + 2\sum_{j_1 < j_2} U^{j_1j_2}_{i_1i_2} U^{j_1j_2}_{i'_1i'_2}\biggr) \\
        &= \sqrt{2}^{\bone{i_1 \neq i_2} + \bone{i'_1 \neq i'_2}}\biggl(\sum_{j_1, j_2\in[r]} U^{j_1j_2}_{i_1i_2} U^{j_1j_2}_{i'_1i'_2}\biggr) \\
        &= \sqrt{2}^{\bone{i_1 \neq i_2} + \bone{i'_1 \neq i'_2}}\iprod{U_{i_1i_2}, U_{i'_1i'_2}} = \bone{(i_1,i_2) = (i'_1,i'_2)} + O(\epsort), \label{eq:Vrowortho}
    \end{align}
    where in the third step we used Lemma~\ref{lem:symmetric}, and in the last step we used \eqref{eq:Uiprod}.
    
    By Lemma~\ref{lem:ortho_sos}, there is a degree-4 SoS proof using \eqref{eq:Vrowortho} that for any $j_1\le j_2$ and $j'_1\le j'_2$, 
    \begin{equation}
        -O(\sqrt{\epsort} r^3) \le \iprod{V^{j_1j_2},V^{j'_1j'_2}} - \bone{(j_1,j_2) = (j'_1,j'_2)} \le O(\sqrt{\epsort} r^3)
    \end{equation} Finally, note that for any $j_1\le j_2$,
    \begin{align}
        \norm{V^{j_1j_2}}^2_F &= \sum_i (V^{j_1j_2}_{ii})^2 +\sum_{i_1 < i_2} (V^{j_1j_2}_{i_1i_2})^2 = 2^{\bone{j_1\neq j_2}}\biggl(\sum_i (U^{j_1j_2}_{ii})^2 + 2\sum_{i_1 < i_2} (U^{j_1j_2}_{i_1i_2})^2\biggr) \\
        &= 2^{\bone{j_1\neq j_2}}\norm{U^{j_1j_2}}^2_F,
    \end{align}
    where in the last step we used Lemma~\ref{lem:symmetric}.
\end{proof}

\subsection{Proof of Corollary~\ref{cor:chain_frob}}
\label{app:defer_chain_frob}

\begin{proof}
    By Lemma~\ref{lem:fakeUsendsQ}, there is a degree-6 SoS proof that $F_U(Q^*_a)\approx_{\epsmap^2} Q_a$ and $F_U(Q^*_b)\approx_{\epsmap^2} Q_b$. So by Parts~\ref{shorthand:multiply} and \ref{shorthand:lincombo} of Fact~\ref{fact:shorthand}, there is a degree-8 SoS proof that \begin{equation}
        \frac{1}{2}(F_U(Q^*_a) F_U(Q^*_b) + F_U(Q^*_b) F_U(Q^*_a)) \approx_{O(\epsmap^2)\cdot(\norm{F_U(Q^*_a)}^2_F + \norm{F_U(Q^*_b)}^2_F)} \frac{1}{2}(Q_a Q_b + Q_b Q_a). \label{eq:intermed}
    \end{equation}
    Additionally, Lemma~\ref{lem:fakeUsendsQ} and Constraint~\ref{constraint:Qbound} imply in degree-6 SoS that $\norm{F_U(Q^*_a)}^2_F, \norm{F_U(Q^*_b)}^2_F \le 2\radius^2 + \epsmap^2 \le O(\radius^2)$ by our assumed bound on $\eta$. By Lemma~\ref{lem:fakeUsendsQaQb}, there is a degree-8 SoS proof that $\frac{1}{2}(Q_aQ_b + Q_bQ_a) \approx_{\epsmap^2} F_U(Q^*_aQ^*_b)$. Combining this with \eqref{eq:intermed} using Part~\ref{shorthand:transitive} of Fact~\ref{fact:shorthand} yields a degree-8 SoS proof of the claimed bound.
\end{proof}

\subsection{Proof of Lemma~\ref{lem:simpler_identity}}
\label{app:defer_simpler_identity}

\begin{proof}
    \noindent\textbf{Case 1}: $i\neq j = j'\neq k$. By Lemma~\ref{lem:main_identity} there is a degree-8 SoS proof that
    \begin{equation}
        2U^{ij} U^{jk} + 2U^{jk} U^{ij} \approx_{\epsrel^2} U^{ik} + \bone{i = k} U^{jj}.
    \end{equation}
    Left-multiplying by $U^{ii}$ and right-multiplying by $U^{kk}$, we conclude by Part~\ref{shorthand:bothsides} of Fact~\ref{fact:shorthand} that there is a degree-12 SoS proof that
    \begin{equation}
        2U^{ii}U^{ij} U^{jk}U^{kk} + 2 U^{ii}U^{jk} U^{ij} U^{kk} \approx_{O(\epsrel^2)} U^{ii} U^{ik} U^{kk} + \bone{i = k} U^{ii} U^{jj} U^{kk}, \label{eq:rewrite_ijjk}
    \end{equation}
    where we used that $\norm{U^{ii}}^2_2, \norm{U^{kk}}^2_2\le O(1)$ by Lemma~\ref{lem:Unorm12}.
    
    Note that $U^{ii} U^{jj} U^{kk}$ contributes negligibly to \eqref{eq:rewrite_ijjk}: in degree-12 SoS we have
    \begin{equation}
        \norm{U^{ii} U^{jj} U^{kk}}^2_F \le \norm{U^{ii} U^{jj}}^2_F \norm{U^{kk}}^2_F \le \epsrel^2\norm{U^{kk}}^2_F \le O(\epsrel^2),
    \end{equation}
    where in the first step we used Fact~\ref{fact:sos_frobenius_mult}, in the second step we used Part~\ref{item:ijzero} of Corollary~\ref{cor:consequences}, and in the last step we used Lemma~\ref{lem:Unorm12}. In the sequel, we will repeatedly employ this sequence of steps without further comment.
    
    Next, we show that $U^{ii} U^{ij} U^{jk} U^{kk}$ is close to $\frac{1}{4}W^{ij} W^{jk}$. First, by Part~\ref{item:iiij} of Corollary~\ref{cor:consequences}, there is a degree-8 SoS proof that
    \begin{equation}
        U^{ii} U^{ij} U^{jk} U^{kk} \approx_{O(\epsrel^2)} U^{ii} (U^{ii} U^{ij} + U^{ij} U^{ii})(U^{jj} U^{jk} + U^{jk} U^{jj}) U^{kk} \label{eq:iiijjkkk}
    \end{equation}
    By Part~\ref{item:Uii_idempotent} of Corollary~\ref{cor:consequences},
    \begin{equation}
        U^{ii}(U^{ii} U^{ij} + U^{ij} U^{ii}) \approx_{O(\epsrel^2)} U^{ii} U^{ij} + U^{ii} U^{ij} U^{ii}, \label{eq:iiijjkkk_1}
    \end{equation}
    and by Part~\ref{item:ijzero} plus the assumption that $j\neq k$,
    \begin{equation}
        (U^{jj} U^{jk} + U^{jk} U^{jj})U^{kk} \approx_{O(\epsrel^2)} U^{jj} U^{jk} U^{kk}. \label{eq:iiijjkkk_2}
    \end{equation}
    Applying \eqref{eq:iiijjkkk_1} and \eqref{eq:iiijjkkk_2} to \eqref{eq:iiijjkkk} and noting that the left-hand sides of \eqref{eq:iiijjkkk_1} and \eqref{eq:iiijjkkk_2} have squared Frobenius norm $O(1)$, we conclude by Part~\ref{shorthand:multiply} of Fact~\ref{fact:shorthand} that in degree-16 SoS,
    \begin{align}
        U^{ii} U^{ij} U^{jk} U^{kk} &\approx_{O(\epsrel^2)} (U^{ii} U^{ij} +U^{ii} U^{ij} U^{ii}) U^{jj} U^{jk} U^{kk} \\
        &\approx_{O(\epsrel^2)} U^{ii} U^{ij} U^{jj} U^{jk} U^{kk} \\
        &\approx_{O(\epsrel^2)} U^{ii} U^{ij} (U^{jj})^2 U^{jk} U^{kk} = \frac{1}{4} W^{ij} W^{jk},\label{eq:firstterm_leftright}
    \end{align}
    where in the second step we used Part~\ref{item:iiij} of Corollary~\ref{cor:consequences}, in the third step we used Part~\ref{item:Uii_idempotent}, and in the fourth step we used the definition of $W^{ij}, W^{jk}$.
    % By Lemma~\ref{lem:main_identity} applied to $j_1 = i, j_2 = k_1 = j$, and $k_2 = k$, we have
    
    % \begin{equation}
    %     2U^{ii} U^{ij} U^{jk} U^{kk} + 2U^{ii} U^{jk} U^{ij} U^{kk} = U^{ii} U^{ik} U^{kk}. \label{eq:leftright}
    % \end{equation}
    % By Part~\ref{item:iiij} of Corollary~\ref{cor:consequences},
    % \begin{align}
    %     U^{ii} U^{ij} U^{jk} U^{kk} &= U^{ii} (U^{ii}U^{ij} + U^{ij} U^{ii}) (U^{jj} U^{jk} +U^{jk} U^{jj}) U^{kk} \\
    %     &= (U^{ii} U^{ij} + U^{ii} U^{ij} U^{ii}) U^{jj} U^{jk} U^{kk} \\
    %     &= U^{ii} U^{ij} U^{jj} U^{jk} U^{kk} \\
    %     &= U^{ii} U^{ij} (U^{jj})^2 U^{jk} U^{kk} \\
    %     &= \frac{1}{4} W^{ij} W^{jk} \label{eq:firstterm_leftright}
    % \end{align}
    % where in the second step we used Part~\ref{item:Uii_idempotent} and \ref{item:ijzero} of Corollary~\ref{cor:consequences} plus the assumption that $j\neq k$, in the third step we used Part~\ref{item:ijzero} of Corollary~\ref{cor:consequences} and the assumption that $i\neq j$, and in the fourth step we used Part~\ref{item:Uii_idempotent} of Corollary~\ref{cor:consequences}. 
    Finally, we handle the $U^{ii} U^{jk} U^{ij} U^{kk}$ term in \eqref{eq:rewrite_ijjk}. By Part~\ref{item:swap} of Corollary~\ref{cor:consequences} and Part~\ref{shorthand:multiply} of Fact~\ref{fact:shorthand}, if $i\neq k$ then there is a degree-16 SoS proof that
    \begin{equation}
        U^{ii} U^{jk} U^{ij} U^{kk} \approx_{O(\epsrel^2)} (-U^{jk} U^{ii}) (-U^{kk} U^{ij}) = U^{jk} (U^{ii} U^{kk}) U^{ij} \approx_{O(\epsrel^2)} 0, \label{eq:secondterm_leftright}
    \end{equation}
    On the other hand, if $i = k$, then by Lemma~\ref{lem:main_identity} applied to $j_1 = i$, $j_2 = j$, $k_1 = j$, $k_2 = i$, then by Part~\ref{shorthand:lincombo} of Fact~\ref{fact:shorthand}, there is a degree-8 SoS proof that
    \begin{equation}
        4(U^{ij})^2 = 2U^{ij} U^{ji} + 2U^{ji} U^{ij} \approx_{\epsrel^2} U^{ii} + U^{jj}
    \end{equation} so left- and right-multiplying this by $U^{ii}$ and applying Lemma~\ref{lem:symmetric}, we conclude that there is a degree-16 SoS proof that
    \begin{equation}
        4 U^{ii} U^{ji} U^{ij} U^{ii} = 4 U^{ii} (U^{ij})^2 U^{jj} \approx_{O(\epsrel^2)} 4U^{ii}(U^{ii} + U^{jj})U^{ii} \approx_{O(\epsrel^2)} 4U^{ii}. \label{eq:secondterm_leftright'}
    \end{equation}
    So substituting \eqref{eq:firstterm_leftright}, \eqref{eq:secondterm_leftright}, and \eqref{eq:secondterm_leftright'} into \eqref{eq:rewrite_ijjk}, we conclude in degree-16 SoS that 
    \begin{equation}
        \frac{1}{2} W^{ij} W^{jk} + \bone{i = k} \cdot \frac{1}{2} U^{ik} \approx_{O(\epsrel^2)} U^{ii} U^{ik} U^{kk} + U^{ii} D[i,j,j,k] U^{kk} \approx_{O(\epsrel^2)} U^{ii} U^{ik} U^{kk}.
    \end{equation}
    If $i = k$, then because $(U^{ii})^3 \approx_{O(\epsrel^2)} U^{ii} = W^{ii}$, this implies that $W^{ij} W^{ji} \approx_{O(\epsrel^2)} W^{ii}$. If $i\neq k$, this implies that $W^{ij} W^{jk} \approx_{O(\epsrel^2)} W^{ik}$, as desired.
    
    \noindent\textbf{Case 2}: $i = j = j'$ or $j = j' = k$. By symmetry, it suffices to consider the former possibility. In that case, the subcase of $j' = k$ was handled above by Part~\ref{item:Uii_idempotent} of Corollary~\ref{cor:consequences}. For the subcase of $j'\neq k$, by Part~\ref{item:Uii_idempotent} of Corollary~\ref{cor:consequences} and Part~\ref{shorthand:bothsides} of Fact~\ref{fact:shorthand} we have $W^{ii} W^{ik} = 2U^{ii} \cdot U^{ii} U^{ik} U^{kk} \approx_{O(\epsrel^2)} 2U^{ii} U^{ik} U^{kk} = W^{ik}$ in degree-12 SoS, concluding the proof of this case.
    
    \noindent\textbf{Case 3}: $j \neq j'$. Note that $W^{ij} W^{j'k}$ is equal, up to a positive integer factor, to \begin{equation}
        U^{ii} U^{ij} U^{jj} U^{j'j'} U^{j'k} U^{kk} \approx_{O(\epsrel^2)} 0,
    \end{equation}
    where the last step follows in degree-24 by Part~\ref{item:ijzero} of Corollary~\ref{cor:consequences}, Part~\ref{shorthand:bothsides} of Fact~\ref{fact:shorthand}, and the assumption that $j\neq j'$.
\end{proof}

\subsection{Proof of Lemma~\ref{lem:WsendsQ}}
\label{app:defer_WsendsQ}

\begin{proof}
    We first show that for any $i\neq j$, there is a degree-12 SoS proof that
    \begin{equation}
        W^{ij} + W^{ji} \approx_{O(\epsrel^2)} 2U^{ij}, \label{eq:WWU}
    \end{equation} 
    or equivalently, 
    \begin{equation}
        U^{ii} U^{ij} U^{jj} +U^{jj} U^{ji} U^{ii} \approx_{O(\epsrel^2)} U^{ij} \label{eq:WWU_equiv}
    \end{equation}
    To show this, recall from Part~\ref{item:iiij} of Corollary~\ref{cor:consequences} that $U^{ii} U^{ij} + U^{ij} U^{ii} \approx_{O(\epsrel^2)} U^{ij}$, and similarly $U^{ij} \approx_{O(\epsrel^2)} U^{jj} U^{ij} + U^{ij} U^{jj}$. Left-multiplying (resp. right-multiplying) the latter on both sides by $U^{ii}$ (resp. $U^{ii}$) and applying Part~\ref{shorthand:bothsides} of Fact~\ref{fact:shorthand} and Lemma~\ref{lem:Unorm12}, we get a degree-12 SoS proof that $U^{ii}U^{ij} \approx_{O(\epsrel^2)} U^{ii}(U^{jj} U^{ij} + U^{ij} U^{jj})$ (resp. that $U^{ij} U^{ii} \approx_{O(\epsrel^2)} (U^{jj}U^{ij} + U^{ij} U^{jj})U^{ii}$). Substituting these into $U^{ii} U^{ij} + U^{ij} U^{ii} \approx_{O(\epsrel^2)} U^{ij}$ and applying Fact~\ref{shorthand:lincombo}, we get a degree-12 SoS proof that
    \begin{equation}
        U^{ii} (U^{jj} U^{ij} +U^{ij} U^{jj}) + (U^{jj} U^{ij} +U^{ij} U^{jj})U^{ii} \approx_{O(\epsrel^2)} U^{ij}.
    \end{equation}
    Note that by Part~\ref{item:ijzero} of Corollary~\ref{cor:consequences} and Fact~\ref{fact:shorthand}, 
    \begin{equation}
        U^{ii} (U^{jj} U^{ij} +U^{ij} U^{jj}) + (U^{jj} U^{ij} +U^{ij} U^{jj})U^{ii} \approx_{O(\epsrel^2)} U^{ii} U^{ij} U^{jj} + U^{jj} U^{ij} U^{ii},
    \end{equation}
    completing the proof of \eqref{eq:WWU_equiv} and thus of \eqref{eq:WWU}.
    
    For any $i_1,i_2,j_1,j_2\in[r]$ for which $j_1\neq j_2$, write $D^{j_1j_2}_{i_1i_2} \triangleq W^{j_1j_2}_{i_1i_2} + W^{j_2j_1}_{i_1i_2} - 2U^{j_1j_2}$. Then
    \begin{align}
        F_W(Q^*_a)_{i_1i_2} &= \sum_{j_1,j_2\in[r]} W^{j_1j_2}_{i_1i_2} (Q^*_a)_{j_1j_2} = \sum_{j_1<j_2} (W^{j_1j_2} + W^{j_2j_1})_{i_1i_2} (Q^*_a)_{j_1j_2} + \sum_j W^{jj}_{i_1i_2} (Q^*_a)_{jj}\\
        &= 2\sum_{j_1<j_2} (D^{j_1j_2}_{i_1i_2} + U^{j_1j_2}_{i_1i_2}) (Q^*_a)_{j_1j_2} + \sum_j U^{jj}_{i_1i_2} (Q^*_a)_{jj} = F_U(Q^*_a)_{j_1j_2} + 2 \sum_{j_1 < j_2} D^{j_1j_2}_{i_1i_2}.
    \end{align}
    We conclude that
    \begin{equation}
        \norm{F_W(Q^*_a) - F_U(Q^*_a)}^2 \le 4\sum_{i_1,i_2}\biggl(\sum_{j_1 < j_2} D^{j_1j_2}_{i_1i_2}\biggr)^2 \le 4r^4\sum_{i_1,i_2,j_1,j_2} (D^{j_1j_2}_{i_1i_2})^2 \le O(r^6 \epsrel^2),
    \end{equation}
    where the second step follows by Cauchy-Schwarz, and the last step follows by \eqref{eq:WWU}. The lemma follows by Lemma~\ref{lem:fakeUsendsQ}.
\end{proof}

\subsection{Proof of Lemma~\ref{lem:QQstar}}
\label{app:defer_QQstar}

\begin{proof}
    Take any $c\in[d]\cup\brc{\lambda,\mu}$. Recall from Lemma~\ref{lem:WsendsQ} and Corollary~\ref{cor:WsendsQlam} that there is a degree-12 SoS proof that $F_W(Q^*_c) \approx_{O(d\epsmap^2 + dr^6\epsrel^2)} Q_c$. Then for any $i,j\in[r]$,
    \begin{equation}
        (Q_c)_{ij} = F_W(Q^*_c)_{ij} \pm O(\sqrt{d}\epsmap + r^3\sqrt{d}\epsrel) = \sum_{k,\ell\in[r]} W^{k\ell}_{ij} (Q^*_c)_{k\ell} \pm O(\sqrt{d}\epsmap + r^3\sqrt{d}\epsrel). \label{eq:Qcij}
    \end{equation}
    Note that
    \begin{equation}
        \biggl(\sum_{(k,\ell)\neq (i,j)} W^{k\ell}_{ij} (Q^*_c)_{k\ell} \biggr)^2 \le \norm{Q^*_c}^2_F \cdot \sum_{(k,\ell)\neq (i,j)} (W^{k\ell}_{ij})^2 \le O(\radius^2 r^2 \epsoffdiag), \label{eq:offdiagQ}
    \end{equation}
    in degree-48 SoS, where in the second step we used Part~\ref{assume:scale} of Assumption~\ref{assume:tensorring} together with \eqref{eq:deg63} from Lemma~\ref{lem:deg6}. \eqref{eq:QQstar} follows from \eqref{eq:Qcij} and \eqref{eq:offdiagQ}.
\end{proof}

\section{Deferred Proofs from Section~\ref{sec:push}}

\subsection{Proof of Lemma~\ref{lem:push_gauge_invariant}}
\label{app:push_gauge}

\begin{proof}
    Define $T'_a \triangleq F_{V^{\otimes\omega}}(T^*_a)$ for all $a\in[d]$. 
    
    \noindent\textbf{Eq.~\eqref{eq:push_moment}:} As $D$ is rotation invariant, we conclude by the definition of $\Sigma$ in \eqref{eq:Sigmadef} that $\iprod{T^*_a,T^*_b}_{\Sigma} = \iprod{T'_a,T'_b}_{\Sigma}$, so $T'_1,\ldots,T'_d$ satisfy \eqref{eq:push_moment}.
    
    \noindent\textbf{Part~\ref{item:push_scale}:} This is immediate from the fact that rotations preserve Frobenius norm.
    
    \noindent\textbf{Part~\ref{item:push_condnumber}:} Let $M'\in\R^{d\times\rchoose}$ denote the matrix whose $(a,(i_1,\ldots,i_\omega))$-th entry, for $a\in[d]$ and $1 \le i_1 \le \cdots \le i_\omega \le r$, is given by $(T'_a)_{i_1\cdots i_\omega}$. We will show that $\sigma_{\min}(M') \ge \sigma_{\min}(M^*)$.
    
    The proof of this is a generalization of that of Lemma~\ref{lem:tensorring_gauge_invariant}. For any sorted tuple $\mb{i}\in\strings$,
    \begin{align}
        (T'_a)_{\mb{i}} &= \sum_{\mb{j}\in\strings} V_{i_1j_1}\cdots V_{i_\omega j_\omega} (T^*_a)_{\mb{j}} \\
        &= \sum_{\mb{j} \ \text{sorted}} (T^*_a)_{\mb{j}} \sum_{\mb{j}': \sort{j}' = \sort{j}} V_{i_1j'_1}\cdots V_{i_\omega j'_\omega}. \label{eq:VTVentry}
    \end{align}
    So consider the $\rchoose\times\rchoose$ matrix $\overline{V}$ given by
    \begin{equation}
        \overline{V}^{\mb{i}}_{\mb{j}} = \sum_{\mb{j}': \sort{j}' = \sort{j}} V_{i_1j'_1}\cdots V_{i_\omega j'_\omega},
    \end{equation}
    noting that \eqref{eq:VTVentry} implies that $M' = M^*\cdot \overline{V}$. Note that the rows of $\overline{V}$ are orthogonal: for any sorted tuples $\mb{j},\mb{j}'$, it is straightforward to check that $\iprod{\overline{V}_{\mb{j}},\overline{V}_{\mb{j}}}$ is zero if $\mb{j}\neq \mb{j}'$ and otherwise at least 1, implying that $\sigma_{\min}(\overline{V}) = 1$ and thus that $\sigma_{\min}(M') \ge \sigma_{\min}(M^*)$. 
    
    \noindent\textbf{Part~\ref{item:push_condnumber2}:} Observe that for $v'_1,\ldots,v'_\ell$ and sorted $\mb{j}^1,\ldots,\mb{j}^{\ell+1}\in\strings$, and $t_1,\ldots,t_{\ell+1}\in[\ell]$, the corresponding entry of $q(Vv'_1,\ldots,Vv'_\ell)$ is given by
    \begin{equation}
        \prod^{\ell+1}_{s=1}((Vv)^{\otimes \omega}_{t_s})_{\mb{j}^s} = \prod^{\ell+1}_{s=1} \sum_{\mb{k}^s} (V^{\otimes\omega})^{\mb{k}^s}_{\mb{j}^s} ({v'}^{\otimes\omega}_{t_s})_{\mb{k}^s} = \sum_{\mb{k}^1,\ldots,\mb{k}^{\ell+1}} (V^{\otimes\omega})^{\mb{k}^1}_{\mb{j}^1}\cdots (V^{\otimes\omega})^{\mb{k}^{\ell+1}}_{\mb{j}^{\ell+1}} \prod^{\ell+1}_{s=1} ({v'}^{\otimes\omega}_{t_s})_{\mb{k}^s},
    \end{equation}
    where we use $V^{\otimes \omega}$ to denote the appropriate $\rchoose\times\rchoose$ matrix. In particular, this implies that there is some matrix $A$ for which 
    \begin{equation}
        A\,q(v'_1,\ldots,v'_\ell) = q(Vv'_1,\ldots,Vv'_\ell) \label{eq:Adef}
    \end{equation} for all $a\in[d]$.
    So if for some $v_1,\ldots,v_\ell$ we have \begin{equation}
        q(V^{\top} v_1,\ldots, V^{\top} v_\ell) = \sum^d_{a=1} \lambda_a q(v^*_{a,1},\ldots,q(v^*_{a,\ell})) \label{eq:preA}
    \end{equation}
    for some $\lambda$ satisfying \begin{equation}
        \norm{\lambda}^2 \le \theta^2\biggl(\sum_t \norm{V^{\top} v_t}^2\biggr)^{\omega(\ell+1)} = \theta^2\biggl(\sum_t \norm{v_t}^2\biggr)^{\omega(\ell+1)},
    \end{equation} then left-multiplying both sides of \eqref{eq:preA} by $A$ and applying \eqref{eq:Adef}, we conclude that
    \begin{equation}
        q(v_1,\ldots,v_\ell) = A\, q(V^{\top} v_1,\ldots,V^{\top} v_\ell) = \sum^d_{a=1} A\lambda_a q(v^*_{a,1},\ldots,q(v^*_{a,\ell})) = \sum^d_{a=1} \lambda_a q(Vv^*_{a,1},\ldots,q(Vv^*_{a,\ell}))
    \end{equation}
    implying that $T'_1,\ldots,T'_d$ also satisfy Part~\ref{item:push_condnumber2} with parameter $\theta$.
\end{proof}

\subsection{Proof of Lemma~\ref{lem:approxVrank}}
\label{app:defer_approxVrank}

\begin{proof}
    Fix any $c\in[d]$ and, to ease notation, denote $T^*_c,T_c$ by $T^*,T$, and let $\Delta\triangleq F_U(T^*_c) - T_c$. Contracting along $V$, we have
    \begin{equation}
        \Delta(V,:,:) = F_U(T^*)(V,:,:) - T(V,:,:).
    \end{equation}
    Recalling from Lemma~\ref{lem:push_almost_map} that $\norm{\Delta}^2_F\le \epsmap^2$,
    \begin{equation}
        \norm{\Delta(V,:,:)}^2_F = \sum_{a,b\in[r]}\biggl(\sum_{\mb{i}\in\strings} \Delta_{\mb{i}ab} V_{\mb{i}}\biggr)^2 \le \sum_{a,b\in[r]} \biggl(\sum_{\mb{i}} \Delta^2_{\mb{i}ab}\biggr)\biggl(\sum_{\mb{i}} V_{\mb{i}}^2\biggr) = \norm{V}^2_F\cdot\norm{\Delta}^2_F \le \epsmap^2. \label{eq:Delta_contract}
    \end{equation}
    Define $\delta_{a,b}\triangleq \Delta(V,:,:)_{a,b}$ for any $a,b\in[r]$ so that \eqref{eq:Delta_contract} implies $-\epsmap \le \delta_{a,b} \le \epsmap$. For any $\brc{a_1,\ldots,a_{\ell+1}}, \brc{b_1,\ldots,b_{\ell+1}}\subset[r]$,
    \begin{align}
        \MoveEqLeft \sum_{\pi\in\calS_{\ell+1}}\sgn(\pi)\prod^{\ell+1}_{s=1} F_U(T^*)(V,:,:)_{a_s,b_{\pi(s)}} = \sum_{\pi\in\calS_{\ell+1}}\sgn(\pi)\prod^{\ell+1}_{s=1} \left(T(V:,:)_{a_s,b_{\pi(s)}} - \delta_{a_s,b_{\pi(s)}}\right) \\
        &= \sum_{\pi\in\calS_{\ell+1}} \sgn(\pi)\sum_{S\subseteq[\ell+1]: S\neq\emptyset}(-1)^{|S|} \prod_{s\in S} \delta_{a_s,b_{\pi(s)}} \prod_{s\not\in S} T(V,:,:)_{a_s,b_{\pi(s)}} \\
        &\le (\ell+1)!\cdot(2\radius)^{\ell+1}\cdot \epsmap \le (3\radius\ell)^{\ell+1}\cdot \epsmap.
    \end{align}
    where the second step follows by Constraint~\ref{constraint:push_lowrank} and Fact~\ref{fact:sos_rank} (this is where we need a degree-$O(\omega\ell)$ SoS proof in $\brc{v_{c,t}}$), and the penultimate step follows by the fact that for any $s$,
    \begin{align}
        \norm{T(V,:,:)^2}_F &= \sum_{a,b\in[r]}\biggl(\sum_{\mb{i}\in\strings}T_{\mb{i}ab}V_{\mb{i}}\biggr)^2 \le \sum_{a,b\in[r]}\biggl(\sum_{\mb{i}}T^2_{\mb{i}ab}\biggr)\biggl(\sum_{\mb{i}} V^2_{\mb{i}}\biggr) = \norm{T}^2\cdot\norm{V}^2_F \le \radius^2. \qedhere
    \end{align}
\end{proof}

\subsection{Proof of Lemma~\ref{lem:basis_sym}}
\label{app:basis_sym}

\begin{proof}
    For any any $1\le i \le m$, define $M'_{ii} \triangleq e_ie_i^{\top}$, and for any $1\le i < j \le m$, define $M'_{ij} \triangleq \frac{1}{2}(e_i + e_j)(e_i + e_j)^{\top}$. Note that for any $i < j$, $e_ie_j^{\top} + e_je_i^{\top} = M'_{ij} - M'_{ji}$. So for any orthogonal $u,v\in\S^{m-1}$, we can write $uv^{\top} + vu^{\top}$ as
    \begin{equation}
        \sum_{i<j} (u_i v_j + u_j v_i) (e_i e_j^{\top} + e_j e_i^{\top}) + \sum_i 2u_i v_i e_ie_i^{\top} = \sum_{i<j} (u_i v_j + u_j v_i) (M'_{ij} - M'_{ji}) + \sum_i 2u_iv_i M'_{ii}.
    \end{equation}
    Equivalently, if $B\in\R^{m^2}$ is the matrix whose columns consist of $\vec(M'_{ij})$ for all $i,j$, then for $\lambda$ whose $(i,j)$-th entry is $u_i v_j + u_j v_i$ for $i \le j$ and $-u_i v_j - u_j v_i$ for $i > j$, we have 
    \begin{equation}
        \vec(uv^{\top} + vu^{\top}) = B\lambda. \label{eq:vvB}
    \end{equation} Note that $\norm{\lambda}^2 = 2$.

    Now suppose we had some rotation $U\in O(m)$ for which 
    \begin{equation}
        |(UM'_{ij}U^{\top})_{ab}| = \Omega(1/m^7) \ \ \forall \ i,j,a,b\in[m]. \label{eq:entries_big}
    \end{equation}
    Consider $W\in O(m^2)$ given by $W\vec(M) = UMU^{\top}$ for all $M\in\R^{m\times m}$. Then if we take $u,v$ in \eqref{eq:vvB} to be $U^{\top}(e_i+e_j)$, then we can left-multiply both sides of \eqref{eq:vvB} by $W$ and take matricizations to conclude that
    \begin{equation}
        (e_i+e_j)(e_i+e_j)^{\top} = \sum_{i,j} \lambda_{ij} \cdot UM'_{ij}U^{\top}.
    \end{equation}
    
    To complete the proof, it suffices to exhibit $U$ satisfying \eqref{eq:entries_big}. Consider a Haar-random $U\in O(r)$. For any $a,i,j$, because $U_a$ is a random unit vector, by standard estimates on hyperspherical caps we have for sufficiently small constant $c > 0$ that
    \begin{equation}
        \Pr*{|(U_a)_i + (U_b)_j| > c/m^{7/2}} = \Pr*{|(U_a)_1| > c/m^3\sqrt{2}} \ge 1 - 1/10m^3,
    \end{equation}
    so by a union bound, this holds for all $a,i,j$ with probability $9/10$. Under this event, $(UM'_{ij}U^{\top})_{ab}$ is equal, up to a constant factor, to $\left((U_a)_i + (U_a)_j\right)\left((U_b)_i + (U_b)_j\right) > c^2/m^7$.
\end{proof}

\subsection{Proof of Corollary~\ref{cor:rank1relation}}
\label{app:defer_rank1relation}

\begin{proof}
    Take $V^{(1)},\ldots,V^{(N)}$ from Lemma~\ref{lem:Vrankimplies1rank}. Suppose we have shown inductively that for any symmetric tensor $S'\in(\R^r)^{\otimes\omega}$ for which there exist $v_1,\ldots,v_j$ satisfying $S' = \sum^j_{t=1}v^{\otimes\omega}_t$ and $\sum_t \norm{v_t}^2 = 1$, $F_U(S')$ is $\xi_j$-approximately $V$-rank at most $j$. 
    
    For the base case of $j = \ell$, this is true for $\xi_\ell = \theta\sqrt{d}(3\radius\ell)^{\ell+1}\epsmap$ by Lemma~\ref{lem:genericflatrank}. By Lemma~\ref{lem:rankrtorankr}, the inductive step follows for
    \begin{equation}
        \xi_{j-1} = r^{O(\omega)}\cdot O(\ell)^{O(\ell)}\cdot \left(\xi_j\cdot \omega^{\omega} r^{\omega/2}\radius\sqrt{d} / \kappa 
        % + O(\omega)^{O(\omega\ell)}\eta\normbound^2 d\right).
        + O(\omega)^{O(\omega\ell)} r^{\omega} \eta d / \kappa^2\right).
    \end{equation}
    Unrolling the recursion, we conclude that
    \begin{align}
        \xi_1 &= r^{O(\omega\ell)}\cdot O(\ell)^{O(\ell^2)}\left((\omega^{\omega}r^{\omega/2}\radius\sqrt{d}/\kappa)^{\ell-1}\xi_\ell + \sum^{\ell-2}_{i=1}
        % (O(\omega)^{O(\omega\ell)}\eta\normbound^2 d)\cdot
        (O(\omega)^{O(\omega\ell)} r^{\omega} \eta d / \kappa^2)\cdot
        (\omega^{\omega} r^{\omega/2}\radius\sqrt{d}/\kappa)^{i}\right) \\
        &= O(\ell^{\ell}\omega^{\omega}r^{2\omega}\radius d/\kappa)^{O(\ell)}\cdot
        % \normbound^2(\theta+\eta).
        (\theta\epsmap+\eta)
    \end{align}
    The result then follows by Lemma~\ref{lem:Vrankimplies1rank}.
\end{proof}

\subsection{Proof of Lemma~\ref{lem:outerprodstructure}}
\label{app:defer_outerprodstructure}

\begin{proof}
    For any $0\le m < \omega$ and $\mb{i},\mb{a}\in\strings$, we show by backwards induction on $m$ that 
    \begin{equation}
        (T_{a_{1:\omega}})^{\omega-m-1} T_{a_{1:m} i_{m+1:\omega}} = \prod^{\omega}_{t=m+1} T_{a_{1:t-1} i_t a_{t+1:\omega}} \pm \tmp_{\omega-m-1}
    \end{equation}
    for certain error terms $\brc{\tmp_0,\ldots,\tmp_{\omega-1}}$. Observe that the base cases of $m = \omega-1$ and $m = \omega-2$ are immediate with $\tmp_0, \tmp_1 = 0$.
    
    Next, by \eqref{eq:2x2vanish} we know that
    \begin{multline}
        T_{a_{1:\omega}} T_{a_{1:m} i_{m+1:\omega}} + T_{a_{1:m} i_{m+1:\omega-2} a_{\omega-1} a_{\omega}} T_{a_{1:\omega-2} i_{\omega-1} i_{\omega}}
        \\ = 
        T_{a_{1:\omega-1} i_{\omega}} T_{a_{1:m} i_{m+1:\omega-1} a_{\omega}} + T_{a_{1:\omega-2} i_{\omega-1} a_{\omega}} T_{a_{1:m} i_{m+1:\omega-2} a_{\omega-1} i_{\omega}} \pm \epsrel. \label{eq:apply2x2}
    \end{multline}
    As $T$ is symmetric, we can rewrite this as
    \begin{multline}
        T_{a_{1:\omega}} T_{a_{1:m} i_{m+1:\omega}} + T_{a_{1:m} a_{\omega-1} a_{\omega} i_{m+1:\omega-2} } T_{a_{1:\omega-2} i_{\omega-1} i_{\omega}}
        \\ = 
        T_{a_{1:\omega-1} i_{\omega}} T_{a_{1:m} a_{\omega} i_{m+1:\omega-1}} + T_{a_{1:\omega-2} a_{\omega} i_{\omega-1}} T_{a_{1:m} a_{\omega-1} i_{m+1:\omega-2} i_{\omega}} \pm \epsrel. \label{eq:apply2x2_rewrite}
    \end{multline}
    Consider the second, third, and fourth terms in \eqref{eq:apply2x2_rewrite}. To bound the second term, we have by the inductive hypothesis and symmetry of $T$ that
    \begin{align}
        (T_{a_{1:\omega}})^{\omega-m-3} T_{a_{1:m} a_{\omega-1} a_{\omega} i_{m+1:\omega-2}} &= \prod^{\omega-2}_{t = m+1} T_{a_{1:m} a_{\omega-1} a_{\omega} a_{m+1:t-1} i_t a_{t+1:\omega-2}} \pm \tmp_{\omega-m-3} \\
        &= \prod^{\omega-2}_{t = m+1} T_{a_{1:t-1} i_t a_{t+1:\omega}} \pm \tmp_{\omega-m-3}. \label{eq:second1}
    \end{align}
    and
    \begin{equation}
        (T_{a_{1:\omega}})^1 T_{a_{1:\omega-2}i_{\omega-1}i_{\omega}} = \prod^{\omega}_{t=\omega-1} T_{a_{1:t-1} i_t a_{t+1:\omega}} \pm \tmp_1.\label{eq:second2}
    \end{equation}
    Similarly, for the third term in \eqref{eq:apply2x2_rewrite},
    \begin{align}
        (T_{a_{1:\omega}})^{\omega-m-2} T_{a_{1:m} a_{\omega} i_{m+1:\omega-1}} &= \prod^{\omega-1}_{t=m+1} T_{a_{1:m}a_{\omega} a_{m+1:t-1} i_t a_{t:\omega-1}} \pm \tmp_{\omega-m-2}\\
        &= \prod^{\omega-1}_{t=m+1} T_{a_{1:t-1} i_t a_{t+1:\omega}} \pm \tmp_{\omega-m-2}, \label{eq:third}
    \end{align}
    and for the fourth term in \eqref{eq:apply2x2_rewrite},
    \begin{align}
        (T_{a_{1:\omega}})^{\omega-m-2} T_{a_{1:m} a_{\omega-1} i_{m+1:\omega-2} i_{\omega}} &= \prod_{t\in\brc{m+1,\ldots,\omega-2,\omega}} T_{a_{1:m}a_{\omega-1} a_{m+1:t-1} i_t a_{t+1:\omega-2} a_{\omega}} \pm \tmp_{\omega-m-2} \\
        &= \prod_{t\in\brc{m+1,\ldots,\omega-2,\omega}} T_{a_{1:t-1} i_t a_{t+1:\omega}} \pm \tmp_{\omega-m-2}. \label{eq:fourth}
    \end{align}
    Multiplying both sides of \eqref{eq:apply2x2_rewrite} by $(T_{a_{1:\omega}})^{\omega-m-2}$ and applying \eqref{eq:second1}, \eqref{eq:second2}, \eqref{eq:third}, \eqref{eq:fourth}, we find that we can take
    \begin{equation}
        \tmp_{\omega-m-1} \le \norm{T}^{\omega-m-2}_{\max} \epsrel + 2\norm{T}^{\omega-m-2}_{\max}\tmp_1 + 2\norm{T}^2_{\max}\tmp_{\omega-m-3} + 2\norm{T}_{\max} \tmp_{\omega-m-2}. \label{eq:recur}
    \end{equation}
    Define $\tmp'_j \triangleq \norm{T}^{\omega-j}\tmp_j$ so that we can rewrite \eqref{eq:recur}, after multiplying by $\norm{T}^{m+1}_{\max}$, as
    \begin{equation}
        \tmp'_{\omega-m-1} \le \norm{T}^{\omega-1}_{\max} \epsrel + 2(\tmp'_1 + \tmp'_{\omega-m-3} + \tmp'_{\omega-m-2}) \le \norm{T}^{\omega-1}_{\max}\epsrel + 6\max_{j < \omega-m-1} \tmp'_j \le \exp(O(\omega))\cdot \norm{T}^{\omega-1}_{\max}\epsrel,
    \end{equation} where the last step follows by unrolling the recurrence. We conclude that $\tmp_{\omega-1} \le \exp(O(\omega))\cdot \norm{T}^{\omega-2}_{\max}\epsrel$. By Lemma~\ref{lem:bootstrap}, $\norm{T}^2_F\le \omega^{O(\omega)}$, so $\norm{T}_{\max} = \pm \omega^{O(\omega)}$, completing the proof.
\end{proof}

\subsection{Proof of Lemma~\ref{lem:sumpowersT}}
\label{app:sumpowersT}

\begin{proof}
    We have
    \begin{align}
        \biggl(\sum_{x\in[r]} (T_{x\cdots x})^e\biggr)^{\omega} &= \sum_{x_1,\ldots,x_{\omega}\in[r]} \left(T_{x_1\cdots x_1} \cdots T_{x_\omega\cdots x_\omega}\right)^{e} \\
        &= \sum_{x_1,\ldots,x_\omega\in[r]} (T_{x_1\cdots x_\omega})^{e\omega} \pm r^{\omega}\epsrel^* \\
        &\ge \frac{1}{(r^{\omega})^{e\omega/2 - 1}} \biggl(\sum_{x_1,\ldots,x_\omega\in[r]} (T_{x_1\cdots x_\omega})^2\biggr)^{e\omega/2} - r^{\omega}\epsrel^*,
    \end{align}
    where the second step follows by Lemma~\ref{lem:perm}, and the last step follows by degree-$(e\omega)$ Holder's.
\end{proof}

\subsection{Proof of Lemma~\ref{lem:nice_outerprod_odd}}
\label{app:defer_nice_outerprod_odd}

\begin{proof}
    We have
    \begin{align}
        \wt{U}^i_{x_1}\cdots \wt{U}^i_{x_\omega} &= \sum_{j^1_1, j^1_2,\ldots, j^{\omega}_{\omega'}} U^{i\cdots i}_{x_1 j^1_1 j^1_1 j^1_2 j^1_2 \cdots j^1_{\omega'} j^1_{\omega'}} U^{i\cdots i}_{j^2_1 x_2 j^2_1 j^2_2 j^2_2 \cdots j^2_{\omega'} j^2_{\omega'}} \cdots U^{i\cdots i}_{j^{\omega}_1 j^{\omega}_1 j^{\omega}_2 j^{\omega}_2 \cdots j^{\omega}_{\omega'} j^{\omega}_{\omega'} x_{\omega}} \\
        % &= U^{i\cdots i}_{x_1\cdots x_\omega} \sum_{j^1_1,j^1_2,\ldots,j^\omega_m} \prod^m_{a=1} \left(U^{i\cdots i}_{j^1_a j^a_1 \cdots j^a_m j^{a+1}_1\cdots j^{a+1}_m}\right)^2 \label{eq:multiplicands_disjoint} \\
        &= U^{i\cdots i}_{x_1\cdots x_\omega} \sum_{j^1_1,j^1_2,\ldots,j^\omega_{\omega'}} \prod^{\omega'}_{a=1} \left(U^{i\cdots i}_{j^1_a j^2_a \cdots j^{\omega}_a}\right)^2 \pm r^{\omega'}\cdot \epsrel^* \label{eq:multiplicands_disjoint} \\
        &= U^{i\cdots i}_{x_1\cdots x_\omega} \prod^{\omega'}_{a=1}\brc*{\sum_{j^1_a,j^2_a,\ldots, j^{\omega}_a}\left(U^{i\cdots i}_{j^1_a j^2_a \cdots j^{\omega}_a}\right)^2} \pm r^{\omega'}\cdot\epsrel^*. \label{eq:bigprod}
        % &= U^{i\cdots i}_{x_1\cdots x_{\omega}},
    \end{align}
    where the first step follows by definition of $\wt{U}^i$, the second follows by Lemma~\ref{lem:perm}, and the third follows by interchanging sum and product (because the multiplicands in \eqref{eq:multiplicands_disjoint} are decoupled). Finally, by Lemma~\ref{lem:diagonal_orth}, we know each multiplicand in the product in \eqref{eq:bigprod} is $1\pm \bigl(\epsort/\omega! + O(r\omega)^{O(\omega^2)}\cdot {\epsrel^*}^{1/2}\bigr)$, so the product is $1\pm O\bigl(\epsort/2^{\omega} + O(r\omega)^{O(\omega^2)}\cdot {\epsrel^*}^{1/2}\bigr)$. The lemma then follows by the fact that $(U^{i\cdots i}_{x_1\cdots x_\omega})^2 \le \norm{U^{i\cdots i}}^2 \le O(1)$.
\end{proof}

\subsection{Proof of Corollary~\ref{cor:dd_orth_odd}}
\label{app:defer_dd_orth_odd}

\begin{proof}
    First suppose $i\neq j$. Then by Lemma~\ref{lem:diagonal_orth}, $\iprod{U^{i\cdots i}, U^{j\cdots j}} = \pm\epsort/\omega!$. We can also express $\iprod{U^{i\cdots i},U^{j\cdots j}}$ as
    \begin{multline}
        \iprod{\wt{U}^i, \wt{U}^j}^{\omega} + \iprod{U^{i\cdots i} - (\wt{U}^i)^{\otimes\omega}, U^{j\cdots j}} + \iprod{U^{i\cdots i}, (\wt{U}^j)^{\otimes\omega} - U^{j\cdots j}} + \iprod{(\wt{U}^i)^{\otimes\omega} - U^{i\cdots i}, (\wt{U}^j)^{\otimes\omega} - U^{j\cdots j}} \\
        % &= \iprod{\wt{U}^i,\wt{U}^j}^{\omega} \pm O(r^{\omega/2}\cdot \epsort/2^{\omega} + r^{\omega}\epsort^2/2^{2\omega}) \\
        = \iprod{\wt{U}^i,\wt{U}^j}^{\omega} \pm O\left(r^{\omega/2}\epsort/2^{\omega} + O(r\omega)^{O(\omega^2)}\cdot {\epsrel^*}^{1/2}\right),
    \end{multline}
    where in the second step we used Cauchy-Schwarz and Lemma~\ref{lem:nice_outerprod_odd}. So by Fact~\ref{fact:nonneg_power_of_two} we conclude that $\iprod{\wt{U}^i,\wt{U}^j} = \pm O(r^{1/2} \cdot \epsort^{1/\omega} + O(r\omega)^{O(\omega)}\cdot {\epsrel^*}^{1/\omega})$.
    
    Now suppose $i = j$. Then by an identical calculation, we conclude that $\iprod{\wt{U}^i,\wt{U}^i}^{\omega} = 1 \pm O\bigl(r^{\omega/2}\epsort/2^{\omega} + O(r\omega)^{O(\omega^2)}\cdot {\epsrel^*}^{1/2}\bigr)$. Taking $x$ in Fact~\ref{fact:root1} to be $\norm{\wt{U}^i}^2$ allows us to conclude that $\norm{\wt{U}^i}^2 = 1 \pm O\bigl(r^{\omega/2}\epsort/2^{\omega} + O(r\omega)^{O(\omega^2)}\cdot {\epsrel^*}^{1/2}\bigr)$. 
    % Note that $r^{\omega/2}\cdot \epsort/2^{\omega} \ll r^{1/2}\epsort^{1/\omega}$ because $\epsort$ is sufficiently small.
    
    Finally, the second part of the corollary follows by Lemma~\ref{lem:ortho_sos}.
\end{proof}

\subsection{Proof of Lemma~\ref{lem:main_userank1_odd}}
\label{app:main_userank1_odd}

\begin{proof}
    \textbf{Case 1:} $i_1=\cdots=i_\omega$. Then by Corollary~\ref{cor:dd_orth_odd},
    \begin{align}
        C^{\mb{i}}(\mb{i}') &= U^{\mb{i}}(\mb{i}') \\
        &= \iprod{\wt{U}^{i_1},\wt{U}^{i'_1}}\cdots \iprod{\wt{U}^{i_1},\wt{U}^{i'_\omega}} + \sum_{x_1,\ldots,x_\omega\in[r]} (U^{\mb{i}} - \wt{U}^{i_1}\otimes\cdots\otimes \wt{U}^{i_1}) \wt{U}^{i'_1}_{x_1}\cdots \wt{U}^{i'_\omega}_{x_\omega} \\
        &= \prod^\omega_{s=1} (\bone{i_1 = i'_1} \pm \epstrueort) + \sum_{x_1,\ldots,x_\omega\in[r]} (U^{\mb{i}} - \wt{U}^{i_1}\otimes\cdots\otimes \wt{U}^{i_1}) \wt{U}^{i'_1}_{x_1}\cdots \wt{U}^{i'_\omega}_{x_\omega} \label{eq:diagonal_case}
    \end{align}
    By Cauchy-Schwarz and Lemma~\ref{lem:nice_outerprod_odd},
    \begin{multline}
        \biggl(\sum_{x_1,\ldots,x_\omega\in[r]} (U^{\mb{i}} - U^{i_1}\otimes\cdots\otimes U^{i_1}) \wt{U}^{i'_1}_{x_1}\cdots \wt{U}^{i'_\omega}_{x_\omega}\biggr)^2 \\
        \le O\left(\epsort^2 r^{\omega}/2^{2\omega} + O(r\omega)^{O(\omega^2)}\cdot\epsrel^*\right) \cdot \norm{\wt{U}^{i'_1}}^2 \cdots \norm{U^{i'_\omega}}^2
        \le O\left(\epsort^2 r^{\omega}/4^{\omega} + O(r\omega)^{O(\omega^2)}\cdot\epsrel^*\right),
    \end{multline} where in the last step we used Corollary~\ref{cor:dd_orth_odd}, so substituting this into \eqref{eq:diagonal_case} we conclude that
    \begin{equation}
        C^{\mb{i}}(\mb{i}') = \bone{i'_\ell = i_1 \ \forall \ \ell\in[\omega]} \pm O(\omega\epstrueort + \epsort r^{\omega/2}/2^{\omega} + O(r\omega)^{O(\omega^2)}\cdot{\epsrel^*}^{1/2}).
    \end{equation}
    For convenience, define
    \begin{equation}
        \epsilon' \triangleq \Theta(\omega\epstrueort + \epsort r^{\omega/2}/2^{\omega} + O(r\omega)^{O(\omega^2)}\cdot{\epsrel^*}^{1/2}) \label{eq:epsprime}
    \end{equation}
    so that $C^{\mb{i}}(\mb{i}') \le \epsilon'$ in this case.
    % This follows immediately by Lemma~\ref{lem:nice_outerprod_odd} and Corollary~\ref{cor:dd_orth_odd}.
    
    \noindent\textbf{Case 2:} $\mb{i}$ contains exactly two distinct indices. Without loss of generality, suppose these two distinct indices are 1 and 2. 
    
    We first introduce some notation. For any $0\le m \le \omega$, let $C^{[m]}$ denote $C^{1\cdots 1 2\cdots 2}$ where there are $m$ 2's in the superscript, and let $a^{[m]}$ denote the string $1\cdots 1 2\cdots 2$ containing $m$ 2's.
    
    For some $\epsilon\in [0,1]$ that we will vary, consider the tensor
    \begin{equation}
        T\triangleq F_U((e_1 + \epsilon\cdot e_2)^{\otimes\omega}) = C^{[0]} + \epsilon C^{[1]} + \epsilon^2 C^{[2]} + \cdots + \epsilon^{\omega} C^{[\omega]}.
    \end{equation}
    
    We will use the following consequence of Lemma~\ref{lem:perm}, specifically \eqref{eq:true2x2}:
    \begin{equation}
        T(a_1\cdots a_{\omega}) T(a'_1\cdots a'_{\omega}) = T(a_1\cdots a_{\omega-1} a'_{\omega}) T(a'_1\cdots a'_{\omega-1} a_{\omega}) \pm \exp(O(\omega))\cdot \epsrel^*\label{eq:rank1combo_odd}
    \end{equation}
    for various choices of $a_1,a'_1\ldots,a_{\omega},a'_{\omega}$ (note that the factor of $\exp(O(\omega))$ comes from the fact that $(e_1 +\epsilon e_2)^{\otimes\omega}$ has Frobenius norm $\exp(O(\omega))$, so we must scale the $T$ that we apply Lemma~\ref{lem:perm} to accordingly).
    
    We begin by bounding $C^{[m]}(a^{[m']})$ for all $m < m'$. We will make use of various linear combinations of the constraints \eqref{eq:rank1combo_odd} under $\epsilon = 1/(2\omega+1),2/(2\omega+1),\ldots,1$. The intuition is that we can regard \eqref{eq:rank1combo_odd} as an (approximate) polynomial identity in the variable $\epsilon$, and we would like to isolate out the coefficient corresponding to a particular power of $\epsilon$ in the monomial expansion of this identity (we applied a similar trick in Lemma~\ref{lem:rankrtorankr}). 
    % Recall by Fact~\ref{fact:vandermonde_interpolate} that for any $0\le m \le 2\omega$, there exists $\lambda\in\R^{2\omega+1}$ for which, for the Vandermonde matrix
    % \begin{equation}
    %     V\triangleq \begin{pmatrix}
    %         1 & \frac{1}{2\omega+1} & \cdots & \left(\frac{1}{2\omega+1}\right)^{2\omega} \\
    %         1 & \frac{2}{\ell+2} & \cdots & \left(\frac{2}{2\omega+1}\right)^{2\omega} \\
    %         \vdots & \vdots & \ddots & \vdots \\
    %         1 & 1 & \cdots & 1
    %     \end{pmatrix},
    % \end{equation}
    % we have $\lambda^{\top} V = e_m$ and furthermore $\norm{\lambda} \le O(\omega)^{O(\omega)}$. $\lambda$ therefore specifies a 
    By Corollary~\ref{cor:general_vandermonde} applied to $D = 1$ and $e = 2\omega$, there is a linear combination of the constraints \eqref{eq:rank1combo_odd} for $\epsilon = 1/(2\omega+1),2/(2\omega+1),\ldots,1$ that implies the following: the coefficient of $\epsilon^m$ in the monomial expansion of \eqref{eq:rank1combo_odd} as a polynomial in $\epsilon$ is upper bounded by $O(\omega)^{O(\omega)}\cdot \epsrel^*$.
    We will refer to the coefficient of the $\epsilon^m$ coefficient as the ``degree-$m$ coefficient of \eqref{eq:rank1combo_odd}.''
    
    We are now ready to bound $C^{[m]}(a^{[m']})$ for all $m < m'$. We will prove inductively in $m$ that \begin{equation}
        C^{[m]}(a^{[m']}) \le ((4m+2)\crude)^m \cdot \epsilon' + 2\sum^{m-1}_{i=1} ((4m+2)\crude)^i \cdot O(\omega)^{O(\omega)}\cdot \epsrel^* \ \ \forall m < m'. \label{eq:case2_induct}
    \end{equation} By Case 1, this is true for $m = 0$ and $m' > 0$.
    
    For general $m < m'$, consider the degree-$m$ coefficient of \eqref{eq:rank1combo_odd} applied to $a_1\cdots a_{\omega} = a^{[0]}$ and $a'_1\cdots a'_{\omega} = a^{[m']}$. This can be written as
    \begin{multline}
        C^{[0]}(a^{[0]}) C^{[m]}(a^{[m']}) + \sum^m_{i=1} C^{[i]}(a^{[0]}) C^{[m-i]}(a^{[m']}) \\ = C^{[0]}(a^{[1]}) C^{[m]}(a^{[m'-1]}) + \sum^m_{i=1} C^{[i]}(a^{[1]}) C^{[m-i]}(a^{[m'-1]}) \pm O(\omega)^{O(\omega)}\cdot \epsrel^*. \label{eq:Ca}
    \end{multline}
    If $m' > m$, then $m' > m' - 1 > m - i$ for all $1 \le i \le m$, so by the inductive hypothesis, we conclude that the terms in the summations on either side of \eqref{eq:Ca}, as well as the first term on the right-hand side, are each bounded in magnitude by 
    \begin{equation}
        ((4m-2)\crude)^{m-1}\cdot \epsilon' + 2\sum^{m-2}_{i=1}((4m-2)\crude)^i \cdot O(\omega)^{O(\omega)}\cdot\epsrel^*.
    \end{equation}
    Finally, recall by Case 1 that $C^{[0]}(a^{[0]}) = 1 \pm \epsilon' \ge 1/2$. 
    So by Part~\ref{fact:divide_both_sides} of Fact~\ref{fact:division}, we conclude that
    \begin{equation}
        C^{[m]}(a^{[m']}) = \pm \biggl((4m+2)\cdot \biggl(((4m-2)\crude)^{m-1} \epsilon' + 2\sum^{m-2}_{i=1}((4m-2)\crude)^i \cdot O(\omega)^{O(\omega)}\cdot\epsrel^*\biggr) + O(\omega)^{O(\omega)}\cdot \epsrel^*\biggr),
    \end{equation}
    which completes the induction. By symmetry, we have the same bound for $C^{[\omega-m]}(a^{[\omega-m']})$ for all $m < m'$, or equivalently, $C^{[m]}(a^{[m']})$ for all $m > m'$. To summarize, upon simplifying \eqref{eq:case2_induct} by noting that $\epsilon' \gg \epsrel^*$,
    \begin{equation}
        C^{[m]}(a^{[m']}) = \pm O(\omega\crude)^{O(\omega)}\cdot\epsilon' \ \ \forall \ m \neq m'. \label{eq:neq}
    \end{equation}
    
    Next, we verify that $C^{[m]}(a^{[m]})$ is close to $1$. Consider \eqref{eq:Ca} for $m' = m$. The summation on the left-hand side and the first term on the right-hand side have total magnitude $O(\omega\crude)^{O(\omega)}\cdot\epsilon'$ by \eqref{eq:neq}, while the contribution of all but the first summand in the summation on the right-hand side has total magnitude $O(\omega\crude)^{O(\omega)}\cdot\epsilon'$. Recalling that $C^{[0]}(a^{[0]}) = 1 \pm \epsilon'$ by Case 1, we conclude by Part~\ref{fact:divide_both_sides_big} of Fact~\ref{fact:division} that
    \begin{equation}
        \left(C^{[m]}(a^{[m]}) - C^{[1]}(a^{[1]}) C^{[m-1]}(a^{[m-1]})\right)^2 \le O(\omega\crude)^{O(\omega)}\cdot{\epsilon'}^2. \label{eq:Cmam}
    \end{equation}
    Summing \eqref{eq:Cmam} over $2 \le m \le \omega$ and applying the following telescoping Cauchy-Schwarz step,
    \begin{equation}
        \sum^{\omega}_{m=2} \bigl[C^{[1]}(a^{[1]})^{\omega-m} C^{[m]}(a^{[m]}) - C^{[1]}(a^{[1]})^{\omega-m+1} C^{[m-1]}(a^{[m-1]})\bigr]^2 \ge \frac{1}{\omega-1} \bigl[C^{[m]}(a^{[m]}) - C^{[1]}(a^{[1]})^{\omega}\bigr]^2,
    \end{equation}
    we conclude that
    \begin{equation}
        \left(C^{[\omega]}(a^{[\omega]}) - C^{[1]}(a^{[1]})^{\omega}\right)^2 \le \omega\sum^\omega_{m=2} O(\omega\crude)^{O(\omega)} \cdot {\epsilon'}^2 \le O(\omega\crude)^{O(\omega)}\cdot{\epsilon'}^2.
    \end{equation}
    so by Fact~\ref{fact:root1} we have a degree-$O(\omega)$ SoS proof that $C^{[1]}(a^{[1]}) = 1 \pm O(\omega\crude)^{O(\omega)} \cdot \epsilon'$, which subsequently implies by \eqref{eq:Cmam} that $C^{[m]}(a^{[m]}) = 1 \pm O(\omega\crude)^{O(\omega)} \cdot \epsilon'$ for all $m$.
    
    To complete the analysis of Case 2, we verify that for any $(i'_1,\ldots,i'_\omega)\not\in[2]^{\omega}$, $C^{[m]}(i'_1\cdots i'_\omega)$ is small for all $m$. Suppose without loss of generality that $i'_\omega\not\in[2]$. It suffices to prove that for any $0 < m < \omega$, if there are \emph{at least} $m$ elements among $i'_1,\ldots,i'_{\omega-1}$ which are distinct from 1, then $C^{[m]}(i'_1\cdots i'_\omega)$ is small for all $m$. The reason this suffices is that then, for the remaining case where there are \emph{fewer than} $m$ elements among $i'_1,\ldots,i'_{\omega-1}$ distinct from 1, then by symmetry of exchanging the roles of $1$ and $2$, $C^{[m]}(i'_1\cdots i'_{\omega})$ can be bounded provided $C^{[\omega-m]}(i''_1\cdots i''_{\omega})$ can be bounded, where
    \begin{equation}
        i''_j \triangleq \begin{cases}
            1 & \text{if} \ i'_j = 2 \\
            2 & \text{if} \ i'_j = 1 \\
            i'_j & \text{otherwise},
        \end{cases}
    \end{equation}
    There are at least $\omega - m$ elements among $i''_1,\ldots,i''_{\omega-1}$ distinct from 1, so we would conclude that $C^{[\omega-m]}(i''_1\cdots i''_{\omega})$ is also sufficiently small. 
    
    Finally, to show that $C^{[m]}(i'_1\cdots i'_\omega)$ is small when there are at least $m$ elements distinct from 1 among $i'_1,\ldots,i'_{\omega-1}$, we proceed by induction on $m$. The base case of $m = 0$ is immediate: we have by Case 1 that $C^{[0]}(i'_1\cdots i'_\omega) = \pm \epsilon'$. 
    
    In general, the degree-$m$ coefficient of \eqref{eq:rank1combo_odd} for $(a_1,\ldots,a_\omega) = (1,\ldots,1)$ and $(a'_1,\ldots,a'_\omega) = (i'_1,\ldots,i'_\omega)$ implies that
    \begin{multline}
        C^{[0]}(1\cdots 1) C^{[m]}(i'_1\cdots i'_\omega) + \sum^m_{j=1} C^{[j]}(1\cdots 1) C^{[m-j]}(i'_1\cdots i'_\omega) \\
        = C^{[0]}(1\cdots 1 i'_\omega) C^{[m]}(i'_1\cdots i'_{\omega-1}1) + \sum^m_{j=1} C^{[j]}(1\cdots 1 i'_\omega) C^{[m-j]}(i'_1\cdots i'_{\omega-1}1) \pm O(\omega)^{O(\omega)}\cdot \epsrel^*. \label{eq:case2_corner}
    \end{multline}
    We can bound the summations on both sides of \eqref{eq:case2_corner}, as well as the first term on the right-hand side. By an identical recurrence as in the proof of \eqref{eq:case2_induct}, we conclude from \eqref{eq:case2_corner} that $C^{[m]}(i'_1\cdots i'_\omega) = \pm O(\omega\crude)^{O(\omega)}\cdot\epsilon'$ as desired.
    
    \textbf{Case 3:} $(i_1,\ldots,i_{\omega})$ consists of $D > 2$ distinct indices. Without loss of generality suppose that these indices are $1,\ldots,D$ which appear $c_1,\ldots,c_{D}$ times respectively. Following the notation introduced at the beginning of Section~\ref{sec:prelims}, we will denote the string $i_1\cdots i_{\omega}$ by $1^{c_1}\cdots D^{c_D}$.
    
    For $0 \le z_1,\ldots,z_D \le 1$ that we will vary, consider the tensor
    \begin{equation}
        T' \triangleq F_U((z_1 e_1 + \cdots + z_D e_D)^{\otimes \omega}).
    \end{equation}
    Analogous to \eqref{eq:rank1combo_odd}, we will use the following consequence of Lemma~\ref{lem:perm}:
    \begin{equation}
        T'(a_1\cdots a_{\omega}) T'(a'_1\cdots a'_{\omega}) = T'(a_1\cdots a_{\omega-1} a'_{\omega}) T'(a'_1\cdots a'_{\omega-1} a_{\omega}) \pm O(\omega)^{O(\omega)}\cdot \epsrel^* \label{eq:rank1combo_odd2}
    \end{equation}
    for various choices of $a_1,a'_1,\ldots,a_\omega,a'_\omega$ (note that the factor of $O(\omega)^{O(\omega)}$ comes from the fact that $\norm{(z_1 e_1+\cdots z_D e_D)^{\otimes\omega}}^2_F \le O(\omega)^{O(\omega)}$, so we must scale the $T$ in Lemma~\ref{lem:perm} appropriately).
    
    Similar to Case 2, we would like to isolate certain monomials in the monomial expansion of \eqref{eq:rank1combo_odd2} as a polynomial in the variables $z_1,\ldots,z_D$. 
    By Corollary~\ref{cor:general_vandermonde} for $e = 2\omega$, for any term $z_1^{c_1}\cdots z_D^{c_D}$ in the monomial expansion of \eqref{eq:rank1combo_odd2}, there is a linear combination of the constraints~\ref{eq:rank1combo_odd2} for $(z_1,\ldots,z_D) \in \brc{1/(2\omega+1),2/(2\omega+1),\ldots,1}^D$ that implies that the coefficient for $z_1^{c_1}\cdots z_D^{c_D}$ is bounded by
    \begin{equation}
        O(\omega)^{O(\omega D)}\cdot \epsrel^* \le O(\omega)^{O(\omega^2)}\cdot \epsrel^*.
    \end{equation}
    We will refer to this coefficient as the ``degree-$(c_1,\ldots,c_D)$ coefficient of \eqref{eq:rank1combo_odd2}.'' 
    
    For any $i'_1\le \cdots\le i'_\omega$, consider the degree-$(c_1,\ldots,c_D)$ coefficient of \eqref{eq:rank1combo_odd2} for $a_1\cdots a_\omega = a\cdots a$ and $a'_1\cdots a'_{\omega} = i'\cdots i'_\omega$ for some $a\in[r]$. For example, for $a = 1$, this can be written as
    % \begin{multline}
    %     C^{1\cdots 1}(1\cdots 1) C^{1^{c_1}\cdots D^{c_D}}(1^{c'_1}\cdots D^{c'_D}) \\
    %     + \sum_{(b_1,\ldots,b_D)\neq (0,\ldots,0)} C^{1^{\omega - b_1} 2^{b_2}\cdots D^{b_D}}(1\cdots 1) C^{1^{c_1+b_1} 2^{c_2-b_2}\cdots D^{c_D - b_D}}(1^{c'_1}\cdots D^{c'_D}) \\
    %     = C^{1\cdots 1}(1\cdots 1D) C^{1^{c_1}\cdots D^{c_D}}(1^{c'_1+1}2^{c'_2}\cdots(D-1)^{c'_{D-1}} D^{c'_D-1}) \\
    %     + \sum_{(b_1, b_2,\ldots,b_D)\neq (0,\ldots,0)} C^{1^{\omega - b_1} 2^{b_2}\cdots D^{b_D}}(1\cdots 1D) C^{1^{c_1+b_1} 2^{c_2-b_2}\cdots D^{c_D - b_D}}(1^{c'_1+1}2^{c'_2}\cdots(D-1)^{c'_{D-1}} D^{c'_D-1}), \label{eq:case3}
    % \end{multline}
    \begin{multline}
        C^{1\cdots 1}(1\cdots 1) C^{1^{c_1}\cdots D^{c_D}}(i'_1\cdots i'_\omega) 
        + \sum_{b_1,\ldots,b_D \neq 0} C^{1^{\omega - b_1} 2^{b_2}\cdots D^{b_D}}(1\cdots 1) C^{1^{c_1+b_1} 2^{c_2-b_2}\cdots D^{c_D - b_D}}(i'_1\cdots i'_\omega) \\
        = C^{1\cdots 1}(1\cdots 1 i'_\omega) C^{1^{c_1}\cdots D^{c_D}}(i'_1\cdots i'_{\omega-1} 1) \\
        + \sum_{b_1,\ldots,b_D \neq 0} C^{1^{\omega - b_1} 2^{b_2}\cdots D^{b_D}}(1\cdots 1 i'_\omega) C^{1^{c_1+b_1} 2^{c_2-b_2}\cdots D^{c_D - b_D}}(i'_1\cdots i'_{\omega-1} 1) \pm O(\omega)^{O(\omega^2)}\cdot \epsrel^*, \label{eq:case3}
    \end{multline}
    where the summations are over all nonzero tuples $(b_1,b_2,\ldots,b_{D})\in \brc{0,\ldots,\omega}^{D}$ for which the summands are well-defined. Note that for this to hold, we must have $b_2 \le c_2,\ldots, b_{D} \le c_{D}$ and
    \begin{equation}
        b_1 = b_2 + \cdots + b_{D}
    \end{equation}
    which implies that $b_1 > 0$.
    
    Suppose inductively that we have shown for any $(i''_1,\ldots,i''_\omega)$ with Hamming distance at most $t - 1$ from $(1,\ldots,1)$ that for some $\tmp_{t-1} > 0$,
    \begin{equation}
        C^{i''_1\cdots i''_\omega}(i'_1\cdots i'_\omega) = \pm \tmp_{t-1} \label{eq:dist_induct} %O(r^{\omega}\crude)^{t-1}\cdot \epsilon'. 
    \end{equation} 
    for all $1\le i'_1 \le\cdots \le i'_\omega\le r$ such that $c'_1 < c''_1$, where $c'_j$ denotes the number of copies of $j$ among $i'_1,\ldots,i'_\omega$ and $c''_j$ denotes the number of copies of $j$ among $i''_1,\ldots,i''_\omega$. This certainly holds for $t - 1 \le 1$, by Cases 1 and 2, with $\tmp_0 \le \epsilon'$ and $\tmp_1 \le O(\omega\crude)^{O(\omega)}\cdot \epsilon'$.
    
    Suppose $i_1,\ldots,i_\omega$ has Hamming distance $t$ from $(1,\ldots,1)$, and consider any $i'_1,\ldots,i'_\omega$ for which $c'_1 < c_1$. 
    
    We first show that the summations on both sides of \eqref{eq:case3} are small. Because $b_1 > 0$, the string $1^{c_1+b_1}2^{c_2-b_2}\cdots D^{c_{D}-b_{D}}$ has strictly smaller Hamming distance to $1^{\omega}$ than does $i_1\cdots i_\omega = 1^{c_1}\cdots D^{c_D}$, and $c'_1 < c'_1 + 1 < c_1 + b_1$. Therefore, by the inductive hypothesis we know that
    \begin{equation}
        C^{1^{c_1+b_1}2^{c_2-b_2}\cdots D^{c_D - b_D}}(i'_1\cdots i'_\omega) = \pm \tmp_{t-1} %O(r^{\omega}\crude)^{t-1}\cdot \epsilon'
    \end{equation}
    \begin{equation}
        C^{1^{c_1+b_1}2^{c_2-b_2}\cdots D^{c_D - b_D}}(i'_1\cdots i'_{\omega-1}1) = \pm \tmp_{t-1}. % O(r^{\omega}\crude)^{t-1}\cdot \epsilon'.
    \end{equation}
    So the summations on both sides of \eqref{eq:case3} each contribute $2(\omega+1)^D\crude \tmp_{t-1} \le O(\omega)^{\omega}\crude\tmp_{t-1}$ to \eqref{eq:case3}.
    
    Next, note that the first term on the right-hand side of \eqref{eq:case3} is $O(\crude\epsilon')$: we have $1\cdots 1i'_{\omega} \neq 1^{\omega}$ because $i'_{\omega} \neq 1$ by assumption (otherwise $i'_1 = \cdots = i'_\omega = 1$, contradicting the assumption that $c'_1 < c_1$), so $C^{1\cdots 1}(1\cdots 1i'_\omega) = \pm \epsilon'$ by Case 1.
    
    Putting everything together, we conclude from \eqref{eq:case3} that
    \begin{equation}
        C^{1\cdots 1}(1\cdots 1) C^{1^{c_1}\cdots D^{c_D}}(i'_1\cdots i'_\omega) = O(\omega)^{\omega}\crude\tmp_{t-1} + \crude\epsilon' %\pm O(r^{\omega}\crude)^t\cdot \epsilon',
    \end{equation}
    so by Part~\ref{fact:divide_both_sides_big} of Fact~\ref{fact:division},
    \begin{equation}
        C^{1^{c_1}\cdots D^{c_D}}(i'_1\cdots i'_\omega) = \pm O(\omega)^{\omega}\crude\tmp_{t-1} + O(\crude\epsilon'),
    \end{equation}
    completing the induction for $\tmp_t = O(\omega)^{\omega} \crude\tmp_{t-1} + O(\crude\epsilon')$. 
    % Unrolling this recurrence, this shows that $\tmp_t \le O(\omega\crude)^{O(\omega t)}\cdot\epsilon'$.
    
    We have thus shown that 
    \begin{equation}
        C^{\mb{i}}(\mb{i}') = \pm O(\omega\crude)^{O(\omega^2)}\cdot\epsilon' \ \forall \ \text{sorted} \ \mb{i},\mb{i}'\in\strings \ \text{s.t.} \ c_1 > c'_1, \label{eq:case3_first_part}
    \end{equation}
    where $c_1$ and $c'_1$ denote the number of appearances of 1 in $\mb{i}$ and $\mb{i}'$. By symmetry, \eqref{eq:case3_first_part} also holds if $c_1 < c'_1$.
    
    It remains to consider the case of $c_1 = c'_1$. We proceed by induction on the Hamming distance between $(i_1,\ldots,i_\omega)$ and $(1,\ldots,1)$. Suppose we have shown for some $\tmp_{t-1} > 0$ that for any $(i''_1,\ldots,i''_\omega)$ with Hamming distance at most $t - 1$ from $(1,\ldots,1)$ that
    \begin{equation}
        C^{i''_1\cdots i''_\omega}(i'_1\cdots i'_\omega) = \bone{i'_s = i''_s \ \forall \ s\in[\omega]} \pm \tmp_{t-1}
    \end{equation}
    for all $1 \le i'_1\le \cdots \le i'_\omega\le r$ such that $c'_1 = c''_1$.
    
    Now suppose $i_1,\ldots,i_\omega$ has Hamming distance $t$ from $(1,\ldots,1)$, and consider any $i'_1,\ldots,i'_\omega$ for which $c'_1 < c_1$. Consider again \eqref{eq:case3}. Because $b_1 > 0$, we can apply \eqref{eq:case3_first_part} to bound every $C^{1^{\omega-b_1}2^{b_2}\cdots D^{b_D}}(1\cdots 1)$ in the summation on the left-hand side of \eqref{eq:case3} by $O(\omega\crude)^{O(\omega^2)}\cdot\epsilon'$. Likewise, for the first term on the right-hand side of \eqref{eq:case3}, $C^{1^{c_1}\cdots D^{c_D}}(i'_1\cdots i'_{\omega-1}1) = \pm O(\omega\crude)^{O(\omega^2)}\cdot\epsilon'$ by \eqref{eq:case3_first_part} because $i'_1\cdots i'_{\omega-1}1$ has $c'_1 + 1 \neq c_1$ copies of 1.
    
    For every summand on the right-hand side of \eqref{eq:case3} for which $b_1 > 1$, we can likewise bound $C^{1^{\omega-b_1}2^{b_2}\cdots D^{b_D}}(1\cdots 1i'_\omega) = \pm O(\omega\crude)^{O(\omega^2)}\cdot\epsilon'$ using \eqref{eq:case3_first_part}. And for the summands on the right-hand side of \eqref{eq:case3} for which $b_1 = 1$ but for which $1^{\omega-b_1}2^{b_2}\cdots D^{b_D} \neq 1\cdots 1 i'_{\omega}$, we can bound them by $\crude\tmp_{t-1}$ by the inductive hypothesis.
    If $i'_{\omega} > D$, then we have accounted for all summands on the right-hand side of \eqref{eq:case3}, and we conclude from \eqref{eq:case3}
    \begin{equation}
        C^{1\cdots 1}(1\cdots 1) C^{1^{c_1}\cdots D^{c_D}}(i'_1\cdots i'_\omega) = \pm \left(O(\omega\crude)^{O(\omega^2)}\cdot\epsilon' + O(\omega)^{O(\omega)}\crude\tmp_{t-1}\right). \label{eq:iomega_not_ell}
    \end{equation}
    Otherwise, if $1 \le i'_\omega \le D$, suppose without loss of generality that $i'_\omega = D$. Then 
    \eqref{eq:case3} becomes
    \begin{multline}
        C^{1\cdots 1}(1\cdots 1) C^{1^{c_1}\cdots D^{c_D}}(i'_1\cdots i'_\omega)
        = C^{1\cdots 1D}(1\cdots 1D) C^{1^{c_1+1} 2^{c_2}\cdots (D-1)^{c_{D-1}} D^{c_D-1}}(i'_1\cdots i'_{\omega-1} 1) \\ \pm \biggl[O(\omega\crude)^{O(\omega^2)}\cdot\epsilon' + O(\omega)^{O(\omega)}\crude\tmp_{t-1}\biggr]. \label{eq:simplify_case3}
    \end{multline}
    We have already shown in Case 2 that \begin{equation}
        C^{1\cdots 1D}(1\cdots 1D) = 1\pm O(\omega\crude)^{O(\omega)}\cdot \epsilon'
    \end{equation} And by the inductive hypothesis,
    \begin{align}
        \MoveEqLeft C^{1^{c_1+1}2^{c_2}\cdots(D-1)^{c_{D-1}}D^{c_{D}-1}}(i'_1\cdots i'_{\omega-1}1) \\
        &= \bone{i'\cdots i'_{\omega-1} 1 = 1^{c_1+1}2^{c_2}\cdots (D-1)^{c_{D-1}}D^{c_D-1}} \pm O(\tmp_{t-1}) \\
        &= \bone{i'_s = i_s \ \forall \ s\in[\omega]} \pm O(\tmp_{t-1}).
    \end{align}
    By \eqref{eq:simplify_case3}, we conclude that
    \begin{equation}
        C^{1^{c_1}\cdots D^{c_\ell}}(i'_1\cdots i'_\omega) = \bone{i'_s = i_s \ \forall \ s\in[\omega]} \pm O(\omega\crude)^{O(\omega^2)}\cdot\epsilon' + O(\omega)^{O(\omega)}\crude\tmp_{t-1}. \label{eq:iomega_ell}
        % \\
        % &= \bone{i'_s = i_s \ \forall \ s\in[\omega]} \pm  (O(r^{\omega}\crude)^{\omega+1}\cdot \epsilon' + 1)\cdot \crude^t \cdot \epsilon' \\
        % &= O(\crude^t\epsilon'),
    \end{equation}
    By \eqref{eq:iomega_not_ell} and \eqref{eq:iomega_ell}, we have completed the induction for $\tmp_t \triangleq O(\omega\crude)^{O(\omega^2)}\cdot\epsilon' + O(\omega)^{O(\omega)}\crude\tmp_{t-1}$. Unrolling this recurrence, we conclude that 
    \begin{align}
        C^{i''_1\cdots i''_\omega}(i'_1\cdots i'_\omega) &= \bone{i'_s = i''_s \ \forall \ s\in[\omega]} \pm O(\omega\crude)^{O(\omega^2)}\cdot\epsilon'. \qedhere
    \end{align}
\end{proof}

\subsection{Proof of Lemma~\ref{lem:UsendsF}}
\label{app:defer_UsendsF}

\begin{proof}
    For convenience, we will refer to $T^*_a, T_a, F^*_a, F_a$ as $T^*, T, F^*, F$. Denoting $A\triangleq T - F_U(T^*)$ so that $\norm{A}^2_F \le \epsmap^2$ by Lemma~\ref{lem:push_almost_map}, we have
    \begin{align}
    % \MoveEqLeft \norm{F - \sum_{\mb{j}\in[r]^{\omega-2}} \left(F_U(T^*)(e_{j_1},\ldots,e_{j_{\omega-2}},:,:)\right)^2}^2_F \\
        \MoveEqLeft \biggl\|F - \biggl(\sum_{\mb{j}\in[r]^{\omega'}} F_U(T^*)_{\mb{j}\mb{j}:} \biggr)\biggl(\sum_{\mb{j}\in[r]^{\omega'}} F_U(T^*)_{\mb{j}\mb{j}:} \biggr)^{\top}\biggr\|^2_F = \sum^r_{x,y=1} \biggl(\sum_{\mb{j},\mb{j}'} T_{\mb{j}\mb{j}x}T_{\mb{j}'\mb{j}'y} - F_U(T^*)_{\mb{j}\mb{j}x} F_U(T^*)_{\mb{j}'\mb{j}'y}\biggr)^2 \\
        &= \sum_{x,y}\biggl(\sum_{\mb{j},\mb{j}'} T_{\mb{j}\mb{j}x}A_{\mb{j}'\mb{j}'y} + A_{\mb{j}\mb{j}x}T_{\mb{j}'\mb{j}'y} + A_{\mb{j}\mb{j}x}A_{\mb{j}'\mb{j}'y}\biggr)^2 \\
        &\le 3r^{\omega-1}\sum_{x,y,\mb{j},\mb{j}'} T^2_{\mb{j}\mb{j}x} A^2_{\mb{j}'\mb{j}'y} + A^2_{\mb{j}\mb{j}x}T^2_{\mb{j}'\mb{j}'y} + A^2_{\mb{j}\mb{j}x}A^2_{\mb{j}'\mb{j}'y} \le r^{O(\omega)}\radius^2\epsmap^2. \label{eq:Fmap_error}
    \end{align}
    We now consider the vector $\sum_{\mb{j}\in[r]^{\omega'}} F_U(T^*)_{\mb{j}\mb{j}:}$. For any $x\in[r]$, its $x$-th entry is given by
    \begin{equation}
        % \sum_{\mb{j}\in[r]^{\omega-2}, z\in[r], \mb{k},\mb{k}'\in\strings} U^{\mb{k}}_{\mb{j}xz} U^{\mb{k}'}_{\mb{j}zy} T^*_{\mb{k}} T^*_{\mb{k}'}. 
        \sum_{\mb{j}\in[r]^{\omega'}, \mb{k}\in\strings} U^{\mb{k}}_{\mb{j}\mb{j}x} T^*_{\mb{k}} .\label{eq:Fxy}
    \end{equation}
    To ease notation, define
    \begin{equation}
        \calE^{\mb{k}} \triangleq U^{\mb{k}} - \frac{1}{\num{k}}\sum_{\mb{j}\in\strings: \sort{j} = \sort{k}} \wt{U}^{k_1}\otimes\cdots\otimes \wt{U}^{k_\omega}
    \end{equation}
    \begin{equation}
        \delta_{ij} \triangleq \iprod{\wt{U}^i,\wt{U}^j} - \bone{i = j}.
    \end{equation}
    \begin{equation}
        \upsilon_{\mb{i}} \triangleq \biggl(\prod^{\omega'}_{s=1} (\bone{i_{2s-1} = i_{2s}} + \delta_{i_{2s-1}i_{2s}})\biggr) - \prod^{\omega'}_{s=1} \bone{i_{2s-1} = i_{2s}}.
    \end{equation}
    Then we can rewrite \eqref{eq:Fxy} as
    \begin{multline}
        \sum_{\mb{j}\in[r]^{\omega'},\mb{k}\in\strings} T^*_{\mb{k}} \biggl(\calE^{\mb{k}}_{\mb{j}\mb{j}x} + \frac{1}{\num{k}}\sum_{\mb{i}:\sort{i} = \sort{k}} \wt{U}^{i_1}_{j_1}\wt{U}^{i_2}_{j_1} \cdots \wt{U}^{i_{\omega-2}}_{j_{\omega'}}\wt{U}^{i_{\omega-1}}_{j_{\omega'}}\wt{U}^{i_{\omega}}_x\biggr)  \\
        = \sum_{\mb{j}\in[r]^{\omega'},\mb{i}\in\strings} T^*_{\mb{i}}\bigl(\calE^{\mb{i}}_{\mb{j}\mb{j}x} + \wt{U}^{i_1}_{j_1}\wt{U}^{i_2}_{j_1} \cdots \wt{U}^{i_{\omega-2}}_{j_{\omega'}}\wt{U}^{i_{\omega-1}}_{j_{\omega'}}\wt{U}^{i_{\omega}}_x\bigr),
        %\left(\calE^{\mb{k}'}_{\mb{j}'\mb{j}'y} + \frac{1}{\num{k'}}\sum_{\mb{i}':\sort{i}' = \sort{k}'} \wt{U}^{i'_1}_{j'_1}\wt{U}^{i'_2}_{j'_1} \cdots \wt{U}^{i'_{\omega-2}}_{j'_{\omega'}}\wt{U}^{i'_{\omega-1}}_{j'_{\omega'}}\wt{U}^{i'_{\omega}}_y\right)
        \label{eq:split_U_E}
        % \sum_{\mb{j}\in[r]^{\omega-2},z\in[r],\mb{k},\mb{k}'\in\strings} \left(\wt{U}^{k_1}_{j_1}\cdots \wt{U}^{k_{\omega-2}}_{j_{\omega-2}} \wt{U}^{k_{\omega-1}}_x \wt{U}^{k_{\omega}}_z + \calE^{\mb{k}}_{\mb{j}xz}\right)\cdot \left(\wt{U}^{k'_1}_{j_1}\cdots \wt{U}^{k'_{\omega-2}}_{j_{\omega-2}} \wt{U}^{k'_{\omega-1}}_z \wt{U}^{k'_{\omega}}_y + \calE^{\mb{k}'}_{\mb{j}zy}\right)T^*_{\mb{k}}T^*_{\mb{k}'}. \label{eq:split_U_E}
    \end{multline}
    where we used the fact that $T^*$ is symmetric.
    Note that
    \begin{align}
        \sum_{\mb{j}\in[r]^{\omega'},\mb{i}\in\strings} T^*_{\mb{i}}\wt{U}^{i_1}_{j_1}\wt{U}^{i_2}_{j_1} \cdots \wt{U}^{i_{\omega-2}}_{j_{\omega'}}\wt{U}^{i_{\omega-1}}_{j_{\omega'}}\wt{U}^{i_{\omega}}_x &= \sum_{\mb{i}\in\strings} \wt{U}^{i_{\omega}}_x T^*_{\mb{i}}\prod^{\omega'}_{s=1} (\bone{i_{2s-1} = i_{2s}} + \delta_{i_{2s-1}i_{2s}}) \\
        &= \sum_{\mb{i}\in\strings} \wt{U}^{i_\omega}_x T^*_{\mb{i}} \biggl(\upsilon_{\mb{i}} + \prod^{\omega'}_{s=1}\bone{i_{2s-1}=i_{2s}}\biggr) \label{eq:upsilons}
    \end{align}
    and that for any $x,y\in[r]$,
    \begin{align}
        \biggl(\sum_{\mb{i}\in\strings} \wt{U}^{i_\omega}_x T^*_\mb{i}\prod^{\omega'}_{s=1}\bone{i_{2s-1}=i_{2s}}\biggr)\biggl(\sum_{\mb{i}\in\strings} \wt{U}^{i_\omega}_y T^*_\mb{i}\prod^{\omega'}_{s=1}\bone{i_{2s-1}=i_{2s}}\biggr) &= \sum_{i_\omega,i'_\omega\in[r]} \wt{U}^{i_\omega}_x F^*_{i_\omega i'_\omega} \wt{U}^{i_\omega}_y\\ &= (\wt{U}^{\top}F^*\wt{U})_{xy}. \label{eq:get_UFU}
    \end{align}
    It remains to control the error terms involving $\calE$'s and $\upsilon$'s from \eqref{eq:split_U_E} and \eqref{eq:upsilons}. For the former, we can bound
    \begin{align}
        \MoveEqLeft \biggl(\sum_{\mb{j},\mb{j}'\in[r]^{\omega'}, \mb{i},\mb{i}'\in\strings} T^*_{\mb{i}}\calE^{\mb{i}}_{\mb{j}\mb{j}x}\cdot T^*_{\mb{i}'}\wt{U}^{i'_1}_{j'_1}\wt{U}^{i'_2}_{j'_1}\cdots \wt{U}^{i'_{\omega-2}}_{j'_{\omega'}}\wt{U}^{i'_{\omega-1}}_{j'_{\omega'}}\wt{U}^{i'_\omega}_y\biggr)^2 \\
        &\le \biggl(\sum_{\mb{j},\mb{j}',\mb{i},\mb{i}'} (\calE^{\mb{i}}_{\mb{j}\mb{j}x})^2 \cdot (\wt{U}^{i'_1}_{j'_1}\wt{U}^{i'_2}_{j'_1}\cdots \wt{U}^{i'_{\omega-2}}_{j'_{\omega'}}\wt{U}^{i'_{\omega-1}}_{j'_{\omega'}}\wt{U}^{i'_\omega}_y)^2\biggr)\cdot r^{\omega-1}\cdot \norm{T^*}^4_F \le r^{O(\omega)}\cdot \radius^4\cdot {\epstrueouter}^2, \label{eq:crossbound1}
    \end{align}
    where in the last step we used Lemma~\ref{lem:tensorouter},
    and similarly
    \begin{equation}
        \biggl(\sum_{\mb{j},\mb{j}'\in[r]^{\omega'}, \mb{i},\mb{i}'\in\strings} T^*_{\mb{i}'}\calE^{\mb{i}'}_{\mb{j}'\mb{j}'y}\cdot T^*_{\mb{i}}\wt{U}^{i_1}_{j_1}\wt{U}^{i_2}_{j_1}\cdots \wt{U}^{i_{\omega-2}}_{j_{\omega'}}\wt{U}^{i_{\omega-1}}_{j_{\omega'}}\wt{U}^{i_\omega}_y\biggr)^2 \le r^{O(\omega)}\cdot \radius^4\cdot {\epstrueouter}^2, \label{eq:crossbound2}
    \end{equation}
    and
    \begin{align}
        \biggl(\sum_{\mb{j},\mb{j}'\in[r]^{\omega'},\mb{i},\mb{i}'\in\strings} T^*_{\mb{i}}\calE^{\mb{i}}_{\mb{j}\mb{j}x}\cdot T^*_{\mb{i}'} \calE^{\mb{i}'}_{\mb{j}'\mb{j}'y}\biggr)^2 \le \biggl(\sum_{\mb{j},\mb{j}',\mb{i},\mb{i}'} (\calE^{\mb{i}}_{\mb{j}\mb{j}x})^2 (\calE^{\mb{i}'}_{\mb{j}'\mb{j}'y})^2\biggr) \cdot r^{\omega-1}\norm{T^*}^4_F \le r^{O(\omega)}\cdot \radius^4 \cdot {\epstrueouter}^4. \label{eq:crossbound3}
    \end{align}
    Finally, for the error terms involving $\upsilon$'s from \eqref{eq:upsilons},
    \begin{align}
        \biggl(\sum_{\mb{i},\mb{i}'\in\strings}\wt{U}^{i_\omega}_x T^*_{\mb{i}}\upsilon_{\mb{i}}\cdot \wt{U}^{i'_\omega}_y T^*_{\mb{i}'}\prod^{\omega'}_{s=1} \bone{i'_{2s-1} = i'_{2s}}\biggr)^2 \le \biggl(\sum_{\mb{i},\mb{i}'} \upsilon^2_{\mb{i}} (\wt{U}^{i_\omega}_x \wt{U}^{i'_\omega}_y)^2\biggr)\norm{T^*}^4_F \le r^{O(\omega)}\radius^4\cdot {\epstrueort}^2 \label{eq:crossbound4}
    \end{align}
    where in the last step we used that for any $\mb{i}\in\strings$,
    \begin{align}
        \upsilon_{\mb{i}}^2 &= \biggl(\sum_{\emptyset\neq S\subseteq[\omega']} \prod_{s\in S} \delta_{i_{2s-1}i_{2s}} \cdot \prod_{t\not\in S} \bone{i_{2t-1} = i_{2t}}\biggr)^2 \le 2^{\omega'}\sum_{S\neq\emptyset} \prod_{s\in S}\delta^2_{i_{2s-1}i_{2s}} \\
        &\le 2^{\omega'}\biggl(\prod^{\omega'}_{s=1}(1 + \delta^2_{i_{2s-1}i_{2s}}) - 1\biggr) \le 2^{O(\omega)}\cdot {\epstrueort}^2,
    \end{align}
    where the last step follows by Corollary~\ref{cor:dd_orth_odd}.
    Similarly
    \begin{equation}
        \biggl(\sum_{\mb{i},\mb{i}'\in\strings}\wt{U}^{i_\omega}_x T^*_{\mb{i}}\upsilon_{\mb{i}}\cdot \wt{U}^{i'_\omega}_y T^*_{\mb{i}'}\prod^{\omega'}_{s=1} \bone{i'_{2s-1} = i'_{2s}}\biggr)^2 \le r^{O(\omega)}\radius^4{\epstrueort}^2 \label{eq:crossbound5}
    \end{equation}
    and
    \begin{equation}
        \biggl(\sum_{\mb{i},\mb{i}'\in\strings}\wt{U}^{i_\omega}_x T^*_{\mb{i}}\upsilon_{\mb{i}} \cdot \wt{U}^{i'_\omega}_y T^*_{\mb{i}'}\upsilon_{\mb{i}'}\biggr)^2 \le \biggl(\sum_{\mb{i},\mb{i}'} \upsilon_{\mb{i}}^2\upsilon_{\mb{i}'}^2\biggr)\biggl(\sum_{\mb{i},\mb{i}'} (T^*_{\mb{i}} T^*_{\mb{i}'})^2 \cdot (\wt{U}^{i_\omega}_x \wt{U}^{i'_{\omega}}_y)^2\biggr) \le r^{O(\omega)}\radius^4{\epstrueort}^2. \label{eq:crossbound6}
    \end{equation}
    The lemma follows from combining \eqref{eq:Fmap_error}, \eqref{eq:split_U_E}, \eqref{eq:upsilons}, and \eqref{eq:get_UFU} with the error bounds \eqref{eq:crossbound1}, \eqref{eq:crossbound2}, \eqref{eq:crossbound3}, \eqref{eq:crossbound4}, \eqref{eq:crossbound5}, \eqref{eq:crossbound6}.
\end{proof}

\subsection{Proof of Lemma~\ref{lem:push_gram}}
\label{app:defer_push_gram}

\begin{proof}
    Let $\Delta_a \triangleq F_a - \wt{U}F^*_a\wt{U}^{\top}$ for any $a\in[d]$, and let $\calE \triangleq \wt{U}^{\top}\wt{U} - \Id$. We have
    \begin{align}
        \iprod{F_a,F_b} &= \iprod{\wt{U}F^*_a\wt{U}^{\top},\wt{U}F^*_b\wt{U}^{\top}} + \iprod{\Delta_a,\wt{U}F^*_b\wt{U}^{\top}} +\iprod{\wt{U}F^*_a\wt{U}^{\top},\Delta_b} + \iprod{\Delta_a,\Delta_b} \\
        &= \iprod{F^*_a,F^*_b} + \Tr(\calE F^*_a F^*_b) + \Tr(F^*_a \calE F^*_b) + \Tr(\calE F^*_a \calE F^*_b) \\
        & \qquad \qquad + \iprod{\Delta_a,\wt{U}F^*_b\wt{U}^{\top}} +\iprod{\wt{U}F^*_a\wt{U}^{\top},\Delta_b} + \iprod{\Delta_a,\Delta_b} \label{eq:FstarFstar_plusterms}
    \end{align}
    It remains to bound the error terms in \eqref{eq:FstarFstar_plusterms}. For the ones involving $\calE$, note that $\norm{\calE}_{\max} \le O(\epstrueort)$ by Corollary~\ref{cor:dd_orth_odd}, so
    \begin{equation}
        \Tr(\calE F^*_a F^*_b)^2 \le \norm{\calE}^2_F \norm{F^*_a F^*_b}^2_F \le O(r^2{\epstrueort}^2)\cdot \norm{f^*_a}^4\norm{f^*_b}^4 \le r^{O(\omega)}\radius^8 {\epstrueort}^2
    \end{equation} and similarly
    \begin{equation}
        \Tr(F^*_a\calE F^*_b)^2 \le r^{O(\omega)}\radius^8 {\epstrueort}^2,
    \end{equation}
    and
    \begin{equation}
        \Tr(\calE F^*_a \calE F^*_b)^2 = \norm{\calE}^4_2\norm{F^*_a}^2_F \norm{F^*_b}^2_F \le O(r^{O(\omega)}\radius^8 {\epstrueort}^4).
    \end{equation}
    For the error terms in \eqref{eq:FstarFstar_plusterms} involving $\Delta_a$, note that $\norm{\Delta_a}^2_F \le r^{O(\omega)}\radius^4({\epstrueouter}^2 + {\epstrueort}^2 + \epsmap^2)$ by Lemma~\ref{lem:UsendsF}, so
    \begin{equation}
        \iprod{\Delta_a,\wt{U}F^*_b\wt{U}^{\top}}^2 \le \norm{\Delta_a}^2_F\norm{\wt{U}F^*_b\wt{U}^{\top}}^2_F \le \norm{\Delta_a}^2_F\cdot  (2\norm{F^*_b}^2 + 2\norm{\Delta_b}^2) \le r^{O(\omega)}\radius^8({\epstrueouter}^2 + {\epstrueort}^2 + \epsmap^2)
    \end{equation}
    and similarly
    \begin{equation}
        \iprod{\wt{U}F^*_a\wt{U}^{\top},\Delta_b}^2 \le r^{O(\omega)}\radius^8({\epstrueouter}^2 + {\epstrueort}^2 + \epsmap^2)
    \end{equation}
    and
    \begin{equation}
        \iprod{\Delta_a,\Delta_b}^2 \le \norm{\Delta_a}^2_F \norm{\Delta_b}^2_F \le r^{O(\omega)}\radius^8({\epstrueouter}^2 + {\epstrueort}^2 + \epsmap^2).
    \end{equation}
    Combining these error estimates with \eqref{eq:FstarFstar_plusterms} yields the lemma.
\end{proof}

\subsection{Proof of Corollary~\ref{cor:UFFU}}
\label{app:defer_UFFU}

\begin{proof}
    For convenience, we will denote $F^*_c, F_c$ by $F^*,F$ as the choice of $c$ will be immaterial to the following argument. By Corollary~\ref{cor:UsendsFlam}, Part~\ref{shorthand:bothsides} of Fact~\ref{fact:shorthand}, and the fact that $\norm{\wt{U}}^2_F \le O(r)$ by Corollary~\ref{cor:dd_orth_odd},
    \begin{equation}
        \wt{U}F^*\wt{U}^{\top}\wt{U} \approx_{O(r)^{4\omega}\cdot \radius^4\cdot d({\epstrueort}^2 + {\epstrueouter}^2 + \epsmap^2)} F\wt{U}. \label{eq:ufu1}
    \end{equation}
    By Corollary~\ref{cor:dd_orth_odd}, $\wt{U}^{\top}\wt{U}\approx_{r^2{\epstrueort}^2} \Id$, so by Part~\ref{shorthand:bothsides} of Fact~\ref{fact:shorthand} and the fact that $\norm{\wt{U}F^*}^2 \le \norm{\wt{U}}^2_F \norm{F^*}^2_F \le r\radius^2$,
    \begin{equation}
        \wt{U}F^*\wt{U}^{\top}\wt{U} \approx_{r^3\radius^2{\epstrueort}^2} \wt{U}F^*. \label{eq:ufu2}
    \end{equation}
    The claim then follows by Part~\ref{shorthand:lincombo} applied to \eqref{eq:ufu1} and \eqref{eq:ufu2} (note that the error in \eqref{eq:ufu1} dominates that in \eqref{eq:ufu2}).
    % Let $\Delta\triangleq F - \wt{U}F^*\wt{U}^{\top}$. Then from $\wt{U}F^* \wt{U}^{\top} + \Delta = F$, we have
    % \begin{equation}
    %     \wt{U}F^*\wt{U}^{\top}\wt{U} + \Delta\wt{U} = F\wt{U}.
    % \end{equation}
    % By Corollary~\ref{cor:dd_orth_odd}, $\norm{\wt{U}^{\top}\wt{U} - \Id}^2_F \le r\cdot\norm{\wt{U}^{\top}\wt{U} - \Id}^2_{\max} \le {\epstrueort}^2r$, so
    % \begin{equation}
    %     \norm{\wt{U}F^*\wt{U}^{\top}\wt{U} - \wt{U}F^*}^2_F \le {\epstrueort}^2r \norm{\wt{U}F^*}^2_F \le O({\epstrueort}^2 r \radius^2). \label{eq:diff_right_mult}
    % \end{equation}
    % And by Lemma~\ref{lem:UsendsF} and Corollary~\ref{cor:UsendsFlam}, 
    % \begin{equation}
    %     \norm{\Delta\wt{U}}^2_F \le O(\norm{\Delta}^2_F) \le O(r)^{4\omega}\cdot \radius^4({\epstrueort}^2 + \epsouter^2 + \epsmap^2)
    % \end{equation}
    % which dominates \eqref{eq:diff_right_mult}, completing the proof of the claim.
\end{proof}

\subsection{Proof of Lemma~\ref{lem:push_downleft_Wentry}}
\label{app:defer_push_downleft_Wentry}

\begin{proof}
    For convenience, we denote $F^*_{\lambda}$ and $F_{\lambda}$ by $F^*$ and $F$. By Corollary~\ref{cor:UFFU}, together with Constraint~\ref{constraint:push_diag} and \eqref{eq:Fdiag_gap}, we have that
    \begin{equation}
        \sum_{i,j\in[r]} \left((F_{ii} - F^*_{jj})\wt{U}^j_i\right)^2 \le O(r)^{4\omega}\cdot \radius^4\cdot d({\epstrueort}^2 + \epsouter^2 + \epsmap^2).
    \end{equation}
    In particular, for any $i,j,k,\ell\in[r]$, we have
    \begin{equation}
        (\wt{U}^j_i)^2 (\wt{U}^{\ell}_k)^2\left((F_{ii} - F^*_{jj})^2 +(F_{kk} - F^*_{\ell\ell})^2\right) \le O(r)^{8\omega}\cdot \radius^8\cdot d^2({\epstrueort}^4 + \epsouter^4).
    \end{equation}
    But if $i,j,k,\ell$ satisfy the hypotheses of Lemma~\ref{lem:push_downleft}, then \eqref{eq:push_downleft_pair_zero} follows by Lemma~\ref{lem:push_downleft}.
\end{proof}

\subsection{Proof of Lemma~\ref{lem:Udiag}}
\label{app:defer_Udiag}

\begin{proof}[Proof of Lemma~\ref{lem:Udiag}]
    For any $i,j,\ell$, we have
    \begin{equation}
        \sum_{k\in[r]} (\wt{U}^j_i \wt{U}^{\ell}_k)^2 = (\wt{U}^j_i)^2 \cdot \sum_{k\in[r]} (\wt{U}^{\ell}_k)^2 = (\wt{U}^j_i)^2 \pm O(\epstrueort),\label{eq:push_expandnorm}
    \end{equation}
    where in the last step we used Corollary~\ref{cor:dd_orth_odd}.
    
    If $j > \ell$, then we can upper bound the terms on the left-hand side of \eqref{eq:push_expandnorm} for which $k\ge i$ by Lemma~\ref{lem:push_downleft_Wentry} to conclude that
    \begin{equation}
        (\wt{U}^j_i)^2 \le \sum^{i-1}_{k=1} (\wt{U}^j_i)^2 (\wt{U}^{\ell}_k)^2 \pm O(\epspair^2\cdot r + \epstrueort). \label{eq:push_pre_induct}
    \end{equation}
    Now sum \eqref{eq:push_pre_induct} over $1\le i \le i^*$ for any $i^*\in[r-1]$ to get
    \begin{align}
        \sum^{i^*}_{i=1} (\wt{U}^j_i)^2 &\le \sum^{i^*}_{i=1} (\wt{U}^j_i)^2 \cdot \biggl(\sum^{i-1}_{k=1} (\wt{U}^{\ell}_k)^2\biggr) + O(\epspair^2\cdot r^2 + \epstrueort r) \\
        &= \sum^{i^*-1}_{i=1} (\wt{U}^{\ell}_i)^2\cdot \sum^{i^*}_{k=i+1} (\wt{U}^j_k)^2 + O(\epspair^2\cdot r^2 + \epstrueort r) \\
        &\le \sum^{i^*-1}_{i=1} (\wt{U}^{\ell}_i)^2 + O(\epspair^2\cdot r^2 + \epstrueort r), \label{eq:push_use_for_induct}
    \end{align}
    where in the second step we swapped the summation over $i\in[i^*]$ and the summation over $k\in[i-1]$ and also the swapped the names of the corresponding indices $i$ and $k$, and in the third step we used Corollary~\ref{cor:dd_orth_odd} to upper bound the inner summation over $i+1\le k \le i^*$ by $\norm{\wt{U}^j}^2 \le 1 + O(\epstrueort)$.
    
    Take any $j > i^*$ and take $\ell = j - 1$. Then by \eqref{eq:push_use_for_induct},
    \begin{equation}
        \sum^{i^*}_{i=1} (\wt{U}^j_i)^2 - \sum^{i^*-1}_{i=1} (\wt{U}^{\ell}_i)^2 \le O(\epspair^2\cdot r^2 + \epstrueort r).
    \end{equation}
    As $j - c > i^* - c$ for any $c\in\mathbb{Z}$, we have more generally that for this choice of $j$,
    \begin{equation}
        \sum^{i^*-c}_{i=1} (\wt{U}^j_i)^2 - \sum^{i^*-c-1}_{i=1} (\wt{U}^{\ell}_i)^2 \le O(\epspair^2\cdot r^2 + \epstrueort r). \label{eq:push_telescope}
    \end{equation}
    Summing \eqref{eq:push_telescope} over $c$ from $0$ to $i^* - 1$, altogether we get a degree-$\poly(\ell,\omega)$ SoS proof using the constraints of Program~\ref{program:lastprogram} that
    \begin{equation}
        \sum^{i^*}_{i=1} (\wt{U}^j_i)^2 \le O(\epspair^2\cdot r^3 + \epstrueort r^2).
    \end{equation}
    As the left-hand side is lower bounded by any individual summand, we conclude that for all $i,j\in[r]$ satisfying $j > i$,
    \begin{equation}
        (\wt{U}^j_i)^2 \le O(\epspair^2\cdot r^3 + \epstrueort r^2) = O(\epsoffdiag).
    \end{equation}
    By symmetry, we can also show this holds for $j < i$ in an analogous fashion. This concludes the proof of Lemma~\ref{lem:deg6} for the off-diagonal entries of $\wt{U}$.
    To complete the proof of the lemma, it remains to bound $(U^j_j)^2$ for any $j\in[r]$. But we know by Corollary~\ref{cor:dd_orth_odd} that
    \begin{align}
        (U^j_j)^2 &= \norm{U^j}^2 - \sum_{k\neq j} (U^j_k)^2 = 1 \pm O(\epspair^2\cdot r^4 + \epstrueort\cdot r^3). \qedhere
    \end{align}
\end{proof}

\subsection{Proof of Lemma~\ref{lem:FFstar}}
\label{app:defer_FFstar}

\begin{proof}
    Take any $c\in[d]\cup\brc{\lambda,\mu}$. Recall from Lemma~\ref{lem:UsendsF} and Corollary~\ref{cor:UsendsFlam} that $\norm{\wt{U}F^*_c\wt{U}^{\top} - F_c}^2_F \le O(r)^{4\omega}\cdot \radius^4\cdot d({\epstrueort}^2 + {\epstrueouter}^2 + \epsmap^2)$. Then for any $i,j\in[r]$,
    \begin{align}
        (F_c)_{ij} &= (\wt{U}F^*_c\wt{U}^{\top})_{ij} \pm O(r)^{2\omega}\cdot \radius^2\sqrt{d}({\epstrueort} + {\epstrueouter} + \epsmap) \\
        &= \sum_{k,\ell\in[r]} \wt{U}_{ik} \wt{U}_{j\ell} (F^*_c)_{k\ell} \pm O(r)^{2\omega}\cdot \radius^2\cdot \sqrt{d}({\epstrueort} + {\epstrueouter} + \epsmap). \label{eq:Fcij}
    \end{align}
    Note that
    \begin{equation}
        \biggl(\sum_{(k,\ell)\neq (i,j)} \wt{U}_{ik} \wt{U}_{j\ell} (F^*_c)_{k\ell}\biggr)^2 \le \norm{F^*_c}^2_F \cdot \sum_{(k,\ell)\neq (i,j)} (\wt{U}_{ik} \wt{U}_{j\ell})^2 \le r^{2\omega}\radius^4\epsoffdiag, \label{eq:offdiag0_push}
    \end{equation}
    where in the last step we used that $\norm{F^*_c}^2_F = \norm{f^*_c}^4_2 \le r^{2\omega}\radius^4$ and that for any $(i,j)\neq (k,\ell)$, $(\wt{U}_{ik}\wt{U}_{j\ell})^2 \le O(\epsoffdiag)$ by Lemma~\ref{lem:Udiag}. The lemma follows from \eqref{eq:Fcij} and \eqref{eq:offdiag0_push}.
\end{proof}

\subsection{Proof of Lemma~\ref{lem:push_usemu}}
\label{app:defer_push_usemu}

\begin{proof}
    Recall that by \eqref{eq:Fstarfirstrow}, $(F^*_\mu)_{1j} \ge \gap$ for all $j \in[r]$, and by Constraint~\ref{constraint:push_firstrow}, $(F_{\mu})_{1j} \ge 0$. Dividing by the scalar quantity $(F^*_{\mu})_{1j}$ on both sides of \eqref{eq:FFstar} from Lemma~\ref{lem:FFstar} for $c = \mu$ and $i = 1$ and rearranging, we find that
    \begin{equation}
        \wt{U}_{11} \wt{U}_{jj} \ge -O(r)^{2\omega}\cdot (\radius^2\sqrt{d}({\epstrueort} + {\epstrueouter} + \epsmap) + \radius^4\epsoffdiag) / \gap.
    \end{equation}
    This implies that
    \begin{equation}
        (\wt{U}_{11} \wt{U}_{jj} - 1)^2 \le (\wt{U}_{11} \wt{U}_{jj} + 1)^2 + O(r)^{2\omega}\cdot (\radius^2\sqrt{d}({\epstrueort} + {\epstrueouter} + \epsmap) + \radius^4\epsoffdiag) / \gap,
    \end{equation}
    so multiplying both sides by $(\wt{U}_{11}\wt{U}_{jj} - 1)^2$ and noting that by Lemma~\ref{lem:Udiag}, $((\wt{U}_{11}\wt{U}_{jj})^2 - 1)^2 \le O(\epsoffdiag^2)$ and $(\wt{U}_{11}\wt{U}_{jj} - 1)^2 \le 2 + O(\epsoffdiag^2) + O(r)^{2\omega}\cdot (\radius^2\sqrt{d}({\epstrueort} + {\epstrueouter} + \epsmap) + \radius^4\epsoffdiag) / \gap = O(1)$, we get
    \begin{equation}
        (\wt{U}_{11}\wt{U}_{jj} - 1)^4 \le O(\epsoffdiag^2) + O(r)^{2\omega}\cdot (\radius^2\sqrt{d}({\epstrueort} + {\epstrueouter} + \epsmap) + \radius^4\epsoffdiag) / \gap.
    \end{equation}
    By Fact~\ref{fact:nonneg_power_of_two} we conclude that
    \begin{equation}
        \wt{U}_{11}\wt{U}_{jj} = 1 \pm \left( O(\epsoffdiag^2) + O(r)^{2\omega}\cdot (\radius^2\sqrt{d}({\epstrueort} + {\epstrueouter} + \epsmap) + \radius^4\epsoffdiag) / \gap\right)^{1/4} \ \ \forall \ j\in[r]. \label{eq:push_1jbound}
    \end{equation}
    From \eqref{eq:push_1jbound} we conclude that for any $i,j\in[r]$,
    \begin{equation}
        (\wt{U}_{11})^2 \wt{U}_{ii} \wt{U}_{jj} = 1 \pm \left( O(\epsoffdiag^2) + O(r)^{2\omega}\cdot (\radius^2\sqrt{d}({\epstrueort} + {\epstrueouter} + \epsmap) + \radius^4\epsoffdiag) / \gap\right)^{1/4},
    \end{equation}
    and because $(\wt{U}_{11})^2 = 1 + O(\epsoffdiag)$ by Lemma~\ref{lem:Udiag}, by Part~\ref{fact:divide_both_sides_big} of Fact~\ref{fact:division} we conclude the proof upon substituting the definitions of $\epsmap, \epstrueort, \epstrueouter, \epsoffdiag$.
\end{proof}

\section{Other Deferred Proofs}
\label{app:otherdefer}

\subsection{Proof of Lemma~\ref{lem:gauge}}
\label{app:defer_gauge}

\begin{proof}
    This follows immediately from the fact that for any $x\in\R^r$,
    \begin{equation}
        \iprod{T^{**}_a,x^{\otimes \omega}} = \sum_{i_1,j_1,\ldots,i_\omega,j_\omega\in[r]} (U_{i_1j_1}x_{i_1})\cdots (U_{i_\omega j_\omega} x_{i_1}) (T^*_a)_{j_1\cdots j_\omega} = \iprod{T^*_a, (U^{\top} x)^{\otimes\omega}},
    \end{equation}
    and the input distribution $\calN(0,\Id_r)$ being pushed forward is rotation-invariant.
\end{proof}

\subsection{Proof of Lemma~\ref{lem:param_to_wasserstein}}
\label{app:param_to_wasserstein}

\begin{proof}
    Let $\epsilon = \gaugedist(\brc{T_a},\brc{T'_a})$. As $\calN(0,\Id)$ is rotation-invariant, we can assume without loss of generality that $\norm{T_a - T'_a}_F\le \epsilon$ for all $a\in[d]$. For any $x\in\R^r$, by Cauchy-Schwarz we have $\iprod{T_a - T'_a, x^{\otimes\omega}} \le \epsilon\norm{x}^{\omega}$. Let $f:\R^d\to\R$ be any $1$-Lipschitz function. Then
    \begin{align}
        \abs*{\E[z\sim\calD]{f(z)} - \E[z\sim\calD']{f(z)}} &= \abs*{\E[g\sim\calN(0,\Id)]*{f(\iprod{T_1,g^{\otimes\omega}},\ldots,\iprod{T_d,g^{\otimes\omega}}) - f(\iprod{T'_1,g^{\otimes\omega}},\ldots,\iprod{T'_d,g^{\otimes\omega}})}} \\
        &\le \E[g\sim\calN(0,\Id)]*{\abs*{f(\iprod{T_1,g^{\otimes\omega}},\ldots,\iprod{T_d,g^{\otimes\omega}}) - f(\iprod{T'_1,g^{\otimes\omega}},\ldots,\iprod{T'_d,g^{\otimes\omega}})}} \\
        &\le \epsilon\sqrt{d}\cdot\E[g\sim\calN(0,\Id)]*{\norm{g}^{\omega}} \le \epsilon\sqrt{d}\cdot O(\omega r)^{\omega/2}.
    \end{align}
    The lemma follows by the dual characterization of Wasserstein-1 distance.
\end{proof}

\subsection{Proof of Lemma~\ref{lem:symmetric_inspan}}
\label{app:defer_symmetric_inspan}

\begin{proof}
    We proceed inductively on $\omega$.
    When $\omega = 1$, the conclusion is clear.
    For any tensor $T$, and any vector $v$, let $v \otimes_\ell T$ denote the tensor where $v$ is tensored into the $\ell$-th index.
    For general $\omega$, and for some choice of $j_1, \ldots, j_\omega$, define 
    \[
    T_{\omega} = \frac{1}{(\omega - 1)!} \sum_{\pi \in \calS_{\omega - 1}} e_{j_{\pi(1)}}\otimes\cdots \otimes e_{j_{\pi(\omega)}}
    \]
    be the right hand side, so that $T_{\omega} = \sum_{\ell} e_{j_{\omega}} \otimes_\ell T_{\omega - 1}$.
    By induction, we know that there exist vectors $z'_1, \ldots, z'_{s' - 1} \in \S^{r - 1}$ and $w' \in \R^{s'}$ so that $\norm{w'}_1 \leq \exp (O(\omega^2 \log^2 \omega))$ and $T_{\omega - 1} = \sum_{i = 1}^{s'} w'_i (z'_i)^{\otimes (w - 1)}$, so that
    \begin{equation}
    \label{eq:symmetric_inspan_induction}
    T_\omega = \frac{1}{(\omega - 1)!} \sum_{i = 1}^{s'} w_i \sum_{\ell} e_{j_\omega} \otimes_\ell (z'_i)^{\otimes (w - 1)} \; .
    \end{equation}
    
    We now claim that for any vectors $e, v \in \S^{d - 1}$, we have that 
    \[
    T_{e, v} \triangleq \sum_{\ell} e \otimes_\ell v^{\otimes (\omega - 1)} = \sum_{i = 1}^{s''} w''_i (z''_i)^{\otimes \omega} \; ,
    \]
    for some $z''_i \in \S^{r - 1}$ and $w'' \in \R^{s''}$ satisfying $\norm{w''}_1 \leq \exp (O(\omega \log \omega))$.
    This suffices to prove the claim, since combining this with~\eqref{eq:symmetric_inspan_induction} proves the induction.
    Equivalently, since all relevant tensors are symmetric, this means that it suffices to find $z''_i$ and $w''$ as above so that for all vectors $u \in \R^r$, we have that
    \begin{equation}
    \label{eq:symmetric_inspan_poly}
    T_{e, v} (u, \ldots, u) = \sum_{i = 1}^{s''} w''_i \iprod{z''_i, u}^{\omega} \; .
    \end{equation}
    Notice that for any $u$, we have that $T_{e, v} (u, \ldots, u) = \omega \iprod{e, u} \iprod{v, u}^{\omega - 1}$, and hence by Lemma~\ref{lem:symmetric_inspan_vandermonde} below, we know there exist coefficients $\alpha_i$ so that
    \[
        T_{e, v} (u, \ldots, u) = \sum_{i = 1}^\omega \alpha \left( \iprod{e + iv, u} \right)^i \; ,
    \]
    where $\norm{\alpha_i}_1 \leq \exp (O(\omega \log \omega))$.
    Thus, if we let $w_i'' = \alpha_i \norm{e + iv}_2^i$, and $z_i'' = \tfrac{e + iv}{\norm{e + iv}_2}$, it is easily verified that this choice satisfies~\eqref{eq:symmetric_inspan_poly}, and moreover, $\norm{w''}_1 \leq \exp (O(\omega \log \omega))$.
    This completes the proof.
\end{proof}

In the proof above, we used the following fact about polynomial interpolation:

\begin{lemma}
\label{lem:symmetric_inspan_vandermonde}
    Let $\omega$ be a nonnegative integer.
    Then, we have the following polynomial identity:
    \[
    x y^{\omega - 1} = \sum_{i = 1}^\omega \alpha_i (x + i y)^\omega \; ,
    \]
    where $\alpha_i$ satisfies $\sum_{i = 1}^{\omega} |\alpha_i| \leq \exp (O(\omega \log \omega))$.
\end{lemma}
\begin{proof}
    Let $V_\omega$ denote the $\omega \times \omega$ Vandermonde matrix with nodes at $1, \ldots, \omega$, and let $W \in \R^{\omega \times \omega}$ be the diagonal matrix so that $W_{ii} = \binom{\omega}{i}$.
    Then, by the binomial formula, to find such $\alpha_i$, it is equivalent to find a vector $\alpha \in \R^{\omega}$ so that $W V_{\omega} \alpha = e_{\omega - 1}$, where $e_{\omega - 1}$ is the $(\omega - 1)$-th basis vector.
    Since $V_\omega$ is nonsingular, such an $\alpha$ clearly exists, and moreover, since $\norm{V_\omega^{-1}} \leq (\omega + 1)!$ (see e.g.~\cite{gautschi1978inverses}), we know that $\sum_{i = 1}^{\omega} |\alpha_i| \leq \exp (O(\omega \log \omega))$, as claimed.
\end{proof}

\subsection{Proof of Corollary~\ref{cor:general_vandermonde}}
\label{app:defer_general_vandermonde}

\begin{proof}
    Fix $\alpha\in\brc{0,\ldots,e}^D$ with $|\alpha| = e$. Suppose inductively that for some $0\le m \le D$ and $\tmp_m > 0$, we have shown that \begin{equation}
        \sum_{\alpha': \alpha'_{1:m} = \alpha_{1:m}} c_{\alpha'} \mb{z}'_{\alpha'_{m+1:D}} = \pm \tmp_m \ \ \forall \ \mb{z}'\in\brc*{\frac{1}{e+1},\frac{2}{e+1},\ldots,1}^{D-m} \label{eq:induct_allz}
    \end{equation} The base case of $m = 0$ is immediate for $\tmp_0 \triangleq \nu$. We can rewrite \eqref{eq:induct_allz} as 
    \begin{equation}
        \sum^{e}_{b=0} (\mb{z}'_{m+1})^b \sum_{\alpha'': \  \alpha''_{1:m+1} = (\alpha_{1:m},b)} c_{\alpha''} \mb{z}'_{\alpha''_{m+2:D}} = \pm \tmp_m \ \ \forall \ \mb{z}'\in\brc*{\frac{1}{e+1},\frac{2}{e+1},\ldots,1}^{D-m}. \label{eq:induct_allz_rewrite}
    \end{equation}
    By Fact~\ref{fact:vandermonde_interpolate}, for any $\mb{z}'' \in \brc*{\frac{1}{e+1},\frac{2}{e+1},\ldots,1}^{D-m-1}$ and $0\le b\le e$, there is a linear combination of the constraints \eqref{eq:induct_allz_rewrite} for $\mb{z}'$ ranging over $(i/(e+1),\mb{z}'')$ for $0 \le i \le e$ that implies that
    \begin{equation}
        \sum_{\alpha'': \  \alpha''_{1:m+1} = (\alpha_{1:m},b)} c_{\alpha''} \mb{z}''_{\alpha''_{m+2:D}} = \pm O(e)^{O(e)}\cdot\tmp_m, % \label{eq:induct_allz_rewrite}
    \end{equation}
    which completes the inductive step for $\tmp_{m+1} = \Theta(e)^{\Theta(e)}\cdot \tmp_m$. We conclude that \eqref{eq:induct_allz} holds with $\tmp_m = \Theta(e)^{\Theta(e m)}\cdot \nu$ for all $0\le m \le D$. In particular, the case of $m = D$ implies that
    \begin{equation}
        c_{\alpha} = \pm O(e)^{\Theta(e D)}\cdot \nu,
    \end{equation} as claimed.
\end{proof}

\subsection{Proof of Lemma~\ref{lem:variance_shift}}
\label{app:defer_variance_shift}

\begin{proof}
    Define the (normalized) Hermite polynomials $\phi_{\ell}(x)\triangleq \frac{1}{\sqrt{\ell!}}\mathrm{He}_\ell(x)$, where $\mathrm{He}_{\ell}(x)$ denotes the degree-$\ell$ probabilist's Hermite polynomial. Given $\alpha\in\mathbb{Z}^r$, let $\phi_{\alpha}\triangleq \prod^r_{j=1} \phi_{\alpha_j}$ denote the $\alpha$-th multivariate Hermite polynomial. We can rewrite the polynomial $x\mapsto p(ax+b)$ in this basis as
    \begin{align}
        p(ax+b) &= \sum_S p_S (ax+b)_S = a^{\omega} p(x) + \sum_S \sum_{T\subsetneq S} a^{|T|}x_T\cdot b^{|S\backslash T|} \\
        &= a^{\omega}\sum_{\alpha\in[\omega]^r: |\alpha| \le \omega} \wh{p}_{\alpha}\cdot \phi_{\alpha}(x) + \sum_{\beta\in[\omega]^r: |\beta| < \omega} c_\beta\phi_\beta(x)
    \end{align}
    where in the first step the sum ranges over multi-subsets $S$ of $[r]$, in the second step we separated out the top-degree homogeneous component of $p(ax+b)$ under the monomial basis, in the third step $\brc{\wh{p}_{\alpha}}$ are the Hermite coefficients of $p$, and in the last step $\brc{c_T}$ are the Hermite coefficients of $p(ax+b)$ of degree strictly less than $\omega$. By Claim 2.1 in \cite{lovett2010elementary}, for any $\alpha$ satisfying $|\alpha| = \omega$, if the multi-subset $S$ consists of $\alpha_j$ copies of $j$ for every $j\in[r]$, then $p_S = (\prod^r_{j=1} 1/\sqrt{a_j!})\wh{p}_{\alpha}$. Note that $\prod^r_{j=1} a_j! \le \omega!$, so
    \begin{equation}
        \Var{p(ag+b)} = a^{2\omega} \sum_{\alpha\in[\omega]^r: |\alpha| \le \omega} \wh{p}^2_{\alpha} + \sum_{\beta\in[\omega]^r: 0 < |\beta| < \omega} c^2_{\beta}\ge a^{2\omega} \sum_{\alpha} \wh{p}^2_{\alpha} \ge a^{2\omega} (\omega!)^{-1/2}\sum_{S} p_{S}^2 = a^{2\omega} (\omega!)^{-1/2}
    \end{equation} as claimed, where in the last step we used the fact that $p\in\S^{\rchoose-1}$.
\end{proof}

\subsection{Proof of Lemma~\ref{fact:nonneg_power_of_two}}
\label{app:defer_nonneg_power_of_two}

\begin{proof}
    By multiplying the constraints $\epsilon^t - x^t \ge 0$ and $x^t + \epsilon^t \ge 0$ and rearranging, we conclude that $x^{2t} \le \epsilon^{2t}$. Given even integer $\ell$ and $z\in\brc{0,2}$, define
    \begin{equation}
        g(\ell,z) = \begin{cases}
            (\frac{\ell+2}{2}, 2) & \ \text{if} \ \ell \equiv 2 \pmod{4} \\
            (\frac{\ell}{2}, 0) & \ \text{if} \ \ell \equiv 0 \pmod{4}
            % (\frac{\ell+2}{2}, 2) & \ \text{if} \ z = 2 \ \text{and} \ \frac{\ell+2}{2}\equiv 2 \pmod{4} \\
            % (\frac{\ell}{2}, 0) & \text{if} \ z = 0 \ \text{and} \ \frac{\ell}{2}\equiv 0 \pmod{4} \\
            % (\frac{\ell}{2}, 2) & \text{if} \ z = 0 \ \text{and} \ \frac{\ell}{2}\equiv 2 \pmod{4}.
        \end{cases}
    \end{equation}
    Denote $i$-fold composition of $g$ by $g^{(i)}$, and let $\ell_i$ and $z_i$ denote the first and second entries of $g^{(i)}(2t,0)$. Note that $\ell_i$ is even for all $i$, so this is well-defined.
    Let $i^*$ be the index for which $\ell_{i^*} = 2$; this is well-defined because for any $\ell > 2$, the first entry of $g(\ell,z)$ is less than $\ell$ and at least $2$.
    
    Consider the following sequence of polynomials
    \begin{equation}
        x^{z_1}(x^{\ell_1 - z_1} - \epsilon^{\ell_1 - z_1})^2, x^{z_2}(x^{\ell_2 - z_2} - \epsilon^{\ell_2 - z_2})^2, \ldots, x^{z_{i^*}}(x^{\ell_{i^*} - z_{i^*}} - \epsilon^{\ell_{i^*} - z_{i^*}})^2.
    \end{equation}
    Because $2\ell_{i+1} - z_{i+1} = \ell_i$, the second monomial in the expansion of the $i$-th polynomial is equal to $2\epsilon^{\ell_i - z_i}$ times the leading term in the expansion of the $(i+1)$-th polynomial. Therefore,
    \begin{equation}
        x^{2t} - \epsilon^{2t} - \sum^{i^*}_{i=1} \biggl(\prod^{i-1}_{j=0}2\epsilon^{\ell_j - z_j}\biggr)\cdot x^{z_i}(x^{\ell_i - z_i} - \epsilon^{\ell_i - z_i})^2 \label{eq:weird_combo}
    \end{equation}
    is a degree-2 polynomial. Recall that $x^{2t} \le \epsilon^{2t}$, and additionally the summands in \eqref{eq:weird_combo} are squares, so \eqref{eq:weird_combo} is at most 0. Furthermore, the constant term in the monomial expansion of \eqref{eq:weird_combo} is strictly negative, so \eqref{eq:weird_combo} is a positive multiple of $x^2 - \epsilon^2$. We conclude that $x^2 - \epsilon^2 \le 0$ as claimed.
\end{proof}

\subsection{Proof of Fact~\ref{fact:root1}}
\label{app:defer_root1}

\begin{proof}
    Note that $-\epsilon \le x^j - 1 \le \epsilon$ implies that $(1 - \epsilon)^2 - 1 \le x^{2j} - 1 \le (1 + \epsilon)^2 - 1$, so $-2\epsilon + \epsilon^2 \le x^{2j} - 1 \le 2\epsilon + \epsilon^2$.
    
    Squaring the latter implies that $(x^{2j} - 1)^2 \le (2\epsilon + \epsilon^2)^2 \le 9\epsilon^2$, which we can rewrite as \begin{equation}
        (x^2 - 1)^2 \biggl(\sum^{j-1}_{i=0} x^{2i}\biggr)^2 \le 9\epsilon^2.
    \end{equation} 
    We can lower bound the left-hand side by $(x^2 - 1)^2$ because $\sum^{j-1}_{i=0} x^{2i} \ge 1$, so $(x^2 - 1)^2 \le 9\epsilon^2$ as desired.
    
    For the second part of the fact, we have an SoS proof that
    \begin{equation}
        \epsilon^2 \ge (x^j - 1)^2 = (x-1)^2 \cdot (x^{j-1} + x^{j-2} + \cdots + 1)^2 \ge \frac{1}{2}(x - 1)^2,
    \end{equation}
    by Fact~\ref{fact:geoseries} below. So $(x - 1)^2 \le \epsilon^2$. Furthermore, $x^2 \ge 1 - 3\epsilon$ by the first part of the lemma, so rearranging, we conclude that $x \ge 1 - (\epsilon^2 + 3\epsilon)/2$ From $x^2 \le 1 + 3\epsilon$, we also have $x \le \sqrt{1 + 3\epsilon} \le 1 + 3\epsilon/2$, concluding the proof of the second part.
\end{proof}

In the proof above, we used the following fact:

\begin{fact}\label{fact:geoseries}
    For any even $\ell\in\mathbb{N}$, the polynomial
    \begin{equation}
        1 + x + \cdots + x^{\ell} - 1/2
    \end{equation}
    in the indeterminate $x$ is a degree-$\ell$ sum of squares in $x$.
\end{fact}

\begin{proof}
    It is a standard fact that any nonnegative polynomial over $\R$ can be written as a sum of squares, so it suffices to show that $1 + x + \cdots + x^{\ell} \ge 1/2$ for all $x\in\R$. When $x \ge 0$, clearly $1 + x + \cdots + x^{\ell} \ge 0$. When $x \le -1$, $x^{2i} +x^{2i-1} \ge 0$ for any $i\in\mathbb{N}$, so again $1 + x + \cdots + x^{\ell} \ge 0$. Finally, when $-1 < x < 0$, we have $1 < 1 - x^{\ell+1} < 2$ and $1 < 1 - x < 2$, so $1 + x + \cdots + x^{\ell} = \frac{1 - x^{\ell+1}}{1 - x} > 1/2$, as desired.
\end{proof}

\subsection{Proof of Fact~\ref{fact:division}}
\label{app:defer_division}

\begin{proof}
    For the first part, as $x^2 \ge \alpha^2$ and $x^2 y^2 \ge \epsilon^2$, we have $\epsilon^2 \ge x^2y^2 \ge \alpha^2 y^2$, which completes the proof.
    
    We now turn to the second part. It suffices to show that $y \ge 0$, as we would then have that $1 - \epsilon \le xy \le (1+\delta)y$ and similarly $1 + \epsilon \ge xy \ge (1 - \delta)y$. In this case, $\frac{1 + \epsilon}{1 - \delta} - 1 \le O(\epsilon+\delta)$ for $\epsilon,\delta$ small, and similarly $1 - \frac{1 - \epsilon}{1+\delta} \le O(\epsilon+\delta)$.
    
    To show that $y \ge 0$, first note that 
    \begin{equation}
        \epsilon^2 \ge (xy - 1)^2 = x^2 y^2 - 2xy + 1 \ge (1 - \delta)^2 y^2 - 2(1 + \epsilon) + 1,
    \end{equation}
    so $y^2 \le \frac{(1+\epsilon)^2}{(1 - \delta)^2} \le 1 + O(\epsilon + \delta)$. We therefore have
    \begin{equation}
        (x - y)^2 = x^2 - 2xy + y^2 \le (1+\delta)^2 - 2(1-\epsilon) + 1 + O(\epsilon + \delta) \le O(\epsilon+\delta),
    \end{equation}
    so $y \ge x - O(\epsilon+\delta)^{1/2} \ge 0$ as desired.
\end{proof}

\subsection{Proof of Lemma~\ref{lem:ortho_sos}}
\label{app:defer_ortho_sos}

\begin{proof}
    By applying Fact~\ref{fact:cauchy_schwarz} to $(x_{i_11},\ldots,x_{i_1d})$ and $(x_{i_21},\ldots,x_{i_2d})$ for all $i_1,i_2$ satisfying $i_1 < i_2$ and summing the results, we have that
    \begin{align}
        \sum_{\substack{i_1,i_2,j_1,j_2\in[d]:\\ i_1 < i_2; j_1< j_2}} (x_{i_1 j_1} x_{i_2 j_2} - x_{i_1 j_2} x_{i_2 j_1})^2 &= \sum_{\substack{i_1,i_2\in[d]: \\ i_1 < i_2}} \Biggl[\biggl(\sum_j x^2_{i_1 j}\biggr)\biggl(\sum_j x^2_{i_2j}\biggr) - \biggl(\sum_j x_{i_1j} x_{i_2j}\biggr)^2\Biggr] \\
        &\ge \binom{d}{2}((1-\epsilon)^2 - \epsilon^2) = \binom{d}{2}(1 - 2\epsilon), \label{eq:many_cauchy}
    \end{align}
    where in the second step we used the constraints \eqref{eq:row_ortho}.
    
    By applying Fact~\ref{fact:cauchy_schwarz} to $(x_{1j_1},\ldots,x_{dj_1})$ and $(x_{1j_2},\ldots,x_{dj_2})$ for all $j_1,j_2$ satisfying $j_1 < j_2$ and summing the results, we find
    \begin{equation}
        \sum_{\substack{i_1,i_2,j_1,j_2\in[d]:\\ i_1 < i_2; j_1< j_2}} (x_{i_1 j_1} x_{i_2 j_2} - x_{i_1 j_2} x_{i_2 j_1})^2 = \sum_{\substack{j_1,j_2\in[d]: \\ j_1 < j_2}} \Biggl[\biggl(\sum_i x^2_{ij_1}\biggr)\biggl(\sum_i x^2_{ij_2}\biggr) - \biggl(\sum_i x_{ij_1} x_{ij_2}\biggr)^2\Biggr]. \label{eq:many_cauchy2}
    \end{equation}
    For indeterminates $\brc{a_1,\ldots,a_d}$, there is a degree-2 SoS proof that
    \begin{equation}
        \sum_{i,j: i < j} a_ia_j = \frac{1}{2}\biggl(\sum_i a_i\biggr)^2 - \frac{1}{2}\sum_i a_i^2 \le \left(\frac{1}{2} - \frac{1}{2d}\right)\biggl(\sum_i a_i\biggr)^2, \label{eq:sumpairs}
    \end{equation}
    so applying this to $a_j = x^2_{1j} + \cdots + x^2_{dj}$, we have a degree-4 SoS proof that
    \begin{equation}
        \sum_{j_1<j_2} \biggl(\sum_i x^2_{i j_1}\biggr) \biggl(\sum_i x^2_{i j_2}\biggr) \le \left(\frac{1}{2} - \frac{1}{2d}\right)\biggl(\sum_{i,j} x^2_{ij}\biggr)^2 = \frac{d-2}{2d}\cdot(d(1+\epsilon))^2 \le \binom{d}{2}(1 + 3\epsilon)\label{eq:csortho}
    \end{equation}
    where in the last step we used the constraints \eqref{eq:row_ortho}.
    Combining this with \eqref{eq:many_cauchy} and \eqref{eq:many_cauchy2}, we get
    \begin{equation}
        \sum_{\substack{i_1,i_2,j_1,j_2\in[d]:\\ i_1 < i_2; j_1< j_2}}\biggl(\sum_i x_{ij_1} x_{ij_2}\biggr)^2 \le \binom{d}{2}\cdot 5\epsilon \le 3\epsilon d^2.
    \end{equation}
    This implies that $-2\sqrt{\epsilon} d \le \sum_i x_{ij_1}x_{ij_2} \le 2\sqrt{\epsilon} d$ for any $j_1,j_2$, by degree-4 SoS. 
    
    Combining~\eqref{eq:many_cauchy} and \eqref{eq:many_cauchy2}, we also know that
    \begin{equation}
        \sum_{j_1 < j_2} \biggl(\sum_i x^2_{ij_1}\biggr)\biggl(\sum_i x^2_{ij_2}\biggr) \ge \binom{d}{2}(1 - 2\epsilon)
    \end{equation}
    This, together with \eqref{eq:csortho}, implies that the difference between the quantities $\left(\frac{1}{2} - \frac{1}{2d}\right)(\sum_{i,j} x^2_{ij})^2$ and $\sum_{j_1 < j_2} (\sum_i x^2_{ij_1})(\sum_i x^2_{ij_2})$ is also upper bounded by $\binom{d}{2}\cdot 5\epsilon \le 3\epsilon d^2$. But we can write this difference as 
    \begin{equation}
        \frac{1}{2d}\sum_{j_1<j_2} \biggl(\sum_i x^2_{ij_1} - \sum_i x^2_{ij_2}\biggr)^2 \le 3\epsilon d^2.
    \end{equation}
    We conclude that
    \begin{equation}
        -3\sqrt{\epsilon d^3} \le \sum_i x^2_{ij_1} - \sum_i x^2_{ij_2} \le -3\sqrt{\epsilon d^3}
    \end{equation}
    for any $j_1,j_2$, by degree-4 SoS. We additionally know that $\sum_j \sum_i x^2_{ij} \le d(1+\epsilon)$, so for any $j$, we have
    \begin{equation}
        d(1+\epsilon) \ge \sum_{j'} \sum_i x^2_{ij'}  \ge d\biggl(\sum_i x^2_{ij} - 3\sqrt{\epsilon d^3}\biggr),
    \end{equation}
    implying that $\sum_i x^2_{ij} \le (1 + \epsilon) + 3\sqrt{\epsilon d^3} < 4\sqrt{\epsilon d^3}$. Similarly, we also get that $\sum_i x^2_{ij} \ge -4\sqrt{\epsilon d^3}$.
\end{proof}

\subsection{Proof of Lemma~\ref{lem:estimate_moments}}
\label{app:defer_estimate_moments}

\begin{proof}
    We will repeatedly use the fact that $\E{(g^{\top}Q^*_a g)^2} = \Tr(Q^*_a)^2 + 2\norm{Q^*_a}^2_F = O(r\radius^2)$. For any $a,b\in[d]$, consider the degree-4 polynomial 
    \begin{equation}
        p(x) = \bigl((x^{\top}Q^*_a x) - \E[g\sim\calN(0,\Id)]{g^{\top} Q^*_a g}\bigr)\bigl((x^{\top}Q^*_b x) - \E[g\sim\calN(0,\Id)]{g^{\top} Q^*_b g}\bigr) = \iprod{xx^{\top} - \Id, Q^*_a}\cdot \iprod{xx^{\top} - \Id, Q^*_b}.
    \end{equation}
    By Lemma~\ref{lem:moments}, $\E{p(g)} = 2\Tr(Q^*_a Q^*_b)$. We can loosely bound $\E{p(g)^2}$ by
    \begin{align}
        \E{p(g)^2} &= \E*{\bigl((g^{\top} Q^*_a g)(g^{\top} Q^*_b g) - \Tr(Q^*_a) g^{\top} Q^*_b g - \Tr(Q^*_b) g^{\top} Q^*_a g + \Tr(Q^*_a)\Tr(Q^*_b)\bigr)^2} \\
        &\le 4\E{(g^{\top} Q^*_a g)^2(g^{\top} Q^*_b g)^2} + O(r^2\radius^4)\\
        &\le 4\E{(g^{\top} Q^*_a g)^4}^{1/2} \E{(g^{\top} Q^*_b g)^4}^{1/2} +  O(r^2\radius^4) \\
        &\le O(\E{(g^{\top}Q^*_a g)^2}\cdot \E{(g^{\top}Q^*_b g)^2}) + O(r^2\radius^4) = O(r^2\radius^4),
    \end{align}
    where in the first inequality we used Cauchy-Schwarz and our bound on $\E{(g^{\top}Q^*_a g)^2}, \E{(g^{\top}Q^*_b g)^2}$, and in the fourth step we used hypercontractivity.

    Therefore, by Lemma~\ref{lem:hypercontractivity}, given independent draws $g_1,\ldots,g_n\sim\calN(0,\Id)$,
    \begin{equation}
        \Pr*{\abs*{\frac{1}{n}\sum^n_{i=1} p(g_i) - 2\Tr(Q^*_a Q^*_b)} \ge \Omega(r\radius^2\log^2(2d/\delta)/\sqrt{n})}  \le \delta/2d.
    \end{equation} So provided that $n \ge \Omega(r^2\radius^4\log^2(2d/\delta)/\eta^2$, \eqref{eq:estimate_second_moment} holds with probability $1 - \delta/2$.
    
    Next, for any $a,b,c\in[d]$, consider the degree-6 polynomial
    \begin{align}
        q(x) &= ((x^{\top}Q^*_a x) - \E{g^{\top}Q^*_a g})((x^{\top}Q^*_b x) - \E{g^{\top}Q^*_b g})((x^{\top}Q^*_c x) - \E{g^{\top}Q^*_c g}) \\
        &= \iprod{xx^{\top} - \Id,Q^*_a}\cdot \iprod{xx^{\top} - \Id,Q^*_b}\cdot \iprod{xx^{\top} - \Id,Q^*_c}. 
    \end{align}
    By Lemma~\ref{lem:moments}, $\E{q(g)} = 8\Tr(Q^*_a Q^*_b Q^*_c)$. We can loosely bound $\E{q(g)^2}$ by
    \begin{align}
        \E{q(g)^2} &= \mathbb{E}\left[\left((g^{\top}Q^*_a g)(g^{\top}Q^*_b g)(g^{\top}Q^*_cg) - \Tr(Q^*_a)\Tr(Q^*_b)\Tr(Q^*_c) \right.\right.\\
        &\qquad - \Tr(Q^*_a)(g^{\top}Q^*_b g)(g^{\top}Q^*_c g) - \Tr(Q^*_b)(g^{\top}Q^*_a g)(g^{\top} Q^*_c g) - \Tr(Q^*_c)(g^{\top}Q^*_a g)(g^{\top} Q^*_b g) \\
        &\qquad \left.\left. - \Tr(Q^*_a)\Tr(Q^*_b)(g^{\top}Q^*_c g) - \Tr(Q^*_a)\Tr(Q^*_c)(g^{\top}Q^*_b g) - \Tr(Q^*_b)\Tr(Q^*_c)(g^{\top}Q^*_c g) \right)^2\right] \\
        &\le 8\E{(g^{\top}Q^*_a g)^2(g^{\top}Q^*_b g)^2(g^{\top}Q^*_cg)^2} + O(r^3\radius^6) \\
        &\le 8\E{(g^{\top}Q^*_a g)^6}^{1/3}\E{(g^{\top}Q^*_b g)^6}^{1/3}\E{(g^{\top}Q^*_c g)^6}^{1/3} + O(r^3\radius^6) \\
        &\le O(\E{(g^{\top}Q^*_a g)^2}\E{(g^{\top}Q^*_b g)^2}\E{(g^{\top}Q^*_c g)^2}) + O(r^3\radius^6) = O(r^3\radius^6),
    \end{align}
    where in the second step we used Cauchy-Schwarz and our previous bounds on terms of the form $\E{(g^{\top}Q^*_a)^2}$ and $\E{(g^{\top}Q^*_a g)^2(g^{\top}Q^*_b g)^2}$, and in the fourth step we used hypercontractivity.
    
    Therefore, by Lemma~\ref{lem:hypercontractivity}, given independent draws $g_1,\ldots,g_n\sim\calN(0,\Id)$,
    \begin{equation}
        \Pr*{\abs*{\frac{1}{n}\sum^n_{i=1} q(g_i) - 8\Tr(Q^*_a Q^*_b Q^*_c)} \ge \Omega(r^{3/2}\radius^3\log^3(2d/\delta)/\sqrt{n})} \le \delta/2d.
    \end{equation} So provided that $n\ge \Omega(r^3\radius^6\log^3(2d/\delta)/\eta^2)$, \eqref{eq:estimate_third_moment} holds with probability $1 - \delta/2$.
\end{proof}

\subsection{Proof of Lemma~\ref{lem:rot_to_gaussian}}
\label{app:defer_rot_to_gaussian}

\begin{proof}
    If $e$ is odd, then the lemma follows immediately as $\E[x\sim D]{q(x)} = \E[g\sim\calN(0,\Id)]{q(g)} = 0$. Suppose $e$ is even. Let $p$ denote the distribution over $\norm{x}^2$ for $x\sim D$. To sample from $D$, we can independently sample $z\sim\S^{r-1}$ and $\nu\sim p$ and output $\sqrt{\nu}\cdot z$. Similarly, to sample from $\calN(0,\Id)$, we can independently sample $z\sim\S^{r-1}$ and $\nu\sim \chi^2(r)$, where $\chi^2(r)$ denotes the chi-squared distribution with $r$ degrees of freedom, and output $\sqrt{\nu}\cdot z$. Then 
    \begin{equation}
        \E[x\sim D]{q(x)} = \E[\nu\sim p]{\nu^{e/2}}\cdot \E[z\sim\S^{r-1}]{q(z)} = \frac{\E[\nu\sim p]{\nu^{e/2}}}{\E[\nu\sim \chi^2(r)]{\nu^{e/2}}}\cdot \E[g\sim\calN(0,\Id)]{q(g)},
    \end{equation} where in the first step we used independence of $\nu,z$ and homogeneity of $q$. The explicit expression for $C_{D,e}$ comes from evaluating $\E[\nu\sim \chi^2(r)]{\nu^{e/2}}$ explicitly.
\end{proof}

\subsection{Proof of Lemma~\ref{lem:estimate_moments_2}}
\label{app:defer_estimate_moments_2}

\begin{proof}
    For any $a,b\in[d]$, consider the degree-$2\omega$ polynomial
    \begin{equation}
        p(x) = \iprod{T^*_a,x^{\otimes\omega}}\iprod{T^*b,x^{\otimes\omega}}.
    \end{equation}
    By Lemma~\ref{lem:moments2}, $\E{p(g)} = \iprod{T^*_a,T^*_b}_{\Sigma}$. We can loosely bound $\E{p(g)^2}$ by
    \begin{equation}
        \E{p(g)^2} = \E*{\norm{T^*_a}^2_F \norm{T^*_b}^2_F \norm{x}^{4\omega}_2} \le O(\omega r)^{2\omega} \cdot \radius^4.
    \end{equation}
    Therefore, by Lemma~\ref{lem:hypercontractivity}, given independent draws $g_1,\ldots,g_n\sim\calN(0,\Id)$,
    \begin{equation}
        \Pr*{\abs*{\frac{1}{n}\sum^n_{i=1}p(g_i) - \iprod{T^*_a,T^*_b}_\Sigma} \ge \radius^2 \Omega(\omega r)^{\omega}\log^{\omega}(d/\delta)/\sqrt{n}} \le \delta/d.
    \end{equation}
    So provided that $n \ge \Omega(\omega r)^{2\omega}\radius^4 \log^{2\omega}(d/\delta)/\eta^2$ \eqref{eq:estimate_second_moment_2} holds with probability $1 - \delta$.
\end{proof}

\end{document}